%% file: main.tex
\documentclass{article}

\PassOptionsToPackage{numbers, compress, sort}{natbib}

\usepackage[main,final]{neurips_2026}

\usepackage[utf8]{inputenc} 
\usepackage[T1]{fontenc}    
\usepackage{hyperref}       
\usepackage{url}            
\usepackage{booktabs}       
\usepackage{amsfonts}       
\usepackage{nicefrac}       
\usepackage{microtype}      
\usepackage{xcolor}         

\usepackage{graphicx}
\usepackage{booktabs}
\usepackage{ulem}

\usepackage{amsmath}
\usepackage{amssymb}
\usepackage{mathtools}
\usepackage{amsthm}

\usepackage{subcaption}
\usepackage{placeins}
\usepackage[most]{tcolorbox}

\usepackage{bbm}

\usepackage{dsfont}
\usepackage{thmtools}
\usepackage{thm-restate}

\usepackage{algorithm}
\usepackage{algorithmic}
\usepackage{wrapfig}

\usepackage{xstring}

\newtheorem{proposition}{Proposition}

\newtheorem{corollary}{Corollary}

\declaretheorem[name=Definition]{definition}

\newcommand{\defer}{\mathrm{defer}}

\usepackage{xparse}
\usepackage{letltxmacro}
\usepackage{etoolbox}

\usepackage{wrapfig}
\usepackage{capt-of}

\newif\ifappendixtocactive
\appendixtocactivefalse

\newwrite\apptocfile

\makeatletter

\newcommand{\apptocline}[3]{%
  \ifstrequal{#1}{section}{%
    \par\medskip\noindent\textbf{#2}\dotfill #3\par
  }{%
  \ifstrequal{#1}{subsection}{%
    \par\noindent\hspace{1.5em}#2\dotfill #3\par
  }{%
    \par\noindent\hspace{3em}{\small #2\dotfill #3}\par
  }}%
}

\newcommand{\printappendixcontents}{%
  \section*{Appendix Contents}
  \begingroup
  \normalsize
  \InputIfFileExists{\jobname.apc}{}{}%
  \endgroup
  \immediate\openout\apptocfile=\jobname.apc\relax
}

\AtEndDocument{%
  \immediate\closeout\apptocfile
}

\newcommand{\addtoappendixcontents}[2]{%
  \ifappendixtocactive
    \begingroup
    \protected@write\apptocfile{}%
      {\string\apptocline{#1}{#2}{\thepage}}%
    \endgroup
  \fi
}

\LetLtxMacro{\oldsection}{\section}
\RenewDocumentCommand{\section}{s o m}{%
  \IfBooleanTF{#1}
    {%
      \oldsection*{#3}%
      \addtoappendixcontents{section}{#3}%
    }
    {%
      \IfNoValueTF{#2}
        {\oldsection{#3}\addtoappendixcontents{section}{\thesection\quad #3}}
        {\oldsection[#2]{#3}\addtoappendixcontents{section}{\thesection\quad #2}}%
    }%
}

\LetLtxMacro{\oldsubsection}{\subsection}
\RenewDocumentCommand{\subsection}{s o m}{%
  \IfBooleanTF{#1}
    {%
      \oldsubsection*{#3}%
      \addtoappendixcontents{subsection}{#3}%
    }
    {%
      \IfNoValueTF{#2}
        {\oldsubsection{#3}\addtoappendixcontents{subsection}{\thesubsection\quad #3}}
        {\oldsubsection[#2]{#3}\addtoappendixcontents{subsection}{\thesubsection\quad #2}}%
    }%
}

\let\oldthmhead\thmhead \renewcommand{\thmhead}[3]{%
\ifappendixtocactive \begingroup \def\thmname##1{##1}%
\def\thmnumber##1{##1}%
\def\thmnote##1{##1}%
\protected@edef\apptocentry{#1~#2: #3}%
\addtoappendixcontents{item}{\apptocentry}%
\endgroup 
\fi 
\oldthmhead{#1}{#2}{#3}%
}

\def\apptoc@figure{figure}
\def\apptoc@table{table}

\newcommand{\apptocshorten}[1]{%
  \StrLeft{#1}{30}[\apptocshort]%
  \StrLen{#1}[\apptoclen]%
  \ifnum\apptoclen>30
    \apptocshort\ldots
  \else
    \apptocshort
  \fi
}

\newcommand{\apptocaddcaption}[1]{%
  \ifappendixtocactive
    \ifx\@captype\apptoc@figure
      \addtoappendixcontents{item}{Figure~\thefigure:\quad\apptocshorten{#1}}%
    \else
      \ifx\@captype\apptoc@table
        \addtoappendixcontents{item}{Table~\thetable:\quad\apptocshorten{#1}}%
      \fi
    \fi
  \fi
}

\LetLtxMacro{\oldcaption}{\caption}

\RenewDocumentCommand{\caption}{o m}{%
  \IfNoValueTF{#1}
    {\oldcaption{#2}\apptocaddcaption{#2}}%
    {\oldcaption[#1]{#2}\apptocaddcaption{#1}}%
}

\makeatother

\title{Multistage Defer Trees for Hybrid Interpretability: If at First You Can't Succeed, Tree Again}

\author{%
  Zakk Heile\thanks{Equal contribution. Corresponding authors: \texttt{zakk.heile@duke.edu} and \texttt{hayden.mctavish@duke.edu}.} \\
  Department of Computer Science\\
  Duke University\\
  Durham, USA \\
  \texttt{zakk.heile@duke.edu} \\
  \And
  Hayden McTavish\footnotemark[1] \\
  Department of Computer Science\\
  Duke University\\
  Durham, USA \\
  \texttt{hayden.mctavish@duke.edu} \\
  \And
  Margo Seltzer \\
  Department of Computer Science\\
  University of British Columbia\\
  Vancouver, Canada \\
  \texttt{mseltzer@cs.ubc.ca} \\
  \And
  Cynthia Rudin \\
  Department of Computer Science\\
  Duke University\\
  Durham, USA \\
  \texttt{cynthia@cs.duke.edu} \\
}

\begin{document}

\maketitle

\begin{abstract}
\input{sections/abstract}

\end{abstract}

\begin{center}
\textbf{Code:} \url{https://github.com/zakk-h/MDT/}
\end{center}

\section{Introduction}

\input{sections/introduction}

\section{Related Work}

\input{sections/related-work}

\section{Methodology}

\input{sections/methods-mdt}

\section{Experiments}\label{sec:experiments}

\input{sections/experiments}

\section{Conclusion}

\input{sections/conclusion}

\bibliographystyle{plainnat}
\bibliography{references}

\newpage

\appendix

\printappendixcontents 

\appendixtocactivetrue 
\newpage

\section{Proofs}\label{appendix:theory}
\input{sections/appendix/proofs}
\FloatBarrier

\newpage
\section{Compression Details} \label{appendix:compression}
\input{sections/appendix/extra-compression-details}
\FloatBarrier

\newpage
\section{Distance to Deferred Region}
\label{appendix:geometry}
\input{sections/appendix/geometry}
\FloatBarrier

\newpage
\section{Full Algorithm}
\label{appendix:fullalgorithm}

\input{sections/appendix/algorithm}
\FloatBarrier

\newpage
\section{Training Objectives}\label{appendix:objectives}
\input{sections/appendix/objectives}
\FloatBarrier

\newpage
\section{Experiment Setup}\label{appendix:experiment-setup}
\input{sections/appendix/experimentsetup}

\FloatBarrier

\newpage
\section{Experimental Results}\label{appendix:more}
\input{sections/appendix/misc}
\FloatBarrier
\newpage
\newpage
\section{Examples of MDTs}
\input{notes/old-methods}
\FloatBarrier

\newpage

\section*{Broader Impacts Statement}
This paper studies interpretable machine learning, a key component of many ethical AI applications. However, the proposed method is general-purpose and is not designed for any specific application.


\end{document}

%% file: sections/abstract.tex
Recent work has shown that well-optimized individual decision trees can match complex black box models in some settings, primarily in noisy domains. For the remaining settings, however, complex ensembled compositions of trees often achieve higher accuracy at the cost of interpretability, leaving practitioners with difficult modeling decisions along an accuracy-interpretability tradeoff. Ideally, we would like to classify as much of the data as possible with one or a small number of trees, achieving interpretability for most samples while maintaining state-of-the-art accuracy. We introduce Multistage Defer Trees: a sequence of sparse decision trees that each make predictions for most samples, while deferring a small proportion to the next tree in the sequence or, ultimately, to a black box. We demonstrate that we can train this model class to match the performance of complex tree-based ensembles while routing most samples through only one or a small number of sparse decision trees. We discuss a range of techniques for training these models while maintaining simplicity. Our method expands the accuracy--interpretability frontier in settings where single-tree methods remain insufficient, demonstrating that even when complex models are necessary, they need not be fully opaque.

%% file: sections/introduction.tex
A central challenge in interpretable machine learning is to construct interpretable models that match state-of-the-art predictive performance. Well-optimized simple models, such as optimal sparse trees, can achieve this goal in certain settings, such as when the data 
is generated by noisy processes \citep{semenova2022existence, semenova2023path, boner2024using}. In such cases, the accuracy-interpretability gap is not observed. Still, there remains such a gap in datasets falling outside of this regime.

One proposed direction to navigate the accuracy-interpretability gap is to use a hybrid- or partially-interpretable model \citep{wang2021hybrid, wang2019gaining, pan2020interpretable, ferry2023learning, frost2024partially, Li2026}. The goal is to learn a simple interpretable model that, for each sample, either makes a prediction for that sample or defers to a black box's prediction. By specifying the proportion of samples classified by a black box (the deferral rate), users can navigate a tradeoff between matching black box accuracy and keeping decisions interpretable for more samples. While hybrid-interpretable models provide a smooth way to navigate compromises between accuracy and interpretability, existing methods often struggle to maintain black box accuracy without deferring the vast majority of samples.

\begin{figure}
    \centering
    \includegraphics[width=1.0\linewidth]{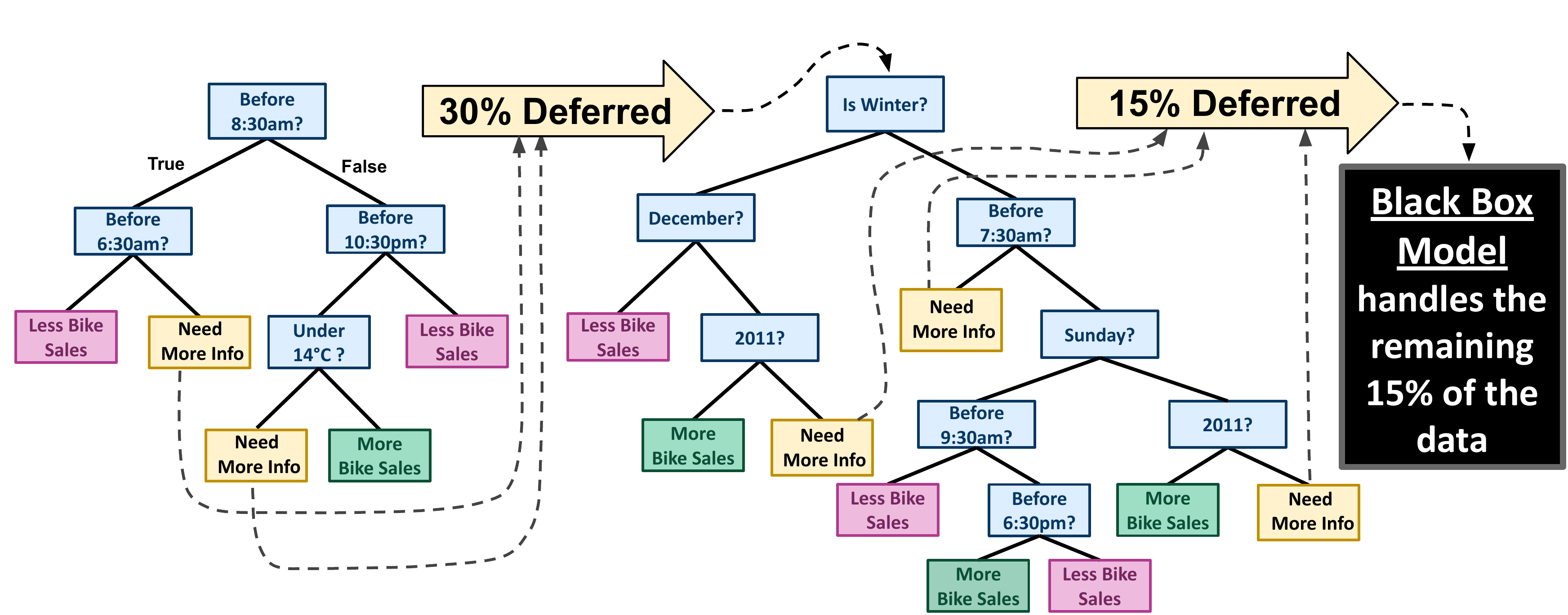}
    \caption{Example of an MDT on the Bike dataset. The model consists of two stages, each with bounded depth. Two leaves in the first stage defer to the second stage, and three leaves in the second stage defer to a black box (here, an XGBoost model). The interpretable component has 13 leaves.
    }
    \label{fig:bike_mdt_example}
\end{figure}

We present a tree-based hybrid-interpretable model that minimizes deferral rate while matching the performance of more complex ensembles. To learn this model, we iteratively focus on subsets of the data where simple models struggle to perform as well as a black box model. This leads to a new model class, \emph{Multistage Defer Trees (MDTs)}, where each tree can either make a prediction or defer to the next stage; the final stage may then defer some samples to a black box model if required. 

To train MDTs, we introduce an alternating optimization procedure in which the deferred-to model evolves with each stage of the MDT. As this process unfolds, the set of deferred points decreases monotonically, and the later stages and fallback model become increasingly specialized to those points. This allows the model to defer only on a small fraction of inputs while maintaining high accuracy. When the fallback is a sparse model (i.e., a single decision tree or EBM \cite{Lou2013AccurateIM}), even these deferral cases remain interpretable; when the fallback is a black box, performance can be matched with minimal reliance on that model. Crucially, because deferred points are confined to a tightly constrained region of the feature space, the fallback model can be compressed, producing simpler but equivalent representations of even complex ensembles.

Our contributions are as follows.
\begin{itemize}
    \item We introduce Multistage Defer Trees (MDTs), a new model class that allocates complexity adaptively, based on the difficulty of different regions of the feature space.
    \item We develop an iterative algorithm for training MDTs that successively narrows the deferral region, while borrowing strength from surrounding regions, to learn specialized trees that generalize well.
    \item We provide algorithms to compress MDTs into sparse single-tree or sparser sequential representations and extend this framework to compress the fallback model when it is a tree or tree ensemble.
    \item We demonstrate improved accuracy–deferral–sparsity trade-offs where such trade-offs exist.
\end{itemize}

%% file: sections/related-work.tex
\paragraph{Hybrid Approaches}
\citet{wang2019gaining} introduces hybrid models for interpretability: an interpretable rule set or linear model makes predictions on some of the data, while deferring to fixed black box predictions on the remaining data. This allows users to navigate a tradeoff between the model accuracy and the proportion of samples that have interpretable model predictions. This framework, in turn, connects to work in deferral models \cite{madras2018predict, mozannar2020consistent}, rejection learning \cite{chow2003optimum, SANTOSPEREIRA2005943}, selective classification \cite{el2010foundations}, and abstention-based learning \cite{schapire, freund1997using}. Hybrid models have primarily been applied to simple rule set models or linear models \cite{frost2024partially, ferry2023learning, wang2019gaining, pan2020interpretable, wang2021hybrid}. One recent work explores a gradient-based method to learn a complex (100-1000 leaves) decision tree, which defers to a black box based on computing the proximity to the decision boundary \cite{Li2026}; by contrast, our work learns a sparser sequence of trees that directly defer where the black box is most helpful, and iteratively reduces this deferral region.

\paragraph{Decision Tree Optimization}
Decision trees are classically optimized with highly scalable greedy algorithms \citep{breiman1984classification, quinlan1986induction}. Such algorithms can be ideal for popular ensembles of many trees \citep{breiman2001random, chen2016xgboost}. 
More recent work, however, revisits the use of less scalable, higher quality individual tree algorithms, to reduce the need for this level of model complexity and to allow use of a single well-performing sparse tree classifier for domains where interpretability is needed. In recent years, global optimization of individual trees has become reasonably tractable for bounded depth and sparsity \citep{dl85, osdt, murtree, gosdt, quantbnb, brița2025optimal, bertsimas2017optimal, verwer2019learning}. 
This problem is nevertheless NP-hard and existing solutions struggle to scale to cases with many features. Recent work sacrifices global optimality in exchange for improved scalability and runtime while maintaining performance empirically close to optimal \citep{babbar2025near, topk, kiossou2024efficient}. Other approaches add anytime behavior, allowing early termination with high-quality solutions \citep{kiossou2022time, demirovic2023blossom, kiossou2025generic}.

We fit a small sequence of trees that nearly matches the performance of a significantly more complex ensemble. Because MDTs rely on only a few sparse trees, we can afford to optimize each component tree more carefully than in standard greedy tree induction while still requiring the overall procedure to scale. Given this need for both higher-quality component trees and scalability, we use a version of the polynomial-time LicketySPLIT algorithm from \citet{babbar2025near}, which we modify to support sample weights and deferral; our theory provides guarantees achievable when each component tree is instead exactly optimal with respect to misclassification error, deferral rate, and tree complexity. Many less-greedy decision tree algorithms, including LicketySPLIT, operate on binarized features, which are often constructed using threshold guessing \citep{mctavish2022fast}; we use the same binarization method.

\paragraph{Ensemble-Guided Tree Search} \citet{vidal2020born} present an algorithm to find the sparsest single decision tree with 0-1 predictive behavior identical to that of a given tree-based ensemble (though not necessarily identical predicted probabilities), inspired by an approach approximating Random Forest behavior with a single tree \citep{breiman1996born}. They additionally provide practical relaxations to allow for faster computation (removing the certificate that the tree is the sparsest equivalent form) or sparser representations (providing extra pruning in areas with no support in the training data). \citet{sagi2021approximating} translate an XGBoost tree ensemble \citep{chen2016xgboost} to a single tree, with some approximation loss.
\citet{mctavish2022fast} take a slightly different approach to learn a single tree when provided with a tree-based ensemble, using the ensemble's predictions to provide heuristic pruning of the search space for an otherwise globally optimal tree. All these methods afford more interpretable single tree alternatives when provided with an ensemble model, though to guarantee matching performance with that ensemble, these approaches often find trees with many more leaves than is ideal for interpretability.

In the hybrid component of our own algorithm, we incorporate information from a black box model in a fundamentally different way than the above approaches do. Our goal is not to match the black box model's predictions exactly, nor to accelerate an optimal tree search. Instead our goal is to use the black box to resolve subspaces of the data for which interpretable algorithms cannot find an accurate simple model. This allows us to find models that are transparent on much of the data, without sacrificing meaningful accuracy relative to the black box. While we discuss several useful properties when the black box is a tree-based ensemble, our approach inherits the more general-purpose properties of hybrid models and can work with any black box model, since we only need to be able to train the black box and obtain its predictions.

%% file: sections/methods-mdt.tex
\paragraph{Notation.} 
Let $\mathcal X \subseteq \mathbb{R}^p$ denote an input space with $p$ features and for simplicity, let $\mathcal Y=\{0,1\}$ denote the label space for binary classification (we additionally support multiclass classification and regression in our code).
 Training data is denoted as $\mathcal D = \{(x_i,y_i)\}_{i=1}^N \subseteq \mathcal X \times \mathcal Y$. When this data is weighted, we denote the dataset as $\mathcal D^w$ for weight vector $w \in \mathbb R_{\ge 0}^N$. We write $\mathds{1}\{\cdot\}$ for the indicator function. For a binary decision tree $T$, let $|T|$ denote its number of leaves, and let $T(x)$ denote the label predicted by $T$ for input $x$. 

A standard binary decision tree recursively partitions $\mathcal X$ by internal boolean split nodes and assigns a prediction at each leaf. 
We generalize this object by allowing a leaf to output either a class label or a special \textit{defer} action, indicating that prediction should be deferred to a subsequent model.

\begin{definition}[Defer tree]
\label{def:defer_tree}
A \emph{defer tree} $T$ is a binary tree whose leaves are labeled by elements of $\mathcal Y \cup \{\defer\}$. For some fallback model B, the prediction of a defer tree is: 
\[
\hat y_{T,B}(x)
=
\begin{cases}
T(x), & \text{if } T(x) \neq \defer ,\\[4pt]
B(x), & \text{otherwise}.
\end{cases}
\]
\end{definition}

For a fixed fallback model $B$, we can train a single defer tree $T$ by optimizing (without a certificate of optimality)
the objective $\mathcal L_{T, B}$ defined by
\[
\mathcal L_{T, B}(\mathcal D^w,  \tau, \eta)
=
\tau\,(|T|-1)
+
\sum_{i=1}^N w_i\,(\mathds{1}\{\hat{y}_{T,B}(x_i) \neq y_i\} + \eta \mathds{1}\{{{T}(x_i) = \defer}\}),
\]
where the hyperparameter $\tau$ penalizes the number of splits in the tree, $\eta$ penalizes deferrals to a black box model $B$, and $w$ denotes the sample weights for dataset $\mathcal D$. We optimize this objective using a modified version of LicketySPLIT \cite{babbar2025near}; details are given in Appendix \ref{appendix:licketysplitsection}.

A single defer tree can improve accuracy over a fully interpretable model by deferring difficult samples to a black box. However, those deferred samples still have a problem. On that subset, we are still relying on a black box, and we would like to recover interpretability there as well. It may seem natural to simply grow the tree further on that region, but if the original defer tree is already optimal, this is no longer possible. An optimal defer tree comes with a certificate that rules out further improvement: any attempt to replace a deferral with an interpretable prediction must increase the objective. In particular, matching the black box’s accuracy on the deferred subproblems would require a complexity increase whose leaf penalty outweighs the savings from eliminating deferral (formalized in Proposition \ref{thm:defer-leaf-tradeoff} in the appendix). 

To overcome this limitation, we move beyond a single tree and introduce a sequential model that allows further improvement on the deferred region of a single defer tree. 

\begin{definition}[Multistage defer tree]
\label{def:mdt}
A \emph{multistage defer tree} (MDT) of length $K$ is an ordered sequence
\[
T_{1:k} = (T_1,\dots,T_K),
\]
where each $T_i:\mathcal X\to \mathcal Y\cup\{\defer\}$ for $i=1,\dots,K$ is a defer tree.
For some fallback model $B:\mathcal X \to \mathcal Y$, the prediction of the MDT is
\[
\hat y_{T_{1:k};B}(x)
=
\begin{cases}
T_{1}(x), & \text{if } T_1(x) \neq \defer ,\\[4pt]
\hat{y}_{T_{2:k}, B}(x), & \text{otherwise}.
\end{cases}
\]
\end{definition}

An MDT routes each input through a sequence of defer trees, where each stage either predicts immediately or defers to the next stage. It also affords interpretability benefits relative to other tree-based ensembles. As discussed in \autoref{thm:rule-list} below, our  multistage defer trees can be represented relatively compactly as a rule list. This is significantly smaller than the rule list representations of other ensembles, such as random forests or boosted trees, where even the number of literals in a single antecedent can grow with the sum of depths across all trees (see  \autoref{appendix:complexity_proof} in the appendix).

\begin{restatable}[Rule-list representation of MDTs]{theorem}{rulelistmdt}\label{thm:rule-list}
An MDT with $K$ stages admits a rule-list representation,
in which each non-deferral leaf in any stage contributes exactly one antecedent. Consequently, if stage $j$ has $n_j$ non-deferral leaves, the total number of antecedents is $\sum_{j=1}^K n_j$.
\end{restatable}

Because of Theorem \ref{thm:rule-list}, we can define a simple objective for MDTs that corresponds to an upper bound on the rule list deferral objective in \citet{ferry2023learning} applied to the rule list form of the MDT (we discuss our relation to existing objectives in Appendix \ref{appendix:objectives}).

For a given MDT $(T_1, \ldots, T_k)$ and fallback model $B$, let its objective $\mathcal L_{T_{1:k}, B}$ be: 
\[
\mathcal L_{T_{1:k}, B}(\mathcal D^w, \tau, \eta)
=
\tau \sum_{j=1}^k \bigl(|T_j|-1\bigr)
+
\sum_{i=1}^N w_i\,(\mathds{1}\{\hat{y}_{T_{1:k,B}}(x_i) \neq y_i\} + \eta \mathds{1}\{{{T_{1:k}}(x_i) = \defer}\}),
\]

where ${T}_{1:k}(x)$ denotes the output of the sequence of trees $(T_1,\dots,T_k)$.

\subsection{Training Multistage Defer Trees}

When optimizing our MDT objective $\mathcal L_{T_{1:k}, B}$, we can expand the structure from the top down, training a defer tree at each stage on the remaining deferred region with our single defer tree objective $\mathcal L_{T,B}$. When we do so, Proposition \ref{prop:monotone-topdown} shows that we monotonically decrease the objective defined above.

\begin{restatable}[Top-down improvement]{proposition}{topdownimprovement}
\label{prop:monotone-topdown}
Consider an iterative procedure where we train a single defer tree on the samples that have been deferred in all previous stages. Define weights for stage $j$ that zero out all samples that are not deferred by that stage:
$$w^{(j)}_i := w_i^{(j-1)}\mathds{1}\{T_{j-1}(x_i)=\mathrm{defer}\},
\quad
w^{(1)}_i := w_i.$$

Let $\Delta_j$ be how much we improve on the objective of the deferred dataset $\mathcal D^{w^{(j)}}$ with defer tree $T_j$, relative to deferring to a black box $B$: 

$$\Delta_j = \mathcal L_{\defer, B}(\mathcal D^{w^{(j)}}, \tau, \eta) - \mathcal L_{T_j, B}(\mathcal D^{w^{(j)}}, \tau, \eta). $$

With this iterative procedure, the total improvement of the model relative to the black box is the sum of improvements of each single defer tree relative to deferral.
$$\mathcal L_{T_{1:k}, B}(\mathcal D^w, \tau, \eta) = L_{\defer, B}(\mathcal D^{w^{}}, \tau, \eta) - \sum_{j=1}^k \Delta_j.$$
\end{restatable}

While Proposition \ref{prop:monotone-topdown} motivates our approach, there is a risk of a hybrid model deferring to the black box primarily in regions where the black box has overfit. Consider subregions of the space where a black box's performance is especially high due to sample bias. These are also regions where it will be difficult for a simple model to match that performance, so a hybrid model is likely to defer in this case, leading to high test error in deferral cases. 

Even for a single defer tree, the generalization of the black-box component can be a substantial issue. Table \ref{tab:defer_tree_error_split} in the appendix shows that the test error of a black box in the deferred region can be more than an order of magnitude higher than the test error of the simple model in the non-deferred region. To address this, we also need to leverage the rest of the data to improve generalization.

We can continue to refine an MDT without inducing overfitting by reweighting the training distribution: non-deferred points are downweighted based on their distance to the deferred region, rather than assigning them no weight. At a high level, we alternate between fitting a black box and a defer tree and updating the training weights to emphasize points within or near the deferred region.

Let $\mathcal{D}_{k-1} \subseteq \mathcal{X}$ denote the set of deferred points after stage $k-1$. For any sample $x_i \not \in \mathcal{D}_{k-1}$, let $\textrm{defer\_dist}_{T_{1:k-1}}(x_i)$ denote the $\ell_1$-norm (taken in quantile space, more information in Appendix \ref{appendix:geometry}) of the minimal change to $x_i$ such that it would be deferred by $T_{\textrm{1:k-1}}$. Using this distance, we define training weights $w_i$ for the next iteration (on the first iteration, all weights are 1).
$$w_i \;=\;
\begin{cases}
1, & \text{if } x_i \in \mathcal{D}_{k-1} \\[6pt]
(1 - \mu)\,(1 +\textrm{defer\_dist}_{T_\textrm{{1:k-1}}}(x_i))^{-\gamma}, & \text{if } x_i \notin \mathcal{D}_{k-1}.
\end{cases}$$
Here, $\mu \in [0,1]$ controls how much we uniformly scale down the weights of the non-deferred region, while $\gamma \geq 0$ controls how strongly the weights decay with distance from the deferred region.
When $\gamma=0$, this reduces to assigning weight $1-\mu$ to each non-deferred point. $\mu=1$ places weight solely on the deferred points.

\begin{algorithm}[t]
\caption{Top-down MDT training}
\label{alg:weighted_top_down_mdt}
\begin{algorithmic}[1]
\REQUIRE Data $(X,y)$, depth $d$, tree-complexity regularization $\tau$, deferral regularization $\eta$, mean reweighting parameter $\mu$, distance downweighting parameter $\gamma$, maximum number of stages $K_{\max}$, tree-complexity regularization rescaling flag \texttt{rescale\_tau}
\STATE $\mathcal T\gets[\,],\ D\gets\{1,\ldots,N\},\
\mathcal S\gets\{\mathbb R^p\},\ w\gets (1,1,\ldots,1) \in \mathbb{R}^n$ \COMMENT{Each data point starts with a weight of 1; we defer discussion of $\mathcal{S}$ to \autoref{subsec:feasible-region-tracking}}
\FOR{$k=1,\ldots,K_{\max}$}
    \STATE $B\gets\textsc{FitFallbackModel}(X,y,w)$ \COMMENT{Train a temporary fallback model; for instance \texttt{XGBoost.fit}}
    \STATE $T_k\gets\textsc{FitDeferTree}(X,y,B(X),w,\mathcal S,\tau,\eta,d)$
\COMMENT{Train the next defer tree pretending it is deferring to $B$, discussed in \autoref{subsec:defer}}
    \STATE $D'\gets\{i\in D:T_k(x_i)=\defer\}$
    \IF{$D'=D$}
        \STATE \textbf{break}
    \ENDIF
    \STATE $\mathcal T\gets\mathcal T+[T_k],\quad D\gets D'$
    \IF{$D=\varnothing$}
        \STATE \textbf{break}
    \ENDIF
   \STATE $(\mathcal S,w)\gets
\textsc{UpdateStage}(X,\mathcal S,T_k,D,\mu,\gamma)$\COMMENT{Get the weights and subregions used in the next stage (discussed in Algorithm \ref{alg:update_stage})}
    \IF{\texttt{rescale\_tau}}
        \STATE $\tau\gets\tau\cdot\frac{\sum_i w_i}{N}$ \COMMENT{Rescale regularization for the size of the subproblem}
    \ENDIF
    \IF{$k=K_{\max}$}
        \STATE $B\gets\textsc{FitFallbackModel}(X,y,w)$ \COMMENT{Train the final fallback model}
    \ENDIF
\ENDFOR
\STATE \textbf{return} $\mathcal T+[B]$
\end{algorithmic}
\end{algorithm}

Algorithm \ref{alg:weighted_top_down_mdt} describes our training procedure. At each stage, we fit a fallback model and a defer tree using the current sample weights, initially all 1. If the deferred set decreases in size, we add the new tree to the sequence and update the feasible deferred subregions $\mathcal{S}$ and sample weights $w$ using \textsc{UpdateStage}, described below. Training stops when the deferred set becomes empty or stops shrinking, or when $K_{\max}$ stages have been reached. In the latter case, we retrain the fallback using the final weights. Algorithm \ref{alg:alternating_backfitting_mdt} in the appendix gives an optional backfitting procedure to improve an initial MDT; in practice, we find it is not needed. In Appendix \ref{appendix:fullalgorithm}, we also describe early stopping criteria based on model complexity (e.g., per-stage size, equivalent rule list size, or equivalent single-tree size), along with efficient online updates to enforce these constraints.

We prove additional properties of our method in Appendix \ref{appendix:theory}. For instance, we show that the distance-based weights arise as a Gibbs distribution minimizing an entropy-regularized objective. 

\subsection{Feasible Region Tracking}
\label{subsec:feasible-region-tracking}

Let $\Lambda_k^{\text{defer}}$ denote the set of defer leaves of tree $k$. Each defer leaf $\lambda \in \Lambda_k^{\text{defer}}$ corresponds to an axis-aligned hyperrectangle in feature space, denoted $R(\lambda)$. The total deferral region after stage $k$ can be written as a set of defer subregions in the following way: 

$$
\mathcal{S}_k
:=
\left(
\left\{
\bigcap_{t=1}^k R(\lambda_t)
\;\middle|\;
(\lambda_1,\dots,\lambda_k)\in \Lambda_1^{\mathrm{defer}}\times\cdots\times\Lambda_k^{\mathrm{defer}}
\right\}
\right)
\setminus \{\emptyset\}.
$$

We refer to each element $\bigcap_{t=1}^k R(\lambda_t)$ in the set $\mathcal{S}_{k}$  as a \emph{defer subregion}. Each such subregion is an axis-aligned hyperrectangle. 
Many such subregions may be empty (infeasible), e.g., $age \leq 20$ in one stage and $age > 25$ in another; we detect and discard these cases. 

\begin{algorithm}[t]
\caption{Update stage}
\label{alg:update_stage}
\begin{algorithmic}[1]
\REQUIRE $X,\mathcal S,T,D,\mu,\gamma$
\STATE $\mathcal S\gets
\{s\cap R(\lambda)\mid
s\in\mathcal S,\lambda\in\Lambda^{\defer}(T)\}
\setminus\{\varnothing\}$
\FOR{$i=1,\ldots,N$}
    \STATE $\delta_i\gets\min_{s\in\mathcal S}
    \sum_{j=1}^p
    \left(
    \max\{0,L_{sj}-z_j(x_{ij})\}
    +\max\{0,z_j(x_{ij})-U_{sj}\}
    \right)$
    \STATE $w_i\gets 1$ if $i\in D$, else
    $(1-\mu)(1+\delta_i)^{-\gamma}$
\ENDFOR
\STATE \textbf{return} $(\mathcal S,w)$
\end{algorithmic}
\end{algorithm}

Because our iterative procedure refines only the set of subregions that defer, we can update this set iteratively. 
Algorithm \ref{alg:update_stage} performs this update efficiently and computes the quantities needed for the next stage. We compute the distance from each sample to the deferred region as the minimum distance to any defer subregion. For continuous features, $z_j$ maps each value to quantile space by counting the empirical mass below it plus half the mass at that value. The values $L_{sj},U_{sj}$ are the boundaries of the defer subregion, snapped inward to the closest observed value, and mapped into quantile space. The axis-aligned and rectangular structure of a defer subregion allows us to efficiently compute the $\ell_1$ distance to each subregion by summing, across features, the distance to the nearest value within the subregion's bounds. For features unrestricted by a subregion, we set the quantile-space bounds to  $L_{sj}=0$ and $U_{sj}=1$, so they contribute zero to the distance. Categorical feature handling is described in Appendix \ref{appendix:geometry}.

\begin{figure*}[!b]
\centering
\begin{minipage}[t]{0.48\textwidth}
\begin{algorithm}[H]
\caption{\textsc{DeferTree}($Z,y,r,w,h,\tau,\eta$)}
\label{alg:fit_defer_tree_main}
\begin{algorithmic}[1]
\REQUIRE Binarized dataset $Z$, labels $y$, fallback predictions $r$, sample weights $w$, depth $h$, sparsity penalty $\tau$, deferral penalty $\eta$

\STATE \COMMENT{Find the best leaf action}
\STATE $(T_\ell,C_\ell)\gets \textsc{Greedy}(Z,y,r,w,0,\tau,\eta)$
\IF{$h=0$ \textbf{or} $|Z|\leq1$ \textbf{or} $\mathcal F(Z)=\varnothing$}
    \STATE \textbf{return} $(T_\ell,C_\ell)$
\ENDIF

\STATE \COMMENT{Evaluate splits using greedy completions}
\FOR{$f\in\mathcal F(Z)$}
    \FOR{$b\in\{0,1\}$}
        \STATE $(Z_b,y_b,r_b,w_b)\gets
        \textsc{Filter}(Z,y,r,w,f,b)$
        \STATE $(T^G_{f,b},C^G_{f,b})\gets
        \textsc{Greedy}(Z_b,y_b,r_b,w_b,h \textrm{ -- } 1,\tau,\eta)$
    \ENDFOR
    \STATE $C^G_f\gets C^G_{f,0}+C^G_{f,1}$
\ENDFOR

\STATE $f^*\gets\arg\min_f C^G_f$

\STATE \COMMENT{Recursively fit the selected split}
\FOR{$b\in\{0,1\}$}
    \STATE $(Z_b,y_b,r_b,w_b)\gets
    \textsc{Filter}(Z,y,r,w,f^*,b)$
    \STATE $(T_b,C_b)\gets
    \textsc{DeferTree}(Z_b,y_b,r_b,w_b,h \textrm{ -- } 1,\tau,\eta)$
\ENDFOR

\STATE \COMMENT{Keep split only if it improves over leaf}
\IF{$C_0+C_1<C_\ell$}
    \STATE \textbf{return}
    $(\mathrm{split}(f^*,T_0,T_1),C_0+C_1)$
\ENDIF
\STATE \textbf{return} $(T_\ell,C_\ell)$
\end{algorithmic}
\end{algorithm}
\end{minipage}
\hfill
\begin{minipage}[t]{0.48\textwidth}
\begin{algorithm}[H]
\caption{\textsc{Greedy}$(Z,y,r,w,h,\tau,\eta)$)}
\label{alg:greedy_defer_tree_main}
\begin{algorithmic}[1]
\REQUIRE Binarized features $Z$, labels $y$, fallback predictions $r$, sample weights $w$, depth $h$, sparsity penalty $\tau$, deferral penalty $\eta$

\STATE \COMMENT{Compute the cost of each leaf action}
\FOR{each class $c\in\mathcal Y$}
    \STATE $L_c\gets\tau+\sum_{i=1}^{|Z|}
    w_i\mathds{1}\{y_i\neq c\}$
\ENDFOR
\STATE $L_{\defer}\gets\tau+\sum_{i=1}^{|Z|}w_i
(\mathds{1}\{r_i\neq y_i\}+\eta)$

\STATE \COMMENT{Find the best leaf action}
\STATE $a^*\gets\arg\min_{a\in\mathcal Y\cup\{\defer\}}L_a$
\STATE $(T_\ell,C_\ell)\gets(\mathrm{leaf}(a^*),L_{a^*})$
\IF{$h=0$ \textbf{or} $|Z|\leq1$ \textbf{or} $\mathcal F(Z)=\varnothing$}
    \STATE \textbf{return} $(T_\ell,C_\ell)$
\ENDIF

\STATE \COMMENT{Select the best split according to weighted-entropy}
\STATE $f^*\gets\textsc{BestEntropySplit}(Z,y,w)$

\STATE \COMMENT{Recursively fit both children}
\FOR{$b\in\{0,1\}$}
    \STATE $(Z_b,y_b,r_b,w_b)\gets
    \textsc{Filter}(Z,y,r,w,f^*,b)$
    \STATE $(T_b,C_b)\gets
    \textsc{Greedy}(Z_b,y_b,r_b,w_b,h \textrm{ -- } 1,\tau,\eta)$
\ENDFOR

\STATE \COMMENT{Keep split only if it improves over leaf}
\IF{$C_0+C_1<C_\ell$}
    \STATE \textbf{return}
    $(\mathrm{split}(f^*,T_0,T_1),C_0+C_1)$
\ENDIF
\STATE \textbf{return} $(T_\ell,C_\ell)$
\end{algorithmic}
\end{algorithm}
\end{minipage}
\end{figure*}

Feasible region tracking can also be used to compress the resulting MDT. 
In Appendix \ref{appendix:compression}, we provide an algorithm to compress the MDT into a single defer tree. We recursively expand each defer leaf by substituting the next-stage tree, while tracking the corresponding subregion of the feature space that we are in. If a split is guaranteed to evaluate the same way throughout this region (i.e., always True or always False), we can replace the node with the corresponding child. We also provide algorithms to further simplify the MDT by eliminating redundant splits; this has the effect of reducing the number of rules or their width in the MDT's rule-list representation. We can apply either of these techniques to each tree in a tree ensemble, using the set of subregions in which that tree may be evaluated.

\subsection{Defer Tree Algorithm}\label{subsec:defer}

Algorithm \ref{alg:fit_defer_tree_main} presents our default procedure for constructing defer trees, which uses the greedy subroutine in Algorithm \ref{alg:greedy_defer_tree_main}. Both algorithms recursively construct a defer tree on the current subproblem, represented by a binarized dataset $Z$, labels $y$, fallback predictions $r$, and sample weights $w$, returning a subtree $T$ and its objective value $C$. Before calling \textsc{DeferTree}, \textsc{FitDeferTree} can update the binarization at each stage in two ways: by generating new thresholds (e.g., using XGBoost) or by starting with a fixed binarization and removing binary columns that are constant across all feasible defer subregions from a fixed binarization (Appendix \ref{appendix:compression}). We compare these approaches against each other and a continuous-feature approach in Appendix \ref{appendix:more}.

Here, $\textsc{Filter}(Z,y,r,w,f,b)$ restricts the dataset to samples satisfying $Z_{if}=b$, returning the corresponding subsets of $Z$, $y$, $r$, and $w$. At each node, we select the best leaf action: predicting a class or deferring to the fallback. We use a per-leaf penalty $\tau$ (equivalent in ranking to a per-split penalty). \textsc{DeferTree} selects splits using greedy subtree completions before recursively fitting the children, while \textsc{Greedy} selects splits using weighted label entropy. Both retain a split only if it improves upon the best leaf.

%% file: sections/experiments.tex
We compare our MDTs and the defer tree method to state-of-the-art tree ensembles (XGBoost \cite{chen2016xgboost}, Random Forest \cite{breiman2001random}), sparse additive tree models (FIGS \cite{tan2025fast}), and existing hybrid interpretable baselines (HyRS \cite{wang2019gaining, wang2021hybrid}, Pre- and Post- CORELS \cite{ferry2023learning}) on a wide range of datasets. We also implemented the hybrid interpretable model from \citet{frost2024partially}, and refer to the method as FLMM. Lastly, we add a logistic regression baseline, which defers when the estimated probability is close to 0.5. 
Our key questions are as follows.
(1) Can MDTs provide accuracy comparable to that of black-box models for a small deferral rate? (2) How do MDT deferral rates and accuracy compare to existing deferral methods? (3) How complex are state-of-the-art MDTs?

\begin{table}[!b]
\centering
\small
\setlength{\tabcolsep}{4pt}
\renewcommand{\arraystretch}{0.88}
\resizebox{\textwidth}{!}{
\begin{tabular}{lccccc}
\toprule
Dataset 
& Non-Deferral Tree 
& FIGS
& Random Forest 
& MDT+XGB $\le$25\% (Ours)
& XGBoost \\
\midrule

Abalone & $0.6314 \pm 0.0060$ & $\mathbf{0.6381 \pm 0.0030}$ & $0.6314 \pm 0.0035$ & $\underline{0.6333 \pm 0.0069}$ & $\underline{0.6323 \pm 0.0070}$ \\

Adult & $0.8598 \pm 0.0017$ & $0.8589 \pm 0.0021$ & $0.8619 \pm 0.0017$ & $0.8667 \pm 0.0021$ & $\mathbf{0.8701 \pm 0.0013}$ \\

Aging & $\mathbf{0.8028 \pm 0.0034}$ & $0.7944 \pm 0.0047$ & $\underline{0.8014 \pm 0.0036}$ & $\mathbf{0.8028 \pm 0.0034}$ & $\underline{0.8000 \pm 0.0028}$ \\

Bank & $0.9051 \pm 0.0008$ & $0.9052 \pm 0.0017$ & $0.9069 \pm 0.0008$ & $\underline{0.9087 \pm 0.0008}$ & $\mathbf{0.9090 \pm 0.0007}$ \\

Bike & $0.9346 \pm 0.0021$ & $0.9113 \pm 0.0039$ & $0.9426 \pm 0.0017$ & $\underline{0.9486 \pm 0.0024}$ & $\mathbf{0.9500 \pm 0.0023}$ \\

California & $0.8893 \pm 0.0024$ & $0.8774 \pm 0.0015$ & $0.9019 \pm 0.0032$ & $\underline{0.9099 \pm 0.0020}$ & $\mathbf{0.9101 \pm 0.0022}$ \\

Churn & $0.9424 \pm 0.0036$ & $0.9434 \pm 0.0057$ & $0.9506 \pm 0.0032$ & $\underline{0.9566 \pm 0.0028}$ & $\mathbf{0.9574 \pm 0.0028}$ \\

Droid & $0.9627 \pm 0.0011$ & $0.9601 \pm 0.0005$ & $\underline{0.9715 \pm 0.0006}$ & $0.9709 \pm 0.0005$ & $\mathbf{0.9718 \pm 0.0004}$ \\

Heloc & $\underline{0.7052 \pm 0.0107}$ & $0.6840 \pm 0.0074$ & $\underline{0.7080 \pm 0.0102}$ & $\underline{0.7056 \pm 0.0107}$ & $\mathbf{0.7116 \pm 0.0108}$ \\

Jasmine & $0.8037 \pm 0.0077$ & $0.7983 \pm 0.0054$ & $\mathbf{0.8218 \pm 0.0085}$ & $\underline{0.8161 \pm 0.0100}$ & $\underline{0.8214 \pm 0.0071}$ \\

Magic
& $0.8541 \pm 0.0066$
& $0.8539 \pm 0.0045$
& $0.8798 \pm 0.0059$
& $0.8756 \pm 0.0111$
& $0.8849 \pm 0.0064$ \\

Phishing & $0.9528 \pm 0.0017$ & $0.9392 \pm 0.0019$ & $\underline{0.9721 \pm 0.0013}$ & $\underline{0.9728 \pm 0.0017}$ & $\mathbf{0.9733 \pm 0.0015}$ \\

Pol & $0.9723 \pm 0.0012$ & $0.9631 \pm 0.0018$ & $0.9823 \pm 0.0007$ & $\mathbf{0.9849 \pm 0.0003}$ & $\underline{0.9845 \pm 0.0007}$ \\

Rl & $0.7614 \pm 0.0052$ & $0.7163 \pm 0.0094$ & $0.8074 \pm 0.0058$ & $0.7968 \pm 0.0089$ & $\mathbf{0.8278 \pm 0.0034}$ \\

Shopping & $0.9010 \pm 0.0025$ & $0.8988 \pm 0.0021$ & $\underline{0.9046 \pm 0.0012}$ & $\underline{0.9047 \pm 0.0016}$ & $\mathbf{0.9057 \pm 0.0015}$ \\

Spambase & $0.9254 \pm 0.0016$ & $0.9222 \pm 0.0034$ & $\underline{0.9517 \pm 0.0014}$ & $\underline{0.9500 \pm 0.0024}$ & $\mathbf{0.9528 \pm 0.0023}$ \\

Wine & $0.8340 \pm 0.0015$ & $0.8391 \pm 0.0039$ & $\underline{0.8848 \pm 0.0042}$ & $0.8788 \pm 0.0052$ & $\mathbf{0.8864 \pm 0.0022}$ \\

\bottomrule
\end{tabular}
}
\caption{Test accuracy (mean $\pm$ standard error across five train/test splits) across representative datasets. Bold indicates the best mean accuracy in each row. Underlining indicates methods whose mean accuracy is within $1.96$ standard errors of the best method, corresponding to overlap with the best method's 95\% confidence interval.
}
\label{tab:clean_results}
\end{table}

To answer these questions, we need to determine the best configurations for hybrid interpretable models subject to a particular constraint on test deferral rate. To do so, we adopt a semi-supervised approach: 
for each train/test split and maximum deferral rate $c$, we choose the hyperparameter that maximizes the average validation accuracy, while deferring on at most $c$ proportion of the test set. Note that deferral is independent of labels, so this is fully semi-supervised. 

In Appendix \ref{appendix:more} (Tables   \ref{tab:train_test_defer_rate_stability_25} and \ref{tab:train_test_defer_rate_stability_10}), we show that the gap between train and test deferral rates is always small, suggesting that a fully supervised selection procedure would yield nearly identical results. Indeed, we show this is the case in the appendix (\autoref{fig:train_deferral_comparison}). 
Additional information on our experiment setup is in Appendix \ref{appendix:experiment-setup}; all results are averaged across five train/test splits. The mean and median MDT training times across datasets were 132.1 and 55.4 seconds, respectively (Table \ref{appendix:timing} in the appendix).

To answer question (1), Table~\ref{tab:clean_results} 
reports test accuracy across datasets. MDT+XGB (with at most 25\% deferral) is consistently competitive even with the strongest black box baselines: it is frequently closer to state-of-the-art performance than random forests and almost always within the 95\% confidence interval of the best method. Furthermore, it outperforms the sparse ensemble model FIGS on nearly all datasets. MDT also outperforms a single-tree, non-deferral ablation, demonstrating that the multistage training and deferral is core to MDT's success.

\begin{figure}[t]
    \centering

    \begin{subfigure}{0.495\linewidth}
        \centering
        \includegraphics[width=\linewidth]{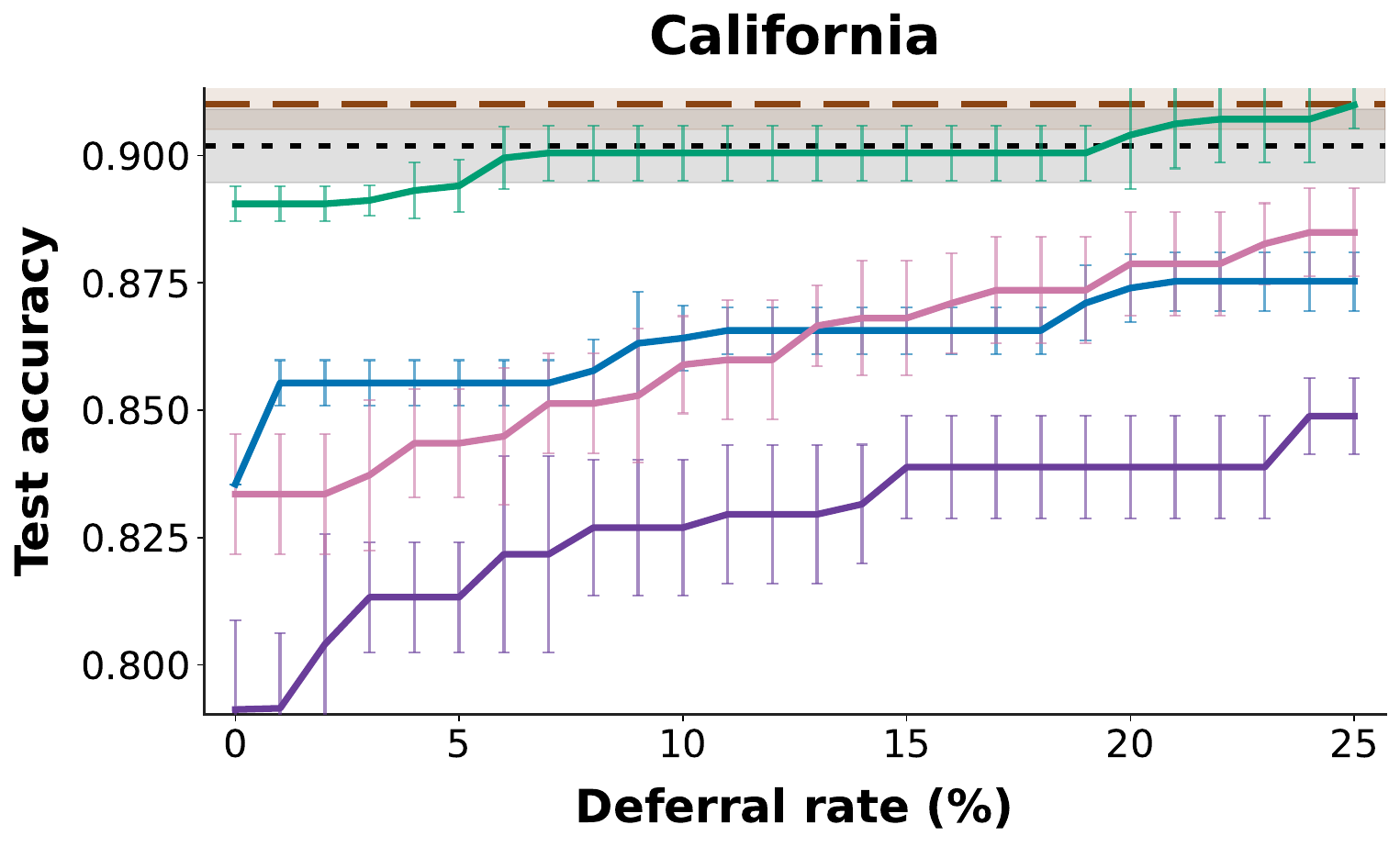}
    \end{subfigure}
    \hfill
    \begin{subfigure}{0.495\linewidth}
        \centering
        \includegraphics[width=\linewidth]{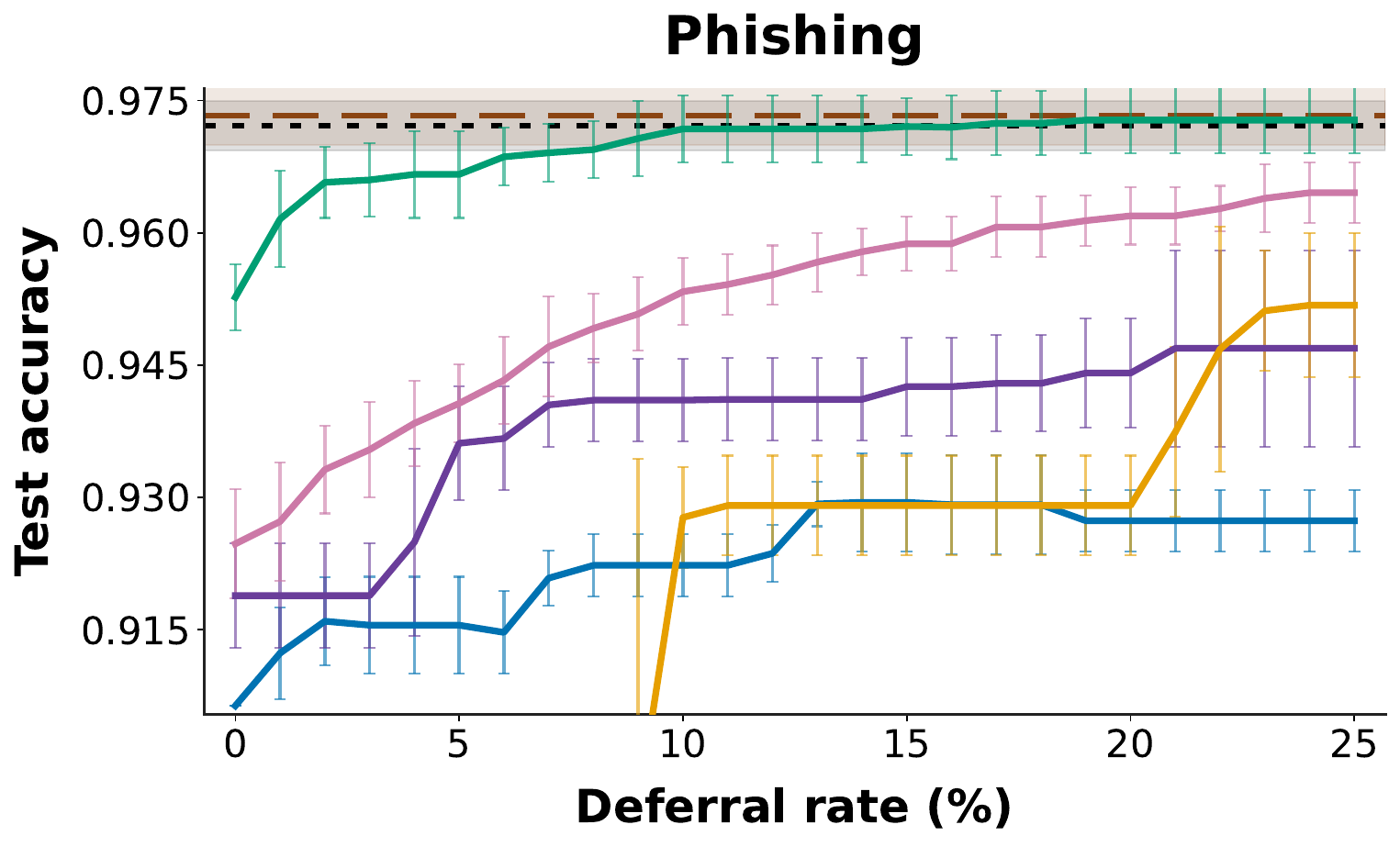}
    \end{subfigure}

    \vspace{0em}

    \begin{subfigure}{0.495\linewidth}
        \centering
        \includegraphics[width=\linewidth]{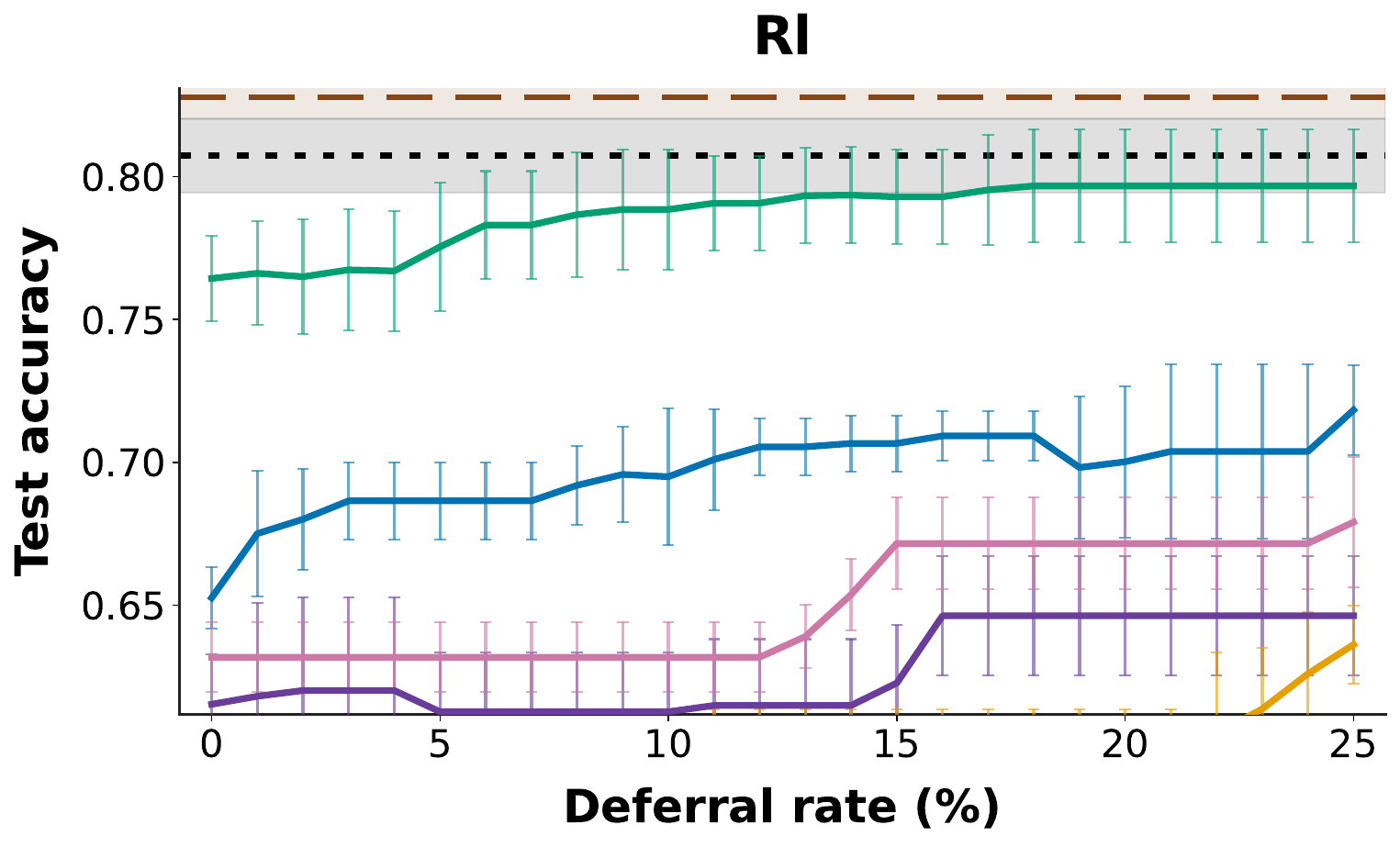}
    \end{subfigure}
    \hfill
    \begin{subfigure}{0.495\linewidth}
        \centering
        \includegraphics[width=\linewidth]{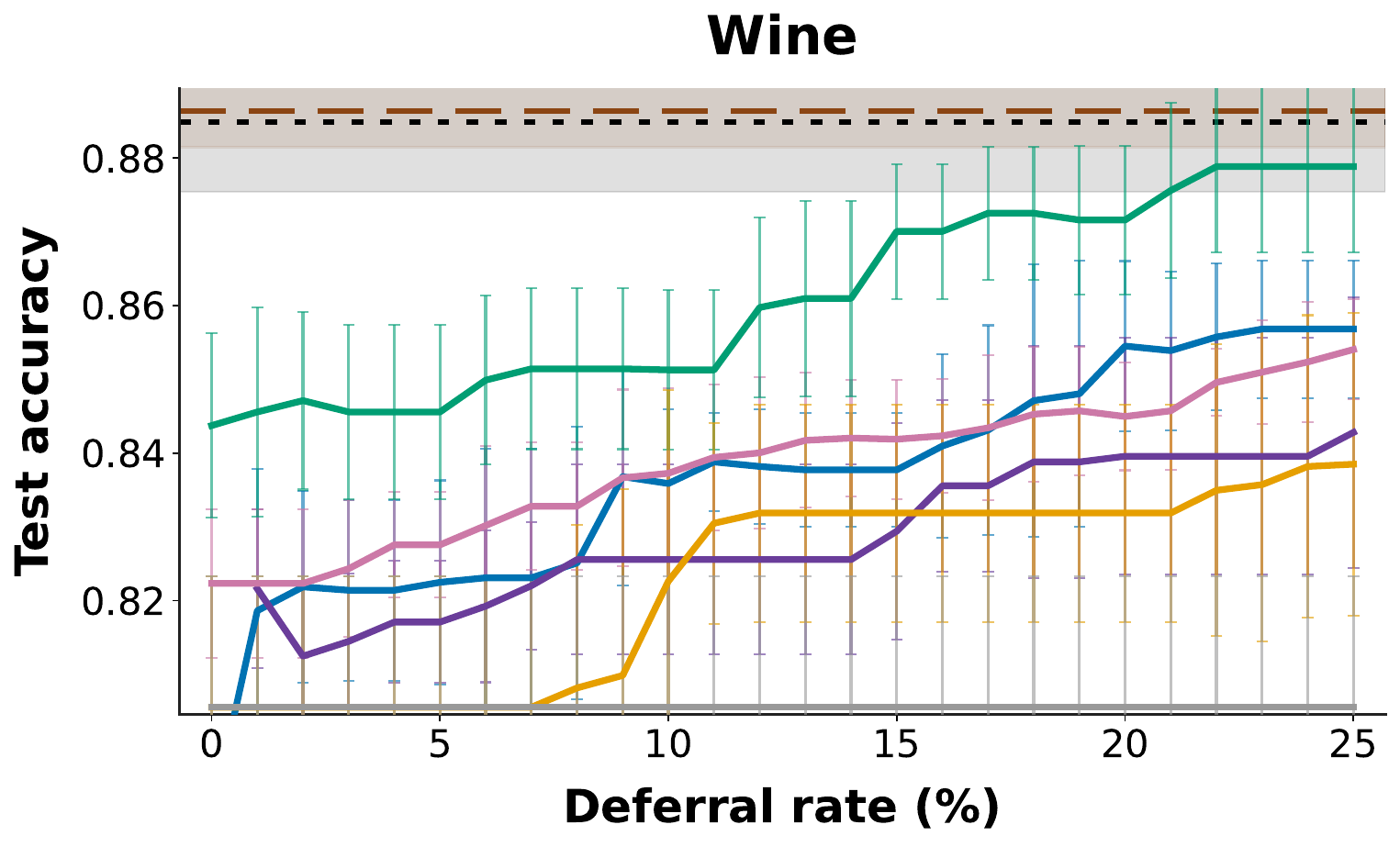}
    \end{subfigure}

    \vspace{-0em}

    \includegraphics[width=\linewidth, trim=0 25 0 25, clip]{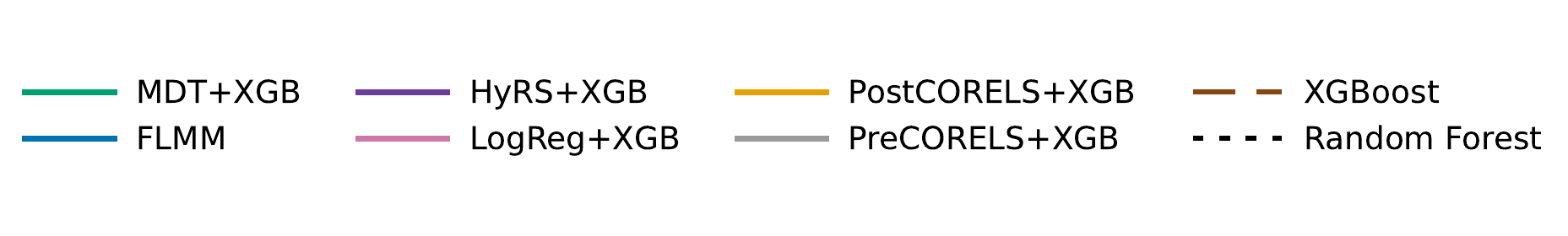}

    \vspace{-0.3em}

    \caption{Test accuracy (mean $\pm$ standard deviation) as a function of deferral rate across datasets.}
    \label{fig:deferral_accuracy_selected}
\end{figure}

\begin{figure}[!t]
    \centering
    \includegraphics[width=0.8\linewidth]{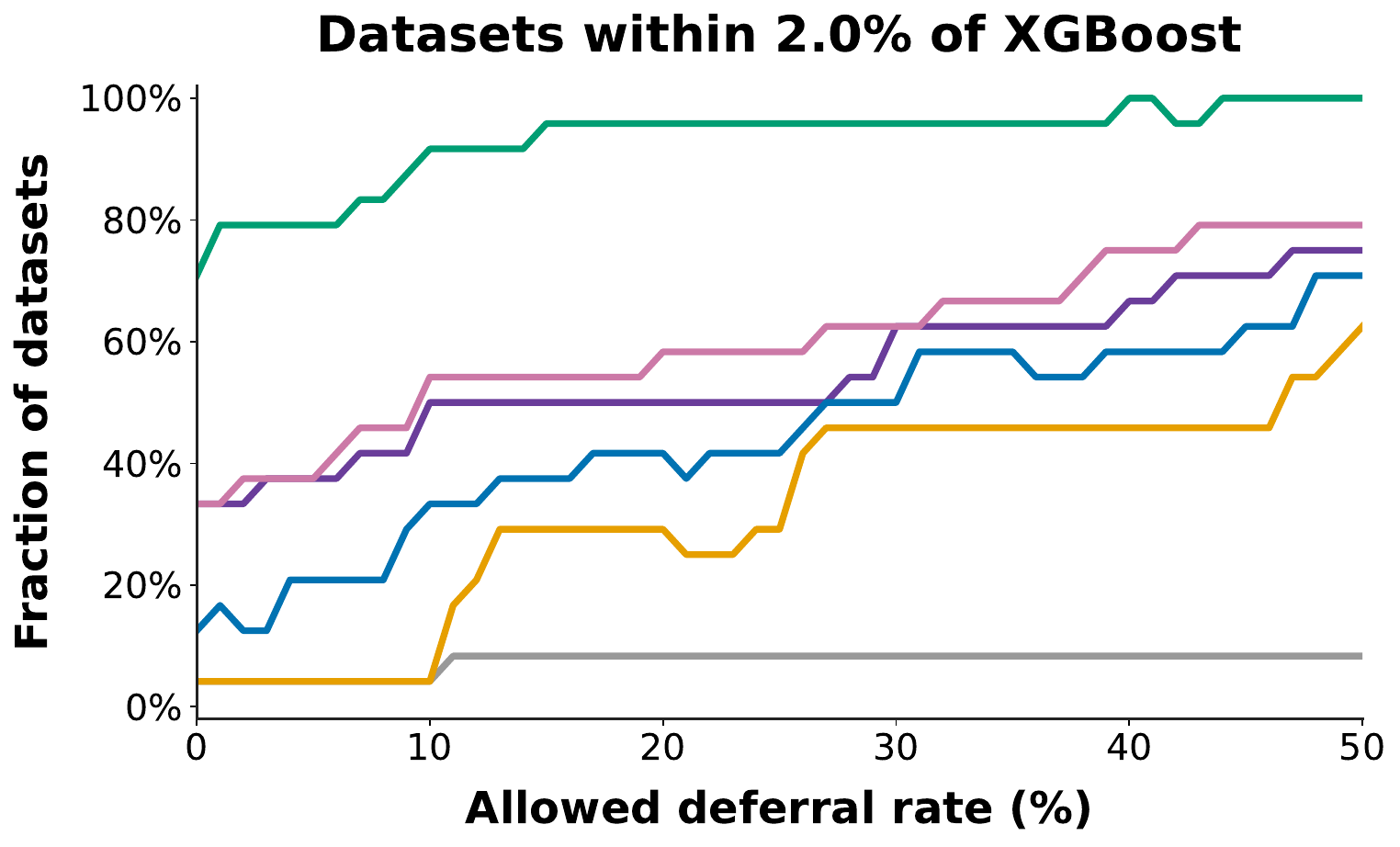}
    \includegraphics[width=0.7\linewidth, trim=0 25 225 25, clip]{new_main_figs/legend_main.pdf}
    \caption{CDF of deferral thresholds required to have accuracy within $2.0\%$ of XGBoost. Each curve shows the fraction of datasets for which a method achieves this accuracy within a given deferral rate.}
    \label{fig:deferral_cdf_eps0025}
    
\end{figure}

To answer question (2), we compare MDTs to other hybrid-interpretable models in Figures \ref{fig:deferral_accuracy_selected} and \ref{fig:deferral_cdf_eps0025}. 
In Figure~\ref{fig:deferral_accuracy_selected}, we compare hybrid-interpretable models on datasets where the accuracy-interpretability tradeoff is high (a black box is $> 2 \%$ more accurate than a single tree). 
On all four datasets, MDT+XGB consistently outperforms all other baselines. On California, MDT+XGB matches the mean accuracy of XGBoost at 25\% deferral, a level that no competing method achieves within the same deferral budget. 
Figure~\ref{fig:deferral_cdf_eps0025} shows the deferral rate required to achieve accuracy within $2.0\%$ of XGBoost across all 25 datasets (the full list of datasets is in Appendix \ref{appendix:experiment-setup}). At around $40\%$ deferral, our MDT+XGB method can come within this amount of XGBoost for all datasets. No existing method can match this without deferring the majority of the time.

\begin{figure}[!h]
    \centering

    \begin{subfigure}[h]{0.6\linewidth}
    \vspace{0pt}
        \centering
        \includegraphics[width=\linewidth]{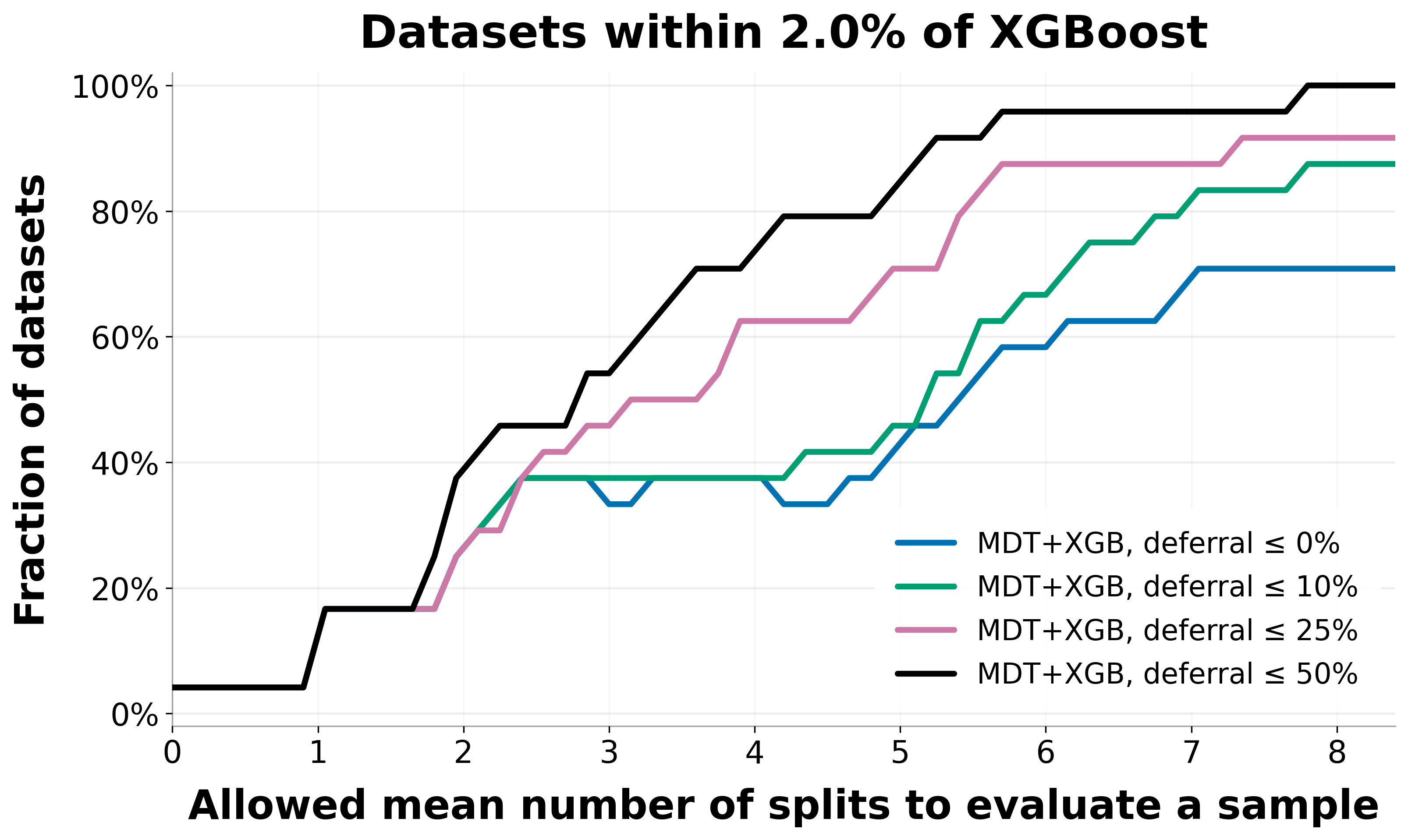}
        \caption{Fraction of datasets within $2.0\%$ of XGBoost as a function of the average number of split decisions encountered by a test sample in the interpretable portion of MDT+XGB.}
        \label{fig:sparsity_cdf_eps0025}
    \end{subfigure}
    \hfill
    \hspace{-0.02\linewidth}%
    \begin{subfigure}[h]{0.38\linewidth}
    \vspace{0pt}
        \centering
        \scriptsize
        \setlength{\tabcolsep}{2.2pt}
        \renewcommand{\arraystretch}{0.83}
        \resizebox{\linewidth}{!}{
        \begin{tabular}{lccc}
        \toprule
        Dataset & Splits & Leaves & $\Delta$ Acc. 95\% CI \\
        \midrule
        Adult    & 1.021 & 1.082  & $[-0.0048,\, 0.0012]$ \\
        Bike
        & 1.064
        & 1.088
        & $[-0.0060,\, 0.0045]$ \\
        Compas   & 1.756 & 1.545  & $[-0.0132,\, 0.0164]$ \\
        Diamonds & 1.112 & 1.458  & $[-0.0066,\, 0.0003]$ \\
        Jasmine  & 1.847 & 1.361  & $[-0.0734,\, 0.0395]$ \\
        Madeline & 1.061 & 11.956 & $[-0.0491,\, 0.0202]$ \\
        Phishing & 1.808 & 1.048  & $[-0.0029,\, 0.0024]$ \\
        Pol      & 1.002 & 1.043  & $[-0.0047,\, 0.0075]$ \\
        Rl       & 1.223 & 1.327  & $[-0.0112,\, 0.0196]$ \\
        Shopping & 1.105 & 1.302  & $[-0.0121,\, 0.0052]$ \\
        \bottomrule
        \end{tabular}
        }
        \caption{Mean compression ratios for splits to evaluate a deferred test sample and leaves of the final MDT XGBoost relative to the original XGBoost fallback. A 95\% confidence interval on the difference ($\text{MDT} - \text{XGB}$) in test deferred region accuracy shows they are indistinguishable. Full table in Table \ref{tab:appendix-compressed_fallback_results} (Appendix).}
        \label{fig:dist_sparsity_sub}
    \end{subfigure}

    \caption{}
    \vspace{-0.8em}
   
    \label{fig:dist_sparsity}
\end{figure}

To answer question (3), Figure~\ref{fig:sparsity_cdf_eps0025} shows the fraction of datasets that achieve accuracy within $2.0\%$ of XGBoost as a function of the average number of split decisions encountered by a test sample while passing through the interpretable portion of the MDT. This average path-length constraint is enforced analogously to the deferral constraint. At 50\% deferral, no dataset requires more than an average of 7.5 split decisions per sample in the interpretable part to reach near-XGBoost performance, and the vast majority reach this threshold with at most an average of 5.5 split decisions per sample. 
Since MDTs only need to use a black-box component for a small proportion of samples, the black-box representations they learn are naturally more compressed. Figure~\ref{fig:dist_sparsity_sub} compares the distribution of the number of splits needed to evaluate samples that are deferred, either by (i) using the final fallback XGBoost model or (ii) using the best XGBoost model for the whole dataset. Both (i) and (ii) are trained with the same hyperparameter configuration, but approach (i) is trained using a weighting that emphasizes deferred points.  
The number of splits required to classify a sample is reduced.

%% file: sections/conclusion.tex
We introduce Multistage Defer Trees, a model class that adaptively allocates complexity by routing inputs through a sequence of sparse decision trees before deferring to a black-box model. Our training procedure progressively refines the defer region, combining distance-based weighting, which allows the model to borrow strength from nearby points, with compression techniques that track and simplify defer subregions. MDTs demonstrate that accuracy and interpretability need not be traded off globally, substantially improving the accuracy-deferral frontier. Future work could explore using multiple fallback models, possibly from different model classes.

\section{Acknowledgments}

We acknowledge funding from the National Institutes of Health (NIH) under grant R01-DA054994. We acknowledge the support of the Natural Sciences and Engineering Research Council of Canada (NSERC). Nous remercions le Conseil de recherches en sciences naturelles et en génie du Canada (CRSNG) de son soutien. We also thank Zachery Boner for comments on our manuscript, and Varun Babbar for suggestions on earlier directions for the work. 

%% file: sections/appendix/proofs.tex
\topdownimprovement*
\begin{proof}

For each stage $j$, define the set of samples that reach stage $j$ as
\[
D_j := \{i : T_1(x_i)=\cdots=T_{j-1}(x_i)=\defer\},
\]
with $D_1=\{1,\ldots,N\}$. 

Making the weight update
\[
w_i^{(j)} = w_i \mathds{1}\{i \in D_j\}.
\]

The all-defer baseline has decision tree objective
\[
L_{\defer,B}(\mathcal D^w,\tau,\eta)
=
\sum_{i=1}^N
w_i\Bigl(
\mathds{1}\{B(x_i)\neq y_i\}+\eta
\Bigr),
\]

Now consider the contribution of a single stage $j$ on the samples that reach it. On the weighted dataset $\mathcal D^{w^{(j)}}$, the single-tree defer objective is
\[
L_{T_j,B}(\mathcal D^{w^{(j)}},\tau,\eta)
=
\tau (|T_j|-1)
+
\sum_{i=1}^N w_i^{(j)}
\Bigl(
\mathds{1}\{\hat y_{T_j,B}(x_i)\neq y_i\}
+
\eta \mathds{1}\{T_j(x_i)=\defer\}
\Bigr).
\]
The all-defer tree (a single defer leaf) on the same weighted dataset has objective
\[
L_{\defer,B}(\mathcal D^{w^{(j)}},\tau,\eta)
=
\sum_{i=1}^N w_i^{(j)}
\Bigl(
\mathds{1}\{B(x_i)\neq y_i\}+\eta
\Bigr).
\]
Therefore,
\[
\Delta_j
=
L_{\defer,B}(\mathcal D^{w^{(j)}},\tau,\eta)
-
L_{T_j,B}(\mathcal D^{w^{(j)}},\tau,\eta)
\]
equals
\[
-\tau(|T_j|-1)
+
\sum_{i=1}^N w_i^{(j)}
\Bigl[
\mathds{1}\{B(x_i)\neq y_i\}+\eta
-
\mathds{1}\{\hat y_{T_j,B}(x_i)\neq y_i\}
-
\eta\mathds{1}\{T_j(x_i)=\defer\}
\Bigr].
\]

Now sum the expression for $\Delta_j$ over $j=1,\ldots,k$:
\[
\sum_{j=1}^k \Delta_j
=
-\tau\sum_{j=1}^k (|T_j|-1)
+
\sum_{j=1}^k
\sum_{i=1}^N
w_i^{(j)}
\Bigl[
\mathds{1}\{B(x_i)\neq y_i\}+\eta
-
\mathds{1}\{\hat y_{T_j,B}(x_i)\neq y_i\}
-
\eta\mathds{1}\{T_j(x_i)=\defer\}
\Bigr].
\]

Using
\[
w_i^{(j)}
=
w_i
\mathds{1}\{T_1(x_i)=\cdots=T_{j-1}(x_i)=\defer\},
\]
we can rewrite the double sum sample-by-sample:
\[
\sum_{j=1}^k
\sum_{i=1}^N
w_i^{(j)}
\Bigl[
\mathds{1}\{B(x_i)\neq y_i\}+\eta
-
\mathds{1}\{\hat y_{T_j,B}(x_i)\neq y_i\}
-
\eta\mathds{1}\{T_j(x_i)=\defer\}
\Bigr]
\]
\[
\begin{aligned}
&=
\sum_{i=1}^N
w_i
\sum_{j=1}^k
\mathds{1}\{T_1(x_i)=\cdots=T_{j-1}(x_i)=\defer\}
\Bigl[
\mathds{1}\{B(x_i)\neq y_i\}+\eta \\
&\hspace{6.5em}
-
\mathds{1}\{\hat y_{T_j,B}(x_i)\neq y_i\}
-
\eta\mathds{1}\{T_j(x_i)=\defer\}
\Bigr].
\end{aligned}
\]

Fix a sample $i$. Let $r_i$ be the first stage that predicts a label:
\[
r_i := \min\{j : T_j(x_i)\neq \defer\},
\]
if such a stage exists. If no such stage exists, set $r_i=k+1$.

There are 3 cases: (1.1, 1.2, and 2.1). 

Case 1.1: $j < r_i \leq k$. Here, sample $i$ reaches stage $j$ but it still defers.

Therefore, $T_j(x_i)=\defer$, so the prediction used in the loss of training the single defer tree is $\hat y_{T_j,B}(x_i)=B(x_i)$. Also, $\mathds{1}\{T_j(x_i)=\defer\}=1$. So, the bracket becomes $$\mathds{1}\{B(x_i)\neq y_i\}+\eta
-
\mathds{1}\{B(x_i)\neq y_i\}
-
\eta = 0.$$

So every stage before the first predicting stage contributes nothing to the improvement in the objective for that sample. Intuitively, before $r_i$, the model is still doing exactly what the all defer tree does: defers to $B$ and pays the penalty.

Case 1.2: $j=r_i\le k$. The tree predicts. 

We know $T_{r_i}(x_i)\neq \defer$ and $\mathds{1}\{T_{r_i}(x_i)=\defer\}=0$. This is the first stage for which we make a prediction for this sample, so $\hat y_{T_{r_i},B}(x_i)
=
\hat y_{T_{1:k},B}(x_i)$.

Therefore, the bracket becomes

$$\mathds{1}\{B(x_i)\neq y_i\}+\eta
-
\mathds{1}\{\hat y_{T_{1:k},B}(x_i)\neq y_i\}.$$

Since the full MDT does not defer on this sample, we also know $\mathds{1}\{\hat T_{1:k}(x_i)=\defer\}=0.$

Appending this zero term will become useful later in the proof.

$$\mathds{1}\{B(x_i)\neq y_i\}+\eta
-
\Bigl(
\mathds{1}\{\hat y_{T_{1:k},B}(x_i)\neq y_i\}
+
\eta\mathds{1}\{\hat T_{1:k}(x_i)=\defer\}
\Bigr).$$

Case 2.1: $r_i=k+1$. Here, the MDT defers to the black box. No stage predicts. The sample is deferred by every stage.

Writing the consequences of this formally: $$T_1(x_i)=T_2(x_i)=\cdots=T_k(x_i)=\defer.$$

For every $j \leq k$

$$\hat y_{T_j,B}(x_i)=B(x_i)$$

and 

$$\mathds{1}\{T_j(x_i)=\defer\}=1.$$

Therefore, every bracket is 

$$\mathds{1}\{B(x_i)\neq y_i\}+\eta
-
\mathds{1}\{B(x_i)\neq y_i\}
-
\eta
=
0.$$

Now, with the three cases solved, we can look back at the double sum.

\[
\sum_{j=1}^k
\sum_{i=1}^N
w_i^{(j)}
\Bigl[
\mathds{1}\{B(x_i)\neq y_i\}+\eta
-
\mathds{1}\{\hat y_{T_j,B}(x_i)\neq y_i\}
-
\eta\mathds{1}\{T_j(x_i)=\defer\}
\Bigr]
\]
\[
\begin{aligned}
&=
\sum_{i=1}^N
w_i
\sum_{j=1}^k
\mathds{1}\{T_1(x_i)=\cdots=T_{j-1}(x_i)=\defer\}
\Bigl[
\mathds{1}\{B(x_i)\neq y_i\}+\eta \\
&\hspace{6.5em}
-
\mathds{1}\{\hat y_{T_j,B}(x_i)\neq y_i\}
-
\eta\mathds{1}\{T_j(x_i)=\defer\}
\Bigr].
\end{aligned}
\]

We now claim that $S_i$ is the entire inner sum over stages $j=1,\ldots,k$ for one sample $i$:

$$S_i
=
\mathds{1}\{B(x_i)\neq y_i\}+\eta
-
\Bigl(
\mathds{1}\{\hat y_{T_{1:k},B}(x_i)\neq y_i\}
+
\eta\mathds{1}\{\hat T_{1:k}(x_i)=\defer\}
\Bigr).$$

Recall we showed that at most one $j$ is nonzero for a fixed sample $i$. If $r_i=k+1$, every stage defers, so all terms are zero. We appended the zero term in Case 1.2 to give us this expression. The other two cases were evaluated to 0. $S_i$ is precisely the sum of all of them, which simplifies to just  Case 1.2.

Continuing, the double sum is

$$=
\sum_{i=1}^N w_i S_i,$$

$$
=
\sum_{i=1}^N w_i
\Bigl[
\mathds{1}\{B(x_i)\neq y_i\}+\eta
-
\mathds{1}\{\hat y_{T_{1:k},B}(x_i)\neq y_i\}
-
\eta\mathds{1}\{\hat T_{1:k}(x_i)=\defer\}
\Bigr].$$

Now plug this back into the sum over improvements,

$$\sum_{j=1}^k \Delta_j
=
-\tau\sum_{j=1}^k (|T_j|-1)
+
\sum_{i=1}^N w_i S_i.$$

$$\sum_{j=1}^k \Delta_j
=
-\tau\sum_{j=1}^k (|T_j|-1)
+
\sum_{i=1}^N w_i
\Bigl[
\mathds{1}\{B(x_i)\neq y_i\}+\eta
-
\mathds{1}\{\hat y_{T_{1:k},B}(x_i)\neq y_i\}
-
\eta\mathds{1}\{\hat T_{1:k}(x_i)=\defer\}
\Bigr].$$

Rearranging,
\[
\begin{aligned}
\sum_{j=1}^k \Delta_j
&=
\left[
\sum_{i=1}^N w_i
\bigl(
\mathds{1}\{B(x_i)\neq y_i\}+\eta
\bigr)
\right] \\
&\quad -
\left[
\tau\sum_{j=1}^k (|T_j|-1)
+
\sum_{i=1}^N w_i
\Bigl(
\mathds{1}\{\hat y_{T_{1:k},B}(x_i)\neq y_i\}
+
\eta\mathds{1}\{\hat T_{1:k}(x_i)=\defer\}
\Bigr)
\right].
\end{aligned}
\]
The second bracket is exactly the MDT objective:

$$\mathcal L_{T_{1:k},B}(\mathcal D^w,\tau,\eta)
=
\tau\sum_{j=1}^k (|T_j|-1)
+
\sum_{i=1}^N w_i
\Bigl(
\mathds{1}\{\hat y_{T_{1:k},B}(x_i)\neq y_i\}
+
\eta\mathds{1}\{\hat T_{1:k}(x_i)=\defer\}
\Bigr).$$

Likewise, the first term is the defer leaf objective.

$$L_{\defer,B}(\mathcal D^w,\tau,\eta)
=
\sum_i w_i(\mathds{1}\{B(x_i)\neq y_i\}+\eta).$$

These two facts establish what we wanted to show.

\end{proof}

Proposition \ref{thm:defer-leaf-tradeoff} establishes the relationship between sparsity, accuracy, and deferral for sparse decision trees of unbounded depth. Consequently, the same claim holds when the model class is restricted to decision trees of any fixed bounded depth $d$.

\begin{proposition}[Accuracy-deferral trade-off]
\label{thm:defer-leaf-tradeoff}
Let $T^*$ be an optimal defer tree and let $v$ be a defer leaf. Define $S_v = \{i : x_i \text{ reaches } v\}$.

For any decision tree $U$ used to replace $v$, define
\[
\Delta(U)
=
\sum_{i \in S_v} \mathds{1}\{U(x_i) \neq y_i\} - \sum_{i \in S_v} \mathds{1}\{B(x_i) \neq y_i\}.
\]
Then
\[
\Delta(U)
\;\ge\;
\eta\,|S_v|
\;-\;
\tau \bigl(|U| - 1\bigr).
\]
Let $U(x_i)$ In particular, if $U(x_i)=B(x_i)$  for all $i \in S_v$, then
\[
|U|
\;\ge\;
1 + \frac{\eta\,|S_v|}{\tau}.
\]
\end{proposition}

\begin{proof}

    We first note that whenever an optimal defer tree places a defer leaf $v$ at a subproblem, then this choice must be at least as good (in objective) as both replacing $v$ with a majority prediction and replacing $v$ with a subtree rooted at that subproblem, regardless of whether the subtree itself contains defer leaves.
    
    The proof is an exchange argument from the optimality of $T^*$. 

    The defer-tree objective when referencing a base model with predictions $B(x_i)$ is:

    $$L_{\mathrm{defer}}(T)
    =
    \tau \bigl(|T|-1\bigr)
    +
    \sum_{i=1}^N
    \ell_T(i),$$

    where 

    $$\ell_T(i)
    =
    \mathds{1}\{T(x_i)\in Y,\; T(x_i)\neq y_i\}
    +
    \mathds{1}\{T(x_i)=\mathrm{defer},\; B(x_i)\neq y_i\}
    +
    \eta\,\mathbf{\mathds{1}}\{T(x_i)=\mathrm{defer}\}.$$

    Take any decision tree $U$, and form a new defer tree $T'$ by replacing the defer leaf $v$ in $T^*$ with $U$.

    Then, one leaf is removed, and $|U|
    $ are added.

    $$|T'|
    =
    |T^*| - 1 + |U|,$$

    $$ |T'| - |T^*| = |U| - 1.$$

    For every $i \notin S_v$, the prediction path is unchanged, so the loss is unchanged.

    For every $i \in S_v$, it had previously faced a penalty of $\mathds{1}\{B(x_i)\neq y_i\}+\eta$, because it was deferred. Now, it faces a cost of $\mathds{1}\{U(x_i)\neq y_i\}$.

    Therefore,

    $$L_{\mathrm{defer}}(T')-L_{\mathrm{defer}}(T^*)
    =
    \tau \bigl(|U|-1\bigr)
    +
    \sum_{i\in S_v}\mathds{1}\{U(x_i)\neq y_i\}
    -
    \sum_{i\in S_v}\bigl(\mathds{1}\{B(x_i)\neq y_i\}+\eta\bigr).$$

    Because $L_{\mathrm{defer}}(T')-L_{\mathrm{defer}}(T^*) \ge 0$, we get that 

    $$\Delta(U)
    \ge
    \eta |S_v|
    -
    \tau \bigl(|U|-1\bigr).$$

    A corollary of this is if $U(x_i)=B(x_i)$, for all $i \in S_v$, then clearly 

    $$\sum_{i\in S_v}\mathds{1}\{U(x_i)\neq y_i\}
    =
    \sum_{i\in S_v}\mathds{1}\{B(x_i)\neq y_i\}.$$

    Given that $\Delta(U)=0$, we get $0
    \ge
    \eta |S_v|
    -
    \tau \bigl(|U|-1\bigr)$.

    Equivalently,

    $$|U|
    \ge
    1+\frac{\eta |S_v|}{\tau}.$$

\end{proof}

\rulelistmdt*

\begin{proof}
We will use either $|T_j|$ or $\ell_j$ interchangeably to denote the number of leaves in stage $j$, and $d_j$ to denote the number of deferral leaves, not to be confused with the depth budget of each tree, which is $d$.

In stage $j$, each leaf corresponds to a unique root-to-leaf path, and hence to a conjunction of split conditions. For each non-deferral leaf $\kappa$ in stage $j$, let $A_\kappa$ denote the conjunction of conditions along its path, and let $y_\kappa \in \{0,1\}$ be its prediction.

We construct a rule list by concatenating $K$ sections, one for each stage. In section $j$, we include one antecedent (and the corresponding prediction) for each non-deferral leaf $\kappa$ of stage $j$:
\[
\text{if } A_\kappa \text{ then predict } y_\kappa.
\]

Note that $A_\kappa$, the conjunction of conditions along a root-to-leaf path in a single-stage tree, contains at most $d$ conditions, since each tree has depth at most $d$.

Because the leaves of a decision tree partition the input space, and each later stage operates only on the deferred region, the stages cover disjoint parts of the input space. Within a stage, there is no interaction between antecedents, so they may be listed in any order without ambiguity. While the choice of ordering may simplify individual antecedent conditions, it does not change the total number of rules.

We claim this rule list is equivalent to the MDT. Consider any input $x$. In the MDT, $x$ is first evaluated at stage $1$. If it reaches a non-deferral leaf, the MDT outputs the corresponding prediction and halts. In the rule list, exactly the antecedent corresponding to that leaf is satisfied (and the consequent is the label of that leaf), so the first matching rule produces the same output.

If instead $x$ reaches a deferral leaf at stage $1$, then no antecedent from section $1$ is satisfied, since that section contains only non-deferral leaves. The evaluation proceeds to section $2$, mirroring the MDT’s transition to stage $2$. Repeating this argument inductively, the antecedent that is satisfied corresponds exactly to the non-deferral leaf reached by $x$ in the MDT, and outputs the same prediction.

Thus, the rule list is equivalent to the MDT.

Finally, in stage $j$, there are $\ell_j - d_j$ non-deferral leaves, each contributing exactly one antecedent. Summing over all stages yields the total number of antecedents:
\[
\sum_{j=1}^K (\ell_j - d_j).
\]

\end{proof}

\begin{corollary}[MDTs do not have wide rules]\label{mdtdepth}
    Let a rule list be defined over conjunctions of feature queries, where each satisfied rule must immediately output a leaf prediction. Then, for any MDT composed of $m$ defer trees each of depth $d$, every rule in the equivalent rule-list representation of the interpretable component has at most $d$ literals.
\end{corollary}
\begin{proof}
    This follows direction from the construction discussed in  \autoref{thm:rule-list}, which uses paths in each tree as rules, each of which must be of depth at most $d$.
\end{proof}

\begin{proposition}[Majority-vote ensembles have wide rules]\label{appendix:complexity_proof}
Let a rule list be defined over any conjunction of feature queries in a dataset, such that the true branch must be classified immediately with a leaf. Then
there exist binary decision trees $T_1,\dots,T_m$, each of depth $d$, such that the classifier defined by their majority vote cannot be represented by an equivalent rule list unless some
rule queries $\Omega(md)$ literals.
\end{proposition}
\begin{proof}
Let each tree compute an XOR over $d$ binary features, and suppose that the feature sets used by different trees are disjoint. Assume that all feature
value combinations are possible. For an XOR over $d$ features, no subset of features can determine the tree's predictions.

Now consider any rule in an equivalent rule list whose condition is satisfied by at least one input, and that assigns a label. Let us consider using that rule as the first rule. Since the rule must classify the affected points immediately, its conjunction must be sufficient to determine the majority-vote prediction for every completion of the unqueried features. 

To force the majority vote, the
rule must determine the outputs of at least $\lceil m/2 \rceil$ trees (for the even case, we assume ties can be broken arbitrarily). Because the trees use disjoint feature sets, determining the output of each such XOR tree requires querying all $d$ of its features. Hence, the rule must
contain at least
\[
d\left(\left\lceil \frac{m}{2}\right\rceil\right)
= \Omega(md)
\]
literals.

\end{proof}

As we show in the next theorem, the distance-based weighting rule admits an entropy-regularized variational interpretation. Among all distributions, the normalized weights best balance the expected distance to being deferred and the amount of deviation from uniform (according to $\gamma$).

\begin{proposition}[Gibbs form of distance-based weights]\label{gibbs}

For each training data point $i$, take $c_i := \log(1 + D_i)$ and let $u$ be the uniform distribution over the training data points.

The following objective (with $\gamma>0)$ has a unique minimizer (over all distributions $q$ on the training data points) given by $q_i \propto (1+D_i)^{-\gamma}$.

$$\mathcal{J}(q) = \mathbb{E}_q[c_i] + \frac{1}{\gamma}\mathrm{KL}(q\|u)$$

\end{proposition}

\begin{proof}
    Plugging in definitions,

    $$\mathcal J(q)=\sum_{i=1}^N q_i c_i+\frac1\gamma \sum_{i=1}^N q_i \log\frac{q_i}{u_i}.$$

    Because $u_i=1/N$,

    $$\log\frac{q_i}{u_i}=\log q_i-\log(1/N)=\log q_i+\log N.$$

    This yields

    $$\mathcal J(q)=\sum_{i=1}^N q_i c_i+\frac1\gamma \sum_{i=1}^N q_i \log q_i+\frac1\gamma \sum_{i=1}^N q_i \log N = \sum_{i=1}^N q_i c_i+\frac1\gamma \sum_{i=1}^N q_i \log q_i+\frac1\gamma \log N.$$

    The third term is constant w.r.t $q$, so we just need to optimize

    $$\sum_{i=1}^N q_i c_i+\frac1\gamma \sum_{i=1}^N q_i \log q_i,$$ subject to $q$ being a probability distribution.

    We form the Lagrangian: 
    
    $$L(q,\lambda)=\sum_{i=1}^N q_i c_i+\frac1\gamma \sum_{i=1}^N q_i \log q_i+\lambda\left(\sum_{i=1}^N q_i-1\right).$$

    Because $\frac{d}{dq_i}(q_i\log q_i)=\log q_i+1$, we get $$\frac{\partial L}{\partial q_i}
    =
    c_i+\frac1\gamma(\log q_i+1)+\lambda.$$

    For an interior minimizer (assume no $q_i = 0$), this is 0.

    $$c_i+\frac1\gamma(\log q_i+1)+\lambda=0.$$

    Writing equivalent forms,

    $$c_i+\frac1\gamma(\log q_i+1)+\lambda=0.$$

    $$\log q_i = -\gamma c_i -1-\gamma\lambda.$$

    $$q_i = \exp(-1-\gamma\lambda)\exp(-\gamma c_i).$$

    Because $\exp(-1-\gamma\lambda)$ does not depend on $i$, we know $q_i \propto e^{-\gamma c_i}=e^{-\gamma c_i}=e^{-\gamma \log(1+D_i)}=(1+D_i)^{-\gamma}.$

    Therefore, $$q_i \propto (1+D_i)^{-\gamma}.$$

    Normalizing, as we know that $q$ is a probability distribution,

    $$q_i=\frac{(1+D_i)^{-\gamma}}{\sum_{j=1}^n (1+D_j)^{-\gamma}}.$$

    We have found the only critical point in the interior ($q_i > 0$). We can show that it is both a minimizer and unique by strict convexity. We have the sum of a linear function (which is convex) and the negative entropy component is strictly convex (with $0 \log 0 := 0$). Since $\gamma>0$, the sum is strictly convex. This suffices to show that the minimizer is unique and that it is an interior critical point.

\end{proof}

\begin{proposition}[Size of the uncompressed single tree representation]
Let an MDT chain be given by
\[
T_1 \to T_2 \to \cdots \to T_K \to W,
\]
where for each stage $j$, $T_j$ has $\ell_j$ leaves, of which $d_j$ are deferral leaves, and let $W$ be a fallback model (such as a single-tree representation of a black box) with $\ell_W$ leaves and no deferral leaves. Then the uncompressed single-tree representation of the chain has
\[
\sum_{j=1}^{K}
\left(\prod_{i=1}^{j-1} d_i\right)(\ell_j-d_j)
+
\left(\prod_{i=1}^{K} d_i\right)\ell_W
\]
leaves.
\end{proposition}

\begin{proof}
A copy of $T_j$ appears only when the input has been deferred through all earlier stages $T_1,\dots,T_{j-1}$.

The tree $T_1$ appears once. Each deferral leaf of $T_1$ is replaced by a copy of $T_2$, so there are $d_1$ copies of $T_2$. Each deferral leaf of each copy of $T_2$ is replaced by a copy of $T_3$, so there are $d_1 d_2$ copies of $T_3$.

Continuing inductively, the number of copies of $T_j$ in the expanded tree is
\[
\prod_{i=1}^{j-1} d_i,
\]
with the convention that this product equals $1$ when $j=1$.

Each copy of $T_j$ contributes $\ell_j - d_j$ final (non-deferral) leaves, since its $d_j$ deferral leaves are replaced by the root of the next stage. Therefore, the total number of final leaves contributed by all copies of $T_j$ is
\[
\left(\prod_{i=1}^{j-1} d_i\right)(\ell_j - d_j).
\]

Summing over all stages $j=1,\dots,K$, the total number of final leaves contributed by the MDT stages is
\[
\sum_{j=1}^{K}
\left(\prod_{i=1}^{j-1} d_i\right)(\ell_j - d_j).
\]

If $d_K \neq 0$, the remaining deferral leaves of stage $K$ are replaced by copies of $W$. The number of copies of $T_K$ is $\prod_{i=1}^{K-1} d_i$, and each such copy has $d_K$ deferral leaves, so the total number of copies of $W$ is
\[
\left(\prod_{i=1}^{K-1} d_i\right)d_K
=
\prod_{i=1}^{K} d_i.
\]

Since $W$ has $\ell_W$ leaves and no deferral leaves, the total number of final leaves contributed by all copies of $W$ is
\[
\left(\prod_{i=1}^{K} d_i\right)\ell_W.
\]

Adding the disjoint contributions yields the total number of leaves in the uncompressed tree:
\[
\sum_{j=1}^{K}
\left(\prod_{i=1}^{j-1} d_i\right)(\ell_j - d_j)
+
\left(\prod_{i=1}^{K} d_i\right)\ell_W.
\]

\end{proof}

To exclude the contribution of $W$, one may set $d_K = 0$. Alternatively, one may set $\ell_W=1$; they are equivalent. 

\begin{proposition}[Online update for uncompressed subchain size]\label{prop:subchain}
Let an MDT subchain be given by
\[
T_{K_{\mathrm{lower}}} \to T_{K_{\mathrm{lower}}+1} \to \cdots \to T_{K_{\mathrm{upper}}} \to W,
\]
where for each stage $j$, $T_j$ has $\ell_j$ leaves, of which $d_j$ are deferral leaves, and let $W$ be a fallback model with $\ell_W$ leaves and no deferral leaves. Define $p_j := \ell_j - d_j$.

Then the number of leaves in the uncompressed single-tree representation of this subchain is
\[
G_W(a,b)
=
\sum_{j=a}^{b}
\left(\prod_{i=a}^{j-1} d_i\right)p_j
\;+\;
\left(\prod_{i=a}^{b} d_i\right)\ell_W,
\]
where $a = K_{\mathrm{lower}}$ and $b = K_{\mathrm{upper}}$.

Moreover, for any $b+1 \leq K_{\mathrm{upper}}$ (i.e. whenever stage $T_{b+1}$ exists with total leaf count $\ell_{b+1}$, deferral leaf count $d_{b+1}$, and $W'$ is the new fallback model), this quantity satisfies the following recurrence:
\begin{align*}
\text{(Top-down)}\quad
G_{W'}(a,b+1)
&=
G_W(a,b)
+
\left(\prod_{i=a}^{b} d_i\right)
\bigl[p_{b+1} + d_{b+1}\ell_{W'} - \ell_W\bigr],
\end{align*}
where $W'$ denotes the retrained fallback model after extending the chain.
\end{proposition}

\begin{proof}
The expression for $G_W(a,b)$ follows from the same counting argument as for the full chain. A copy of $T_j$ appears only when the input has been deferred through all earlier stages $T_a,\dots,T_{j-1}$, yielding $\prod_{i=a}^{j-1} d_i$ copies. Each copy contributes $p_j = \ell_j - d_j$ final leaves, since deferral leaves are replaced by the next stage. The final deferrals from stage $b$ produce $\prod_{i=a}^{b} d_i$ copies of $W$, each contributing $\ell_W$ leaves.

For the recurrence, extending the chain by appending $T_{b+1}$ and retraining the fallback model to $W'$ yields
\[
G_{W'}(a,b+1)
=
\sum_{j=a}^{b+1}
\left(\prod_{i=a}^{j-1} d_i\right)p_j
+
\left(\prod_{i=a}^{b+1} d_i\right)\ell_{W'}.
\]
Separating the $j=b+1$ term,
\[
=
\sum_{j=a}^{b}
\left(\prod_{i=a}^{j-1} d_i\right)p_j
+
\left(\prod_{i=a}^{b} d_i\right)p_{b+1}
+
\left(\prod_{i=a}^{b+1} d_i\right)\ell_{W'}.
\]
Using the definition of $G_W(a,b)$,
\[
G_W(a,b)
=
\sum_{j=a}^{b}
\left(\prod_{i=a}^{j-1} d_i\right)p_j
+
\left(\prod_{i=a}^{b} d_i\right)\ell_W,
\]
we subtract and add appropriately to obtain
\[
G_{W'}(a,b+1)
=
G_W(a,b)
-
\left(\prod_{i=a}^{b} d_i\right)\ell_W
+
\left(\prod_{i=a}^{b} d_i\right)p_{b+1}
+
\left(\prod_{i=a}^{b+1} d_i\right)\ell_{W'}.
\]
Factoring $\prod_{i=a}^{b} d_i$ and using
\[
\prod_{i=a}^{b+1} d_i = \left(\prod_{i=a}^{b} d_i\right)d_{b+1},
\]
gives
\[
G_{W'}(a,b+1)
=
G_W(a,b)
+
\left(\prod_{i=a}^{b} d_i\right)
\bigl[p_{b+1} + d_{b+1}\ell_{W'} - \ell_W\bigr].
\]

This establishes the desired recurrence.
\end{proof}

In our algorithm, taking $\ell_W=\ell_{W'}=1$ (counting the leaves of deferring to the black box, but not the black box complexity) is helpful in online updates for early stopping.

At each stage, we update this quantity using the top-down recurrence, treating each current deferral path as contributing one leaf. This yields an incremental update of the form
\[
G \leftarrow G + \left(\prod_{i=a}^{b} d_i\right)(\ell_{b+1}-1),
\]
which can be computed efficiently using running quantities that are already maintained for other early-stopping criteria based on the sequential model complexity. If this estimate exceeds a specified leaf budget, we terminate training early. 

%% file: sections/appendix/extra-compression-details.tex
\subsection{Subregion Representation}

A subregion \(s\) is represented by two kinds of constraints. For each numerical or ordinal feature \(f\), the subregion may store an interval
\[
    x_f \in (\ell_f,u_f],
\]
where absent bounds are interpreted as \((-\infty,\infty]\). For each one-hot categorical group \(g\) (of size $m_g$), the subregion stores a vector
\[
    a_g \in \{-1,0,1\}^{m_g},
\]
where \(a_{gh}=1\) means category \(h\) is forced to be active, \(a_{gh}=-1\) means category \(h\) is ruled out, and \(a_{gh}=0\) means its status is still unknown. We use these constraints to determine whether a split is already implied by the current subregion, and to refine the subregion when a new branch literal is added. For instance, we know that no category can be ruled out, and that there cannot be more than 1 active category.

\subsection{Basic Methods}

\paragraph{\textsc{AddLiteral}.}
The \textsc{AddLiteral} routine refines a subregion \(s\) by incorporating a new branching constraint \(\ell = (f,\nu,b)\). This is done in two steps. For continuous or ordinal variables, \textsc{TightenNumericalBounds} updates the interval for feature \(f\) by intersecting it with the constraint induced by the branch. For categorical variables, \textsc{DeduceCategoricalImplications} updates any categorical group constraints implied by this literal. Together, these operations produce a new subregion \(s'\) that represents the intersection of the original region with the branch condition.

\begin{algorithm}[H]
\caption{\textsc{AddLiteral}$(s,\ell)$}
\label{alg:add_literal}
\begin{algorithmic}[1]
\REQUIRE Subregion \(s\), literal \(\ell=(f,\nu,b)\), where \(b=\texttt{true}\) denotes \(f\le \nu\)
\ENSURE Refined subregion \(s'\)
\STATE \(s' \leftarrow \textsc{TightenNumericalBounds}(s,\ell)\)
\STATE \(s' \leftarrow \textsc{DeduceCategoricalImplications}(s',\ell)\)
\STATE \textbf{return} \(s'\)
\end{algorithmic}
\end{algorithm}

\paragraph{\textsc{TightenNumericalBounds}.}
This routine updates the interval constraint for a numerical (or ordinal) feature \(f\). If the literal corresponds to taking the True branch (\(f \le \nu\)), the upper bound is tightened to \(\min\{u_f,\nu\}\). If the literal corresponds to the False branch (\(f > \nu\)), the lower bound is tightened to \(\max\{\ell_f,\nu\}\). 

\begin{algorithm}[H]
\caption{\textsc{TightenNumericalBounds}$(s,(f,\nu,b))$}
\label{alg:tighten_numerical_bounds}
\begin{algorithmic}[1]
\REQUIRE Subregion \(s\), literal \((f,\nu,b)\)
\ENSURE Subregion with updated numerical interval for \(f\)
\STATE \(s' \leftarrow \textsc{Copy}(s)\)
\STATE \((\ell_f,u_f] \leftarrow s'.\mathrm{bounds}(f)\), using \((-\infty,\infty]\) if absent
\IF{\(b=\texttt{true}\)}
    \STATE \(u_f \leftarrow \min\{u_f,\nu\}\)
\ELSE
    \STATE \(\ell_f \leftarrow \max\{\ell_f,\nu\}\)
\ENDIF
\STATE \(s'.\mathrm{bounds}(f) \leftarrow (\ell_f,u_f]\)
\STATE \textbf{return} \(s'\)
\end{algorithmic}
\end{algorithm}

\paragraph{\textsc{DeduceCategoricalImplications}.}
This routine handles the case where features are part of a one-hot categorical group. For binary decision trees, split thresholds of the form \(f \le \nu\) force the feature to be \(0\), while \(f > \nu\) forces it to be \(1\). We encode this by setting the corresponding entry in the group vector to \(-1\) (inactive) or \(1\) (active), respectively.

 Multiple categories may be inactive (ruled out, set to \(-1\)), but at most one category can be active (set to \(1\)). If any entry is set to \(1\), all other entries in the group must be set to \(-1\). Conversely, if all but one entry are \(-1\), the remaining entry must be \(1\), since exactly one category must be active.

\begin{algorithm}[H]
\caption{\textsc{DeduceCategoricalImplications}$(s,(f,\nu,b))$}
\label{alg:deduce_categorical_implications}
\begin{algorithmic}[1]
\REQUIRE Subregion \(s\), literal \((f,\nu,b)\)
\ENSURE Subregion with updated categorical constraints
\STATE \(s' \leftarrow \textsc{Copy}(s)\)
\IF{\(f\) is not a one-hot categorical feature}
    \STATE \textbf{return} \(s'\)
\ENDIF
\IF{\(\nu\) is not a binary split threshold}
    \STATE \textbf{return} \(s'\)
\ENDIF
\STATE Let \(g\) be the categorical group containing feature \(f\)
\STATE Let \(h\) be the position of feature \(f\) within group \(g\)
\STATE Initialize \(a_g \in \{-1,0,1\}^{m_g}\) to all zeros if absent
\IF{\(b=\texttt{true}\)}
    \STATE \(a_{gh} \leftarrow -1\) \COMMENT{\(f\le \nu\) means this one-hot feature is inactive, consider $\nu=\frac{1}{2}$}
\ELSE
    \STATE \(a_{gh} \leftarrow 1\) \COMMENT{\(f>\nu\) means this one-hot feature is active, consider $\nu=\frac{1}{2}$}
\ENDIF
\IF{there exists \(h\) such that \(a_{gh}=1\)}
    \STATE Set \(a_{gr}\leftarrow -1\) for all \(r\neq h\)
\ELSIF{exactly one entry \(h\) has \(a_{gh}=0\)}
    \STATE \(a_{gh}\leftarrow 1\)
\ENDIF
\STATE \textbf{return} \(s'\)
\end{algorithmic}
\end{algorithm}

\paragraph{\textsc{ImpliesTrue}.}
The routine \textsc{ImpliesTrue} determines whether a split condition \(f \le \nu\) is guaranteed to hold for all possible points in the subregion \(s\). For numerical features, this occurs when the upper bound of the interval satisfies \(u_f \le \nu\), meaning every feasible point lies on the True side. For categorical features, \(f \le \nu\) corresponds to the feature being \(0\). Thus, if the categorical status is \(-1\) (meaning the feature is forced to be \(0\)), the condition is always true. Otherwise, it is either always forced to be $1$, or we do not know, so the outcome is not implied, and the routine returns false.

\begin{algorithm}[H]
\caption{\textsc{ImpliesTrue}$(s,f,\nu)$}
\label{alg:implies_true}
\begin{algorithmic}[1]
\REQUIRE Subregion \(s\), split \(f\le \nu\)
\ENSURE Whether any hypothetical point in \(s\) satisfies \(f\le \nu\)
\STATE \((\ell_f,u_f] \leftarrow s.\mathrm{bounds}(f)\), using \((-\infty,\infty]\) if absent
\IF{\(u_f \le \nu\)}
    \STATE \textbf{return} \texttt{true}
\ENDIF
\IF{\(f\) is a one-hot categorical feature and \(\nu\) is a binary split threshold}
    \STATE Let \(a_{gh}\) be the status of feature \(f\) in its categorical group
    \IF{\(a_{gh}=-1\)}
        \STATE \textbf{return} \texttt{true}
    \ENDIF
\ENDIF
\STATE \textbf{return} \texttt{false}
\end{algorithmic}
\end{algorithm}

\paragraph{\textsc{ImpliesFalse}.}
The routine \textsc{ImpliesFalse} determines whether a split condition \(f \le \nu\) is guaranteed to fail for all possible points in the subregion \(s\). For numerical features, this occurs when the lower bound satisfies \(\ell_f \ge \nu\), meaning all feasible points lie strictly on the False side. For categorical features, \(f \le \nu\) is always false precisely when the feature equals \(1\). 

\begin{algorithm}[H]
\caption{\textsc{ImpliesFalse}$(s,f,\nu)$}
\label{alg:implies_false}
\begin{algorithmic}[1]
\REQUIRE Subregion \(s\), split \(f\le \nu\)
\ENSURE Whether any hypothetical point in \(s\) satisfies \(f>\nu\)
\STATE \((\ell_f,u_f] \leftarrow s.\mathrm{bounds}(f)\), using \((-\infty,\infty]\) if absent
\IF{\(\ell_f \ge \nu\)}
    \STATE \textbf{return} \texttt{true}
\ENDIF
\IF{\(f\) is a one-hot categorical feature and \(\nu\) is a binary split threshold}
    \STATE Let \(a_{gh}\) be the status of feature \(f\) in its categorical group
    \IF{\(a_{gh}=1\)}
        \STATE \textbf{return} \texttt{true}
    \ENDIF
\ENDIF
\STATE \textbf{return} \texttt{false}
\end{algorithmic}
\end{algorithm}

\subsection{Compression into a Single Defer Tree}
A compressed single-tree representation can substantially reduce root-to-leaf path lengths by specializing later-stage trees to the particular subregion in which they are evaluated. For example, a later-stage split such as \( \texttt{age} \le 23 \) may help distinguish inputs in subregions with \(18 < \texttt{age} \le 25\) and \(\texttt{age} > 25\), but is redundant for any path already restricted to \(\texttt{age} \le 18\); when unrolling the MDT into a single tree, such splits can be pruned along the branches where they have no effect. On a high level, we will do this by putting the tree of stage $i+1$ in the defer leaves of stage $i$, simplifying it, and repeating once we reach the new defer leaves.

Algorithm \ref{alg:build_compressed_single_tree_subroutine} constructs a compressed \emph{single-tree representation} of the full MDT by recursively expanding the staged decision process into one tree.

\begin{algorithm}[H]
\caption{\textsc{BuildCompressedSingleTree}$(v,k,s,R_1,\dots,R_K)$}
\label{alg:build_compressed_single_tree_subroutine}
\begin{algorithmic}[1]
\REQUIRE Node $v$, stage index $k$, subregion $s$, stage roots $R_1,\dots,R_K$
\ENSURE Compressed subtree rooted at $v$ under subregion $s$
\IF{$v$ is a leaf}
    \IF{$v.action \neq \texttt{defer}$}
        \STATE \textbf{return} leaf with action $v.action$
    \ENDIF
    \STATE $k' \leftarrow k+1$
    \IF{$k' \le K$}
        \STATE \textbf{return} $\textsc{BuildCompressedSingleTree}(R_{k'},k',s,R_1,\dots,R_K)$
    \ENDIF
    \STATE \textbf{return} leaf with action \texttt{black\_box}
    
\ENDIF
\STATE $f \leftarrow v.feature$
\STATE $\nu \leftarrow v.threshold$
\IF{$\textsc{ImpliesTrue}(s,f,\nu)$}
    \STATE \textbf{return} $\textsc{BuildCompressedSingleTree}(v.left,k,s,R_1,\dots,R_K)$
\ENDIF
\IF{$\textsc{ImpliesFalse}(s,f,\nu)$}
    \STATE \textbf{return} $\textsc{BuildCompressedSingleTree}(v.right,k,s,R_1,\dots,R_K)$
\ENDIF
\STATE $s_L \leftarrow \textsc{AddLiteral}(s,(f,\nu,\texttt{true}))$
\STATE $s_R \leftarrow \textsc{AddLiteral}(s,(f,\nu,\texttt{false}))$
\STATE $u_L \leftarrow \textsc{BuildCompressedSingleTree}(v.left,k,s_L,R_1,\dots,R_K)$
\STATE $u_R \leftarrow \textsc{BuildCompressedSingleTree}(v.right,k,s_R,R_1,\dots,R_K)$
\STATE \textbf{return} internal node with split $(f,\nu)$, left child $u_L$, right child $u_R$
\end{algorithmic}
\end{algorithm}

The procedure takes as input a node $v$ from stage $k$, along with a subregion $s$ describing the set of inputs that may reach this node. The subregion $s$ encodes all constraints accumulated along the path taken so far, potentially across multiple stages.

If $v$ is a leaf, there are two cases. If the leaf corresponds to a terminal prediction (i.e., not \texttt{defer}), then the prediction is returned directly. If the leaf corresponds to \texttt{defer}, then we transition to the next stage: we recurse on the root of stage $k+1$ with the same subregion $s$. If no further stages remain, we return a \texttt{black\_box} leaf, representing the final fallback model.

If $v$ is an internal node with split $(f,\nu)$, we again leverage the subregion $s$ to simplify the structure. If $s$ provably implies that all possible inputs would satisfy the split (i.e., always take the True branch), then we recurse only on the left child. Similarly, if $s$ implies the False branch, we recurse only on the right child. In either case, the split can be removed because it does not affect any possible point in the subregion. 

Otherwise, the split is ambiguous under $s$, and both branches must be explored. We construct refined subregions $s_L$ and $s_R$ by adding the corresponding literal to $s$ (whether we are supposing we take the True or False branch), thereby capturing the additional constraint imposed by each branch. The algorithm then recursively builds the left and right subtrees under these refined subregions and returns a new internal node with these children.

\subsection{Sequential Compression}

In this section, we show how to (possibly) further compress a sequential representation of the MDT. Because the points outside the defer region are downweighted, it is somewhat unlikely that there exist splits in the MDT that provably do not separate any possible pair of points in the defer region; however, it is still possible. We note that compression of the sequential representation is exactly the compression done to reduce the number of literals in the rules of the rule list representation.

Algorithm \ref{alg:simplify_dag_tree} simplifies a tree relative to a set of possible incoming subregions $\mathcal{S}$ (i.e., multiple defer leaves may lead to it, not just one). Intuitively, compared to constructing a single tree from the MDT, we must evaluate each split with respect to all subregions simultaneously, rather than a single input region.

\begin{algorithm}[H]
\caption{\textsc{SimplifySequential}$(v,\mathcal{S})$}
\label{alg:simplify_dag_tree}
\begin{algorithmic}[1]
\REQUIRE Node $v$, incoming set of subregions $\mathcal{S}$
\ENSURE Compressed subtree rooted at $v$
\IF{$v$ is a leaf}
    \STATE \textbf{return} leaf with action $v.action$
\ENDIF

\STATE $f \leftarrow v.feature$
\STATE $\tau \leftarrow v.threshold$

\STATE $all\_true \leftarrow \texttt{true}$
\STATE $all\_false \leftarrow \texttt{true}$

\FORALL{$s \in \mathcal{S}$}
    \IF{not $\textsc{ImpliesTrue}(s,f,\tau)$}
        \STATE $all\_true \leftarrow \texttt{false}$
    \ENDIF
    \IF{not $\textsc{ImpliesFalse}(s,f,\tau)$}
        \STATE $all\_false \leftarrow \texttt{false}$
    \ENDIF
\ENDFOR

\IF{$all\_true$}
    \STATE \textbf{return} $\textsc{SimplifySequential}(v.left,\mathcal{S})$
\ENDIF

\IF{$all\_false$}
    \STATE \textbf{return} $\textsc{SimplifySequential}(v.right,\mathcal{S})$
\ENDIF

\STATE $\mathcal{S}_L \leftarrow [\ ]$
\STATE $\mathcal{S}_R \leftarrow [\ ]$

\FORALL{$s \in \mathcal{S}$}
    \IF{not $\textsc{ImpliesFalse}(s,f,\tau)$}
        \STATE $s' \leftarrow \textsc{TightenNumericalBounds}(s,(f,\tau,\texttt{true}))$
        \STATE $s' \leftarrow \textsc{DeduceCategoricalImplications}(s',(f,\tau,\texttt{true}))$
        \STATE append $s'$ to $\mathcal{S}_L$
    \ENDIF
    \IF{not $\textsc{ImpliesTrue}(s,f,\tau)$}
        \STATE $s' \leftarrow \textsc{TightenNumericalBounds}(s,(f,\tau,\texttt{false}))$
        \STATE $s' \leftarrow \textsc{DeduceCategoricalImplications}(s',(f,\tau,\texttt{false}))$
        \STATE append $s'$ to $\mathcal{S}_R$
    \ENDIF
\ENDFOR

\STATE $u_L \leftarrow \textsc{SimplifySequential}(v.left,\mathcal{S}_L)$
\STATE $u_R \leftarrow \textsc{SimplifySequential}(v.right,\mathcal{S}_R)$

\STATE \textbf{return} internal node with split $(f,\tau)$, left child $u_L$, right child $u_R$
\end{algorithmic}
\end{algorithm}

The procedure is recursive and begins at the root node. For each subregion $s \in \mathcal{S}$, we determine whether the split always routes every possible input consistent with $s$ (not just training samples, but any hypothetical point that satisfies the subregion constraints) to the True branch, the False branch, or whether the outcome is ambiguous (lines 8--15). If and only if every subregion is guaranteed to take the True branch, then $\textsc{all\_true}$ is set to True. Similarly, if and only if every subregion is guaranteed to take the False branch, then $\textsc{all\_false}$ is set to True. In either of these cases, the split can be pruned, and we recurse on the corresponding child. Otherwise, the split cannot be removed, and we must recurse on both branches.

In lines 24--34, we construct two (potentially overlapping) collections of subregions before recursing:
\[
\mathcal{S}_L
=
\left\{
\, s^{\le} \;:\;
s \in \mathcal{S},\;
s \text{ does not provably imply } f > t,\;
s^{\le} \text{ is } s \text{ refined with } f \le t
\right\},
\]
\[
\mathcal{S}_R
=
\left\{
\, s^{>} \;:\;
s \in \mathcal{S},\;
s \text{ does not provably imply } f \le t,\;
s^{>} \text{ is } s \text{ refined with } f > t
\right\}.
\]

Each collection is formed by selecting subregions that are compatible with the corresponding branch and refining them with the constraint induced by the split. Subregions that already contradict a branch are excluded from that side.

These two collections are not disjoint. Any subregion $s$ that does not imply either outcome will appear in both sets (after being refined by the respective branch on which we recurse).

Algorithm \ref{alg:simplify_dag_tree} simplifies a tree as much as we provably can in this traversal when given a set of subregions that the tree will be used on. However, we still need to shrink this set of subregions before simplifying the next tree in the MDT. 

\begin{algorithm}[H]
\caption{Build compressed stage trees}
\label{alg:build_compressed_dag_stage_trees}
\begin{algorithmic}[1]
\REQUIRE Stage roots $R_1,\dots,R_K$
\ENSURE Compressed stage roots $\widetilde{R}_1,\dots,\widetilde{R}_K$
\STATE $\mathcal{S} \leftarrow [\ ]$
\STATE $\widetilde{\mathcal{R}} \leftarrow [\ ]$
\FOR{$k=1,\dots,K$}
    \STATE $\widetilde{R}_k \leftarrow \textsc{SimplifySequential}(R_k,\mathcal{S})$
    \STATE append $\widetilde{R}_k$ to $\widetilde{\mathcal{R}}$
    \STATE $\mathcal{P} \leftarrow \textsc{CollectPathsToAction}(\widetilde{R}_k,\texttt{defer})$
    
    \STATE $\mathcal{S} \leftarrow \textsc{ExtendSubregionsWithFeasiblePaths}(\mathcal{S},\mathcal{P})$
\ENDFOR
\STATE \textbf{return} $\widetilde{\mathcal{R}}$
\end{algorithmic}
\end{algorithm}
Algorithm \ref{alg:build_compressed_dag_stage_trees} performs this update iteratively across stages. At stage $k$, after simplifying the tree $R_k$ into $\widetilde{R}_k$, we identify all root-to-leaf paths that result in a \texttt{defer} decision. Each such path $p \in \mathcal{P}$ corresponds to a conjunction of literals describing a region of the input space that will be deferred by the current stage.
To construct the subregions for the next stage, we combine the incoming subregions $\mathcal{S}$ with these defer paths. Concretely, Algorithm \ref{alg:extend_states_with_feasible_paths} forms a Cartesian product between $\mathcal{S}$ and $\mathcal{P}$: for each subregion $s \in \mathcal{S}$ and each defer path $p \in \mathcal{P}$, we attempt to extend $s$ by sequentially incorporating the literals along $p$. This process can be viewed as intersecting the region represented by $s$ with the region defined by the path $p$. By adding one literal at a time, we allow for early termination.  A contradiction occurs if a literal violates the existing numerical bounds of the subregion or is incompatible with previously deduced categorical constraints. If any such violation occurs, the extended subregion is discarded. Otherwise, we tighten the numerical bounds and propagate any implied categorical constraints, producing a refined subregion $c$.

The resulting set $\mathcal{S}_{next}$ therefore consists precisely of those subregions that are simultaneously:
(i) reachable under the previous stages (captured by $\mathcal{S}$), and
(ii) routed to deferral by the current stage (captured by $\mathcal{P}$).

\begin{algorithm}[H]
\caption{\textsc{ExtendSubregionsWithFeasiblePaths}$(\mathcal{S},\mathcal{P})$}
\label{alg:extend_states_with_feasible_paths}
\begin{algorithmic}[1]
\REQUIRE Incoming subregions $\mathcal{S}$, defer paths $\mathcal{P}$
\ENSURE Next-stage incoming subregions
\STATE $\mathcal{S}_{next} \leftarrow [\ ]$
\FORALL{$s \in \mathcal{S}$}
    \FORALL{$p \in \mathcal{P}$}
        \STATE $c \leftarrow s$
        \STATE $feasible \leftarrow \texttt{true}$
        \FORALL{literal $\ell \in p$}
            \IF{$\textsc{ViolatesNumericBounds}(c,\ell)$}
                \STATE $feasible \leftarrow \texttt{false}$
                \STATE \textbf{break}
            \ENDIF
                \IF{$\textsc{ViolatesCategoricalConstraints}(c,\ell)$}
                    \STATE $feasible \leftarrow \texttt{false}$
                    \STATE \textbf{break}
                \ENDIF
                \STATE $c \leftarrow \textsc{TightenNumericalBounds}(c,\ell)$
                \STATE $c \leftarrow \textsc{DeduceCategoricalImplications}(c,\ell)$
        \ENDFOR
        \IF{$feasible$}
            \STATE append $c$ to $\mathcal{S}_{next}$
        \ENDIF
    \ENDFOR
\ENDFOR
\STATE \textbf{return} $\mathcal{S}_{next}$
\end{algorithmic}
\end{algorithm}

\subsection{Trivial Extensions}
After recursively simplifying the left and right subtrees of a split, we may encounter a situation where both children reduce to identical leaf nodes with the same action. In this case, the split no longer affects the prediction for any input and can be safely removed. We refer to such splits as trivial extensions and perform this pruning as a post-order traversal at the end of any compression scheme.

\begin{algorithm}[H]
\caption{\textsc{PruneTrivialExtensions}$(v)$}
\label{alg:prune_trivial_extensions}
\begin{algorithmic}[1]
\REQUIRE Node \(v\)
\ENSURE Tree rooted at \(v\) with trivial extensions removed
\IF{\(v\) is a leaf}
    \STATE \textbf{return} leaf with action \(v.action\)
\ENDIF

\STATE \(u_L \leftarrow \textsc{PruneTrivialExtensions}(v.left)\)
\STATE \(u_R \leftarrow \textsc{PruneTrivialExtensions}(v.right)\)

\IF{\(u_L\) and \(u_R\) are leaves and \(u_L.action = u_R.action\)}
    \STATE \textbf{return} leaf with action \(u_L.action\)
\ENDIF

\STATE \textbf{return} internal node with split \((v.feature,v.threshold)\), left child \(u_L\), right child \(u_R\)
\end{algorithmic}
\end{algorithm}

%% file: sections/appendix/geometry.tex
Each feasible subregion $s \in \mathcal{S}_{k-1}$ stores the constraints induced by the current deferred region. Let $\mathcal{J}_{\mathrm{num}}(s)$ denote the set of continuous (or ordinal) features that have been constrained in subregion $s$, and let $\mathcal{J}_{\mathrm{cat}}(s)$ denote the set of categorical feature groups that have been constrained. For each $j \in \mathcal{J}_{\mathrm{num}}(s)$, let $(\ell_{sj}, u_{sj}]$ denote the interval induced by subregion $s$. For each $j \in \mathcal{J}_{\mathrm{cat}}(s)$, let $A_{sj}$ denote the set of categories consistent (allowed by) with subregion $s$. Additional details on handling categorical variables are provided in \autoref{appendix:compression}. At a high level, for a one-hot encoded categorical group, a subregion may rule out a category via a False split, leave it unconstrained, or identify it as active either through a True split or by eliminating all other categories. The categories which are consistent with the subregion are exactly the categories that are not ruled out by the subregion.

We first map each continuous or ordinal feature into its empirical quantile coordinate so that distances across coordinates are on a common scale. Since features may be discrete, we use the mid-rank empirical CDF, assigning each value $v$ the quantile
\[
z_j(v)
=
\frac{\#\{X_j < v\} + \tfrac{1}{2}\#\{X_j = v\}}{N}.
\]

Because interval endpoints may not coincide with observed feature values (e.g., midpoints), we evaluate endpoint quantiles by snapping to the nearest realized value that remains within the interval. Specifically, define
\[
\underline{z}_j(t)
=
\inf\{ z_j(v) : v > t,\ v \in \mathrm{supp}(X_j)\},
\qquad
\overline{z}_j(t)
=
\sup\{ z_j(v) : v \le t,\ v \in \mathrm{supp}(X_j)\}.
\]

Then, for each training point $x_i$ that is not deferred, define its distance to the current deferred region by
\[
\textrm{defer\_dist}_{T_\textrm{{1:k-1}}}(x_i)
=
\min_{s \in \mathcal{S}_{k-1}}
\left(
\sum_{j \in \mathcal{J}_{\mathrm{num}}(s)} d_j^{(s)}(x_i)
+
\sum_{j \in \mathcal{J}_{\mathrm{cat}}(s)} d_j^{(s)}(x_i)
\right).
\]

For continuous (or ordinal) features $j \in \mathcal{J}_{\mathrm{num}}(s)$, we measure distance in quantile space as
\[
d_j^{(s)}(x_i)
=
\max\{\underline{z}_j(\ell_{sj}) - z_j(x_{ij}), 0\}
+
\max\{z_j(x_{ij}) - \overline{z}_j(u_{sj}), 0\}.
\]

For categorical features $j \in \mathcal{J}_{\mathrm{cat}}(s)$, let $c_{ij}$ denote the category of point $x_i$ on feature $j$. We define
\[
d_j^{(s)}(x_i)
=
\begin{cases}
0, & c_{ij} \in A_{sj},\\[4pt]
\tfrac{1}{2}, & c_{ij} \notin A_{sj}.
\end{cases}
\]

This choice is consistent with treating a binary feature as continuous under the mid-rank empirical CDF, where a mismatch incurs a quantile distance of \(\tfrac{1}{2}\); we extend this to general categorical features.

Unconstrained features (i.e., $j \notin \mathcal{J}_{\mathrm{num}}(s) \cup \mathcal{J}_{\mathrm{cat}}(s)$) contribute zero distance as no value for them will change what leaf you are in.

Using these distances, we convert each $\textrm{defer\_dist}_{T_\textrm{{1:k-1}}}(x_i)$ into a similarity score
\[
(1+\textrm{defer\_dist}_{T_\textrm{{1:k-1}}}(x_i))^{-\gamma}.
\]

Let $\mathcal{D}_{k-1} \subseteq \mathcal{X}$ denote the set of deferred points after stage $k-1$. We define training weights $w_i$ for the next iteration (on the first iteration, all weights are 1).
$$w_i \;=\;
\begin{cases}
1, & \text{if } x_i \in \mathcal{D}_{k-1} \\[6pt]
(1 - \mu)\,(1 +\textrm{defer\_dist}_{T_\textrm{{1:k-1}}}(x_i))^{-\gamma}, & \text{if } x_i \notin \mathcal{D}_{k-1}
\end{cases}$$

When $\gamma=0$, this reduces to assigning a uniform weight of $1-\mu$ to each non-deferred point and a weight of $1$ to each deferred point.

%% file: sections/appendix/algorithm.tex
\subsection{MDT Algorithm}

Algorithm \ref{alg:weighted_top_down_mdt_appendix_full} presents the full version of our algorithm. There are two additions relative to the algorithm presented in the main paper. The first is that we update the set of subregions $\mathcal{S}$ (lines 4, 21, 22, 28) and pass it into our \textsc{FilterSplits} and \textsc{GetWeights} so it does not have to be recomputed from scratch each iteration. The second is that we additionally support early termination of training if the model complexity becomes too large. This is measured in 3 ways: (i) a stage of the tree has too many leaves, (ii) the sum of leaves across stages becomes too large, or (iii) the number of leaves in the na\"{\i}ve representation of an MDT as a single tree becomes too large. (iii) takes advantage of a recurrence derived in Proposition \ref{prop:subchain}. Given that we have already computed that tree, we choose to add it to MDT anyway, stopping after that. 

\begin{algorithm}[H]
\caption{Top-down MDT training with feasible-region tracking}
\label{alg:weighted_top_down_mdt_appendix_full}
\begin{algorithmic}[1]
\REQUIRE Data $(X,X_{\mathrm{binarized}},y)$ with $X\in\mathbb{R}^{N\times p}$ and $X_{\mathrm{binarized}}\in\{0,1\}^{N\times q}$, depth $d$, $\tau$, $\eta$, $\mu\ge 0$, boolean flag $\mathrm{rescale\_tau}$, maximum number of stages $K_{\max}$, $\gamma\ge 0$, optional budgets $L_{\mathrm{stage}},L_{\mathrm{total}},L_{\mathrm{expanded}}$
\ENSURE MDT stages and fallback model

\STATE $B \leftarrow \textsc{FitFallbackModel}(X,y,\mathbf{1})$
\STATE $\mathcal{T}_{\mathrm{interp}} \leftarrow [\ ]$
\STATE $D \leftarrow \{1,\dots,N\}$
\STATE $\mathcal{S} \leftarrow \{s_{\emptyset}\}$ \COMMENT{Initial unconstrained subregion}
\STATE $\Pi \leftarrow 1$ \COMMENT{Running product of deferral leaves}
\STATE $G \leftarrow 1$ \COMMENT{Expanded-leaf count with fallback counted as one leaf}

\FOR{$k=1,\dots,K_{\max}$}
    \IF{$k=1$}
        \STATE $T_k \leftarrow \textsc{FitDeferTree}(X_{\mathrm{binarized}},y,B(X),\mathbf{1},\tau, \eta, d)$
    \ELSE
        \STATE $X_{\mathrm{filtered}} \leftarrow \textsc{FilterSplits}(X_{\mathrm{binarized}},\mathcal{S})$
        \STATE $w \leftarrow \textsc{GetWeights}(X,D,\mathcal{S},\gamma,\mu)$
        \STATE $B \leftarrow \textsc{FitFallbackModel}(X,y,w)$
        \STATE $\tau_k \leftarrow \tau\cdot \frac{\sum_{i=1}^N w_i}{N}$ \textbf{ if } $\mathrm{rescale\_tau}$ \textbf{ else } $\tau$
        \STATE $T_k \leftarrow \textsc{FitDeferTree}(X_{\mathrm{filtered}},y,B(X),w,\tau_k, \eta, d)$
    \ENDIF

    \STATE $\mathcal{T}_{\mathrm{interp}}\leftarrow \mathcal{T}_{\mathrm{interp}}+[T_k]$
    \STATE $\ell_k \leftarrow \textsc{NumLeaves}(T_k)$
    \STATE $d_k \leftarrow \textsc{NumDeferLeaves}(T_k)$
    \STATE $D_{\mathrm{new}} \leftarrow \{i\in D:T_k(x_i)=\texttt{defer}\}$

    \STATE $\mathcal{P}\leftarrow \textsc{CollectPathsToAction}(T_k,\texttt{defer})$
    \STATE $\mathcal{S}_{\mathrm{new}}\leftarrow \textsc{ExtendSubregionsWithFeasiblePaths}(\mathcal{S},\mathcal{P})$

    \STATE $(\texttt{should\_stop},G,\Pi) \leftarrow
    \textsc{ShouldStopTraining}(D,D_{\mathrm{new}},\ell_k,d_k,\mathcal{T}_{\mathrm{interp}},G,\Pi,L_{\mathrm{stage}},L_{\mathrm{total}},L_{\mathrm{expanded}})$

    \IF{$\texttt{should\_stop}$}
        \STATE \textbf{break}
    \ENDIF

    \STATE $D\leftarrow D_{\mathrm{new}}$
    \STATE $\mathcal{S}\leftarrow \mathcal{S}_{\mathrm{new}}$
\ENDFOR

\IF{$D\neq\varnothing$}
    \STATE $w \leftarrow \textsc{GetWeights}(X,D,\mathcal{S},\gamma,\mu)$
    \STATE $B \leftarrow \textsc{FitFallbackModel}(X,y,w)$
\ENDIF

\STATE \textbf{return} $\mathcal{T}_{\mathrm{interp}}+[B]$
\end{algorithmic}
\end{algorithm}

Algorithm~\ref{alg:should_stop_training} includes the aforementioned stopping criteria. It terminates training if the deferred set becomes empty, fails to shrink, or if any complexity budget is exceeded. The expanded tree size is tracked online using a recurrence, allowing early stopping without constructing a single tree representation.

\begin{algorithm}[H]
\caption{\textsc{ShouldStopTraining}$(D,D_{\mathrm{new}},\ell_k,d_k,\mathcal{T}_{\mathrm{interp}},G,\Pi,L_{\mathrm{stage}},L_{\mathrm{total}},L_{\mathrm{expanded}})$}
\label{alg:should_stop_training}
\begin{algorithmic}[1]
\REQUIRE Previous deferred set $D$, new deferred set $D_{\mathrm{new}}$, current tree leaf count $\ell_k$, current defer-leaf count $d_k$, current stage list $\mathcal{T}_{\mathrm{interp}}$, running expanded count $G$, running defer-product $\Pi$, optional budgets $L_{\mathrm{stage}},L_{\mathrm{total}},L_{\mathrm{expanded}}$
\ENSURE Boolean $\texttt{should\_stop}$, updated $G$, updated $\Pi$

\IF{$D_{\mathrm{new}}=\varnothing$}
    \STATE \textbf{return} $(\textbf{true},G,\Pi)$ \COMMENT{No points remain deferred}
\ENDIF

\IF{$D_{\mathrm{new}}=D$}
    \STATE \textbf{return} $(\textbf{true},G,\Pi)$ \COMMENT{Deferred set did not shrink}
\ENDIF

\IF{$L_{\mathrm{expanded}}$ is specified}
    \STATE $G \leftarrow G+\Pi(\ell_k-1)$
    \STATE $\Pi \leftarrow \Pi d_k$
\ENDIF

\IF{$L_{\mathrm{stage}}$ is specified and $\ell_k>L_{\mathrm{stage}}$}
    \STATE \textbf{return} $(\textbf{true},G,\Pi)$ \COMMENT{Current stage exceeds per-tree budget}
\ENDIF

\IF{$L_{\mathrm{total}}$ is specified and $\sum_{T\in\mathcal{T}_{\mathrm{interp}}}|T|>L_{\mathrm{total}}$}
    \STATE \textbf{return} $(\textbf{true},G,\Pi)$ \COMMENT{Sequential model exceeds total leaf budget}
\ENDIF

\IF{$L_{\mathrm{expanded}}$ is specified and $G>L_{\mathrm{expanded}}$}
    \STATE \textbf{return} $(\textbf{true},G,\Pi)$ \COMMENT{Expanded single-tree representation exceeds budget}
\ENDIF

\STATE \textbf{return} $(\textbf{false},G,\Pi)$
\end{algorithmic}
\end{algorithm}

Algorithm~\ref{alg:filter_splits} removes binary features that do not distinguish between feasible deferred subregions. A feature is discarded if it is provably always true or always false across all subregions (note that we can't discard a feature if it is provably true on half of the subregions and provably false on the others). This has the added bonus of reducing the search space, which is beneficial for near-optimal decision tree algorithms. We implement this by storing, for each column, a string representation of the split (e.g., \texttt{age<=25}), which we parse to recover the feature and threshold. The column itself in the binarized dataset includes which samples satisfy the split.
\begin{algorithm}[H]
\caption{\textsc{FilterSplits}$(X_{\mathrm{binarized}},\mathcal{S})$}
\label{alg:filter_splits}
\begin{algorithmic}[1]
\REQUIRE Binarized feature matrix $X_{\mathrm{binarized}}$, feasible deferred subregions $\mathcal{S}$
\ENSURE Filtered binarized feature matrix
\STATE $I\leftarrow [\ ]$
\FOR{$j=1,\dots,q$}
    \STATE Get the split $(f \le \nu)$ from the metadata associated with column $j$
    \STATE $\texttt{all\_true}\leftarrow \textbf{true}$
    \STATE $\texttt{all\_false}\leftarrow \textbf{true}$
    \FORALL{$s\in\mathcal{S}$}
        \IF{not $\textsc{ImpliesTrue}(s,f,\nu)$}
            \STATE $\texttt{all\_true}\leftarrow \textbf{false}$
        \ENDIF
        \IF{not $\textsc{ImpliesFalse}(s,f,\nu)$}
            \STATE $\texttt{all\_false}\leftarrow \textbf{false}$
        \ENDIF
    \ENDFOR
    \IF{not $\texttt{all\_true}$ and not $\texttt{all\_false}$}
        \STATE append $j$ to $I$
    \ENDIF
\ENDFOR
\STATE \textbf{return} $X_{\mathrm{binarized}}[:,I]$
\end{algorithmic}
\end{algorithm}

Algorithm~\ref{alg:get_weights} assigns weights based on distance to the deferred region. Deferred points receive full weight, while non-deferred points are down-weighted according to their proximity to the nearest feasible subregion. This emphasizes points near the deferral boundary during training.

\begin{algorithm}[H]
\caption{\textsc{GetWeights}$(X,D,\mathcal{S},\gamma,\mu)$}
\label{alg:get_weights}
\begin{algorithmic}[1]
\REQUIRE Data $X$, current deferred index set $D$, feasible deferred subregions $\mathcal{S}$, decay parameter $\gamma$, non-deferred downweighting parameter $\mu$
\ENSURE Weight vector $w\in\mathbb{R}_{\ge 0}^N$
\FOR{$i=1,\dots,N$}
    \IF{$i\in D$}
        \STATE $w_i\leftarrow 1$
    \ELSE
        \STATE $\Delta_i\leftarrow \min_{s\in\mathcal{S}}\textsc{DistanceToSubregion}(x_i,s)$
        \STATE $w_i\leftarrow (1-\mu)(1+\Delta_i)^{-\gamma}$
    \ENDIF
\ENDFOR
\STATE \textbf{return} $w$
\end{algorithmic}
\end{algorithm}

Algorithm~\ref{alg:distance_to_subregion} computes the distance from a point to a subregion by summing per-feature deviations. Numerical features are measured in quantile space, while categorical features incur a fixed penalty of $\frac{1}{2}$ if the subregion does not allow the category that the point is in (i.e, the path to the subregion took a split of false on that category). Unconstrained features contribute zero.

\begin{algorithm}[h]
\caption{\textsc{DistanceToSubregion}$(x_i,s)$}
\label{alg:distance_to_subregion}
\begin{algorithmic}[1]
\REQUIRE Point $x_i$, feasible subregion $s$
\ENSURE Distance from $x_i$ to $s$
\STATE $\Delta\leftarrow 0$
\FORALL{numerical or ordinal features $j\in\mathcal{J}_{\mathrm{num}}(s)$}
    \STATE Let $(\ell_{sj},u_{sj}]$ be the interval for feature $j$ in $s$
    \STATE $\Delta\leftarrow \Delta+
    \max\{\underline z_j(\ell_{sj})-z_j(x_{ij}),0\}
    +\max\{z_j(x_{ij})-\overline z_j(u_{sj}),0\}$
\ENDFOR
\FORALL{categorical groups $j\in\mathcal{J}_{\mathrm{cat}}(s)$}
    \STATE Let $A_{sj}$ be the set of categories allowed by $s$
    \STATE Let $c_{ij}$ be the category of $x_i$ in group $j$
    \IF{$c_{ij}\notin A_{sj}$}
        \STATE $\Delta\leftarrow \Delta+\frac12$
    \ENDIF
\ENDFOR
\STATE \textbf{return} $\Delta$
\end{algorithmic}
\end{algorithm}

\newpage
\subsection{Defer Tree Algorithm}
\label{appendix:licketysplitsection}

In Algorithm \ref{alg:fit_defer_tree} construct each defer tree using a recursive, top-down procedure that jointly optimizes over prediction and deferral decisions. This algorithm uses the same heuristics to select splits as are used in the LicketySPLIT algorithm of \cite{babbar2025near}. Our implementation extends the efficient C++ codebase of~\citet{heile2026}, which implements a modified version of LicketySPLIT. Their algorithm considers an unweighted dataset and builds standard decision trees (without deferral). Our changes lie in supporting weights and deferral, thus optimizing a different objective. We solve the optimization problem using a per-leaf penalty, which is equivalent to optimizing a per-split penalty, since the number of leaves in a decision tree is always exactly one more than the number of splits, differing only by an additive constant $\tau$.

In our code, we also modify two other algorithms so that they can construct defer trees: the greedy decision tree algorithm implemented in scikit-learn \citep{scikit-learn}, which operates directly on continuous features, and a modified version of CLARITree \citep{wang2026claritree}. 

Algorithm \ref{alg:fit_defer_tree}, at any node, considers three possible leaf actions: predicting class $0$, predicting class $1$, or deferring to the fallback model. Each action is assigned an objective value that combines a fixed leaf penalty $\tau$ with a weighted misclassification cost. 

At a high level, we will recursively consider every possible split. We query a greedy method to build a subtree on the left and right child of every split, and choose the split that minimizes our objective under greedy completions. After choosing this split, we recurse on the left and right children of this split, and repeat.

\begin{algorithm}[H]
\caption{\textsc{FitDeferTree}$(X_{\mathrm{binarized}},y,r,w,\tau,\eta,d)$}
\label{alg:fit_defer_tree}
\begin{algorithmic}[1]
\REQUIRE Binarized data $X_{\mathrm{binarized}}$, labels $y$, fallback predictions $r$, weights $w$, leaf penalty $\tau$, defer penalty $\eta$, depth budget $d$
\ENSURE A defer tree $T$
\STATE $A \leftarrow \{1,\dots,N\}$
\STATE $(T,C) \leftarrow \textsc{BuildDeferTree}(A,d)$
\STATE \textbf{return} $T$
\end{algorithmic}
\end{algorithm}

\begin{algorithm}[h]
\caption{\textsc{BuildDeferTree}$(A,d)$}
\label{alg:build_defer_tree}
\begin{algorithmic}[1]
\REQUIRE Index set $A$, remaining depth $d$
\ENSURE A pair $(T,C)$, where $T$ is a defer subtree on $A$ and $C$ is its objective value
\STATE $W_A \leftarrow \sum_{i\in A} w_i$
\STATE $C_0 \leftarrow \tau+\sum_{i\in A} w_i\mathbf{1}\{y_i=1\}$
\STATE $C_1 \leftarrow \tau+\sum_{i\in A} w_i\mathbf{1}\{y_i=0\}$
\STATE $C_{\mathrm{defer}} \leftarrow \tau+\sum_{i\in A} w_i\mathbf{1}\{r_i\neq y_i\}+\eta W_A$
\STATE $(C_{\mathrm{leaf}},a_{\mathrm{leaf}})\leftarrow \min\{(C_0,0),(C_1,1),(C_{\mathrm{defer}},\texttt{defer})\}$

\IF{$d=0$ or $|A|\le 1$}
    \STATE \textbf{return} $(\text{leaf with action }a_{\mathrm{leaf}}, C_{\mathrm{leaf}})$
\ENDIF

\STATE $f^\star \leftarrow \bot$
\STATE $C^\star \leftarrow \infty$

\FOR{$f=1,\dots,q$}
    \STATE $A_L \leftarrow \{i\in A:X_{\mathrm{binarized},if}=1\}$
    \STATE $A_R \leftarrow \{i\in A:X_{\mathrm{binarized},if}=0\}$
    \IF{$A_L=\varnothing$ or $A_R=\varnothing$}
        \STATE \textbf{continue}
    \ENDIF
    \STATE $(\widetilde{T}_L,\widetilde{C}_L)\leftarrow \textsc{Greedy}(A_L,d-1)$
    \STATE $(\widetilde{T}_R,\widetilde{C}_R)\leftarrow \textsc{Greedy}(A_R,d-1)$
    \STATE $C_f \leftarrow \widetilde{C}_L+\widetilde{C}_R$
    \IF{$C_f<C^\star$}
        \STATE $C^\star \leftarrow C_f$
        \STATE $f^\star \leftarrow f$
    \ENDIF
\ENDFOR

\IF{$f^\star=\bot$}
    \STATE \textbf{return} $(\text{leaf with action }a_{\mathrm{leaf}}, C_{\mathrm{leaf}})$
\ENDIF

\STATE $A_L^\star \leftarrow \{i\in A:X_{\mathrm{binarized},if^\star}=1\}$
\STATE $A_R^\star \leftarrow \{i\in A:X_{\mathrm{binarized},if^\star}=0\}$

\STATE $(T_L,C_L) \leftarrow \textsc{BuildDeferTree}(A_L^\star,d-1)$
\STATE $(T_R,C_R) \leftarrow \textsc{BuildDeferTree}(A_R^\star,d-1)$
\STATE $C_{\mathrm{split}}\leftarrow C_L+C_R$

\IF{$C_{\mathrm{leaf}}\le C_{\mathrm{split}}$}
    \STATE \textbf{return} $(\text{leaf with action }a_{\mathrm{leaf}}, C_{\mathrm{leaf}})$
\ENDIF

\STATE Get the non-binarized feature $f_{\mathrm{orig}}$ and threshold $\nu$ (in $\leq \nu$) from the metadata of column $f^\star$
\STATE $T\leftarrow$ internal node with $T.\texttt{feature}=f_{\mathrm{orig}}$, $T.\texttt{threshold}=\nu$, left child $T_L$, and right child $T_R$
\STATE \textbf{return} $(T,C_{\mathrm{split}})$
\end{algorithmic}
\end{algorithm}

The greedy subroutine (Algorithm \ref{alg:greedy_defer_tree}) follows a very similar structure: it considers all splits and chooses the one that minimizes the weighted label entropy of the resulting partition (we minimize this quantity in Algorithm \ref{alg:best_entropy_split}).

\begin{algorithm}[!t]
\caption{\textsc{Greedy}$(A,d)$}
\label{alg:greedy_defer_tree}
\begin{algorithmic}[1]
\REQUIRE Index set $A$, remaining depth $d$
\ENSURE A pair $(T,C)$, where $T$ is a greedy defer subtree on $A$ and $C$ is its objective value
\STATE $W_A \leftarrow \sum_{i\in A} w_i$
\STATE $C_0 \leftarrow \tau+\sum_{i\in A} w_i\mathbf{1}\{y_i=1\}$
\STATE $C_1 \leftarrow \tau+\sum_{i\in A} w_i\mathbf{1}\{y_i=0\}$
\STATE $C_{\mathrm{defer}} \leftarrow \tau+\sum_{i\in A} w_i\mathbf{1}\{r_i\neq y_i\}+\eta W_A$
\STATE $(C_{\mathrm{leaf}},a_{\mathrm{leaf}})\leftarrow \min\{(C_0,0),(C_1,1),(C_{\mathrm{defer}},\texttt{defer})\}$

\IF{$d=0$ or $|A|\le 1$}
    \STATE \textbf{return} $(\text{leaf with action }a_{\mathrm{leaf}}, C_{\mathrm{leaf}})$
\ENDIF

\STATE $f^\star \leftarrow \textsc{BestEntropySplit}(A)$

\IF{$f^\star=\bot$}
    \STATE \textbf{return} $(\text{leaf with action }a_{\mathrm{leaf}}, C_{\mathrm{leaf}})$
\ENDIF

\STATE $A_L^\star \leftarrow \{i\in A:X_{\mathrm{binarized},if^\star}=1\}$
\STATE $A_R^\star \leftarrow \{i\in A:X_{\mathrm{binarized},if^\star}=0\}$

\IF{$A_L^\star=\varnothing$ or $A_R^\star=\varnothing$}
    \STATE \textbf{return} $(\text{leaf with action }a_{\mathrm{leaf}}, C_{\mathrm{leaf}})$
\ENDIF

\STATE $(T_L,C_L)\leftarrow \textsc{Greedy}(A_L^\star,d-1)$
\STATE $(T_R,C_R)\leftarrow \textsc{Greedy}(A_R^\star,d-1)$
\STATE $C_{\mathrm{split}}\leftarrow C_L+C_R$

\IF{$C_{\mathrm{leaf}}\le C_{\mathrm{split}}$}
    \STATE \textbf{return} $(\text{leaf with action }a_{\mathrm{leaf}}, C_{\mathrm{leaf}})$
\ENDIF

\STATE $T\leftarrow$ internal node splitting on $X_{\mathrm{binarized},f^\star}=1$, with left child $T_L$ and right child $T_R$
\STATE \textbf{return} $(T,C_{\mathrm{split}})$
\end{algorithmic}
\end{algorithm}

Algorithm~\ref{alg:best_entropy_split} selects splits by minimizing the weighted label entropy of the resulting partition. This focuses on separating the labels rather than directly optimizing the defer-aware objective. As this is the standard heuristic in greedy decision tree algorithms, we adopt it here.
\begin{algorithm}[!t]
\caption{\textsc{BestEntropySplit}$(A)$}
\label{alg:best_entropy_split}
\begin{algorithmic}[1]
\REQUIRE Index set $A$
\ENSURE Feature index $f^\star$ minimizing weighted entropy (or $\bot$ if none valid)

\STATE $f^\star \leftarrow \bot$
\STATE $H^\star \leftarrow \infty$

\FOR{$f=1,\dots,q$}
    \STATE $A_L \leftarrow \{i\in A : X_{\mathrm{binarized},if}=1\}$
    \STATE $A_R \leftarrow \{i\in A : X_{\mathrm{binarized},if}=0\}$

    \IF{$A_L=\varnothing$ or $A_R=\varnothing$}
        \STATE \textbf{continue}
    \ENDIF

    \STATE $W_L \leftarrow \sum_{i\in A_L} w_i$
    \STATE $W_R \leftarrow \sum_{i\in A_R} w_i$
    \STATE $W \leftarrow W_L + W_R$

    \IF{$W=0$}
        \STATE \textbf{continue}
    \ENDIF

    \STATE $p_L \leftarrow 
        \begin{cases}
            \dfrac{\sum_{i\in A_L} w_i y_i}{W_L}, & W_L>0\\
            0, & \text{otherwise}
        \end{cases}$
    \STATE $p_R \leftarrow 
        \begin{cases}
            \dfrac{\sum_{i\in A_R} w_i y_i}{W_R}, & W_R>0\\
            0, & \text{otherwise}
        \end{cases}$

    \STATE $H_L \leftarrow 
        \begin{cases}
            -p_L\log p_L - (1-p_L)\log(1-p_L), & p_L\in(0,1)\\
            0, & \text{otherwise}
        \end{cases}$
    \STATE $H_R \leftarrow 
        \begin{cases}
            -p_R\log p_R - (1-p_R)\log(1-p_R), & p_R\in(0,1)\\
            0, & \text{otherwise}
        \end{cases}$

    \STATE $H_f \leftarrow \dfrac{W_L}{W} H_L + \dfrac{W_R}{W} H_R$

    \IF{$H_f < H^\star$}
        \STATE $H^\star \leftarrow H_f$
        \STATE $f^\star \leftarrow f$
    \ENDIF
\ENDFOR

\STATE \textbf{return} $f^\star$
\end{algorithmic}
\end{algorithm}

\FloatBarrier

\newpage

\subsection{Improving Optimization with Backfitting}

The MDT training procedure constructs the model as a sequence of defer trees that are trained with respect to a temporary black box fallback model (Algorithms \ref{alg:weighted_top_down_mdt} and \ref{alg:weighted_top_down_mdt_appendix_full}). At the first stage, we train $T_1$ to either predict or defer to an initial fallback model $B_1$. After this stage identifies a remaining deferred region, we reweight the data toward that region, train a new fallback $B_2$, and then train $T_2$ to either predict or defer to $B_2$. Continuing this process, we train $T_i$ to defer to $B_i$, but in the final model, $T_i$ does not necessarily defer to $B_i$. For instance, if there are 4 trees in the MDT, $T_1$ defers to another MDT, not a black box.

Therefore, once full MDT has been constructed, it may be beneficial to retrain stage $j$ to defer to the entire suffix $T_{j+1:K};B$. This motivates the following procedure: choose a component of the MDT, hold the rest of the model fixed, and retrain that component against the suffix that follows it. As we are optimizing using near-optimal algorithms, it will also be helpful to have an acceptance step to accept the new defer tree only if it is better in the objective that the algorithm tried to optimize.

The backward pass (Algorithm \ref{alg:backward_pass_mdt}) starts from the final fallback model and moves upstream through the MDT, retraining each stage conditioned on the rest of it. We then visit stages in reverse order (optionally starting with the fallback model) doing exactly that. In contrast, the forward pass (Algorithm \ref{alg:forward_pass_mdt}) does exactly the same thing but it starts from the first stage. We put these together in an alternating approach (Algorithm \ref{alg:alternating_backfitting_mdt}). It combines these two sweeps in the following update order: $B,\ T_K,T_{K-1},\ldots,T_1,\ T_2,T_3,\ldots,T_K,\ldots$. The results of using this approach are shown in \autoref{fig:backfit_mdt_xgb_selected}. In short, backfitting gives clear improvements on some datasets, such as Tictactoe, where revisiting stages may help capture higher-order interactions, but it degrades performance on several others.

\begin{algorithm}[H]
\caption{\textsc{BackwardPass}: Backfitting MDT stages in reverse order}
\label{alg:backward_pass_mdt}
\begin{algorithmic}[1]
\REQUIRE Data $(X,X_{\mathrm{binarized}},y)$, MDT stages $\mathcal{T}_{\mathrm{interp}}=[T_1,\dots,T_K]$, fallback model $B$, depth $d$, $\tau$, $\eta$, $\mu\ge 0$, $\gamma\ge 0$, boolean flag $\mathrm{rescale\_tau}$, endpoints $j_{\mathrm{start}}$ and $j_{\mathrm{end}}$
\ENSURE Updated MDT stages and fallback model

\FOR{$j=j_{\mathrm{start}},j_{\mathrm{start}}-1,\dots,j_{\mathrm{end}}$}
    \STATE $D_{j-1} \leftarrow \{i:T_1(x_i)=\cdots=T_{j-1}(x_i)=\texttt{defer}\}$
    \STATE $\mathcal{S}_{j-1} \leftarrow \textsc{GetDeferredSubregions}(T_{1:j-1})$
    \STATE $w \leftarrow \textsc{GetWeights}(X,D_{j-1},\mathcal{S}_{j-1},\gamma,\mu,T_{1:j-1})$
    \STATE $\widehat{y}^{\mathrm{suffix}} \leftarrow \widehat{y}_{T_{j+1:K},B}(X)$
    \STATE $\tau_j \leftarrow \tau\cdot \frac{\sum_{i=1}^N w_i}{N}$ \textbf{ if } $\mathrm{rescale\_tau}$ \textbf{ else } $\tau$
    \STATE $X_{\mathrm{filtered}} \leftarrow \textsc{FilterSplits}(X_{\mathrm{binarized}},\mathcal{S}_{j-1})$
    \STATE $\widetilde{T}_j \leftarrow \textsc{FitDeferTree}(X_{\mathrm{filtered}},y,\widehat{y}^{\mathrm{suffix}},w,\tau_j,\eta,d)$

    \STATE $L_{\mathrm{old}} \leftarrow
    L_{T_j,T_{j+1:K};B}(D^w,\tau_j,\eta)$
    \STATE $L_{\mathrm{new}} \leftarrow
    L_{\widetilde{T}_j,T_{j+1:K};B}(D^w,\tau_j,\eta)$

    \IF{$L_{\mathrm{new}}\le L_{\mathrm{old}}$}
        \STATE $T_j \leftarrow \widetilde{T}_j$
    \ENDIF
\ENDFOR

\STATE \textbf{return} $[T_1,\dots,T_K]+[B]$
\end{algorithmic}
\end{algorithm}

\begin{algorithm}[H]
\caption{\textsc{ForwardPass}: Backfitting MDT stages in forward order}
\label{alg:forward_pass_mdt}
\begin{algorithmic}[1]
\REQUIRE Data $(X,X_{\mathrm{binarized}},y)$, MDT stages $\mathcal{T}_{\mathrm{interp}}=[T_1,\dots,T_K]$, fallback model $B$, depth $d$, $\tau$, $\eta$, $\mu\ge 0$, $\gamma\ge 0$, boolean flag $\mathrm{rescale\_tau}$, endpoints $j_{\mathrm{start}}$ and $j_{\mathrm{end}}$
\ENSURE Updated MDT stages and fallback model

\FOR{$j=j_{\mathrm{start}},j_{\mathrm{start}}+1,\dots,j_{\mathrm{end}}$}
    \STATE $D_{j-1} \leftarrow \{i:T_1(x_i)=\cdots=T_{j-1}(x_i)=\texttt{defer}\}$
    \STATE $\mathcal{S}_{j-1} \leftarrow \textsc{GetDeferredSubregions}(T_{1:j-1})$
    \STATE $w \leftarrow \textsc{GetWeights}(X,D_{j-1},\mathcal{S}_{j-1},\gamma,\mu,T_{1:j-1})$
    \STATE $\widehat{y}^{\mathrm{suffix}} \leftarrow \widehat{y}_{T_{j+1:K},B}(X)$
    \STATE $\tau_j \leftarrow \tau\cdot \frac{\sum_{i=1}^N w_i}{N}$ \textbf{ if } $\mathrm{rescale\_tau}$ \textbf{ else } $\tau$
    \STATE $X_{\mathrm{filtered}} \leftarrow \textsc{FilterSplits}(X_{\mathrm{binarized}},\mathcal{S}_{j-1})$
    \STATE $\widetilde{T}_j \leftarrow \textsc{FitDeferTree}(X_{\mathrm{filtered}},y,\widehat{y}^{\mathrm{suffix}},w,\tau_j,\eta,d)$

    \STATE $L_{\mathrm{old}} \leftarrow
    L_{T_j,T_{j+1:K};B}(D^w,\tau_j,\eta)$
    \STATE $L_{\mathrm{new}} \leftarrow
    L_{\widetilde{T}_j,T_{j+1:K};B}(D^w,\tau_j,\eta)$

    \IF{$L_{\mathrm{new}}\le L_{\mathrm{old}}$}
        \STATE $T_j \leftarrow \widetilde{T}_j$
    \ENDIF
\ENDFOR

\STATE \textbf{return} $[T_1,\dots,T_K]+[B]$
\end{algorithmic}
\end{algorithm}

\begin{algorithm}[H]
\caption{Backfitting for MDTs}
\label{alg:alternating_backfitting_mdt}
\begin{algorithmic}[1]
\REQUIRE Data $(X,X_{\mathrm{binarized}},y)$, MDT stages $\mathcal{T}_{\mathrm{interp}}=[T_1,\dots,T_K]$, fallback model $B$, number of sweeps $M$, depth $d$, $\tau$, $\eta$, $\mu\ge 0$, $\gamma\ge 0$, boolean flag $\mathrm{rescale\_tau}$
\ENSURE Updated MDT stages and fallback model

\FOR{$m=1,\dots,M$}
    \STATE $D_K \leftarrow \{i:T_1(x_i)=\cdots=T_K(x_i)=\texttt{defer}\}$
    \STATE $\mathcal{S}_K \leftarrow \textsc{GetDeferredSubregions}(T_{1:K})$
    \STATE $w \leftarrow \textsc{GetWeights}(X,D_K,\mathcal{S}_K,\gamma,\mu,T_{1:K})$
    \STATE $B \leftarrow \textsc{FitFallbackModel}(X,y,w)$

    \STATE $\mathcal{T}_{\mathrm{interp}}+[B] \leftarrow
    \textsc{BackwardPass}(X,X_{\mathrm{binarized}},y,\mathcal{T}_{\mathrm{interp}},B,d,\tau,\eta,\mu,\gamma,\mathrm{rescale\_tau},K,1)$

    \STATE $\mathcal{T}_{\mathrm{interp}}+[B] \leftarrow
    \textsc{ForwardPass}(X,X_{\mathrm{binarized}},y,\mathcal{T}_{\mathrm{interp}},B,d,\tau,\eta,\mu,\gamma,\mathrm{rescale\_tau},2,K)$
\ENDFOR

\STATE \textbf{return} $\mathcal{T}_{\mathrm{interp}}+[B]$
\end{algorithmic}
\end{algorithm}
\newpage

%% file: sections/appendix/objectives.tex
Learning a sparse decision tree is commonly formulated as minimizing a regularized empirical
risk of the form

$$    \frac{1}{N}\sum_{i=1}^N \mathds{1}\{T(x_i) \neq y_i\}
    +
    \lambda |T|,$$

where $|T|$ is the number of leaves in the tree \citep{gosdt, mctavish2022fast, babbar2025near}. For binary trees, the number of internal
split nodes is always $|T|-1$. Therefore, minimizing a penalty on leaves is equivalent, up to
an additive constant $\lambda$, to minimizing a penalty on split nodes:

$$    \frac{1}{N}\sum_{i=1}^N \mathds{1}\{T(x_i) \neq y_i\}
    +
    \lambda |T|
    =
    \frac{1}{N}\sum_{i=1}^N \mathds{1}\{T(x_i) \neq y_i\}
    +
    \lambda (|T|-1)
    +
    \lambda.$$

Since the two objectives differ only by a constant term, they have the same set of minimizers. We choose to work with the number of splits for a reason that will become clear in the MDT objective.

Hybrid interpretable rule-list models instead optimize an objective that balances the error
of the full hybrid system, the complexity of the interpretable rule list, and the amount of
deferral to a black-box model. In the notation of \citet{ferry2023learning}, for a rule list
$r$ and black box $B$, this objective has the form
\[
    \frac{\widehat L_{\mathcal D}(B,r)}{N}
    +
    \lambda |r|
    +
    \beta \frac{|\mathcal D \setminus \mathcal D_r|}{N},
\]
where $\widehat L_{\mathcal D}(B,r)$ is the empirical error of the full hybrid model,
$|r|$ is the number of rules, and $\mathcal D_r$ is the subset of samples captured by the
rule list. Equivalently, $|\mathcal D \setminus \mathcal D_r|/N$ is the empirical deferral
rate.

Our single defer tree objective combines these two perspectives. It is a sparse decision tree
objective combined with a soft penalty for deferral:
\[
\mathcal L_{T,B}(\mathcal D^w,\tau,\eta)
=
\tau(|T|-1)
+
\sum_{i=1}^N
w_i
\left(
\mathds{1}\{\hat y_{T,B}(x_i)\neq y_i\}
+
\eta \mathds{1}\{T(x_i)=\defer\}
\right).
\]

When all weights
are $w_i=1$ and  $\eta=0$, we recover the same loss (except we do not normalize by the dataset size $N$) as \citep{gosdt, mctavish2022fast, babbar2025near}, with $\tau = \lambda N$.

The MDT objective is then the natural extension of the single defer-tree
objective: it penalizes the predictions and deferral rate after routing through all stages of the MDT.
\[
\mathcal L_{T_{1:k},B}(\mathcal D^w,\tau,\eta)
=
\tau \sum_{j=1}^k (|T_j|-1)
+
\sum_{i=1}^N
w_i
\left(
\mathds{1}\{\hat y_{T_{1:k},B}(x_i)\neq y_i\}
+
\eta \mathds{1}\{T_{1:k}(x_i)=\defer\}
\right).
\]
When $k=1$, this reduces exactly to the single defer-tree objective. For MDTs, it is more natural to regularize the number of splits rather than leaves because if there was a stage with a single defer leaf, regularizing by the number of leaves would increase the MDT objective, though nothing was changed about the predictions or the amount that the MDT deferred.

Finally, Theorem~\ref{thm:rule-list} shows that any MDT has an associated rule-list
representation with
\[
    |r_{\mathrm{MDT}}|
    =
    \sum_{j=1}^k (|T_j|-d_j),
\]
where $|T_j|$ is the number of leaves of stage $j$ and $d_j$ is the number of
deferral leaves in that stage. Assume the MDT has a non-zero deferral rate to the fallback model: this means that every stage defers. Since $d_j\geq 1$,
\[
    \sum_{j=1}^k (|T_j|-d_j)
    \leq
    \sum_{j=1}^k (|T_j|-1)
\]
Therefore, the MDT split penalty upper bounds the rule-list sparsity penalty of the
corresponding rule-list representation. 

%% file: sections/appendix/experimentsetup.tex
\subsection{Machines}
All experiments were performed on an institutional computing cluster.
Each experiment was executed on a single compute node equipped with an AMD EPYC 9554 processor (2.75 GHz), with 64 physical cores. Resources were restricted to 1 CPU core and 50 GB of RAM.

\subsection{Datasets}

We provide a description of all datasets used in this work and the binary classification task we chose.

For all datasets, we remove rows containing missing values and one-hot encode all categorical variables.

\paragraph{\textbf{Abalone \citep{openml_abalone_44956}} (\textit{4{,}177 samples})}
Predict whether an abalone is male based on its physical measurements. CC BY 4.0 License.

\paragraph{\textbf{Adult \citep{adult_2}} (\textit{48{,}842 samples})}
Predict whether an individual earns more than \$50{,}000 per year based on demographic and occupational attributes. CC BY 4.0 License.

\paragraph{\textbf{Aging \citep{national_poll_on_healthy_aging_(npha)_936, malani2017npha}} (\textit{714 samples})}
Predict whether an individual has visited at least two doctors. CC BY 4.0 License.

\paragraph{\textbf{Bank \citep{moro2014ada, bank_marketing_222}} (\textit{45{,}211 samples})}
Predict whether a customer subscribes after being contacted. CC BY 4.0 License.

\paragraph{\textbf{Bike \citep{FanaeeT2013EventLC, bike_sharing_275}} (\textit{17{,}379 samples})}
Predict whether bike rental demand exceeds the median. CC BY 4.0 License.

\paragraph{\textbf{California \citep{openml_california_44090}} (\textit{20{,}634 samples})}
Predict whether the provided binary target variable \texttt{price} is true. We believe this dataset is licensed under CC0, although the OpenML page does not explicitly specify a license. The dataset is derived from California census data.

\paragraph{\textbf{Churn \citep{erickson2025tabarenalivingbenchmarkmachine, marcoulides2005data}} (\textit{5{,}000 samples})}
Predict whether a customer will churn. License unspecified; publicly available on OpenML and included in a recent benchmark \citep{erickson2025tabarenalivingbenchmarkmachine}.

\paragraph{\textbf{Compas \citep{bao2021compaslicated}} (\textit{4{,}966 samples})}
Predict whether a defendant will recidivate within two years. Derived from anonymized public criminal records; available through the GitHub repository accompanying \citet{LarsonMaKiAn16}, which does not specify a license (\url{https://github.com/propublica/compas-analysis}).

\paragraph{\textbf{Coupon \citep{in-vehicle_coupon_recommendation_603}} (\textit{108 samples})}
Predict whether an individual accepts a recommended coupon. CC BY 4.0 License.

\paragraph{\textbf{Credit \citep{default_of_credit_card_clients_350, Yeh2009TheCO}} (\textit{30{,}000 samples})}
Predict whether a client will default on their credit card payment in the following month. CC BY 4.0 License.

\paragraph{\textbf{Diamonds \citep{openml_diamonds_42225}} (\textit{53{,}940 samples})}
Predict whether a diamond belongs to cut category 2 based on its physical attributes. License unspecified on OpenML, although its linked GitLab mirror indicates a CC0 1.0 Universal license. The dataset is included in a standard benchmark \citep{grinsztajn2022tree}.

\paragraph{\textbf{Droid \citep{naticusdroid_(android_permissions)_722, Mathur2021PosterNA}} (\textit{29{,}332 samples})}
Predict whether an Android application is malicious. CC BY 4.0 License. 

\paragraph{\textbf{Heloc \citep{fico2018heloc}} (\textit{2{,}502 samples})}
Predict whether an individual is high- or low-risk for a home equity line of credit. The HELOC dataset is subject to a license permitting research use, originally documented at \url{https://community.fico.com/s/explainable-machine-learning-challenge?tabset-158d9=2}, although this page is no longer accessible.

\paragraph{\textbf{Jasmine \citep{openml_jasmine_41143, automlchallenges}} (\textit{2{,}984 samples})}
Binary classification using the provided target column. License unspecified on OpenML, although its linked GitLab mirror indicates a CC0 1.0 Universal license. The dataset was also featured in an open machine learning challenge \citep{automlchallenges}.

\paragraph{\textbf{Madeline \citep{openml_madeline_41144}} (\textit{3{,}140 samples})}
Binary classification using the provided target column. License unspecified on OpenML, although its linked GitLab mirror indicates a CC0 1.0 Universal license. The dataset was included in an AutoML Challenge \citep{automlchallenges}.

\paragraph{\textbf{Magic \citep{magic_gamma_telescope_159}} (\textit{19{,}020 samples})}
Predict whether a Cherenkov telescope image corresponds to a gamma ray or background noise. CC BY 4.0 License.

\paragraph{\textbf{Monk2 \citep{Thrun1991TheMP, monk's_problems_70}} (\textit{601 samples})}
Predict the logical rule where the label is 1 if exactly two of six attributes take value 1. CC BY 4.0 License.

\paragraph{\textbf{Phishing \citep{phishing_websites_327, Mohammad2012AnAO}} (\textit{11{,}055 samples})}
Predict whether a website is phishing or legitimate. CC BY 4.0 License.

\paragraph{\textbf{Pol \citep{openml_pol_44082}} (\textit{10{,}082 samples})}Binary classification using the provided target column. License unspecified on OpenML, although its linked GitLab mirror indicates a CC0 1.0 Universal license. The dataset is included in a standard benchmark \citep{grinsztajn2022tree}.

\paragraph{\textbf{Rl \citep{openml_rl_43949}} (\textit{4{,}970 samples})}Binary classification using the provided target column. License unspecified on OpenML, although its linked GitLab mirror indicates a CC0 1.0 Universal license. The dataset is included in a standard benchmark \citep{grinsztajn2022tree}.

\paragraph{\textbf{Shopping \citep{online_shoppers_purchasing_intention_dataset_468, Sakar2018RealtimePO}} (\textit{12{,}330 samples})}
Predict whether an online shopping session ends in a purchase. CC BY 4.0 License.

\paragraph{\textbf{Spambase \citep{spambase_94}} (\textit{4{,}601 samples})}
Predict whether an email is spam. CC BY 4.0 License.

\paragraph{\textbf{Student \citep{student_performance_320, Cortez2008UsingDM}} (\textit{649 samples})}
Predict whether a student passes a course (final grade $\geq 10$). CC BY 4.0 License.

\paragraph{\textbf{Tic-Tac-Toe \citep{tic-tac-toe_endgame_101}} (\textit{958 samples})}
Predict whether a player has won given a terminal board configuration. CC BY 4.0 License.

\paragraph{\textbf{Wine \citep{openml_wine_47041}} (\textit{6{,}497 samples})}
Predict whether wine quality is at least 7.  CC BY 4.0 License.

\subsection{Hyperparameter Selection}

All cross-validation was done with 3 folds. Any results reported are averaged across 5 train/test splits.

\subsubsection{Black Boxes:}

\paragraph{Random Forest.}
Random Forest hyperparameters were tuned with Optuna using a 12-hour timeout per train/test split:
\[
\begin{aligned}
n_{\text{estimators}} &\in \{100,200,\ldots,1200\},\\
\text{max\_depth} &\in [2,32],\\
\text{min\_samples\_split} &\in [2,40],\\
\text{min\_samples\_leaf} &\in [1,20],\\
\text{max\_features} &\in \{\sqrt{p}, \log_2(p), \alpha p:\alpha\in[0.1,1.0]\},\\
\text{bootstrap} &\in \{\text{True},\text{False}\},\\
\text{class\_weight} &\in \{\text{None},\text{balanced},\text{balanced\_subsample}\}.
\end{aligned}
\]

\paragraph{XGBoost.}
XGBoost hyperparameters were tuned with Optuna using a 12-hour timeout per train/test split:
\[
\begin{aligned}
n_{\text{estimators}} &\in \{100,200,\ldots,1200\},\\
\text{max\_depth} &\in [2,12],\\
\text{learning\_rate} &\in [10^{-3},3\cdot 10^{-1}] \text{ on a log scale},\\
\text{min\_child\_weight} &\in [10^{-2},20] \text{ on a log scale},\\
\text{subsample} &\in [0.5,1.0],\\
\text{colsample\_bytree} &\in [0.5,1.0],\\
\gamma &\in [10^{-8},10] \text{ on a log scale},\\
\alpha_{\text{reg}} &\in [10^{-8},10] \text{ on a log scale},\\
\lambda_{\text{reg}} &\in [10^{-8},10] \text{ on a log scale}.
\end{aligned}
\]

\paragraph{AdaBoost.}
AdaBoost \citep{adaboost} hyperparameters were tuned with Optuna using a 12-hour timeout per train/test split. We tuned th AdaBoost ensemble using the following search space:
\[
\begin{aligned}
n_{\text{estimators}} &\in \{50,100,\ldots,2000\},\\
\text{learning\_rate} &\in [10^{-3},2] \text{ on a log scale},\\
\text{base\_max\_depth} &\in [1,8],\\
\text{base\_min\_samples\_split} &\in [2,50],\\
\text{base\_min\_samples\_leaf} &\in [1,50],\\
\text{base\_max\_features} &\in \{\text{None},\sqrt{p},\log_2(p)\},\\
\text{base\_criterion} &\in \{\text{gini},\text{entropy},\text{log\_loss}\},\\
\end{aligned}
\]
Hyperparameters were selected to maximize mean cross-validation accuracy. The best configuration was retrained on the full training split and evaluated on the held-out test split.

\subsubsection{Additive Models with Pairwise Interactions}

\paragraph{Explainable Boosting Machine (EBM).}

EBM hyperparameters were tuned with Optuna using a 12-hour timeout per train/test split:
\[
\begin{aligned}
\text{interactions} &\in [0,50],\\
\text{learning\_rate} &\in [10^{-3},5\cdot 10^{-2}] \text{ on a log scale},\\
\text{max\_rounds} &\in \{200,400,\ldots,5000\},\\
\text{max\_bins} &\in \{64,128,256,512,1024\},\\
\text{max\_interaction\_bins} &\in \{16,32,64,128\},\\
\text{min\_samples\_leaf} &\in [2,50],\\
\text{max\_leaves} &\in [2,5],\\
\text{outer\_bags} &\in \{4,8,12,16\},\\
\text{validation\_size} &\in \{0.10,0.15,0.20\},\\
\text{early\_stopping\_rounds} &\in \{50,100,200\},\\
\text{early\_stopping\_tolerance} &\in [10^{-6},10^{-3}] \text{ on a log scale}.
\end{aligned}
\]

\subsubsection{Sparse Ensembles}

\paragraph{FIGS.}

FIGS hyperparameters were tuned with Optuna using a 12-hour timeout per train/test split. The following search space was used:
\[
\begin{aligned}
\text{max\_rules} &\in [4,64],\\
\text{max\_trees} &\in [2,20],\\
\text{max\_depth} &\in [2,12],\\
\text{max\_features} &\in \{\text{None}, \sqrt{p}, \log_2(p)\},\\
\text{min\_impurity\_decrease} &\in [10^{-8},10^{-4}] \text{ on a log scale}.
\end{aligned}
\]
The interpretability constraints $\text{max\_rules} \le 64$, $\text{max\_trees} \le 20$, and $\text{max\_depth} \le 12$ were always enforced.

\subsubsection{Interpretable Models}

\paragraph{SingleTree.}

For a single decision tree algorithm, we used our DeferTree algorithm with $\eta=10^9$ to ensure no deferral (we also verified this on all results reported). This corresponds to a version of the near-optimal decision tree algorithm LicketySPLIT \citep{babbar2025near}.

\[
\begin{aligned}
\lambda &\in \{10^{-5},5\cdot 10^{-5},10^{-4},5\cdot 10^{-4},10^{-3},5\cdot 10^{-3},10^{-2}\},\\
\end{aligned}
\]
We fixed a depth budget of $10$, using $\lambda$ to grow a sparser tree. This method requires a binarized dataset. We binarized the dataset with ThresholdGuessing \citep{mctavish2022fast}, using $n_{\text{estimators}} = 150$ and $\text{max\_depth} = 2$.

\subsubsection{Hybrid Models}

\paragraph{HyRS.}

For HyRS, we tuned
\[
\begin{aligned}
\text{nrules} &\in \{5000\},\\
\text{supp} &\in \{5,20\},\\
\text{maxlen} &\in \{2,3\},\\
\alpha &\in \{10^{-4},10^{-3},10^{-2}\},\\
\beta &\in \{0.01,0.05,0.10,0.15,0.25,0.50\},\\
\text{niter} &\in \{5000\}.
\end{aligned}
\]
Non-binary numeric features were binarized using $10$ quantile levels. We used their provided Random Forest rule mining approach.

\paragraph{FLMM.}

For FLMM, we tuned
\[
\begin{aligned}
r_{\max} &\in \{1,2,3\},\\
\ell &\in \{0.01,0.02,0.05,0.1,0.2\},\\
s &\in \{0.5,0.6,0.7,0.8,0.9,1.0\},\\
\text{min\_support\_frac} &\in \{0.005,0.01,0.02\}.
\end{aligned}
\]

We could not find an exact description of the rule mining done in \cite{frost2024partially}, except that $r_{\max}$ was tuned in $\{1,2,3\}$. Therefore, we use the following procedure.

Non-binary numeric features were binarized using $10$ quantile levels. Then, candidate rules were exhaustively enumerated as all conjunctions of between $1$ and $r_{\max}$ binary indicators. Any rule whose support was below
\[
\max\!\left(1,\left\lceil \text{min\_support\_frac}\cdot n \right\rceil\right)
\]
was discarded. For each remaining conjunction, the predicted label was chosen as the majority class among covered samples. Duplicate rules inducing identical coverage and prediction behavior were removed.

\paragraph{HybridCORELS-POST.}

For the HybridCORELS-POST variant, we tuned
\[
\begin{aligned}
c &\in \{0.001,0.01,0.1\},\\
\text{min\_coverage} &\in \{0.1,0.25,0.5,0.75,0.9\},\\
\text{max\_card} &\in \{1,2\},\\
\text{min\_support} &\in \{1,5\},\\
\text{n\_rules} &\in \{300\},\\
\end{aligned}
\]

\paragraph{HybridCORELS-PRE.}

For the HybridCORELS-PRE variant, we tuned
\[
\begin{aligned}
c &\in \{0.001,0.01,0.1\},\\
\text{min\_coverage} &\in \{0.1,0.25,0.5,0.75,0.9\},\\
\text{max\_card} &\in \{1,2\},\\
\text{min\_support} &\in \{5\},\\
\text{n\_rules} &\in \{300\},\\
\alpha &\in \{0,5,10\}.
\end{aligned}
\]

For both optimal methods (HybridCORELS PRE and POST), we set a time limit of 10 minutes per hyperparameter to fit the rule list. This does not include training the black box or performing rule mining.

For the HybridCORELS variants, we first binarized the dataset with ThresholdGuessing \citep{mctavish2022fast}, using $n_{\text{estimators}} = 150$ and $\text{max\_depth} = 2$. We then used their provided rule mining method on these binary features, which uses FPGrowth. Because the HybridCORELS implementation exposes prediction through a \texttt{predict} method that accepts only the already rule-mined dataset, the fallback XGBoost model was required to operate on this same mined feature space. As such, we retuned the XGBoost model for each hyperparameter combination for HybridCORELS, because the HybridCORELS hyperparameters can change the rule mining and thus what XGBoost is trained on. Aside from that, the method of choosing hyperparameters is discussed in \autoref{appendix:cv}.

\paragraph{DeferTree.}
For our DeferTree, we tuned
\[
\begin{aligned}
\lambda &\in \{10^{-5},5\cdot 10^{-5},10^{-4},5\cdot 10^{-4},10^{-3},5\cdot 10^{-3},10^{-2}\},\\
\eta &\in \{0.001,0.005,0.01,0.05,0.1,0.3,0.6,0.8,1.0,10^9\}.
\end{aligned}
\]

This method requires a binarized dataset. We binarized the dataset with ThresholdGuessing \citep{mctavish2022fast}, using $n_{\text{estimators}} = 150$ and $\text{max\_depth} = 2$.

\paragraph{MDT.}

For our MDT, we tuned
\[
\begin{aligned}
\lambda &\in \{10^{-5},5\cdot 10^{-5},10^{-4},5\cdot 10^{-4},10^{-3},5\cdot 10^{-3},10^{-2}\},\\
\eta &\in \{0.001,0.005,0.01,0.05,0.1,0.3,0.6,0.8,1.0\},\\
\text{rescale\_tau} &\in \{0,1\},\\
c &\in \{0,0.2,0.5,1.0\},\\
\gamma &\in \{0,1,2,4,8\},\\
\end{aligned}
\]
The fixed settings were depth budget $10$, maximum stages $6$, maximum total stage leaves $500$, maximum leaves in an uncompressed single tree $10^6$, and maximum leaves per tree $129$. In practice, we do not hit depth 10 or anywhere near these limits (shown in Figure \ref{fig:sparsity_cdf_eps0025}); these are upper bounds that would only prune very undesired hyperparameters. In our algorithm, the calls to fit a DeferTree require a binarized dataset. We binarized the dataset with ThresholdGuessing \citep{mctavish2022fast}, using $n_{\text{estimators}} = 150$ and $\text{max\_depth} = 2$.

\subsection{Cross Validation Procedure}\label{appendix:cv}

For hybrid interpretable models with fallback black boxes, hyperparameter selection was performed separately for each dataset train/test split. For a given train/test split, we constructed a table whose rows corresponded to hybrid-model hyperparameter combinations and whose columns stored the average validation accuracy, average validation deferral rate, test accuracy, and test deferral rate.

To populate this table, we first created three subtrain/validation partitions from the training set. For each subtrain partition, we independently tuned an XGBoost model over the ranges described earlier using an inner 3-fold cross-validation procedure restricted to that subtrain partition. After selecting the best XGBoost hyperparameters for that subtrain partition, we fixed that XGBoost model and evaluated every hyperparameter combination in the hybrid interpretable grid using the associated validation split. Repeating this process across the three subtrain/validation partitions yielded an average validation accuracy and average validation deferral rate for each hybrid hyperparameter configuration.

Next, we tuned XGBoost once more on the entire training set using the same XGBoost search procedure. The resulting XGBoost model was then used as the fallback model for every hybrid hyperparameter configuration trained on the full training set, producing a corresponding test accuracy and test deferral rate.

Given this matrix for a dataset train/test split, if we wished to select the best model subject to at most $x\%$ test deferral, we first filtered out all hyperparameter configurations whose test deferral exceeded the threshold. Among the remaining configurations, we selected the one with the highest average validation accuracy. We then reported the corresponding test accuracy and test deferral for that configuration. We note that the former (the test accuracy) had not been looked at before then, and was not used to choose any hyperparameters. This entire procedure was repeated independently across the five train/test splits for each dataset.

%% file: sections/appendix/misc.tex
\begin{table}[h]
\centering
\small
\begin{tabular}{lcc}
\toprule
Dataset 
& Defer Tree Non-Deferred Error 
& Defer Tree Deferred Error \\
\midrule
Bike 
& 1.57\% $\pm$ 0.42\% 
& 14.88\% $\pm$ 2.35\% \\

Churn 
& 3.04\% $\pm$ 0.33\% 
& 14.16\% $\pm$ 4.21\% \\

Compas 
& 31.59\% $\pm$ 2.11\% 
& 36.24\% $\pm$ 1.08\% \\

Heloc 
& 25.59\% $\pm$ 2.61\% 
& 36.43\% $\pm$ 6.78\% \\

Jasmine 
& 13.97\% $\pm$ 2.46\% 
& 32.39\% $\pm$ 4.45\% \\

Phishing 
& 0.68\% $\pm$ 0.24\% 
& 11.84\% $\pm$ 1.54\% \\

Pol 
& 0.07\% $\pm$ 0.08\% 
& 6.56\% $\pm$ 0.81\% \\

RL 
& 16.57\% $\pm$ 2.30\% 
& 24.34\% $\pm$ 2.99\% \\

Wine 
& 4.67\% $\pm$ 0.88\% 
& 19.49\% $\pm$ 1.51\% \\

\bottomrule
\end{tabular}
\caption{Test error comparison between the DeferTree non-deferred and deferred regions. The deferred regions consistently exhibit substantially higher test error, motivating that our recursive approach with adaptive weights could reduce the amount we rely on the black box with little change in accuracy. All values are mean $\pm$ standard deviation across 5 splits.}
\label{tab:defer_tree_error_split}
\end{table}

In \autoref{fig:train_deferral_comparison}, we compare the semi-supervised hyperparameter selection procedure to a variant that uses training deferral rate as a proxy for test deferral rate. We utilized this semi-supervised procedure for all methods, and here we compare this approach to using the train deferral rate for MDT+XGB. For each target deferral level $x$, we choose the hyperparameters with the highest average validation accuracy among those with deferral rate at most $x$, where deferral is measured on the test set for MDT+XGB and on the training set for MDT+XGB Train. After selecting hyperparameters separately for each of the five train/test splits, we plot the average test accuracy against the average test deferral rate of the selected models. Across datasets, the two curves are nearly indistinguishable: using training deferral is sometimes slightly better and sometimes slightly worse, but it usually selects models with essentially the same test accuracy--deferral tradeoff. The only clear degradation occurs on Coupon, which has only 108 total samples after removing missing values, so little weight should be placed on that outlier. Overall, these results suggest that training deferral rate is an acceptable practical proxy for test deferral rate when selecting MDT+XGB hyperparameters.

\clearpage
\thispagestyle{empty}

\begin{figure*}[p]
    \centering
    \vspace*{-1.2cm}

    \includegraphics[width=0.49\textwidth,height=0.145\textheight,keepaspectratio]{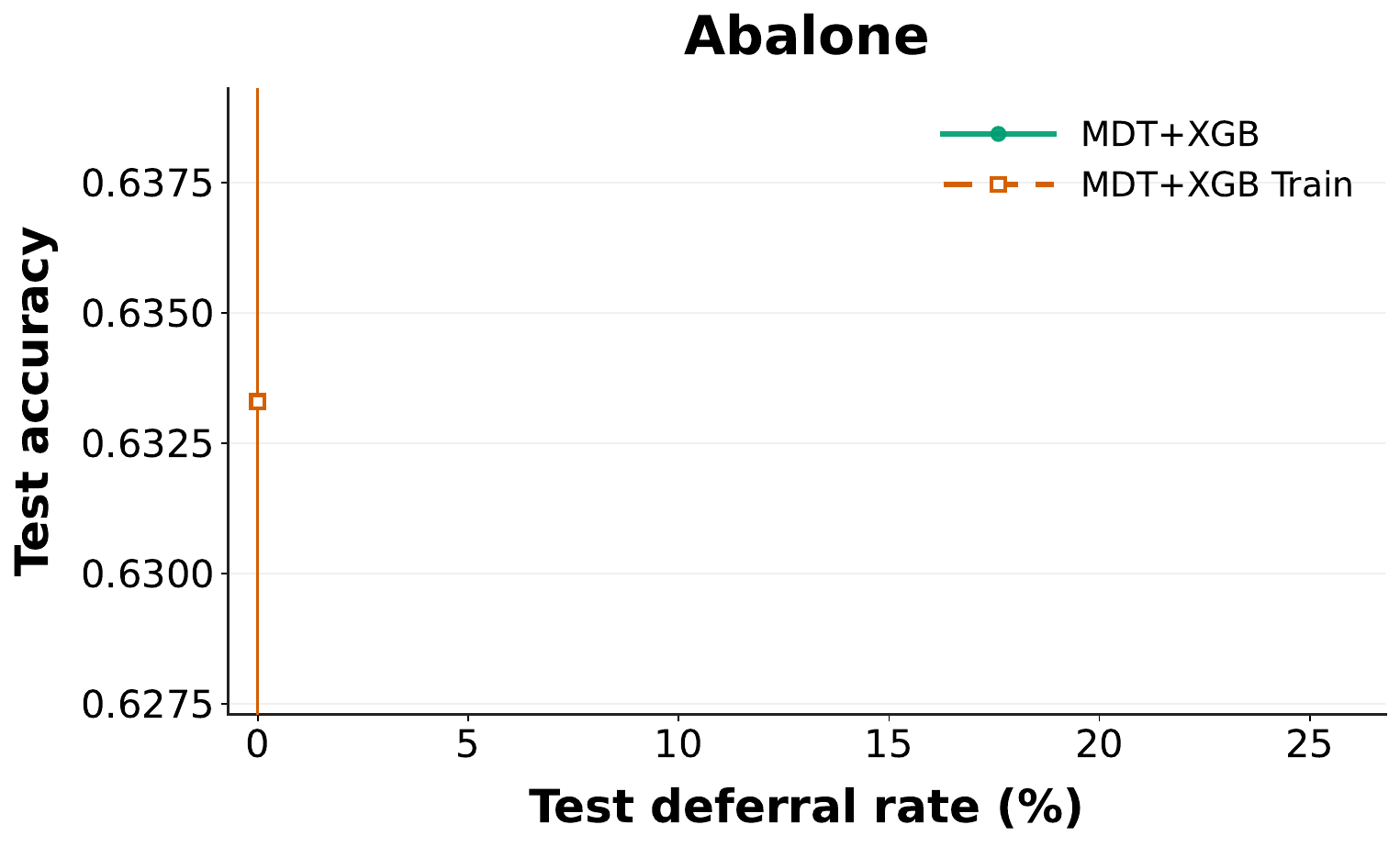}
    \hfill
    \includegraphics[width=0.49\textwidth,height=0.145\textheight,keepaspectratio]{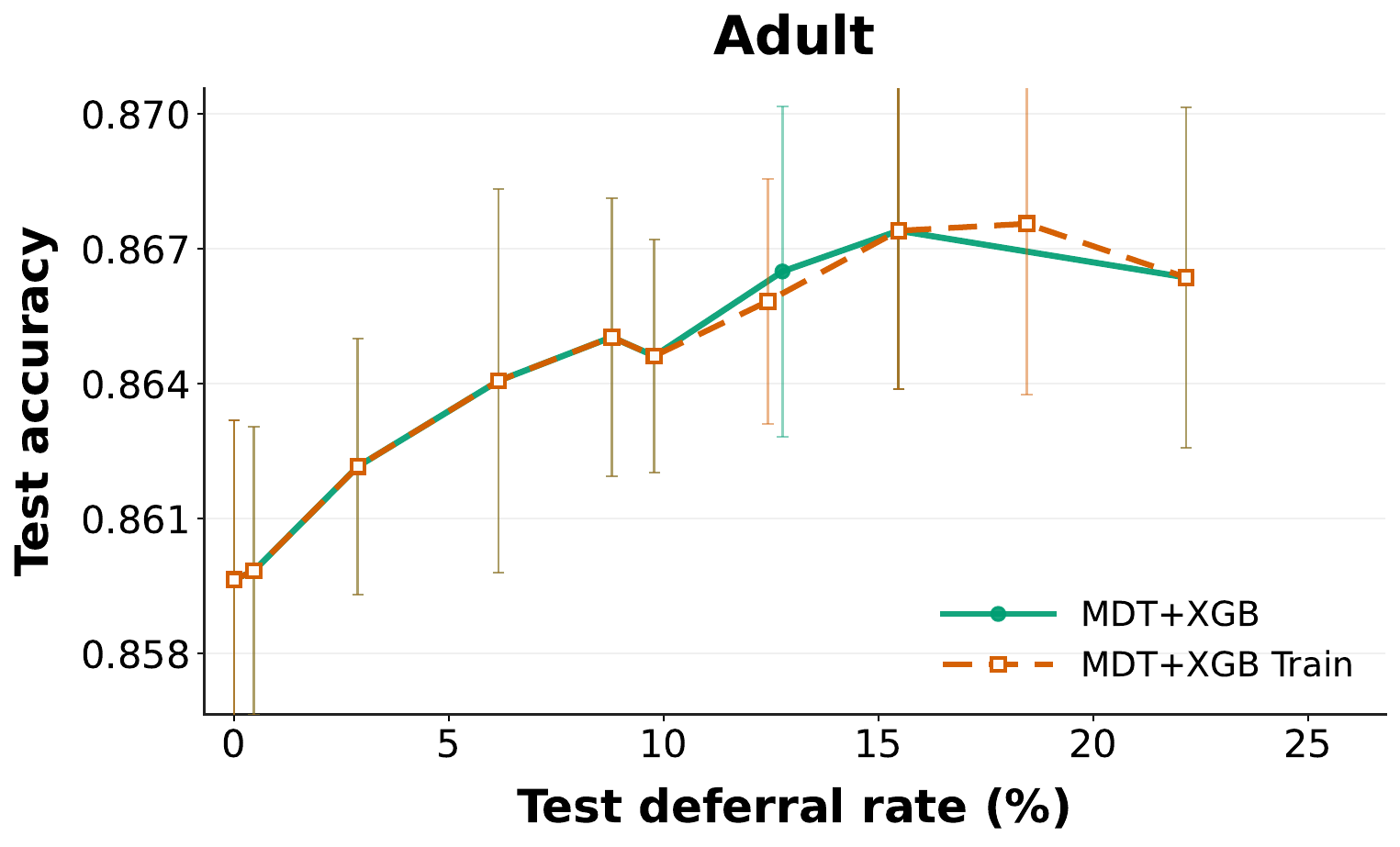}

    \vspace{0.25em}

    \includegraphics[width=0.49\textwidth,height=0.145\textheight,keepaspectratio]{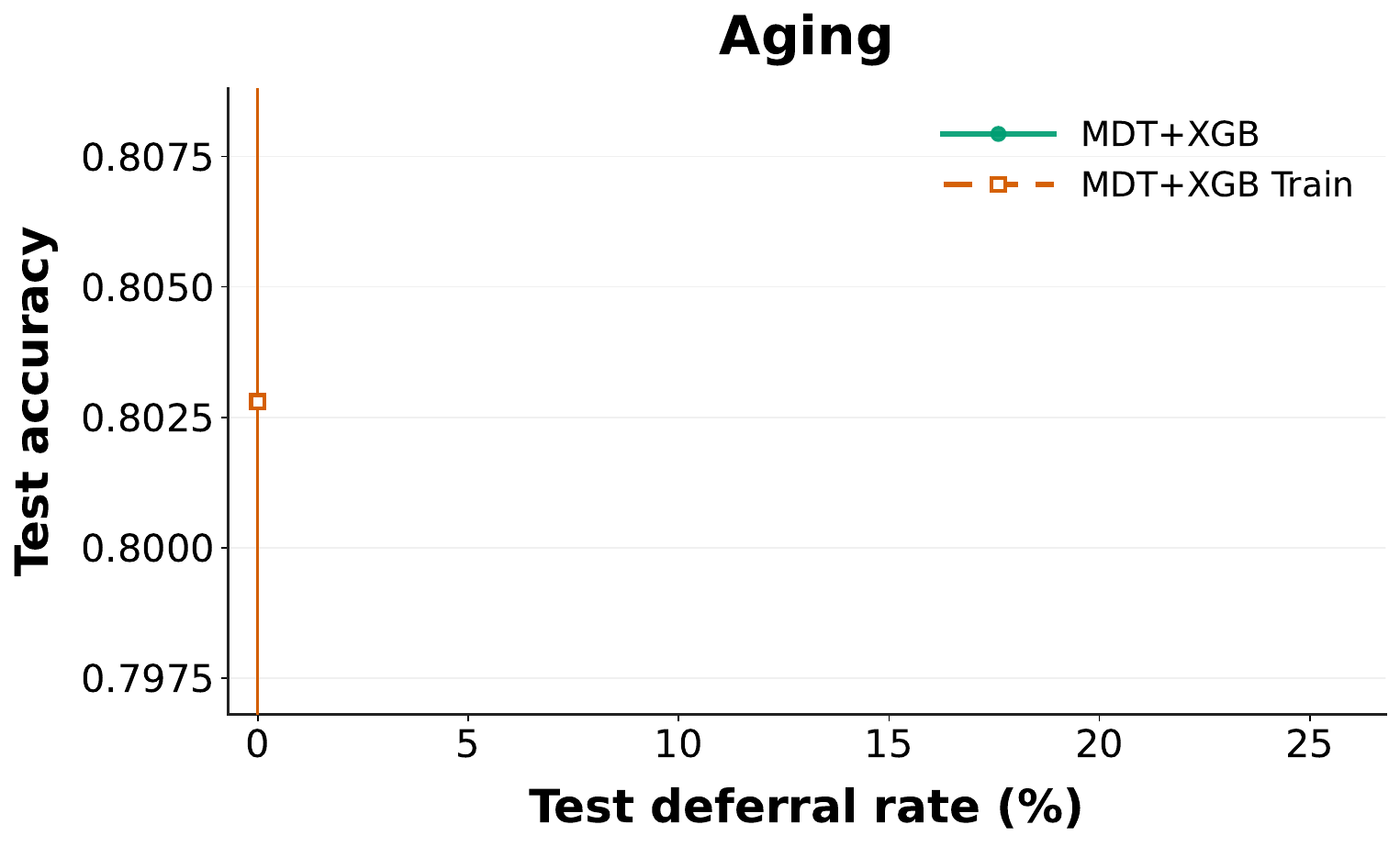}
    \hfill
    \includegraphics[width=0.49\textwidth,height=0.145\textheight,keepaspectratio]{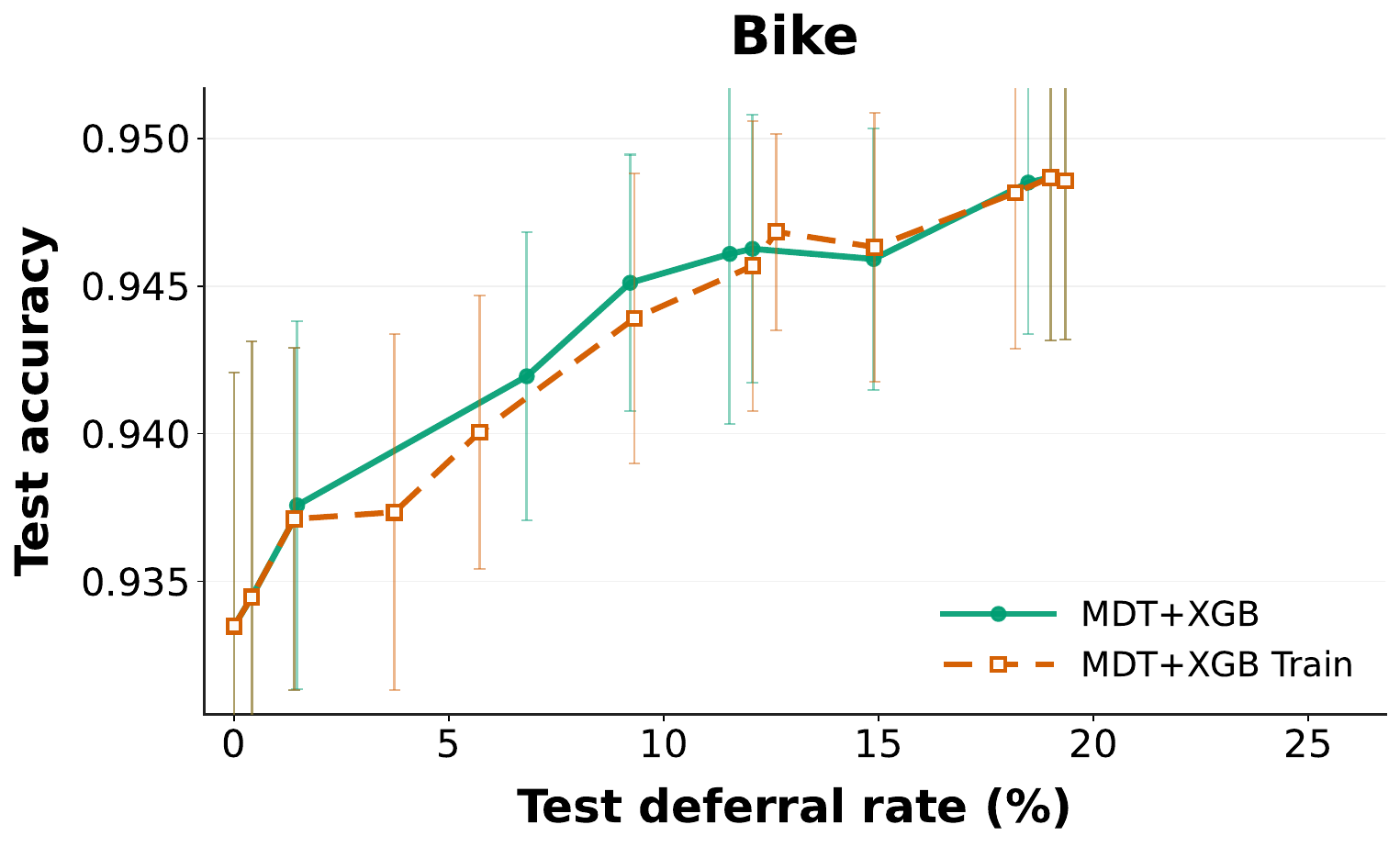}

    \vspace{0.25em}

    \includegraphics[width=0.49\textwidth,height=0.145\textheight,keepaspectratio]{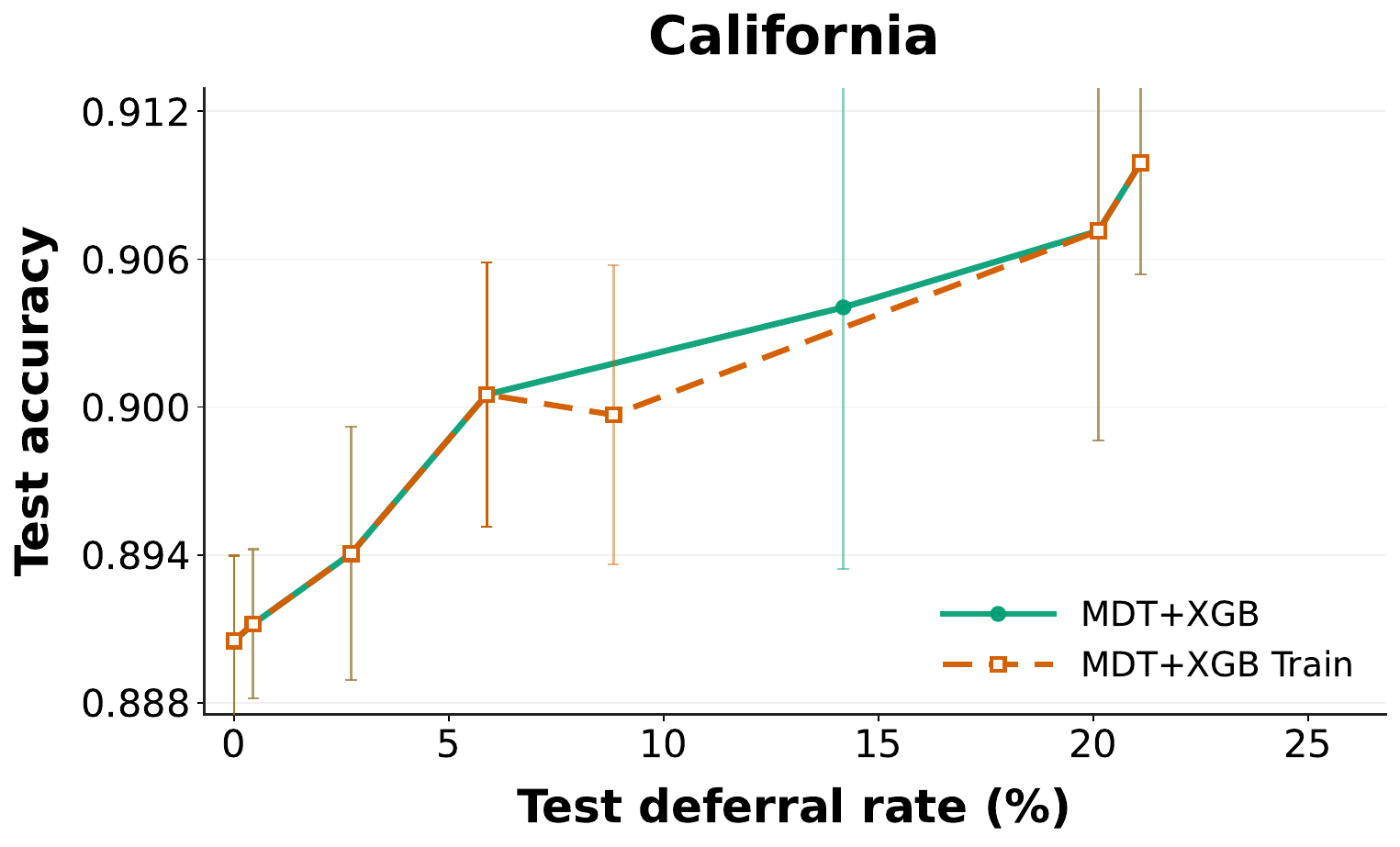}
    \hfill
    \includegraphics[width=0.49\textwidth,height=0.145\textheight,keepaspectratio]{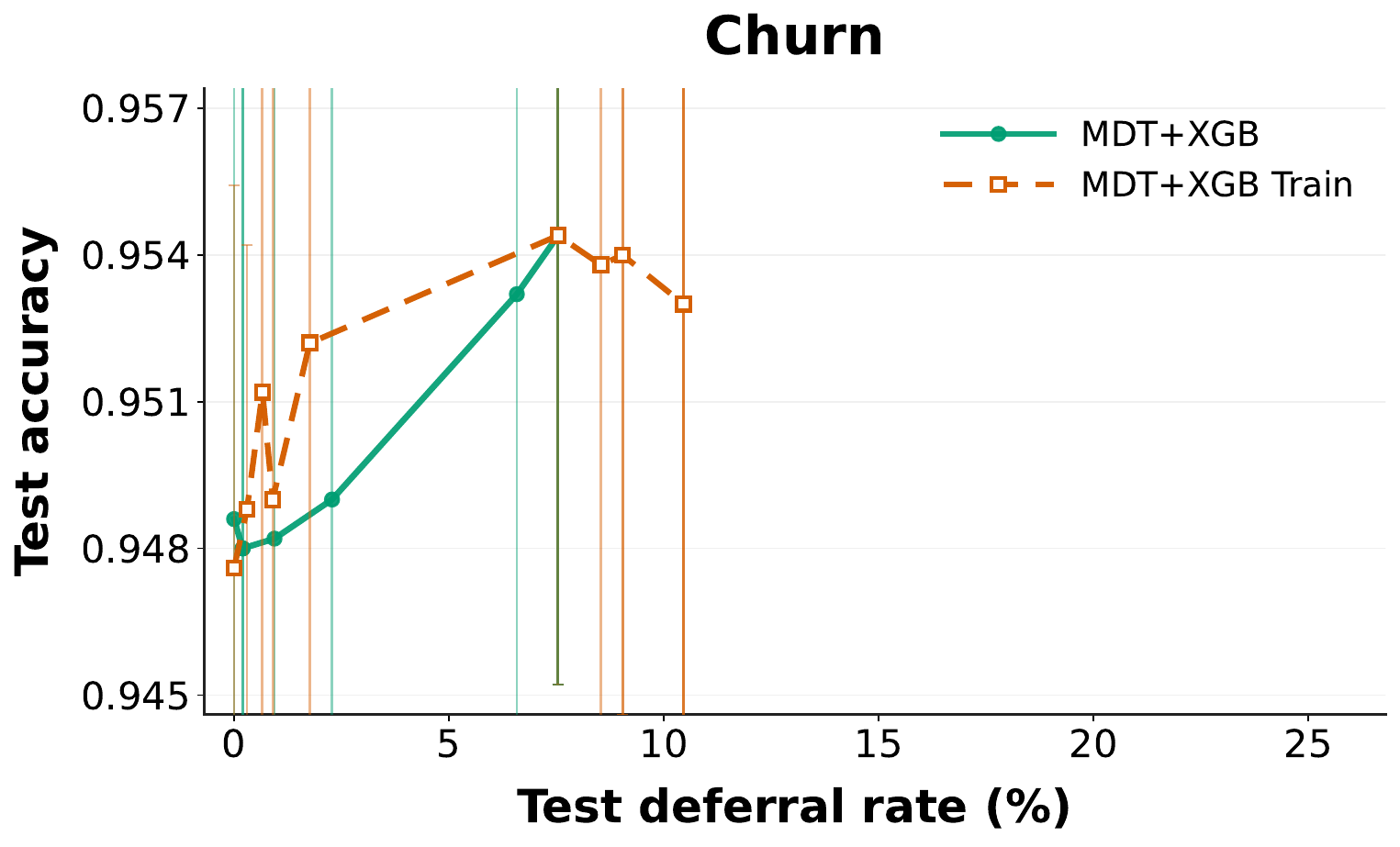}

    \vspace{0.25em}

    \includegraphics[width=0.49\textwidth,height=0.145\textheight,keepaspectratio]{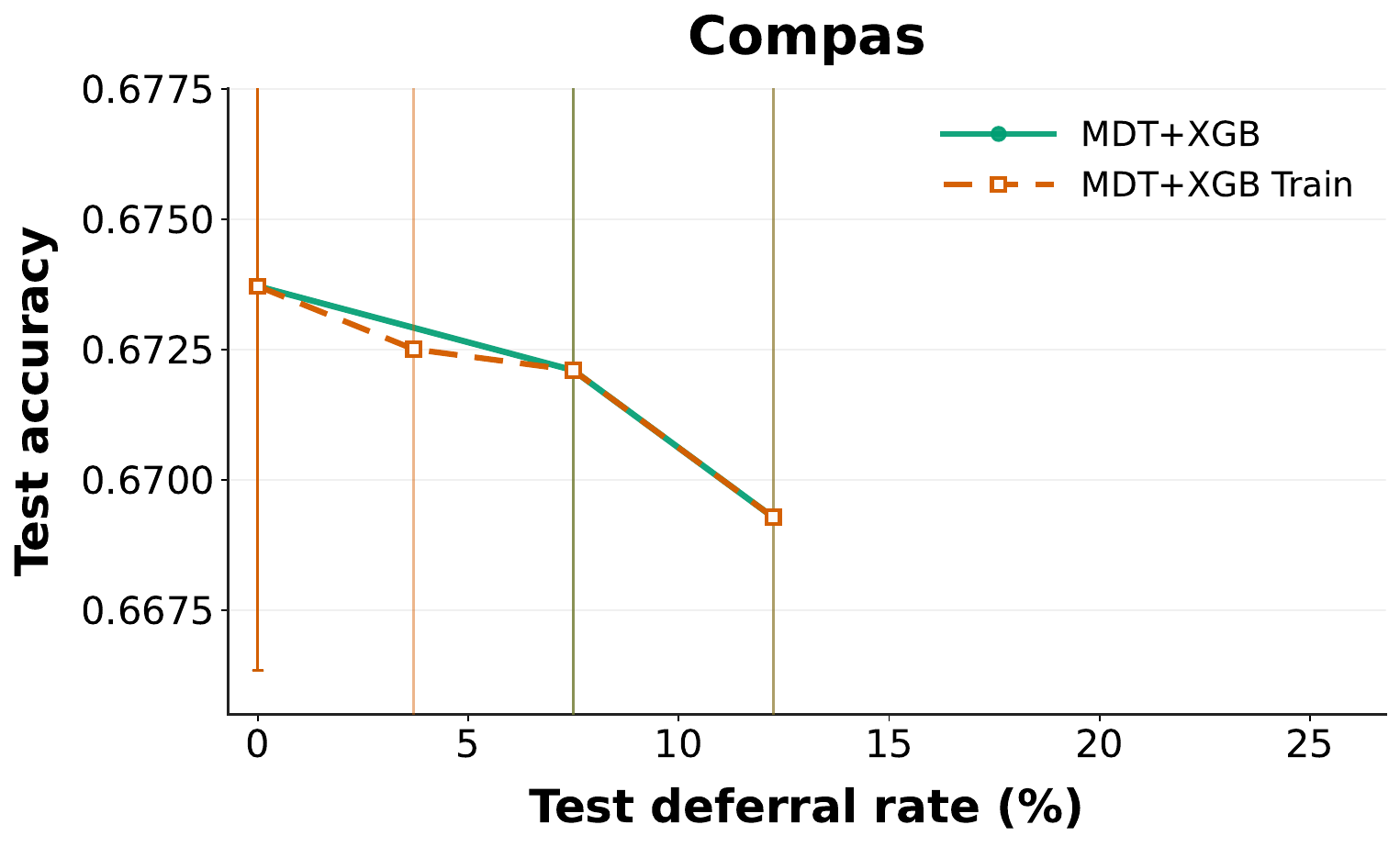}
    \hfill
    \includegraphics[width=0.49\textwidth,height=0.145\textheight,keepaspectratio]{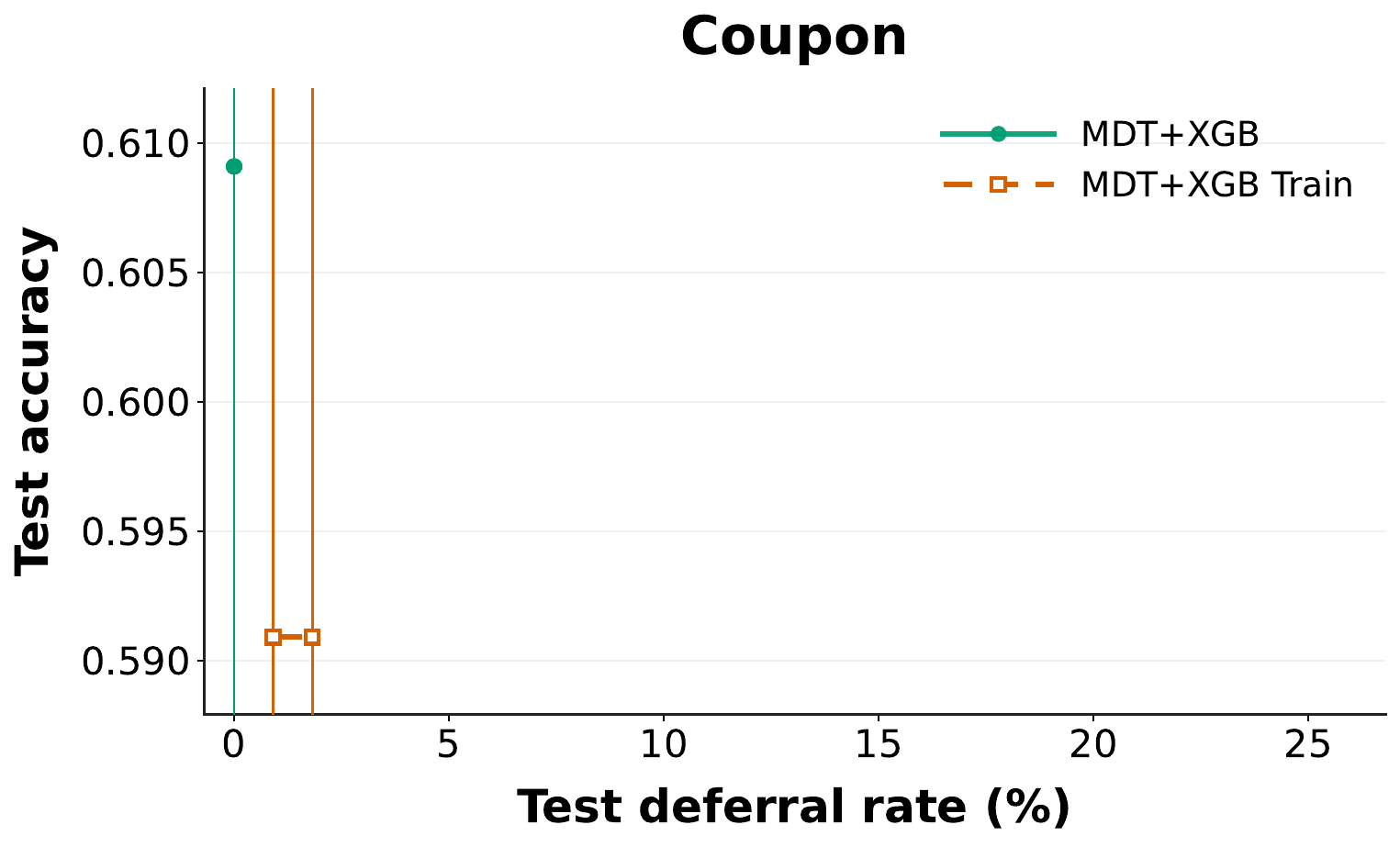}

    \vspace{0.25em}

    \includegraphics[width=0.49\textwidth,height=0.145\textheight,keepaspectratio]{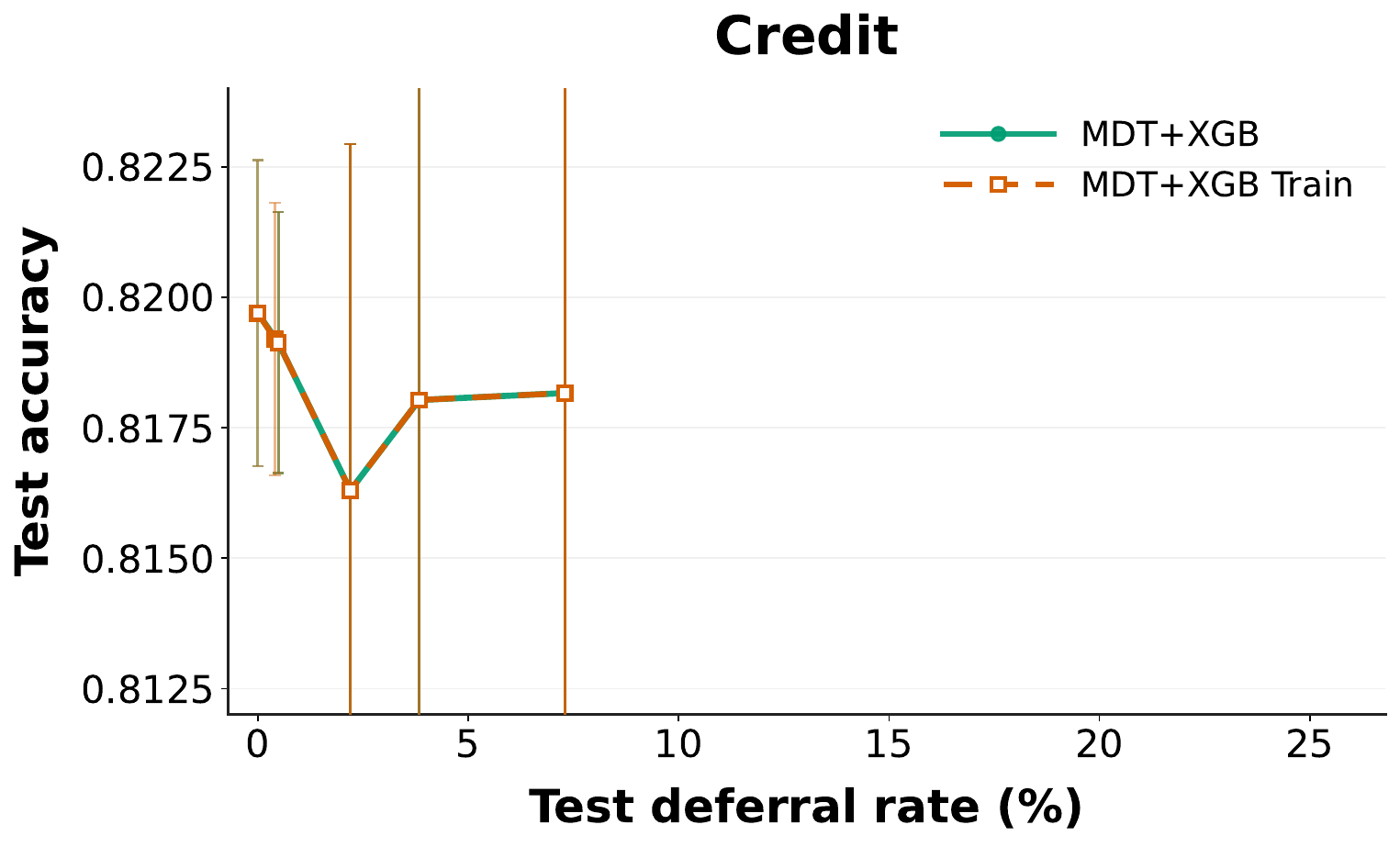}
    \hfill
    \includegraphics[width=0.49\textwidth,height=0.145\textheight,keepaspectratio]{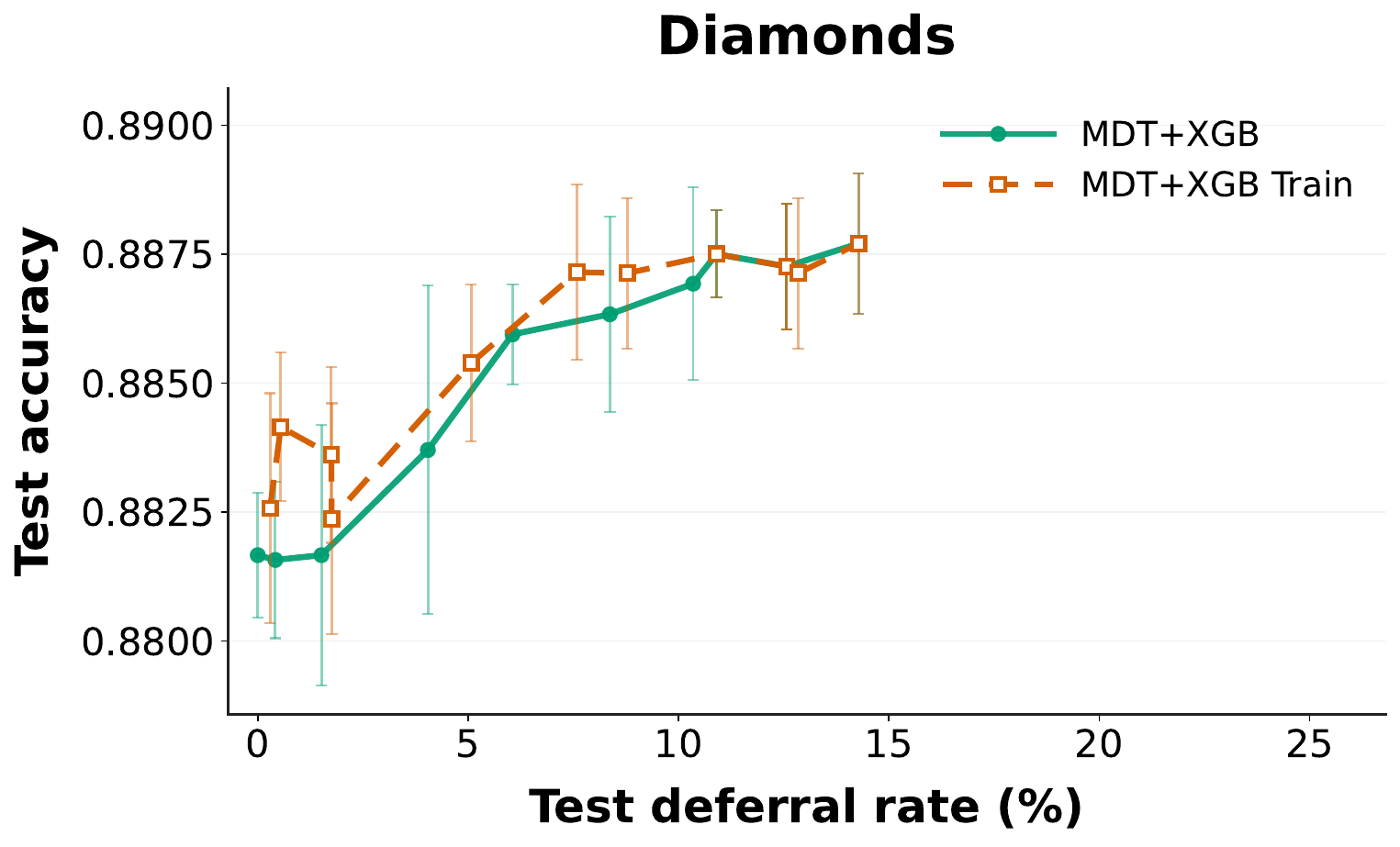}

    \vspace{0.25em}

    \includegraphics[width=0.49\textwidth,height=0.145\textheight,keepaspectratio]{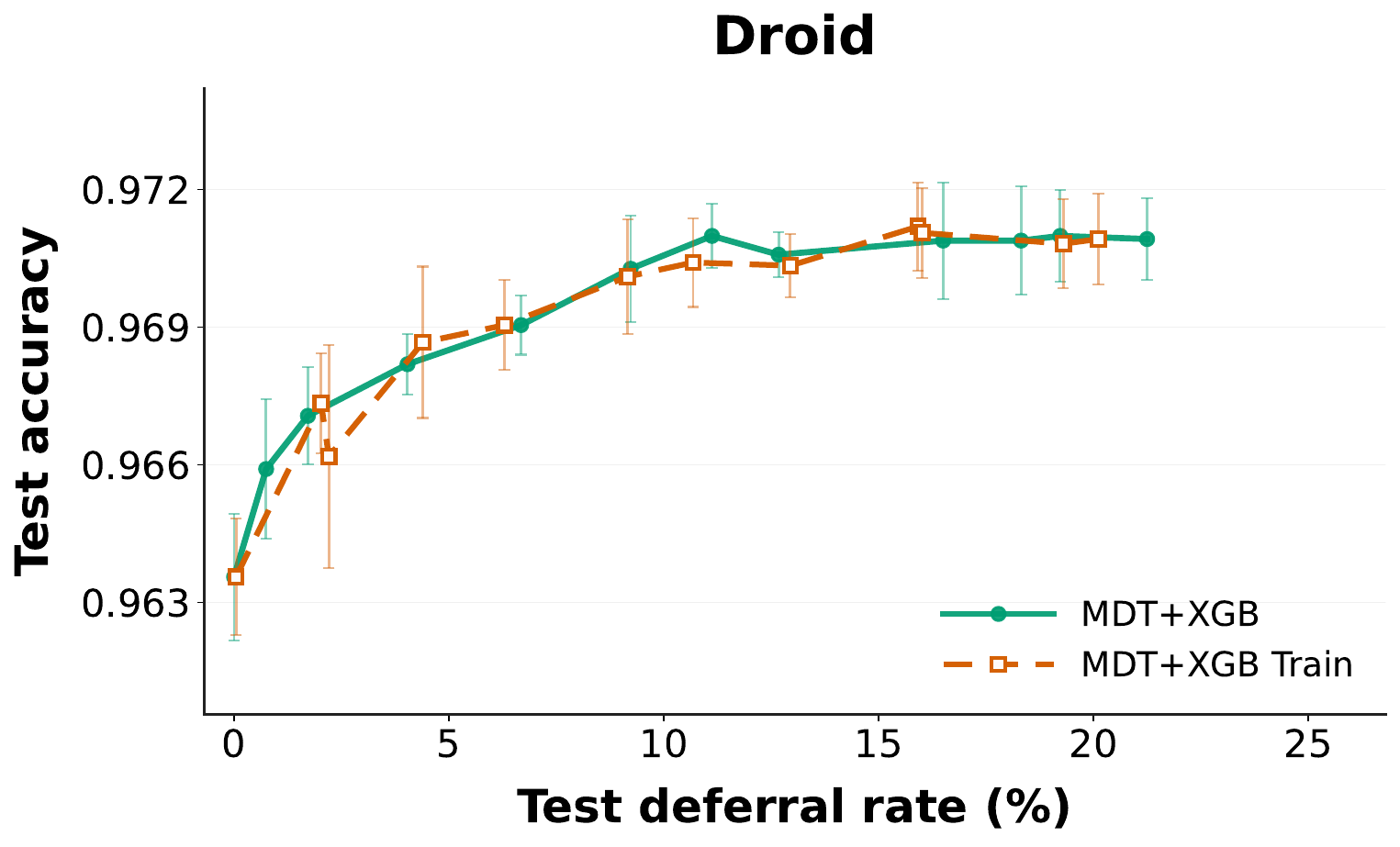}
    \hfill
    \includegraphics[width=0.49\textwidth,height=0.145\textheight,keepaspectratio]{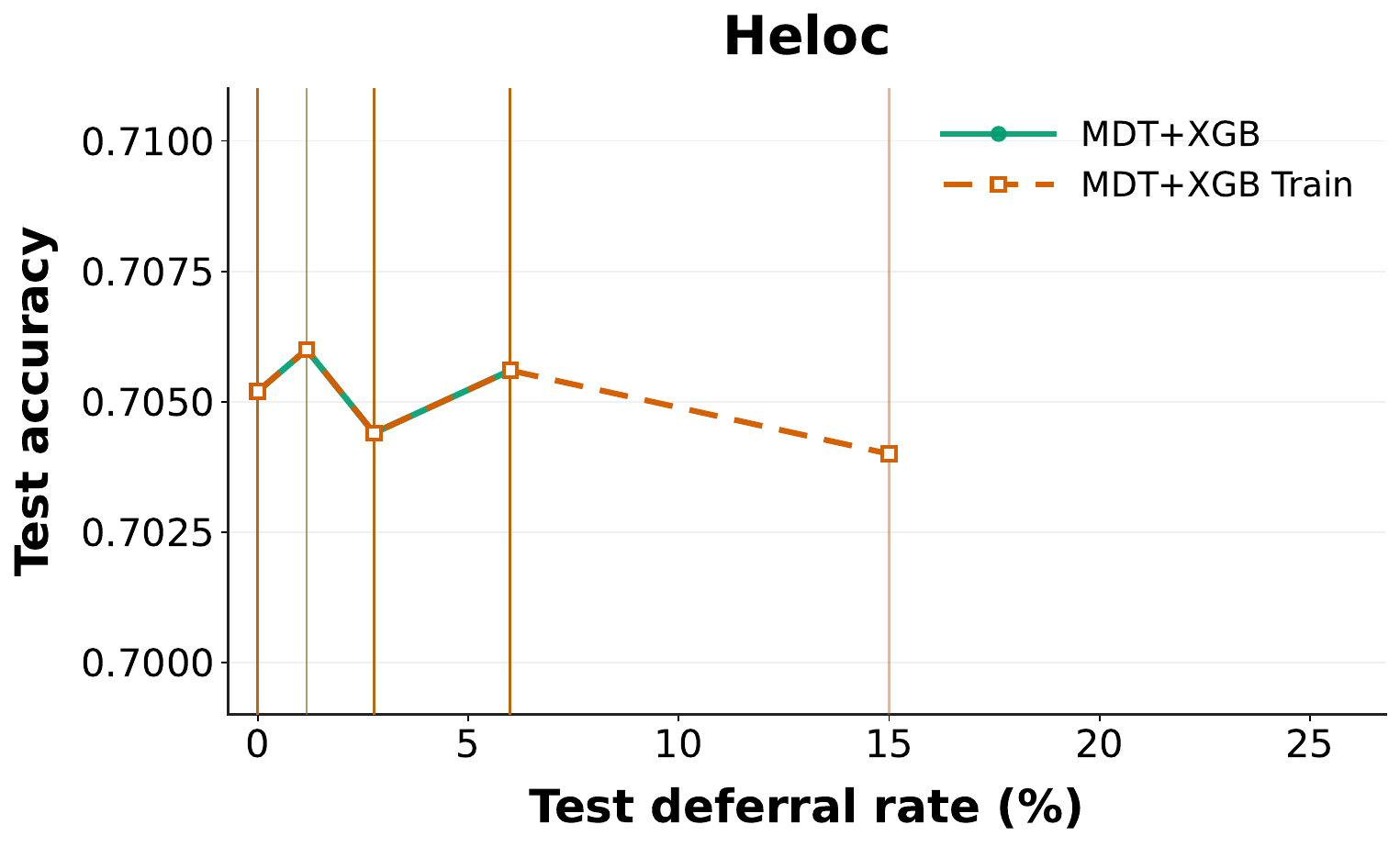}

    \vspace{-0.25em}

    \caption{
    Comparison between MDT+XGB hyperparameter selection using test deferral rate and selection using training deferral rate.
    }
    \label{fig:train_deferral_comparison}
\end{figure*}

\clearpage

\begin{table}[!t]
\centering
\small
\setlength{\tabcolsep}{4pt}
\renewcommand{\arraystretch}{0.92}
\resizebox{\textwidth}{!}{
\begin{tabular}{lccc}
\toprule
Dataset
& Difference in Train/Test Defer Rates
& Train Defer Rate
& Test Defer Rate\\
\midrule

Abalone
& $0.0000 \pm 0.0000$
& $0.0000 \pm 0.0000$
& $0.0000 \pm 0.0000$ \\

Adult
& $0.0052 \pm 0.0046$
& $0.2176 \pm 0.0307$
& $0.2216 \pm 0.0309$ \\

Aging
& $0.0000 \pm 0.0000$
& $0.0000 \pm 0.0000$
& $0.0000 \pm 0.0000$ \\

Bike
& $0.0055 \pm 0.0047$
& $0.1595 \pm 0.0399$
& $0.1539 \pm 0.0397$ \\

California
& $0.0105 \pm 0.0073$
& $0.2208 \pm 0.0142$
& $0.2110 \pm 0.0218$ \\

Churn
& $0.0037 \pm 0.0040$
& $0.0508 \pm 0.0556$
& $0.0532 \pm 0.0544$ \\

Compas
& $0.0084 \pm 0.0102$
& $0.1240 \pm 0.1148$
& $0.1225 \pm 0.1137$ \\

Coupon
& $0.0142 \pm 0.0317$
& $0.0233 \pm 0.0520$
& $0.0091 \pm 0.0203$ \\

Credit
& $0.0043 \pm 0.0029$
& $0.0731 \pm 0.0689$
& $0.0730 \pm 0.0721$ \\

Diamonds
& $0.0043 \pm 0.0029$
& $0.1658 \pm 0.0450$
& $0.1624 \pm 0.0479$ \\

Droid
& $0.0033 \pm 0.0021$
& $0.1841 \pm 0.0621$
& $0.1841 \pm 0.0648$ \\

Heloc
& $0.0105 \pm 0.0118$
& $0.0495 \pm 0.0555$
& $0.0600 \pm 0.0671$ \\

Jasmine
& $0.0067 \pm 0.0072$
& $0.0613 \pm 0.0729$
& $0.0593 \pm 0.0782$ \\

Madeline
& $0.0139 \pm 0.0105$
& $0.1150 \pm 0.0843$
& $0.1102 \pm 0.0811$ \\

Magic
& $0.0027 \pm 0.0023$
& $0.2299 \pm 0.0129$
& $0.2301 \pm 0.0128$ \\

Monk2
& $0.0000 \pm 0.0000$
& $0.0000 \pm 0.0000$
& $0.0000 \pm 0.0000$ \\

Phishing
& $0.0109 \pm 0.0061$
& $0.1807 \pm 0.0300$
& $0.1873 \pm 0.0376$ \\

Pol
& $0.0033 \pm 0.0036$
& $0.1474 \pm 0.0497$
& $0.1463 \pm 0.0524$ \\

Rl
& $0.0094 \pm 0.0079$
& $0.1520 \pm 0.0609$
& $0.1441 \pm 0.0617$ \\

Shopping
& $0.0057 \pm 0.0041$
& $0.2056 \pm 0.0277$
& $0.2015 \pm 0.0250$ \\

Spambase
& $0.0146 \pm 0.0124$
& $0.1717 \pm 0.0349$
& $0.1863 \pm 0.0301$ \\

Student
& $0.0084 \pm 0.0157$
& $0.0470 \pm 0.0880$
& $0.0554 \pm 0.1037$ \\

Tictactoe
& $0.0047 \pm 0.0065$
& $0.0350 \pm 0.0346$
& $0.0323 \pm 0.0358$ \\

Wine
& $0.0124 \pm 0.0102$
& $0.2101 \pm 0.0539$
& $0.1982 \pm 0.0465$ \\

\bottomrule
\end{tabular}
}
\caption{Mean and standard deviation across train/test splits of the absolute difference between train and test defer rates, along with the train and test  defer rates themselves, for MDT+XGB models selected under a 25\% test defer-rate cutoff.}
\label{tab:train_test_defer_rate_stability_25}
\end{table}

\begin{table}[t]
\centering
\small
\setlength{\tabcolsep}{2.8pt}
\renewcommand{\arraystretch}{0.88}
\resizebox{\textwidth}{!}{
\begin{tabular}{lccccc}
\toprule
Dataset 
& Split Compression 
& Leaf Compression
& MDT Deferred Accuracy 
& XGB Deferred Accuracy
& $\Delta$ Accuracy (MDT $-$ XGB) \\
\midrule

Adult
& $1.0214 \pm 0.0167$
& $1.0819 \pm 0.0556$
& $0.7049\ [0.6959, 0.7139]$
& $0.7067\ [0.6977, 0.7156]$
& $-0.0018\ [-0.0048, 0.0012]$ \\

Bike
& $1.0642 \pm 0.0277$
& $1.0875 \pm 0.0803$
& $0.8082\ [0.7933, 0.8228]$
& $0.8090\ [0.7940, 0.8236]$
& $-0.0007\ [-0.0060, 0.0045]$ \\

Compas
& $1.7564 \pm 0.1381$
& $1.5450 \pm 0.6991$
& $0.6349\ [0.5970, 0.6727]$
& $0.6332\ [0.5954, 0.6711]$
& $0.0016\ [-0.0132, 0.0164]$ \\

Diamonds
& $1.1116 \pm 0.1536$
& $1.4575 \pm 0.5455$
& $0.7221\ [0.7126, 0.7314]$
& $0.7252\ [0.7158, 0.7345]$
& $-0.0031\ [-0.0066, 0.0003]$ \\

Jasmine
& $1.8466 \pm 0.5427$
& $1.3607 \pm 0.8066$
& $0.6723\ [0.6045, 0.7401]$
& $0.6893\ [0.6215, 0.7571]$
& $-0.0169\ [-0.0734, 0.0395]$ \\

Madeline
& $1.0614 \pm 0.0752$
& $11.9562 \pm 24.2114$
& $0.6994\ [0.6503, 0.7486]$
& $0.7139\ [0.6647, 0.7601]$
& $-0.0145\ [-0.0491, 0.0202]$ \\

Phishing
& $1.8082 \pm 0.3270$
& $1.0475 \pm 0.0443$
& $0.8798\ [0.8658, 0.8938]$
& $0.8798\ [0.8658, 0.8938]$
& $0.0000\ [-0.0029, 0.0024]$ \\

Pol
& $1.0024 \pm 0.0221$
& $1.0434 \pm 0.0628$
& $0.9186\ [0.9044, 0.9322]$
& $0.9173\ [0.9031, 0.9308]$
& $0.0014\ [-0.0047, 0.0075]$ \\

Rl
& $1.2228 \pm 0.0952$
& $1.3269 \pm 0.2278$
& $0.7458\ [0.7137, 0.7779]$
& $0.7416\ [0.7095, 0.7737]$
& $0.0042\ [-0.0112, 0.0196]$ \\

Shopping
& $1.1049 \pm 0.0796$
& $1.3019 \pm 0.3199$
& $0.6973\ [0.6787, 0.7150]$
& $0.7009\ [0.6828, 0.7186]$
& $-0.0036\ [-0.0121, 0.0052]$ \\

\bottomrule
\end{tabular}
}
\caption{
Compression ratios (mean $\pm$ standard deviation across splits) for splits to evaluate a deferred test sample and leaves of the final compressed MDT XGBoost relative to the original XGBoost fallback. A 95\% confidence interval on the difference ($\text{MDT} - \text{XGB}$) in test deferred region accuracy shows they are indistinguishable. In some settings, the MDT's XGBoost has higher mean accuracy. In other cases, it has lower mean accuracy. However, the confidence intervals almost always contain 0. This shows that we can obtain meaningful compression on the black box without any loss in accuracy.  Hyperparameters were selected according to maximum validation accuracy without exceeding 25\% test deferral for a train-test split. To compute confidence intervals, we concatenate the deferred-sample correctness indicators across all train/test splits and apply bootstrap resampling to the pooled vector of deferred samples.
}
\label{tab:appendix-compressed_fallback_results}
\end{table}

\begin{table}[t]
\centering
\caption{Runtime statistics of the chosen hyperparameters (best validation accuracy with under 25\% test deferral) across datasets. Mean $\pm$ standard deviation across 5 train/test splits is shown.}
\label{appendix:timing}
\begin{tabular}{lc}
\toprule
Dataset & Mean $\pm$ Std (s) \\
\midrule
Abalone & $27.99 \pm 2.44$ \\
Adult & $367.56 \pm 106.81$ \\
Aging & $0.65 \pm 0.12$ \\
Bike & $102.39 \pm 48.71$ \\
California & $315.83 \pm 28.55$ \\
Churn & $57.19 \pm 20.99$ \\
Compas & $9.08 \pm 2.34$ \\
Coupon & $0.50 \pm 0.13$ \\
Credit & $459.45 \pm 82.95$ \\
Diamonds & $364.63 \pm 185.63$ \\
Droid & $53.54 \pm 29.69$ \\
Heloc & $25.59 \pm 2.00$ \\
Jasmine & $44.13 \pm 16.54$ \\
Madeline & $323.47 \pm 126.18$ \\
Magic & $427.09 \pm 15.86$ \\
Monk2 & $0.52 \pm 0.04$ \\
Phishing & $18.49 \pm 12.89$ \\
Pol & $159.86 \pm 37.14$ \\
Rl & $37.65 \pm 43.64$ \\
Shopping & $146.12 \pm 11.98$ \\
Spambase & $160.96 \pm 33.55$ \\
Student & $2.08 \pm 0.43$ \\
Tictactoe & $2.22 \pm 0.07$ \\
Wine & $63.69 \pm 5.84$ \\
\bottomrule
\end{tabular}
\end{table}

\begin{table}[!t]
\centering
\small
\setlength{\tabcolsep}{4pt}
\renewcommand{\arraystretch}{0.92}
\resizebox{\textwidth}{!}{
\begin{tabular}{lccc}
\toprule
Dataset
& Difference in Train/Test Defer Rates
& Train Defer Rate
& Test Defer Rate\\
\midrule

Abalone
& $0.0000 \pm 0.0000$
& $0.0000 \pm 0.0000$
& $0.0000 \pm 0.0000$ \\

Adult
& $0.0017 \pm 0.0016$
& $0.0890 \pm 0.0097$
& $0.0879 \pm 0.0099$ \\

Aging
& $0.0000 \pm 0.0000$
& $0.0000 \pm 0.0000$
& $0.0000 \pm 0.0000$ \\

Bike
& $0.0060 \pm 0.0050$
& $0.0851 \pm 0.0300$
& $0.0791 \pm 0.0295$ \\

California
& $0.0022 \pm 0.0016$
& $0.0608 \pm 0.0070$
& $0.0588 \pm 0.0059$ \\

Churn
& $0.0053 \pm 0.0044$
& $0.0437 \pm 0.0414$
& $0.0444 \pm 0.0380$ \\

Compas
& $0.0000 \pm 0.0000$
& $0.0000 \pm 0.0000$
& $0.0000 \pm 0.0000$ \\

Coupon
& $0.0142 \pm 0.0317$
& $0.0233 \pm 0.0520$
& $0.0091 \pm 0.0203$ \\

Credit
& $0.0021 \pm 0.0019$
& $0.0234 \pm 0.0139$
& $0.0220 \pm 0.0145$ \\

Diamonds
& $0.0030 \pm 0.0021$
& $0.0840 \pm 0.0186$
& $0.0810 \pm 0.0180$ \\

Droid
& $0.0048 \pm 0.0040$
& $0.0873 \pm 0.0127$
& $0.0832 \pm 0.0103$ \\

Heloc
& $0.0048 \pm 0.0066$
& $0.0228 \pm 0.0322$
& $0.0276 \pm 0.0386$ \\

Jasmine
& $0.0044 \pm 0.0071$
& $0.0258 \pm 0.0360$
& $0.0214 \pm 0.0312$ \\

Madeline
& $0.0094 \pm 0.0068$
& $0.0482 \pm 0.0289$
& $0.0576 \pm 0.0332$ \\

Magic
& $0.0072 \pm 0.0053$
& $0.0664 \pm 0.0298$
& $0.0669 \pm 0.0304$ \\

Monk2
& $0.0000 \pm 0.0000$
& $0.0000 \pm 0.0000$
& $0.0000 \pm 0.0000$ \\

Phishing
& $0.0048 \pm 0.0042$
& $0.0914 \pm 0.0039$
& $0.0903 \pm 0.0047$ \\

Pol
& $0.0059 \pm 0.0021$
& $0.0728 \pm 0.0242$
& $0.0719 \pm 0.0188$ \\

Rl
& $0.0046 \pm 0.0047$
& $0.0436 \pm 0.0463$
& $0.0390 \pm 0.0416$ \\

Shopping
& $0.0062 \pm 0.0046$
& $0.0691 \pm 0.0336$
& $0.0648 \pm 0.0289$ \\

Spambase
& $0.0046 \pm 0.0027$
& $0.0702 \pm 0.0220$
& $0.0739 \pm 0.0194$ \\

Student
& $0.0011 \pm 0.0026$
& $0.0066 \pm 0.0146$
& $0.0077 \pm 0.0172$ \\

Tictactoe
& $0.0047 \pm 0.0065$
& $0.0350 \pm 0.0346$
& $0.0323 \pm 0.0358$ \\

Wine
& $0.0069 \pm 0.0077$
& $0.0527 \pm 0.0412$
& $0.0594 \pm 0.0468$ \\

\bottomrule
\end{tabular}
}
\caption{Mean and standard deviation across train/test splits of the absolute difference between train and test defer rates, along with the train and test defer rates themselves, for MDT+XGB models selected under a 10\% test defer-rate cutoff.}
\label{tab:train_test_defer_rate_stability_10}
\end{table}

\begin{table}[!t]
\centering
\scriptsize
\setlength{\tabcolsep}{4pt}
\renewcommand{\arraystretch}{0.88}
\resizebox{\textwidth}{!}{
\begin{tabular}{lccccc}
\toprule
Dataset 
& Non-Deferral Tree 
& MDT+XGB $\le$25\% 
& FIGS
& Random Forest 
& XGBoost \\
\midrule

Abalone 
& $0.6314 \pm 0.0135$
& $0.6333 \pm 0.0155$
& $0.6381 \pm 0.0066$
& $0.6314 \pm 0.0079$
& $0.6323 \pm 0.0157$ \\

Adult
& $0.8598 \pm 0.0038$
& $0.8667 \pm 0.0047$
& $0.8589 \pm 0.0048$
& $0.8619 \pm 0.0038$
& $0.8701 \pm 0.0030$ \\

Aging
& $0.8028 \pm 0.0077$
& $0.8028 \pm 0.0077$
& $0.7944 \pm 0.0106$
& $0.8014 \pm 0.0080$
& $0.8000 \pm 0.0063$ \\

Bank & $0.9051 \pm 0.0008$ & $0.9052 \pm 0.0017$ & $0.9069 \pm 0.0008$ & $\underline{0.9087 \pm 0.0008}$ & $\mathbf{0.9090 \pm 0.0007}$ \\

Bike
& $0.9346 \pm 0.0046$
& $0.9486 \pm 0.0054$
& $0.9113 \pm 0.0087$
& $0.9426 \pm 0.0037$
& $0.9500 \pm 0.0052$ \\

California
& $0.8893 \pm 0.0053$
& $0.9099 \pm 0.0045$
& $0.8774 \pm 0.0034$
& $0.9019 \pm 0.0072$
& $0.9101 \pm 0.0050$ \\

Churn
& $0.9424 \pm 0.0081$
& $0.9566 \pm 0.0063$
& $0.9434 \pm 0.0127$
& $0.9506 \pm 0.0072$
& $0.9574 \pm 0.0063$ \\

Compas
& $0.6725 \pm 0.0076$
& $0.6693 \pm 0.0083$
& $0.6657 \pm 0.0112$
& $0.6741 \pm 0.0121$
& $0.6757 \pm 0.0072$ \\

Coupon
& $0.6091 \pm 0.0689$
& $0.6000 \pm 0.0674$
& $0.5727 \pm 0.0407$
& $0.5364 \pm 0.0498$
& $0.5545 \pm 0.0674$ \\

Credit
& $0.8196 \pm 0.0032$
& $0.8182 \pm 0.0065$
& $0.8200 \pm 0.0032$
& $0.8195 \pm 0.0029$
& $0.8202 \pm 0.0029$ \\

Diamonds
& $0.8815 \pm 0.0016$
& $0.8877 \pm 0.0014$
& $0.8829 \pm 0.0023$
& $0.8871 \pm 0.0009$
& $0.8883 \pm 0.0021$ \\

Droid
& $0.9627 \pm 0.0024$
& $0.9709 \pm 0.0012$
& $0.9601 \pm 0.0011$
& $0.9715 \pm 0.0014$
& $0.9718 \pm 0.0009$ \\

Heloc
& $0.7052 \pm 0.0239$
& $0.7056 \pm 0.0239$
& $0.6840 \pm 0.0165$
& $0.7080 \pm 0.0228$
& $0.7116 \pm 0.0241$ \\

Jasmine
& $0.8037 \pm 0.0173$
& $0.8161 \pm 0.0224$
& $0.7983 \pm 0.0120$
& $0.8218 \pm 0.0190$
& $0.8214 \pm 0.0158$ \\

Madeline
& $0.8408 \pm 0.0281$
& $0.8627 \pm 0.0122$
& $0.8099 \pm 0.0328$
& $0.8793 \pm 0.0158$
& $0.8697 \pm 0.0115$ \\

Magic
& $0.8541 \pm 0.0066$
& $0.8756 \pm 0.0111$
& $0.8539 \pm 0.0045$
& $0.8798 \pm 0.0059$
& $0.8849 \pm 0.0064$ \\

Monk2
& $1.0000 \pm 0.0000$
& $1.0000 \pm 0.0000$
& $0.8867 \pm 0.0402$
& $0.9933 \pm 0.0149$
& $1.0000 \pm 0.0000$ \\

Phishing
& $0.9528 \pm 0.0039$
& $0.9728 \pm 0.0038$
& $0.9392 \pm 0.0042$
& $0.9721 \pm 0.0029$
& $0.9733 \pm 0.0034$ \\

Pol
& $0.9723 \pm 0.0026$
& $0.9849 \pm 0.0006$
& $0.9631 \pm 0.0040$
& $0.9823 \pm 0.0016$
& $0.9845 \pm 0.0015$ \\

Rl
& $0.7614 \pm 0.0117$
& $0.7968 \pm 0.0198$
& $0.7163 \pm 0.0210$
& $0.8074 \pm 0.0129$
& $0.8278 \pm 0.0075$ \\

Shopping
& $0.9010 \pm 0.0055$
& $0.9047 \pm 0.0036$
& $0.8988 \pm 0.0048$
& $0.9046 \pm 0.0027$
& $0.9057 \pm 0.0034$ \\

Spambase
& $0.9254 \pm 0.0035$
& $0.9500 \pm 0.0053$
& $0.9222 \pm 0.0077$
& $0.9517 \pm 0.0031$
& $0.9528 \pm 0.0051$ \\

Student
& $0.8292 \pm 0.0305$
& $0.8154 \pm 0.0361$
& $0.8046 \pm 0.0241$
& $0.8215 \pm 0.0358$
& $0.8046 \pm 0.0301$ \\

Tictactoe
& $0.9375 \pm 0.0165$
& $0.9885 \pm 0.0113$
& $0.9552 \pm 0.0291$
& $0.9896 \pm 0.0037$
& $1.0000 \pm 0.0000$ \\

Wine
& $0.8340 \pm 0.0033$
& $0.8788 \pm 0.0116$
& $0.8391 \pm 0.0087$
& $0.8848 \pm 0.0095$
& $0.8864 \pm 0.0049$ \\

\bottomrule
\end{tabular}
}
\caption{Test accuracy (mean $\pm$ standard deviation) across all datasets.}
\label{tab:all_results_std}
\end{table}

\FloatBarrier

\begin{table}[htbp]
\centering
\caption{Black-box evaluation complexity at a 25\%  deferral rate cap, for the best-performing MDT at no more than that test deferral rate. For each deferred test example, we count the split decisions traversed across all XGBoost trees, then average over deferred examples from the completed train/test splits. \textit{Root XGBoost} uses the root black-box model as the fallback without retraining. \textit{Retrained XGBoost} retrains the black box at each stage as in our algorithm, without compression. \textit{Compressed single XGBoost} applies compression using the union of deferred subregions. \textit{Subregion-specific XGBoost} compresses a model separately within each deferred subregion rather than over their union. Lower is better.}
\label{tab:blackbox_complexity}
\resizebox{\linewidth}{!}{%
\begin{tabular}{lrrrr}
\toprule
Dataset & Subregion-specific XGBoost & Compressed single XGBoost & Retrained XGBoost & Root XGBoost \\
\midrule
Abalone   & 1,863 & 3,546 & 3,569 & 3,579 \\
Adult     & 4,007 & 5,442 & 5,514 & 5,516 \\
Bike      & 3,241 & 5,319 & 5,319 & 5,549 \\
Churn     & 3,428 & 4,115 & 4,116 & 4,903 \\
Compas    & 798   & 798   & 1,030 & 1,049 \\
Coupon    & 2,266 & 2,266 & 2,689 & 2,662 \\
Diamonds  & 5,268 & 7,193 & 7,508 & 7,848 \\
Droid     & 8,042 & 9,287 & 9,287 & 9,009 \\
Heloc     & 899   & 953   & 967   & 971   \\
Jasmine   & 2,480 & 2,533 & 2,533 & 3,969 \\
Madeline  & 3,896 & 4,366 & 4,366 & 4,593 \\
Phishing  & 1,985 & 2,550 & 2,611 & 4,071 \\
Pol       & 3,106 & 4,047 & 4,078 & 4,122 \\
Rl        & 3,046 & 4,458 & 4,634 & 6,265 \\
Shopping  & 2,540 & 2,579 & 2,616 & 2,963 \\
Spambase  & 3,190 & 3,937 & 3,943 & 4,095 \\
Tictactoe & 2,240 & 2,656 & 2,656 & 2,834 \\
Wine      & 4,076 & 6,862 & 6,862 & 7,004 \\
\bottomrule
\end{tabular}%
}
\end{table}

\begin{table}[htbp]
\centering
\caption{Normalized black-box evaluation complexity at a 25\% target deferral cap, for the best-performing MDT at that deferral rate. For each deferred test example, we count the split decisions traversed across all XGBoost trees, then average over deferred examples from the completed train/test splits. \textit{Root XGBoost} uses the root black-box model as the fallback without retraining. \textit{Retrained XGBoost} retrains the black box at each stage as in our algorithm, without compression. \textit{Compressed single XGBoost} applies compression using the union of deferred subregions. \textit{Subregion-specific XGBoost} compresses a model separately within each deferred subregion rather than over their union. Every value is divided by the corresponding Root XGBoost count; lower is better.}
\label{tab:blackbox_complexity_normalized}
\resizebox{\linewidth}{!}{%
\begin{tabular}{lrrrr}
\toprule
Dataset & Subregion-specific XGBoost & Compressed single XGBoost & Retrained XGBoost & Root XGBoost \\
\midrule
abalone   & 0.521 & 0.991 & 0.997 & 1.000 \\
adult     & 0.726 & 0.987 & 1.000 & 1.000 \\
bike      & 0.584 & 0.959 & 0.959 & 1.000 \\
churn     & 0.699 & 0.839 & 0.839 & 1.000 \\
compas    & 0.761 & 0.761 & 0.982 & 1.000 \\
coupon    & 0.851 & 0.851 & 1.010 & 1.000 \\
diamonds  & 0.671 & 0.917 & 0.957 & 1.000 \\
droid     & 0.893 & 1.031 & 1.031 & 1.000 \\
heloc     & 0.926 & 0.981 & 0.996 & 1.000 \\
jasmine   & 0.625 & 0.638 & 0.638 & 1.000 \\
madeline  & 0.848 & 0.951 & 0.951 & 1.000 \\
phishing  & 0.488 & 0.626 & 0.641 & 1.000 \\
pol       & 0.754 & 0.982 & 0.989 & 1.000 \\
rl        & 0.486 & 0.712 & 0.740 & 1.000 \\
shopping  & 0.857 & 0.870 & 0.883 & 1.000 \\
spambase  & 0.779 & 0.961 & 0.963 & 1.000 \\
tictactoe & 0.790 & 0.937 & 0.937 & 1.000 \\
wine      & 0.582 & 0.980 & 0.980 & 1.000 \\
\bottomrule
\end{tabular}%
}
\end{table}

\FloatBarrier

\begin{figure}[t]
    \centering
    \includegraphics[width=\linewidth]{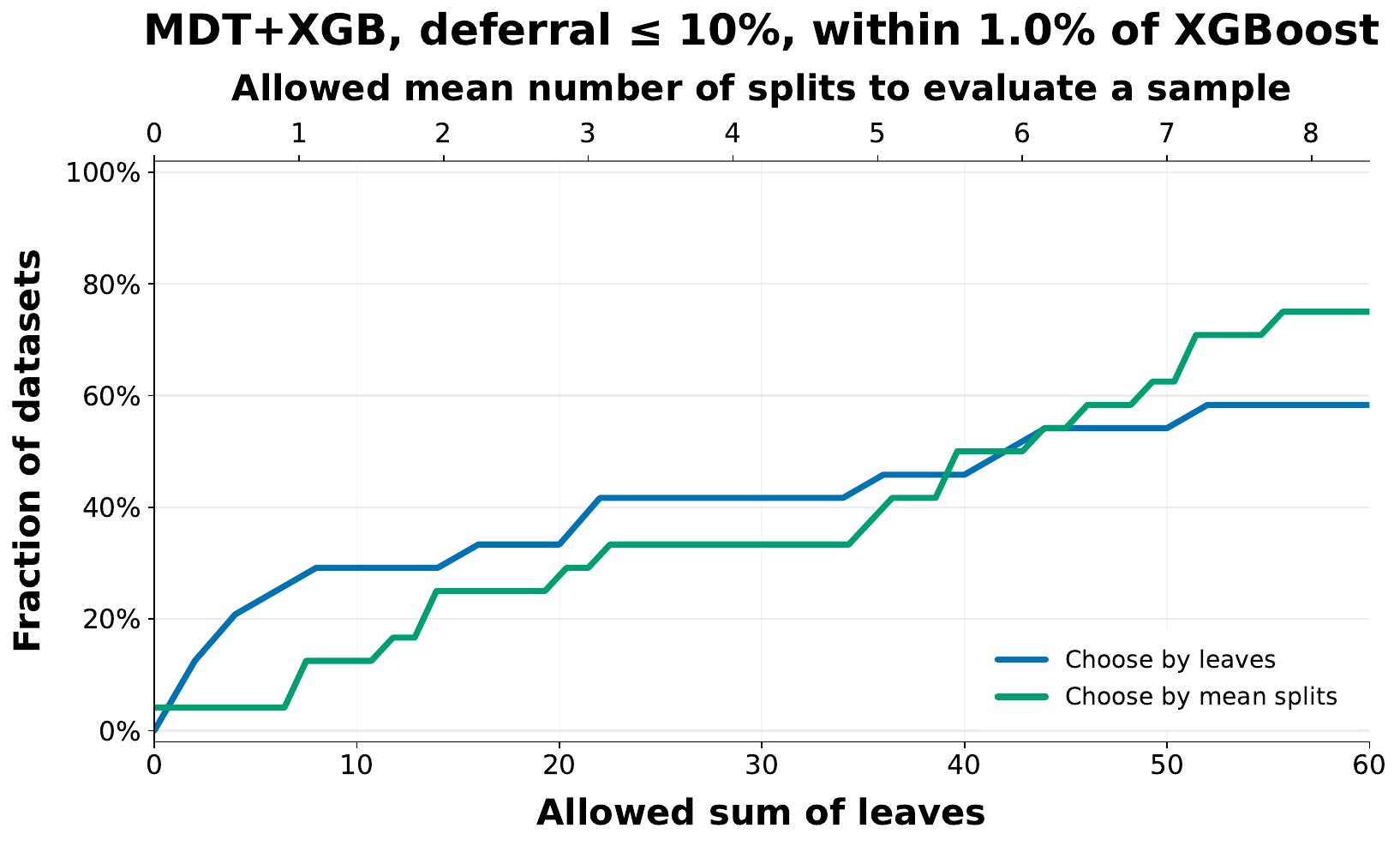}
    \caption{
    Comparison between selecting hyperparameters for MDTs using a budget on the total number of leaves versus a budget on the mean number of split decisions required to evaluate a sample. The curves are closely but not exactly aligned. A rule of thumb is that a total of 50 leaves corresponds to a mean of approximately 7 split evaluations per sample. This is much better than a full depth 7 tree, which has 128 leaves.
    }
    \label{fig:leaves_vs_mean_splits}
\end{figure}

\begin{figure}[!b]
    \centering
    \includegraphics[
        width=\linewidth,
        trim=0 25 0 25,
        clip
    ]{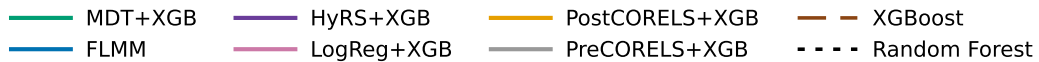}
    \caption{Legend for the accuracy-deferral trade-off plots on the following pages.}

\end{figure}

\begin{figure}[p]
    \thispagestyle{empty}
    \centering
    \captionsetup{aboveskip=2pt, belowskip=0pt}

    \begin{subfigure}{0.49\linewidth}
        \includegraphics[width=\linewidth]{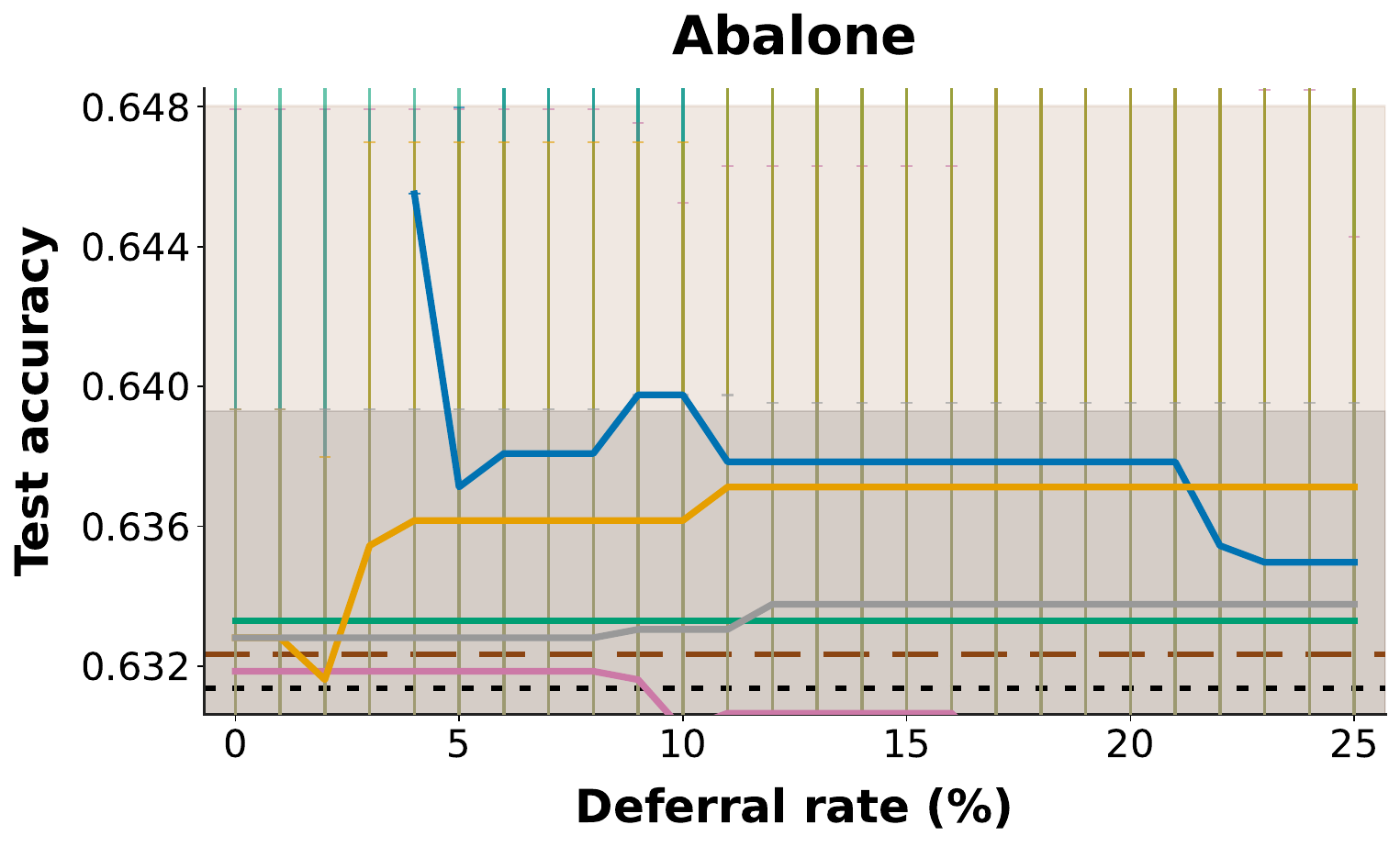}
    \end{subfigure}
    \hfill
    \begin{subfigure}{0.49\linewidth}
        \includegraphics[width=\linewidth]{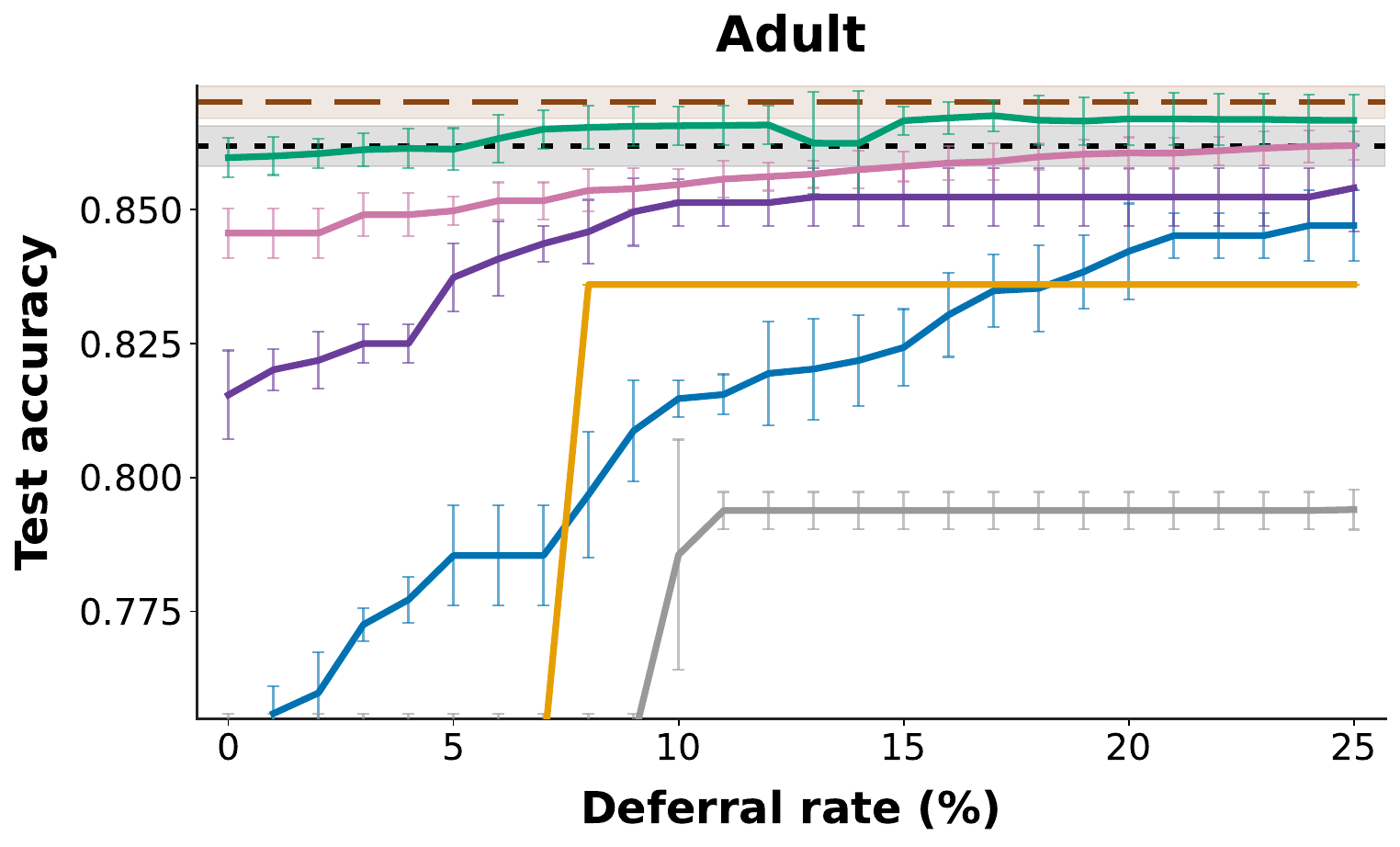}
    \end{subfigure}

    \vspace{-0.05em}

    \begin{subfigure}{0.49\linewidth}
        \includegraphics[width=\linewidth]{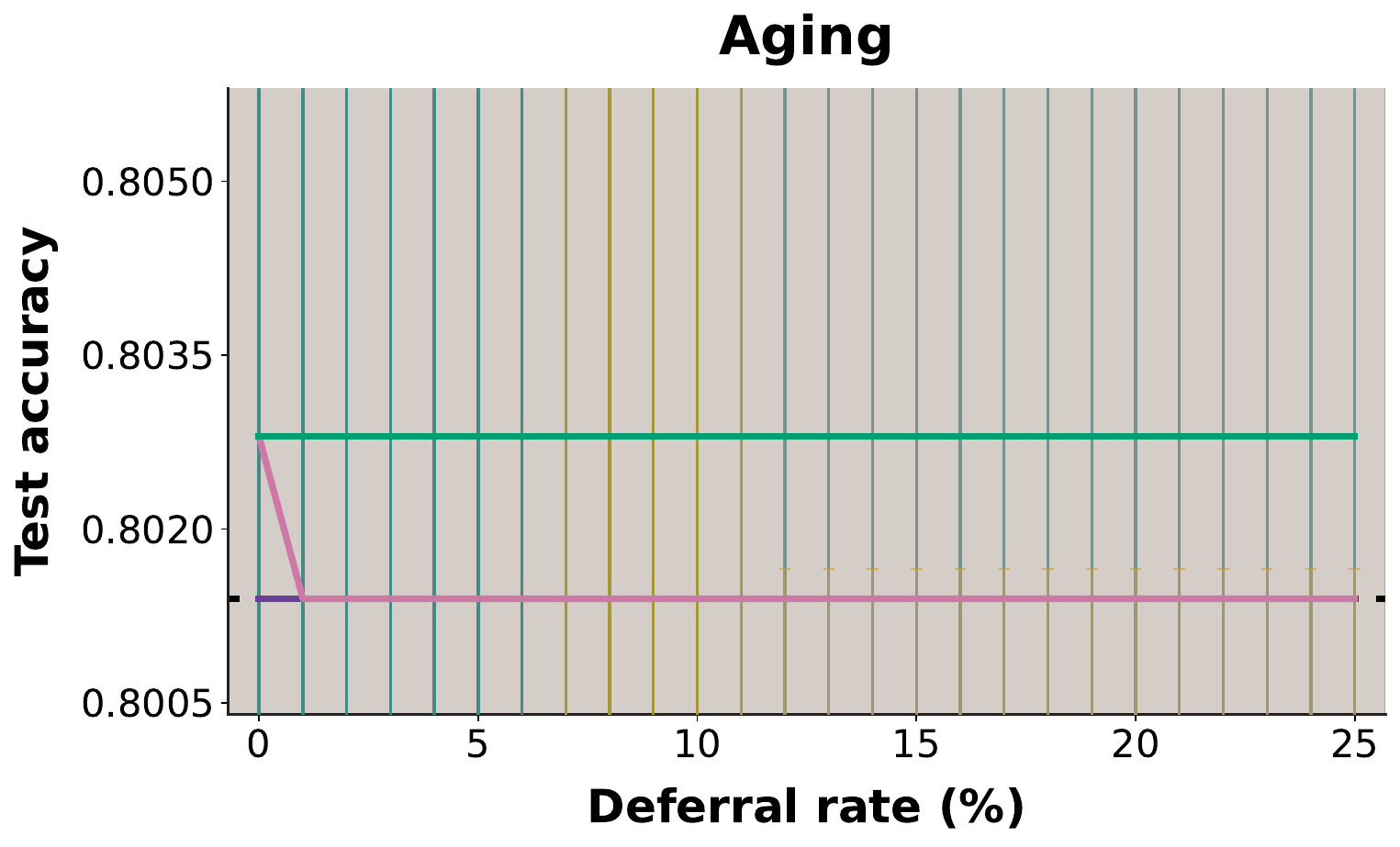}
    \end{subfigure}
    \hfill
    \begin{subfigure}{0.49\linewidth}
        \includegraphics[width=\linewidth]{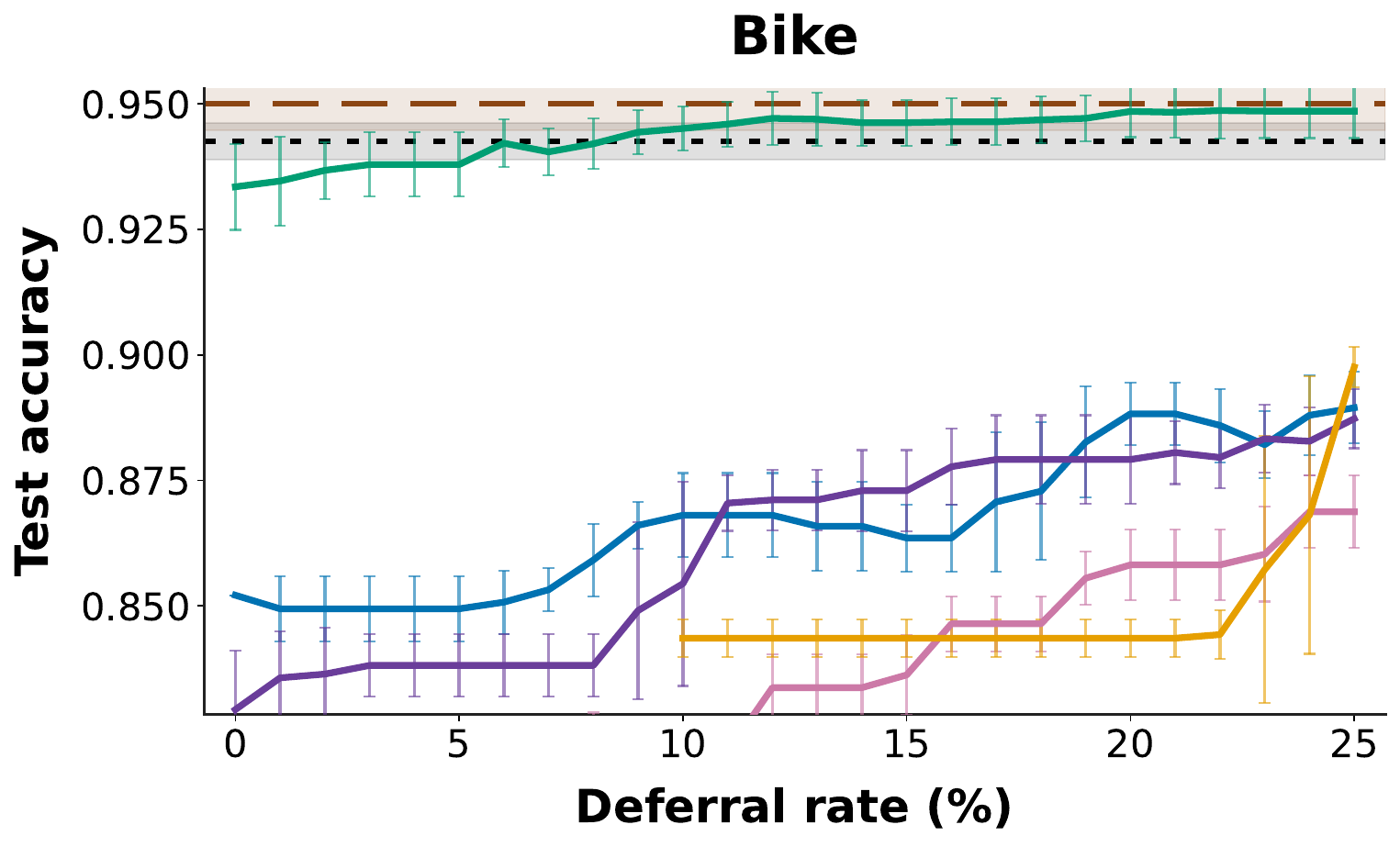}
    \end{subfigure}

    \vspace{-0.05em}

    \begin{subfigure}{0.49\linewidth}
        \includegraphics[width=\linewidth]{new_main_figs/california_main.pdf}
    \end{subfigure}
    \hfill
    \begin{subfigure}{0.49\linewidth}
        \includegraphics[width=\linewidth]{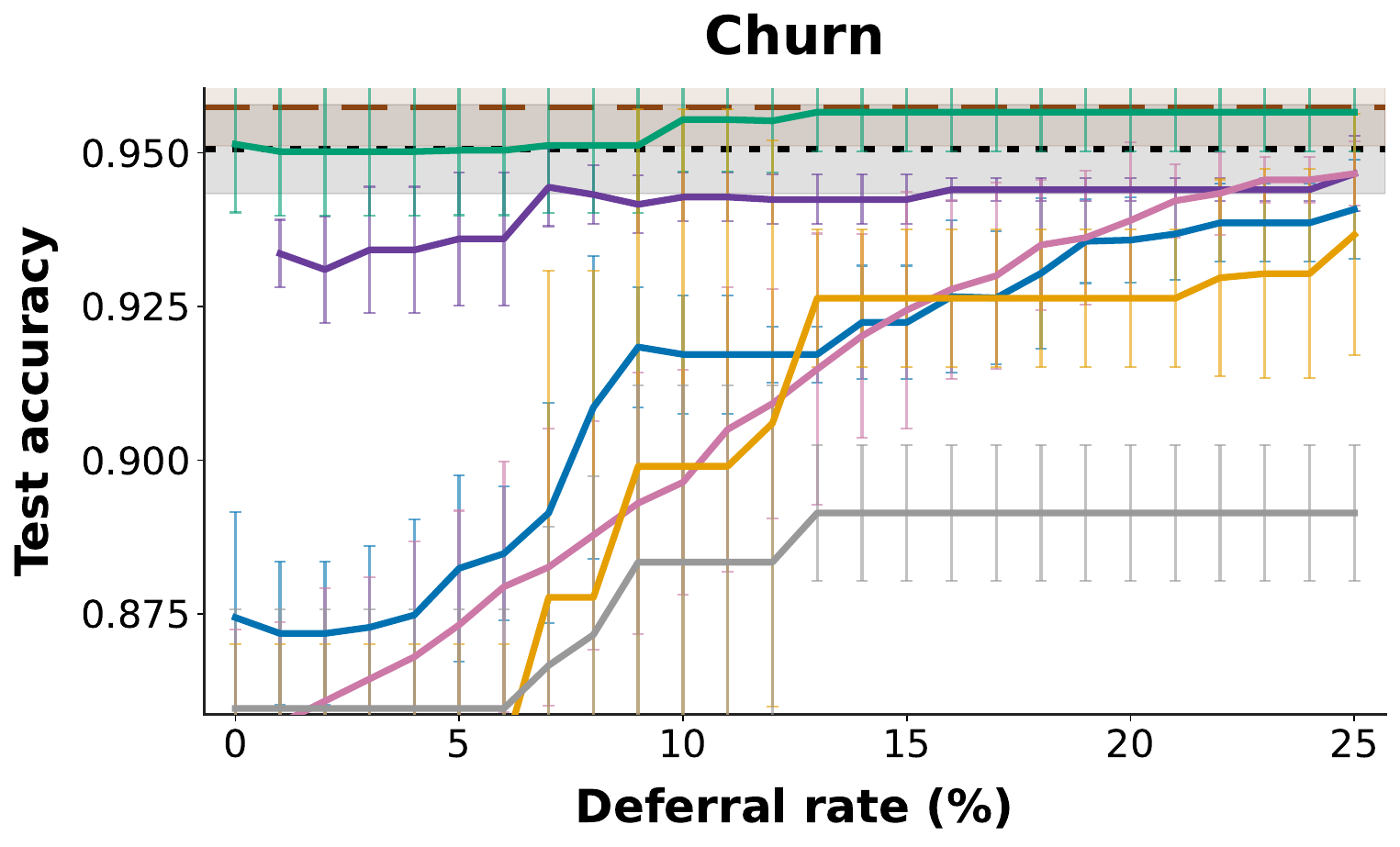}
    \end{subfigure}

    \vspace{-0.05em}

    \begin{subfigure}{0.49\linewidth}
        \includegraphics[width=\linewidth]{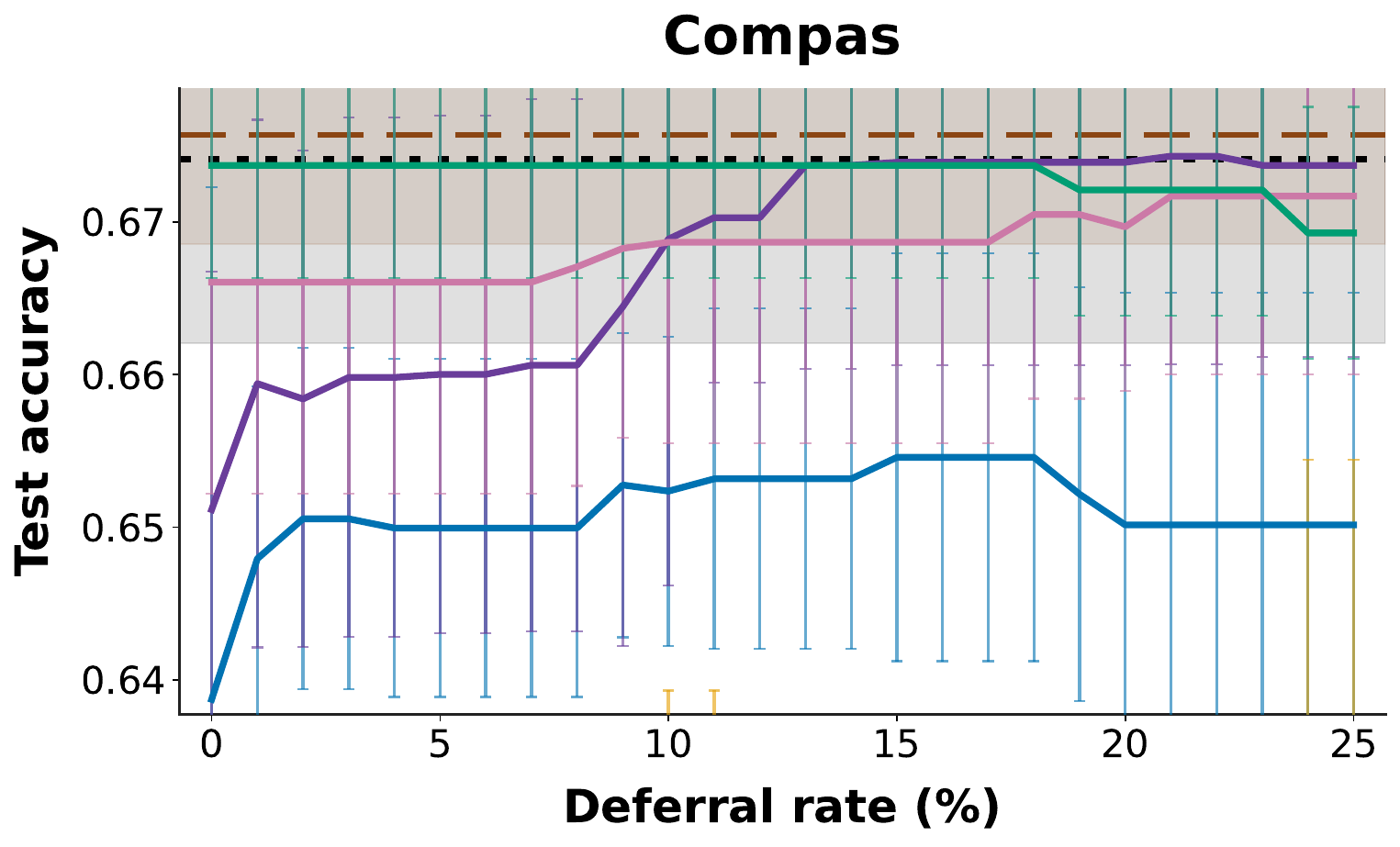}
    \end{subfigure}
    \hfill
    \begin{subfigure}{0.49\linewidth}
        \includegraphics[width=\linewidth]{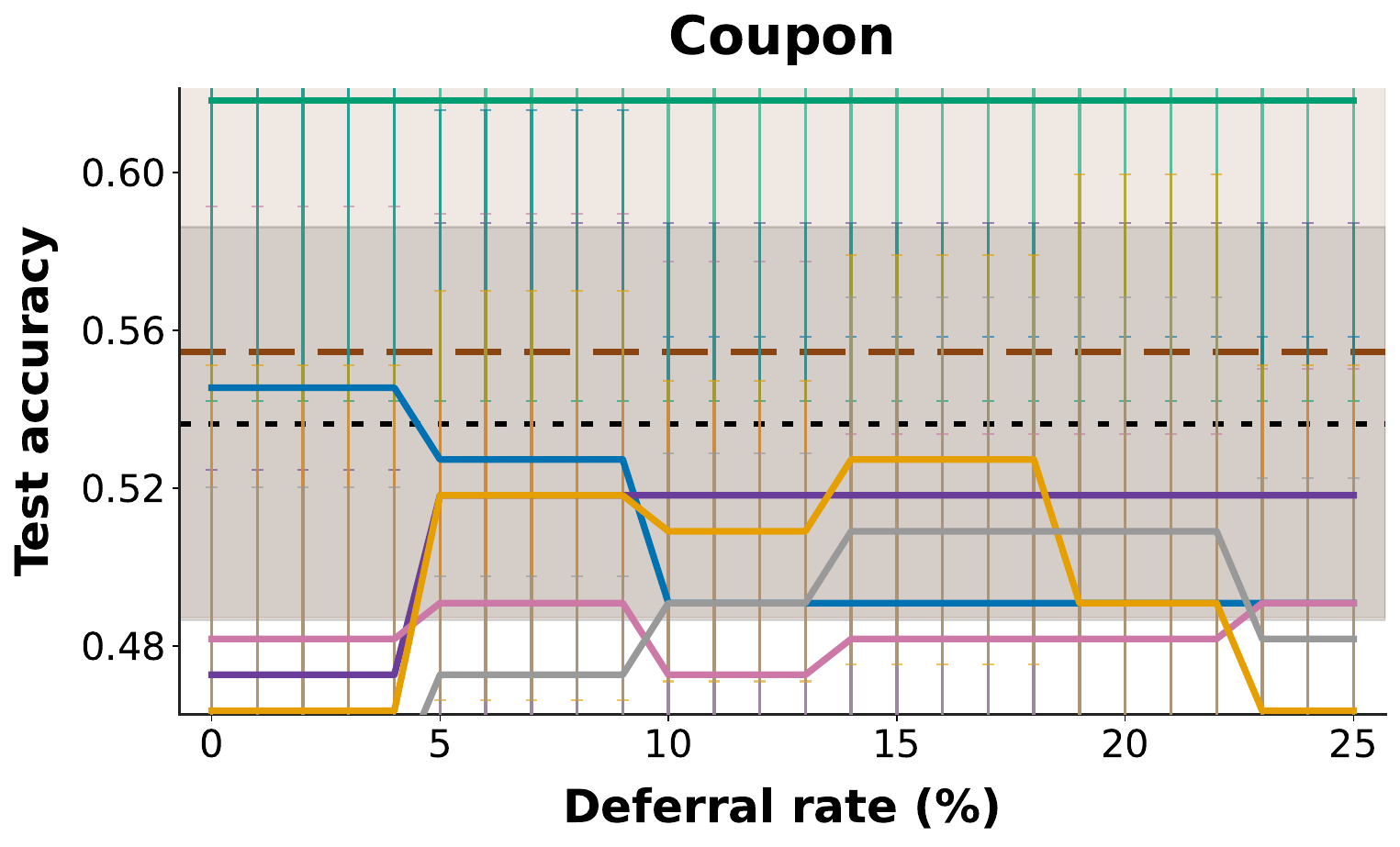}
    \end{subfigure}

    \vspace{-0.05em}

    \begin{subfigure}{0.49\linewidth}
        \includegraphics[width=\linewidth]{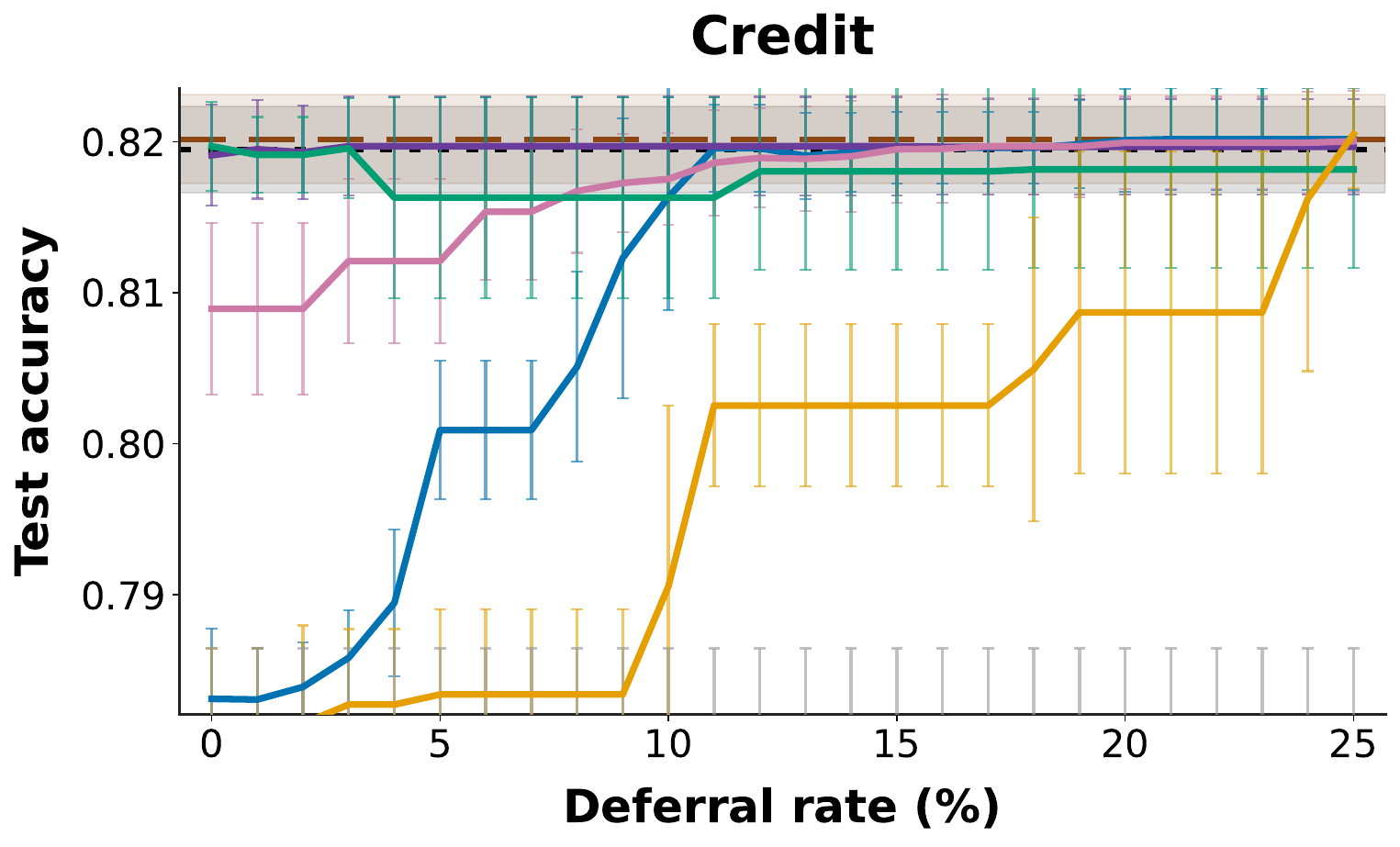}
    \end{subfigure}
    \hfill
    \begin{subfigure}{0.49\linewidth}
        \includegraphics[width=\linewidth]{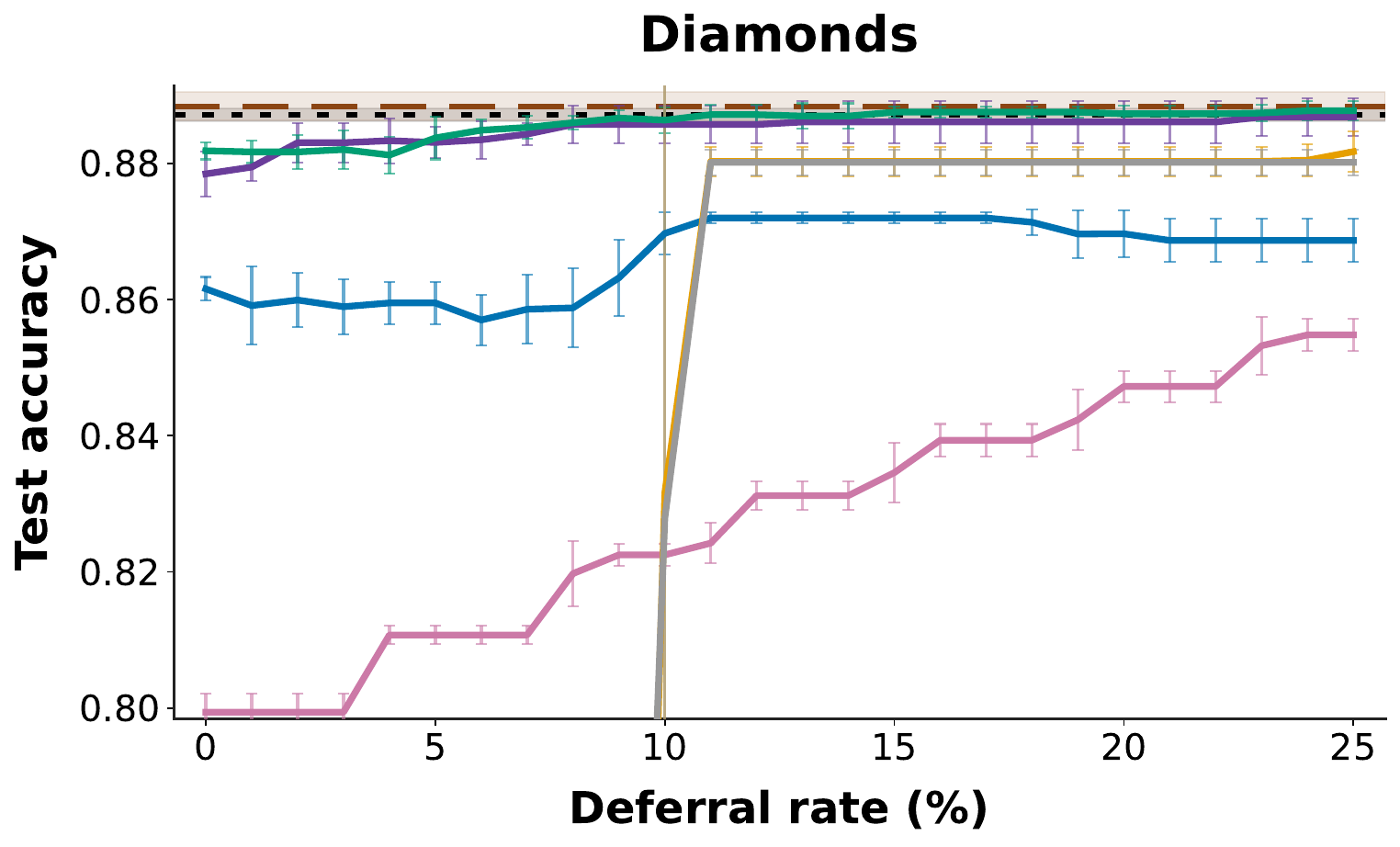}
    \end{subfigure}

    \vspace{-0.05em}

    \begin{subfigure}{0.49\linewidth}
        \includegraphics[width=\linewidth]{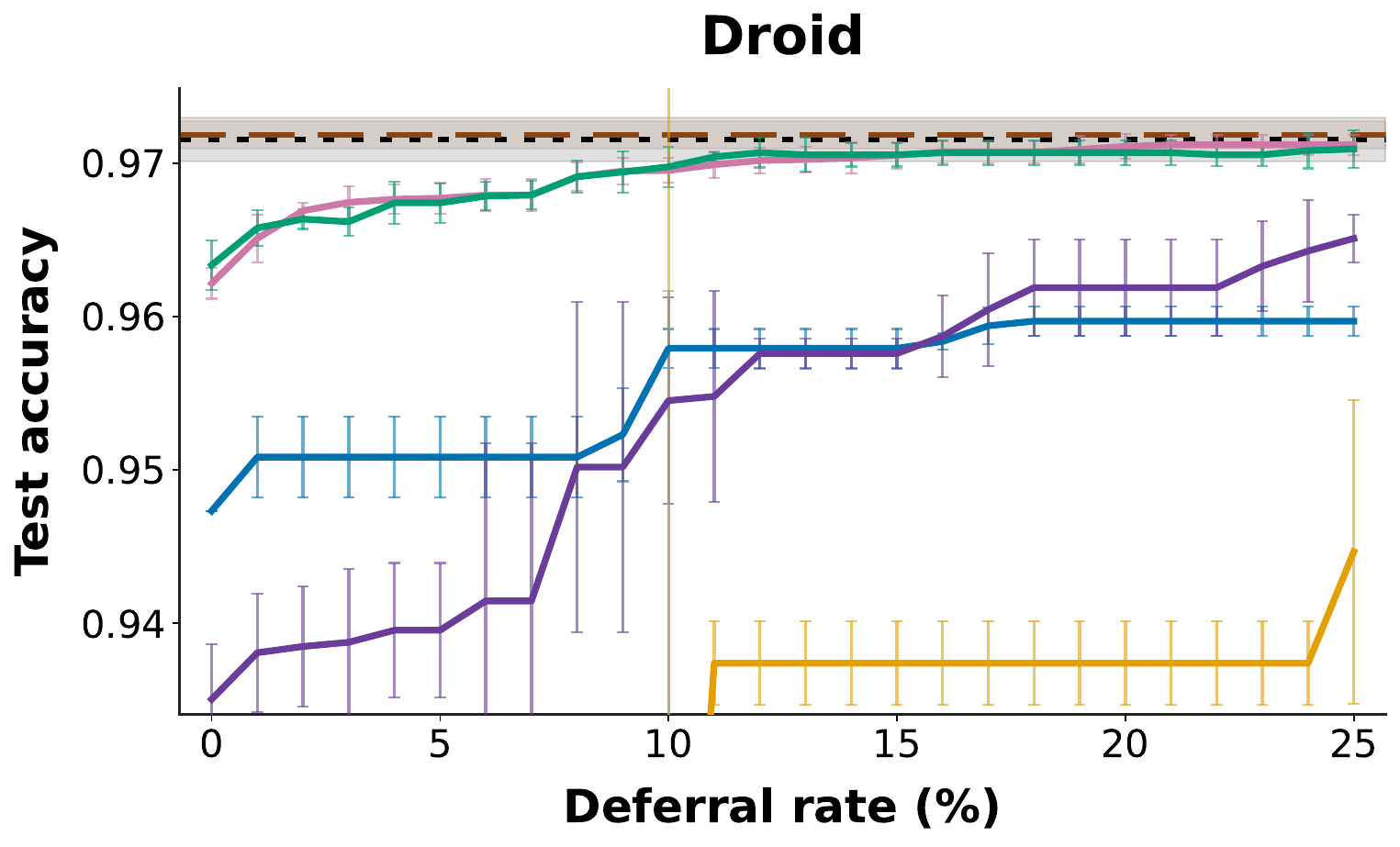}
    \end{subfigure}
    \hfill
    \begin{subfigure}{0.49\linewidth}
        \includegraphics[width=\linewidth]{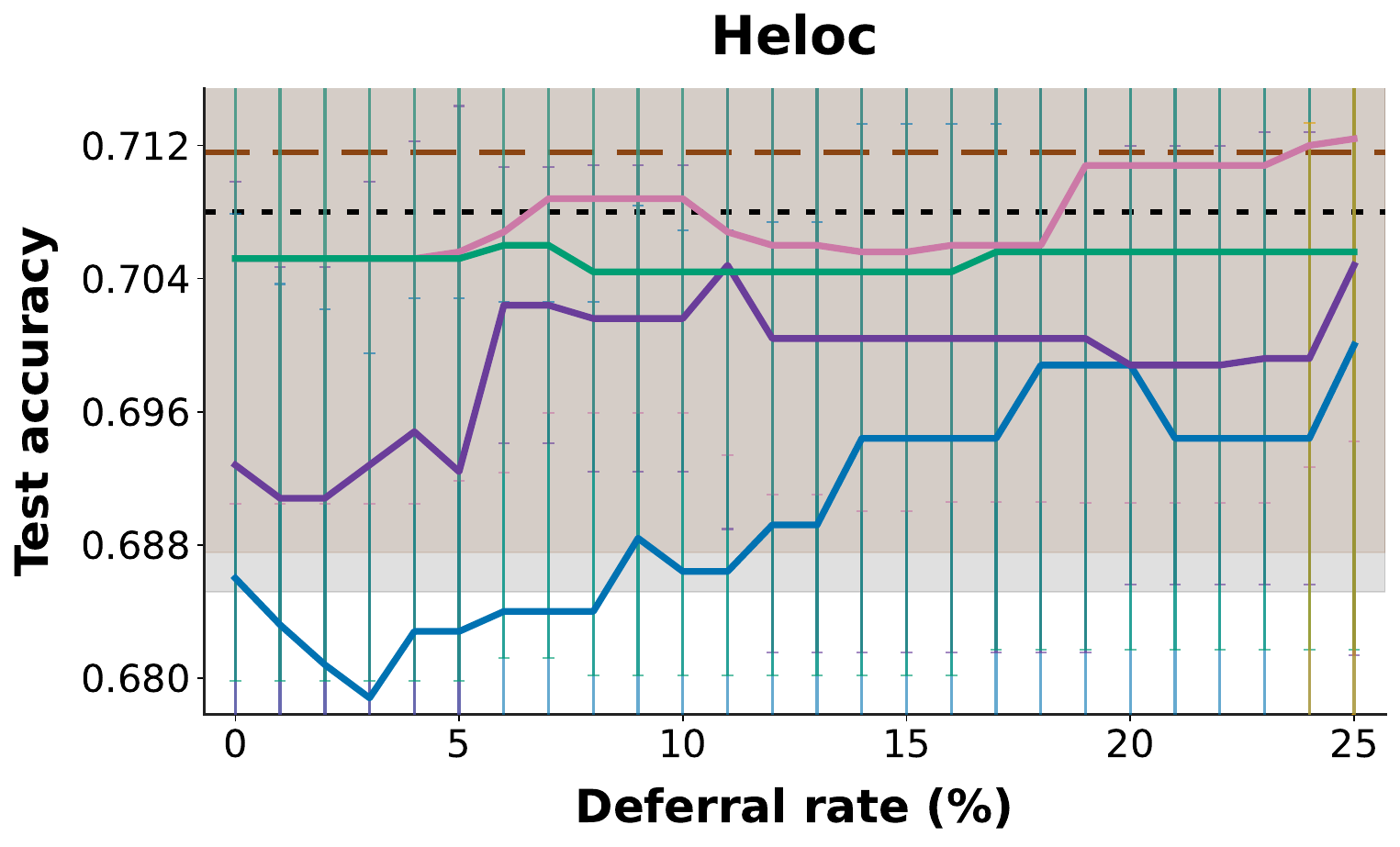}
    \end{subfigure}

    \vspace{-0.3em}
    \caption{Deferral-accuracy trade-offs for all datasets.}
    \label{fig:main_dataset_curves_part1}
\end{figure}

\begin{figure}[p]
    \thispagestyle{empty}
    \centering
    \captionsetup{aboveskip=2pt, belowskip=0pt}

    \begin{subfigure}{0.49\linewidth}
        \includegraphics[width=\linewidth]{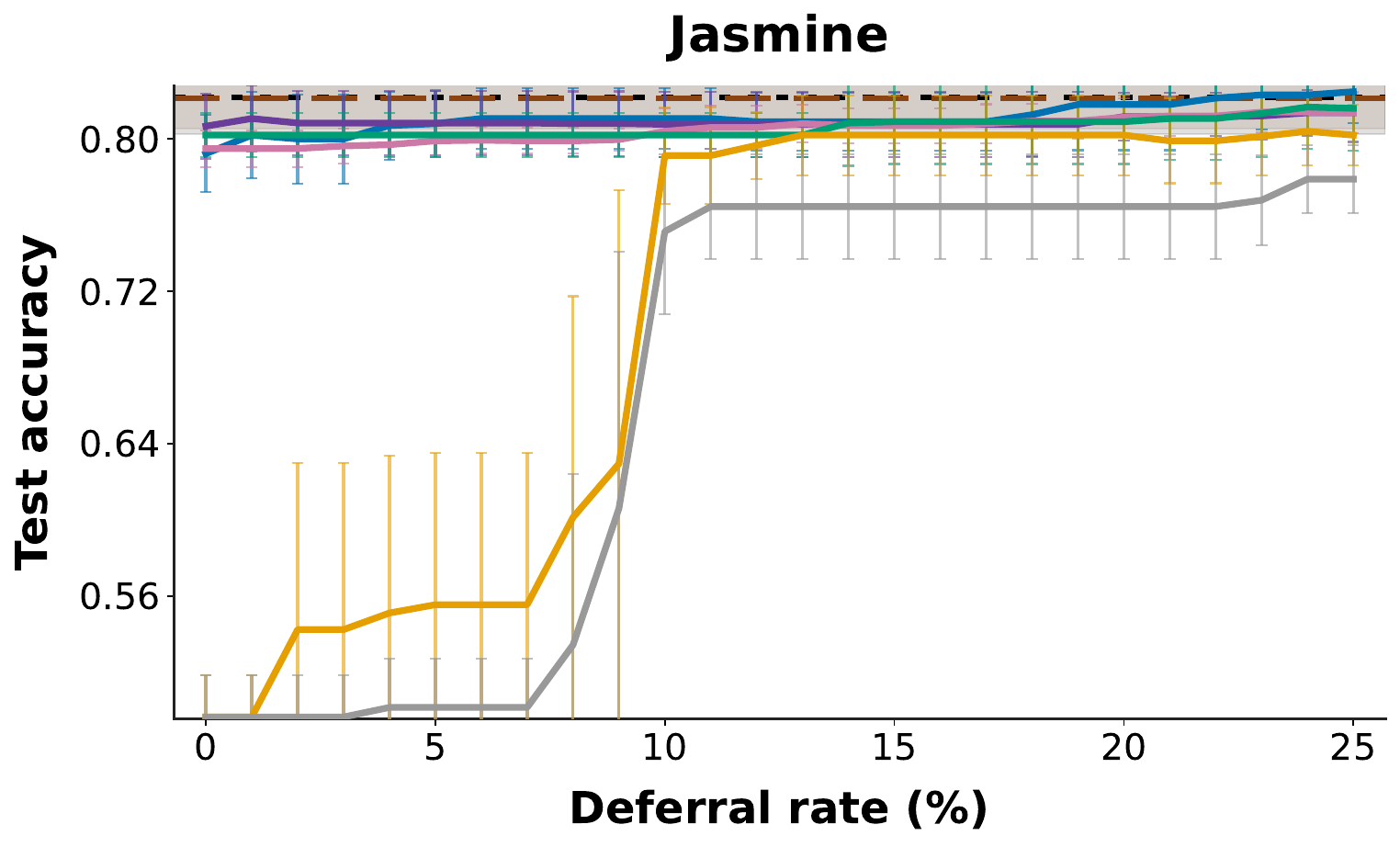}
    \end{subfigure}
    \hfill
    \begin{subfigure}{0.49\linewidth}
        \includegraphics[width=\linewidth]{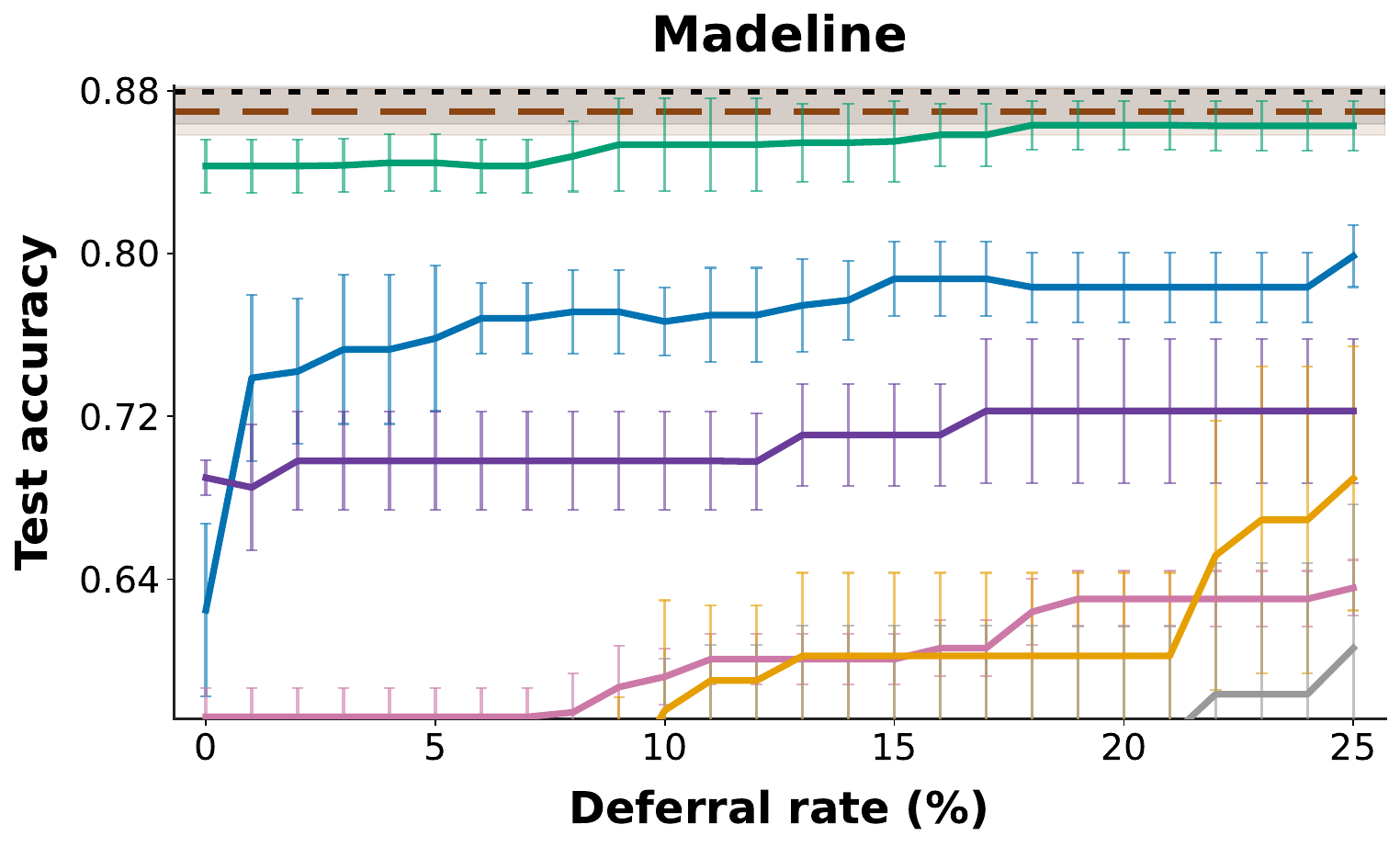}
    \end{subfigure}

    \vspace{-0.05em}

    \begin{subfigure}{0.49\linewidth}
        \includegraphics[width=\linewidth]{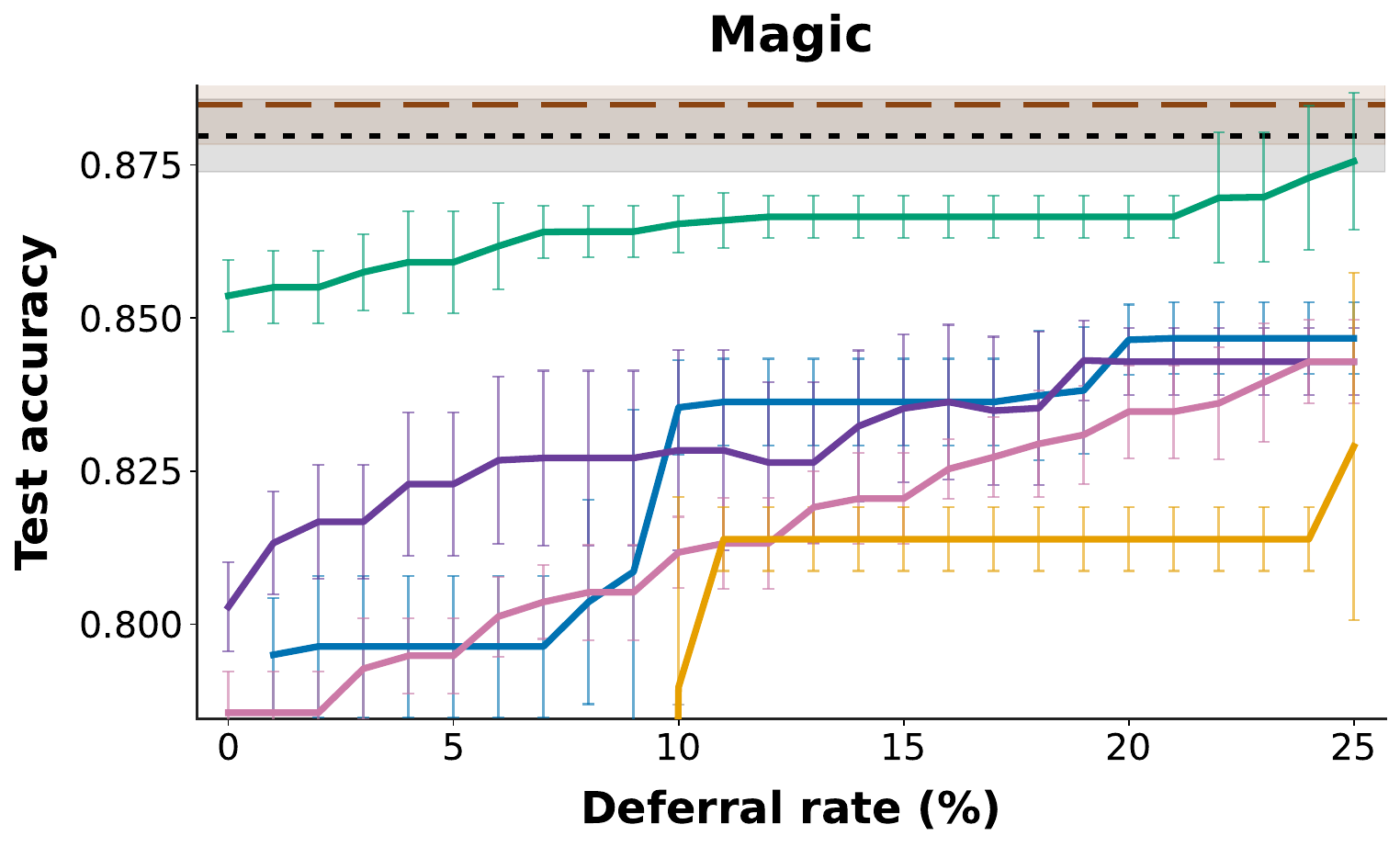}
    \end{subfigure}
    \hfill
    \begin{subfigure}{0.49\linewidth}
        \includegraphics[width=\linewidth]{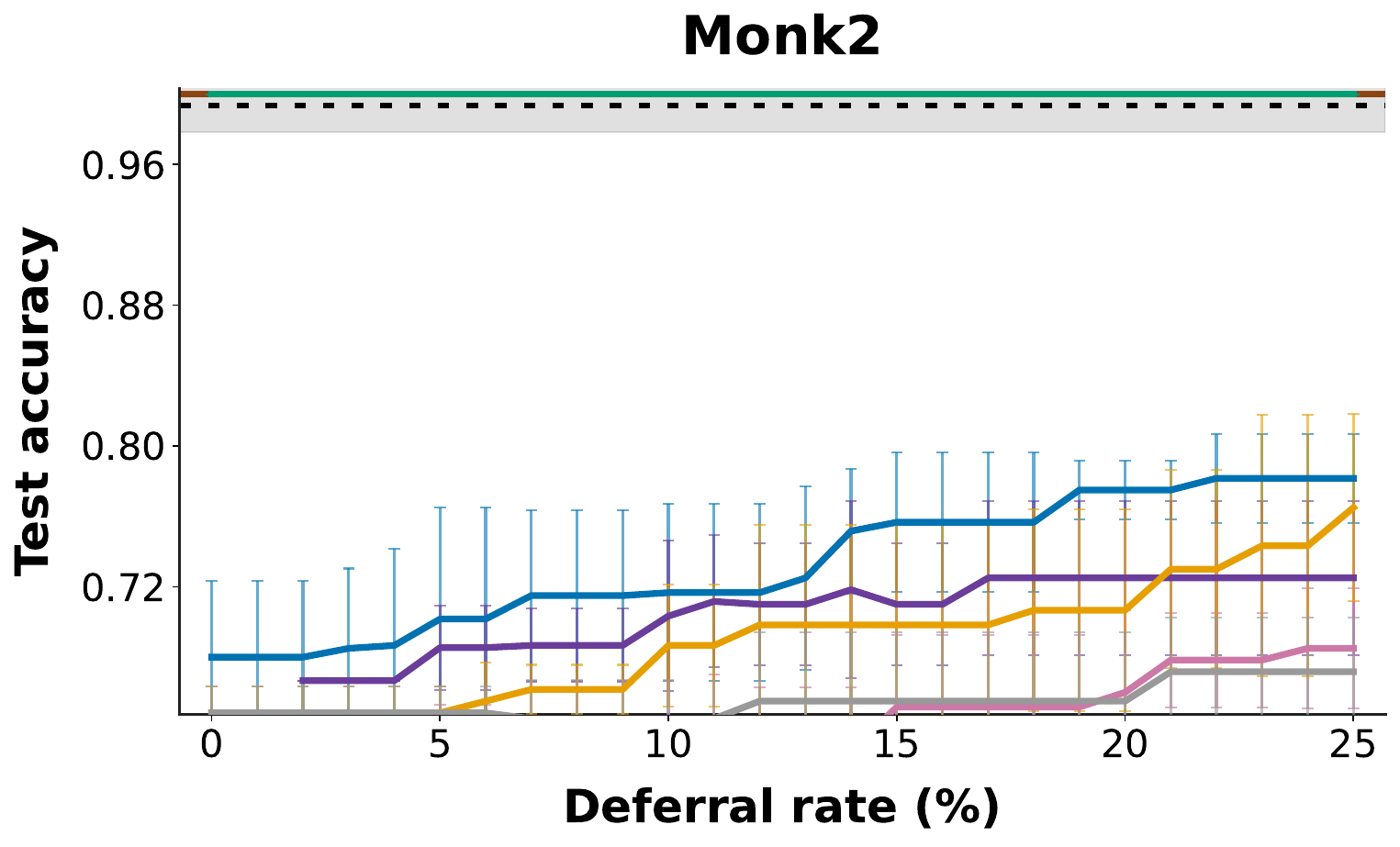}
    \end{subfigure}

    \vspace{-0.05em}

    \begin{subfigure}{0.49\linewidth}
        \includegraphics[width=\linewidth]{new_main_figs/phishing_main.pdf}
    \end{subfigure}
    \hfill
    \begin{subfigure}{0.49\linewidth}
        \includegraphics[width=\linewidth]{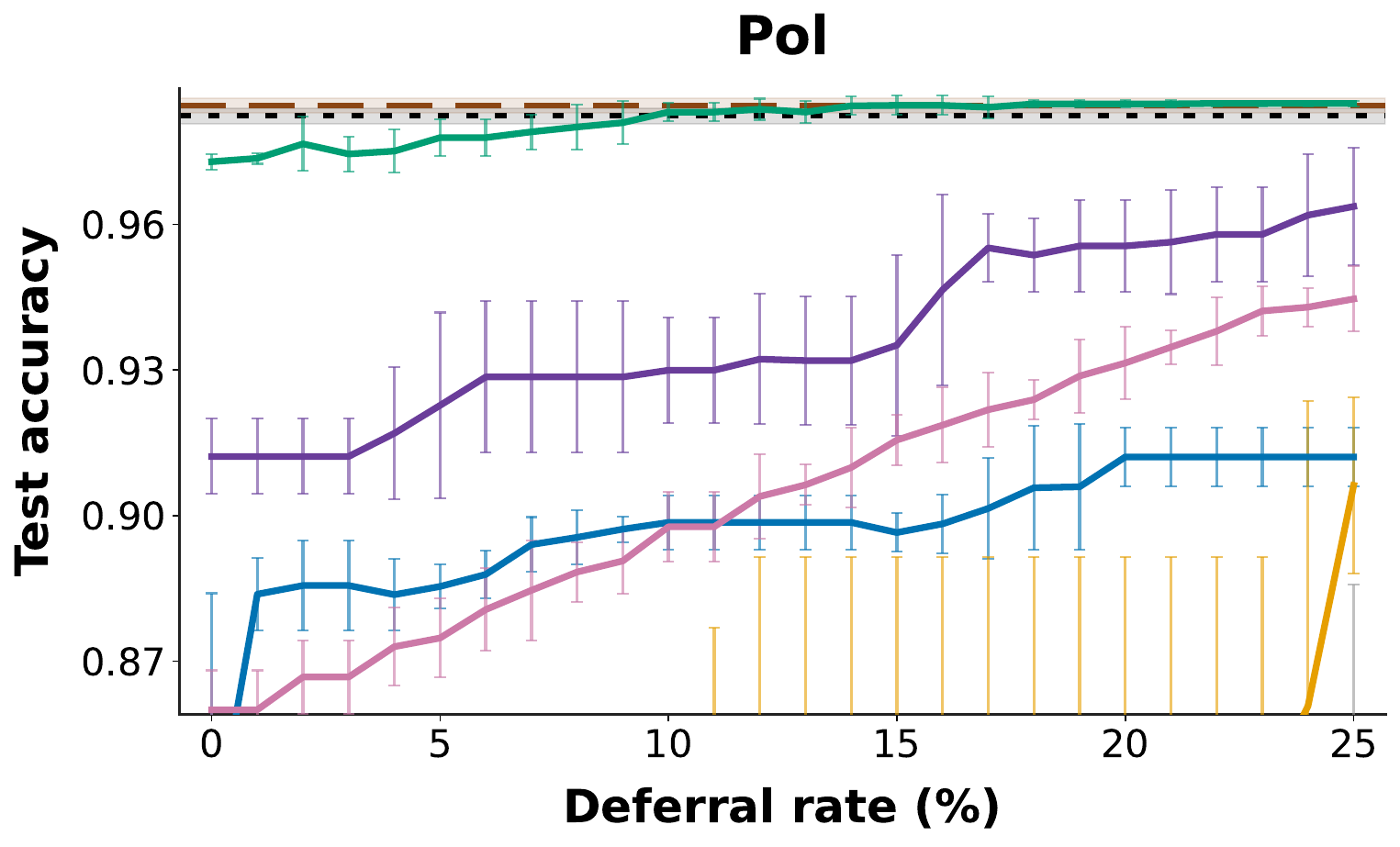}
    \end{subfigure}

    \vspace{-0.05em}

    \begin{subfigure}{0.49\linewidth}
        \includegraphics[width=\linewidth]{new_main_figs/rl_main.pdf}
    \end{subfigure}
    \hfill
    \begin{subfigure}{0.49\linewidth}
        \includegraphics[width=\linewidth]{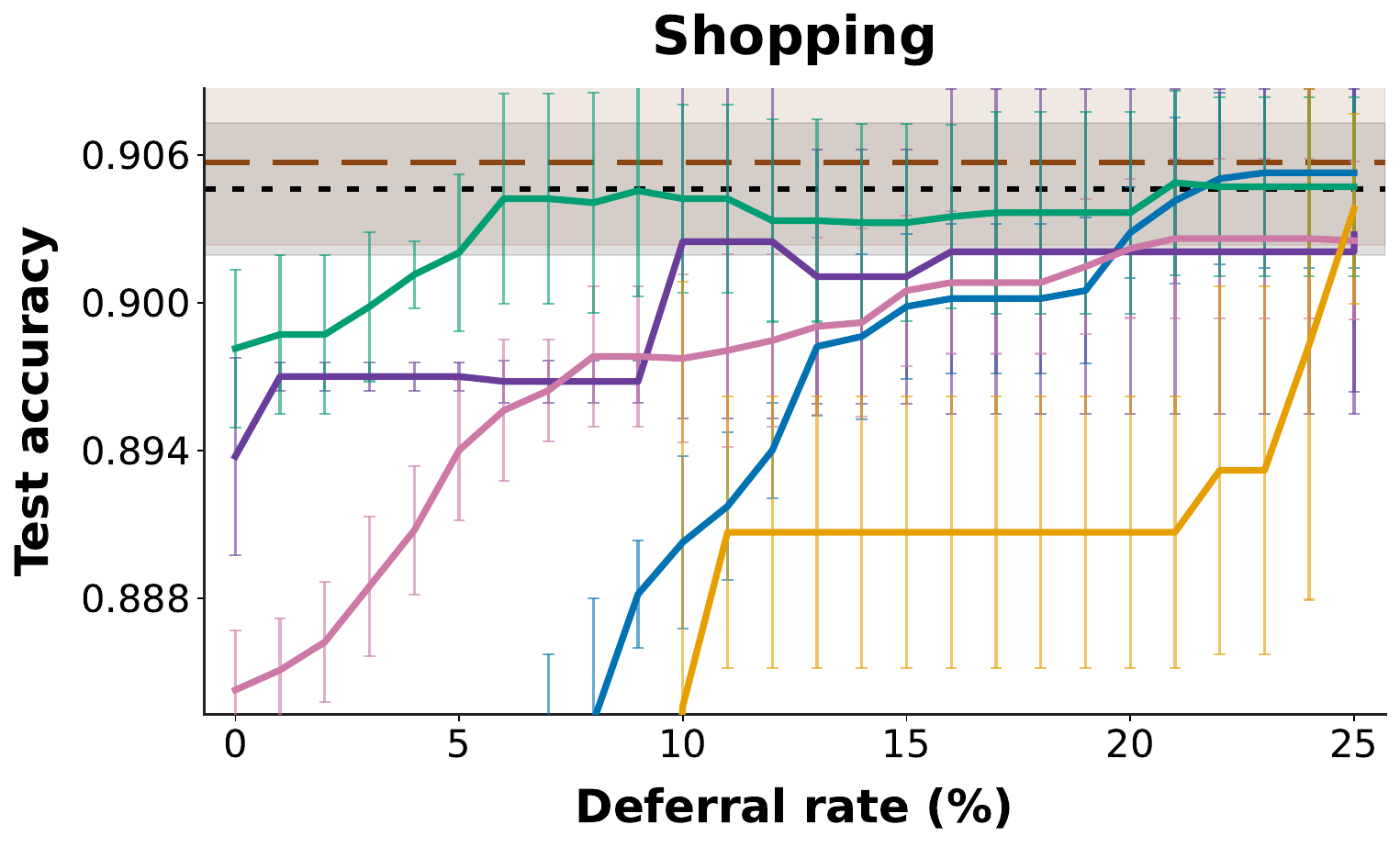}
    \end{subfigure}

    \vspace{-0.05em}

    \begin{subfigure}{0.49\linewidth}
        \includegraphics[width=\linewidth]{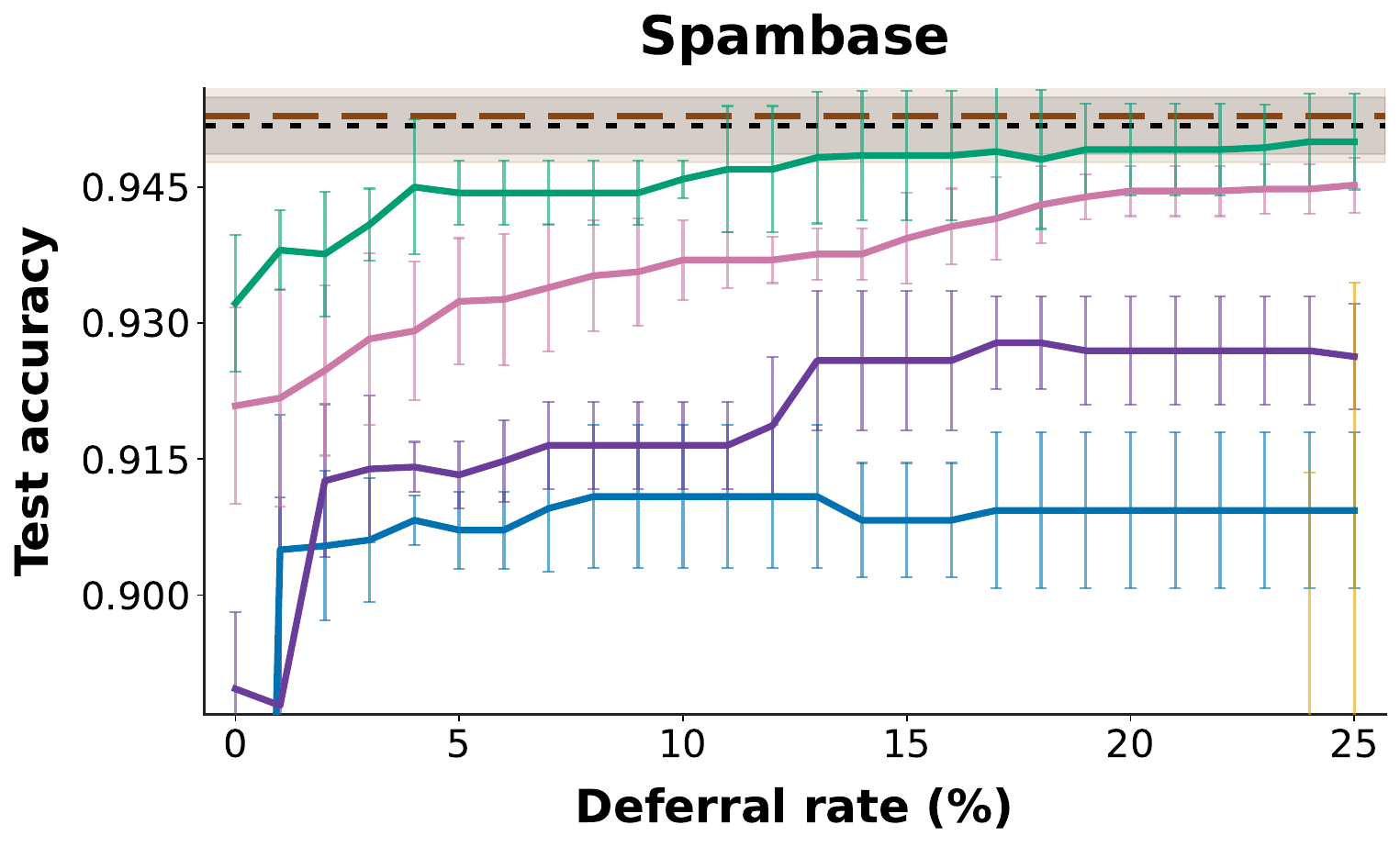}
    \end{subfigure}
    \hfill
    \begin{subfigure}{0.49\linewidth}
        \includegraphics[width=\linewidth]{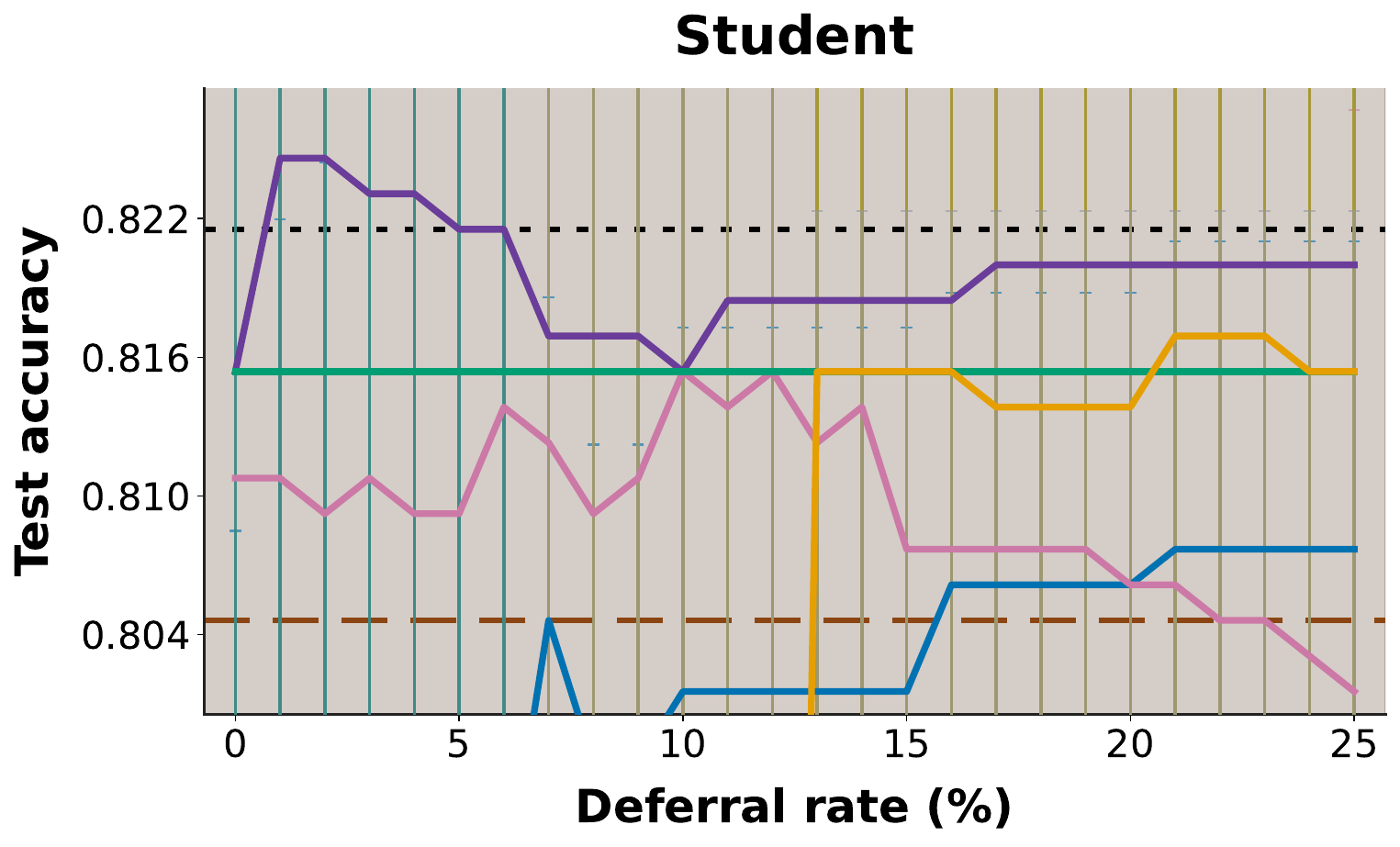}
    \end{subfigure}

    \vspace{-0.05em}

    \begin{subfigure}{0.49\linewidth}
        \includegraphics[width=\linewidth]{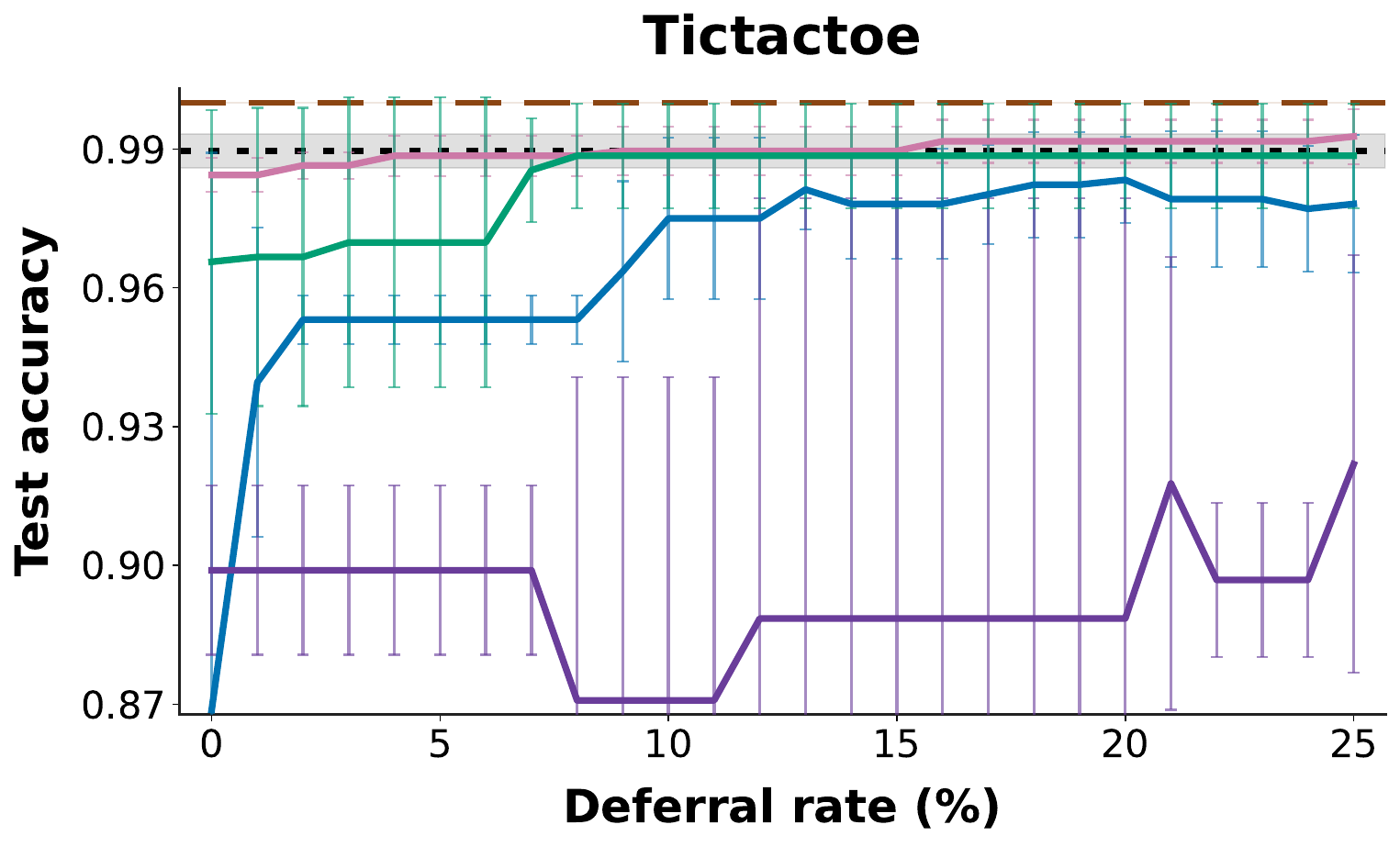}
    \end{subfigure}
    \hfill
    \begin{subfigure}{0.49\linewidth}
        \includegraphics[width=\linewidth]{new_main_figs/wine_main.pdf}
    \end{subfigure}

    \vspace{-0.3em}
    \caption{Deferral-accuracy trade-offs for all datasets, continued.}
    \label{fig:main_dataset_curves_part2}
\end{figure}

\FloatBarrier

\begin{table}[t]
    \centering
    \small
    \setlength{\tabcolsep}{5pt}
    \caption{Comparison of multistage defer trees with XGBoost fallback (MDT+XGB), constrained to at most 25\% deferral, against AdaBoost. Entries report mean test classification accuracy $\pm$ standard error over five train/test splits (higher is better).}
    \label{tab:mdt-adaboost}
    \begin{tabular}{lcc}
        \toprule
        Dataset & MDT+XGB ($\leq$25\%) & AdaBoost \\
        \midrule
        Abalone & $0.6333 \pm 0.0069$ & $0.6326 \pm 0.0026$ \\
        Adult & $0.8667 \pm 0.0021$ & $0.8684 \pm 0.0016$ \\
        Aging & $0.8028 \pm 0.0034$ & $0.8000 \pm 0.0028$ \\
        Bank & $0.9087 \pm 0.0008$ & $0.9092 \pm 0.0009$ \\
        Bike & $0.9486 \pm 0.0024$ & $0.9461 \pm 0.0020$ \\
        California & $0.9009 \pm 0.0020$ & $0.9088 \pm 0.0018$ \\
        Churn & $0.9566 \pm 0.0028$ & $0.9548 \pm 0.0021$ \\
        Droid & $0.9709 \pm 0.0005$ & $0.9711 \pm 0.0004$ \\
        Heloc & $0.7056 \pm 0.0107$ & $0.7124 \pm 0.0115$ \\
        Jasmine & $0.8161 \pm 0.0100$ & $0.8241 \pm 0.0089$ \\
        Phishing & $0.9728 \pm 0.0017$ & $0.9731 \pm 0.0016$ \\
        Pol & $0.9849 \pm 0.0003$ & $0.9843 \pm 0.0010$ \\
        RL & $0.7968 \pm 0.0089$ & $0.8052 \pm 0.0071$ \\
        Shopping & $0.9047 \pm 0.0016$ & $0.9062 \pm 0.0017$ \\
        Spambase & $0.9500 \pm 0.0024$ & $0.9524 \pm 0.0024$ \\
        Wine & $0.8788 \pm 0.0052$ & $0.8870 \pm 0.0035$ \\
        \bottomrule
    \end{tabular}
\end{table}

\FloatBarrier
\subsection{Ablations}

\clearpage
\thispagestyle{empty}

\begin{figure}[p]
\centering
\vspace*{-0.8cm}

\includegraphics[width=0.49\textwidth,height=0.24\textheight,keepaspectratio]{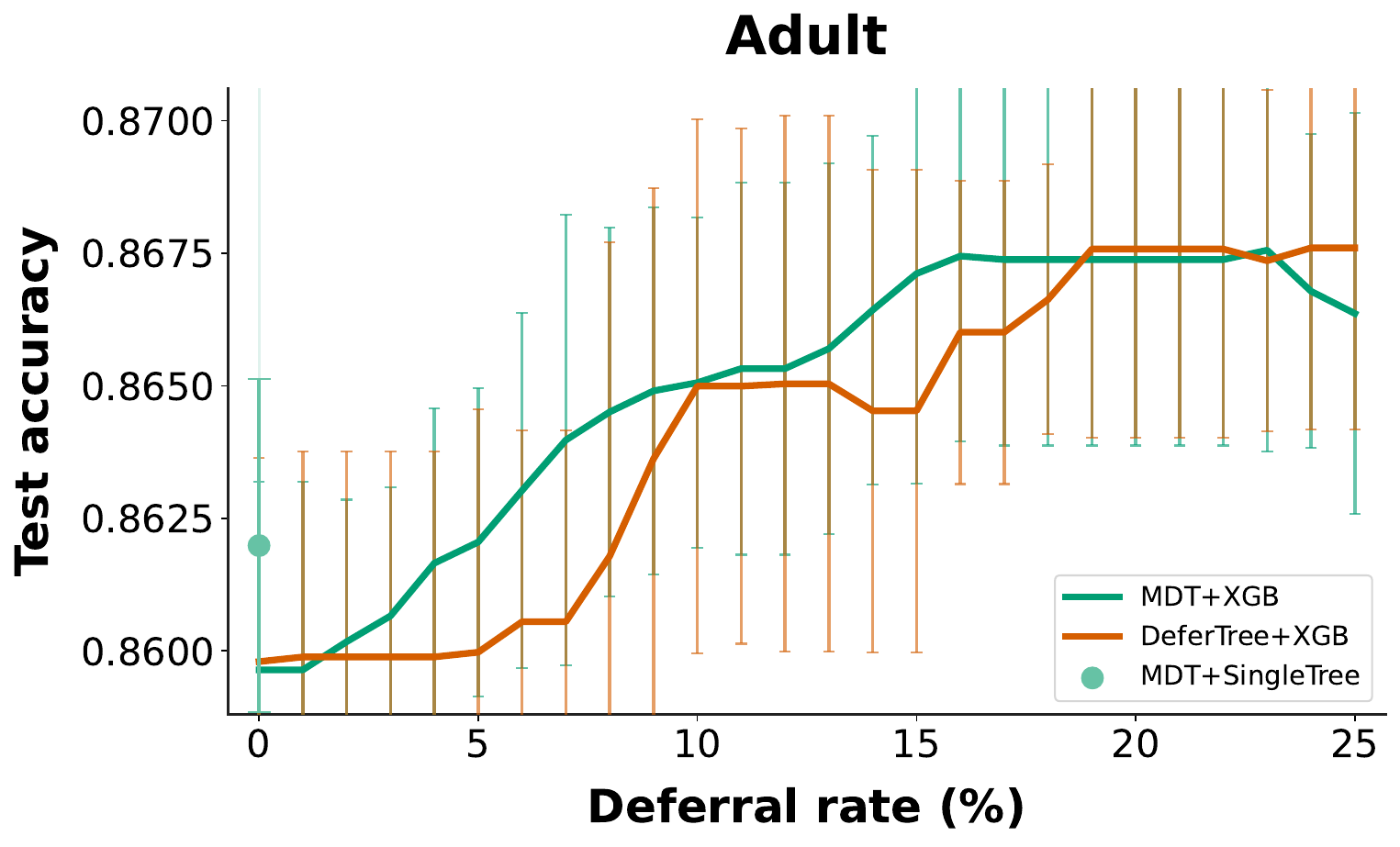}
\hfill
\includegraphics[width=0.49\textwidth,height=0.24\textheight,keepaspectratio]{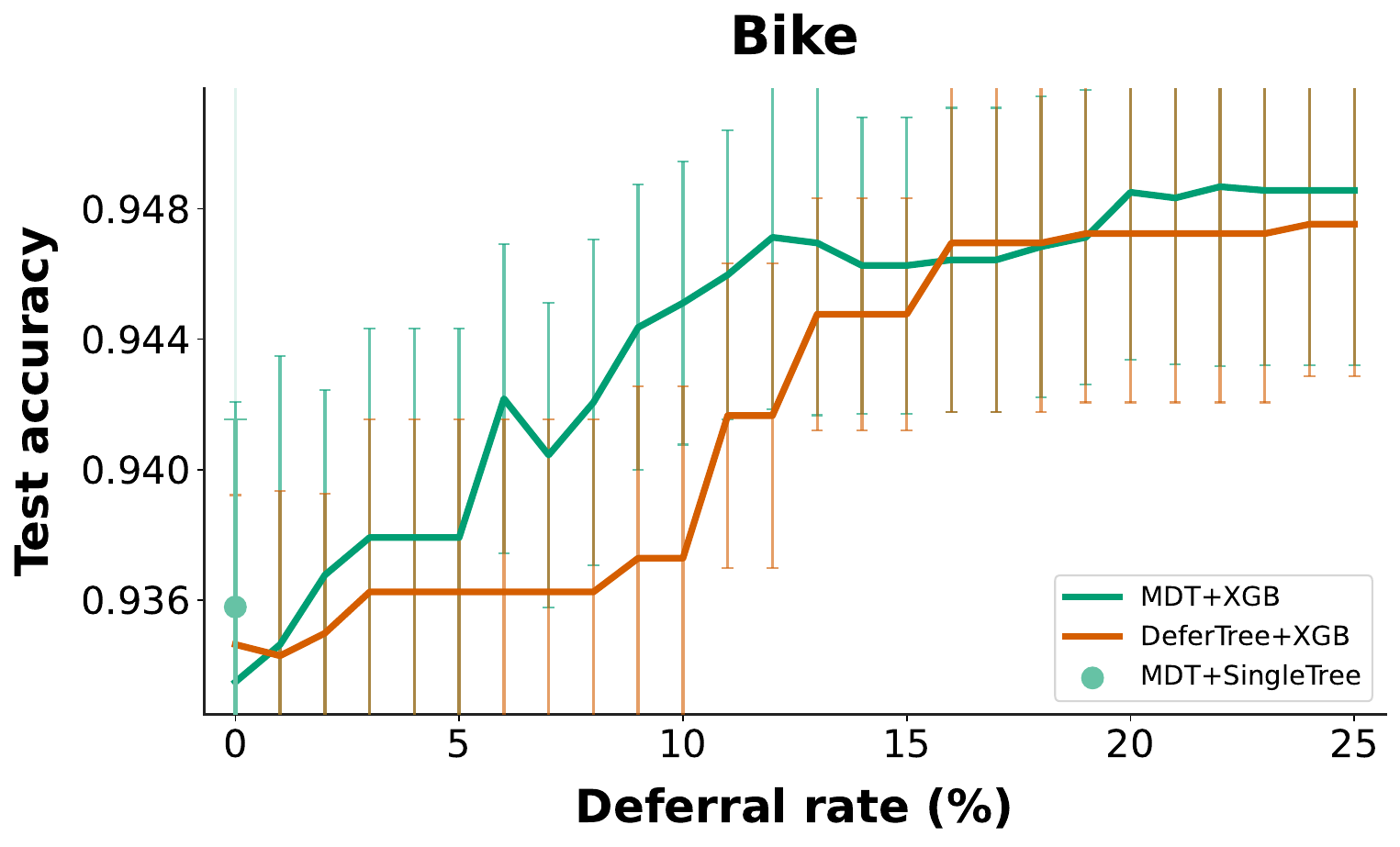}

\vspace{0.3em}

\includegraphics[width=0.49\textwidth,height=0.24\textheight,keepaspectratio]{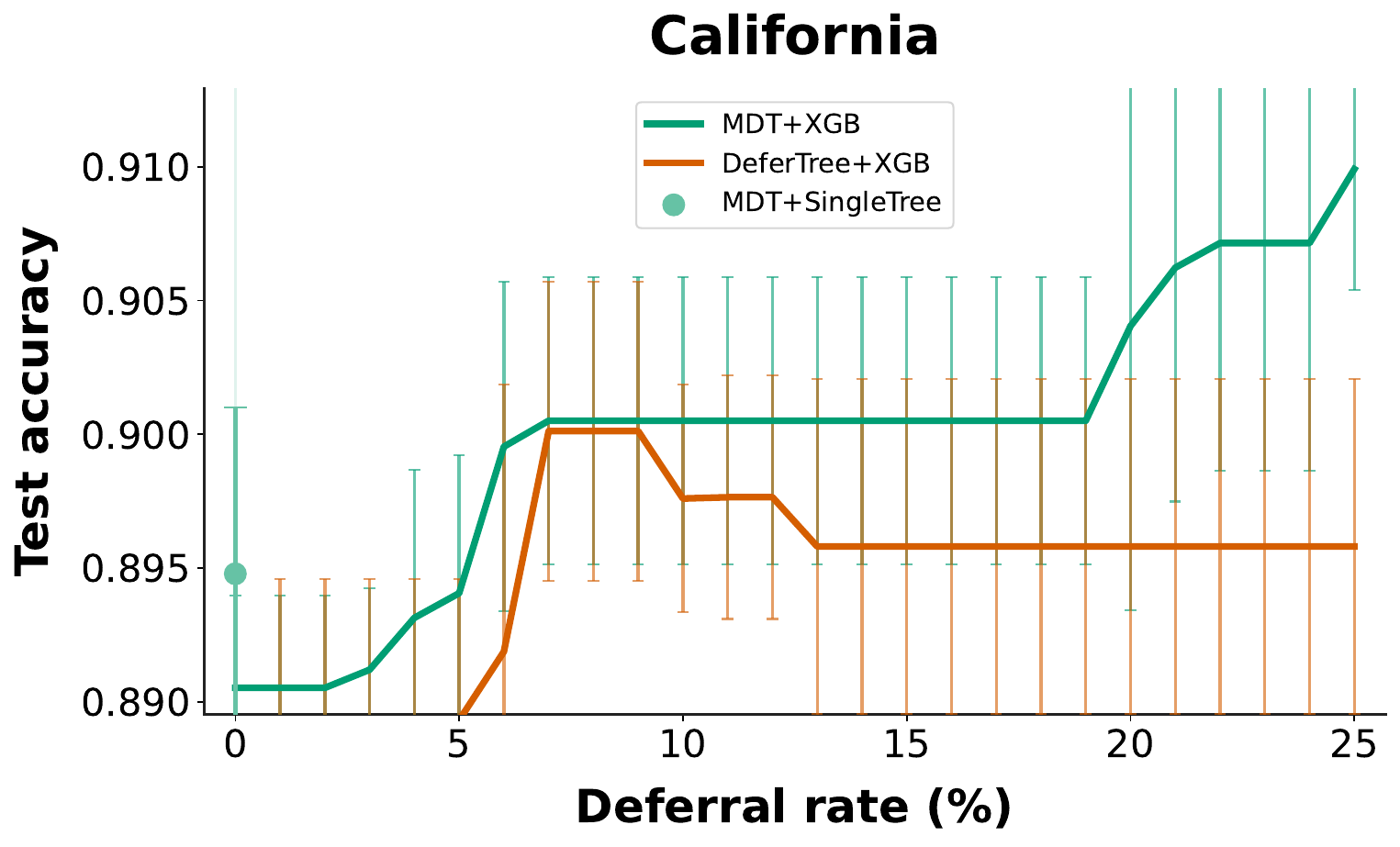}
\hfill
\includegraphics[width=0.49\textwidth,height=0.24\textheight,keepaspectratio]{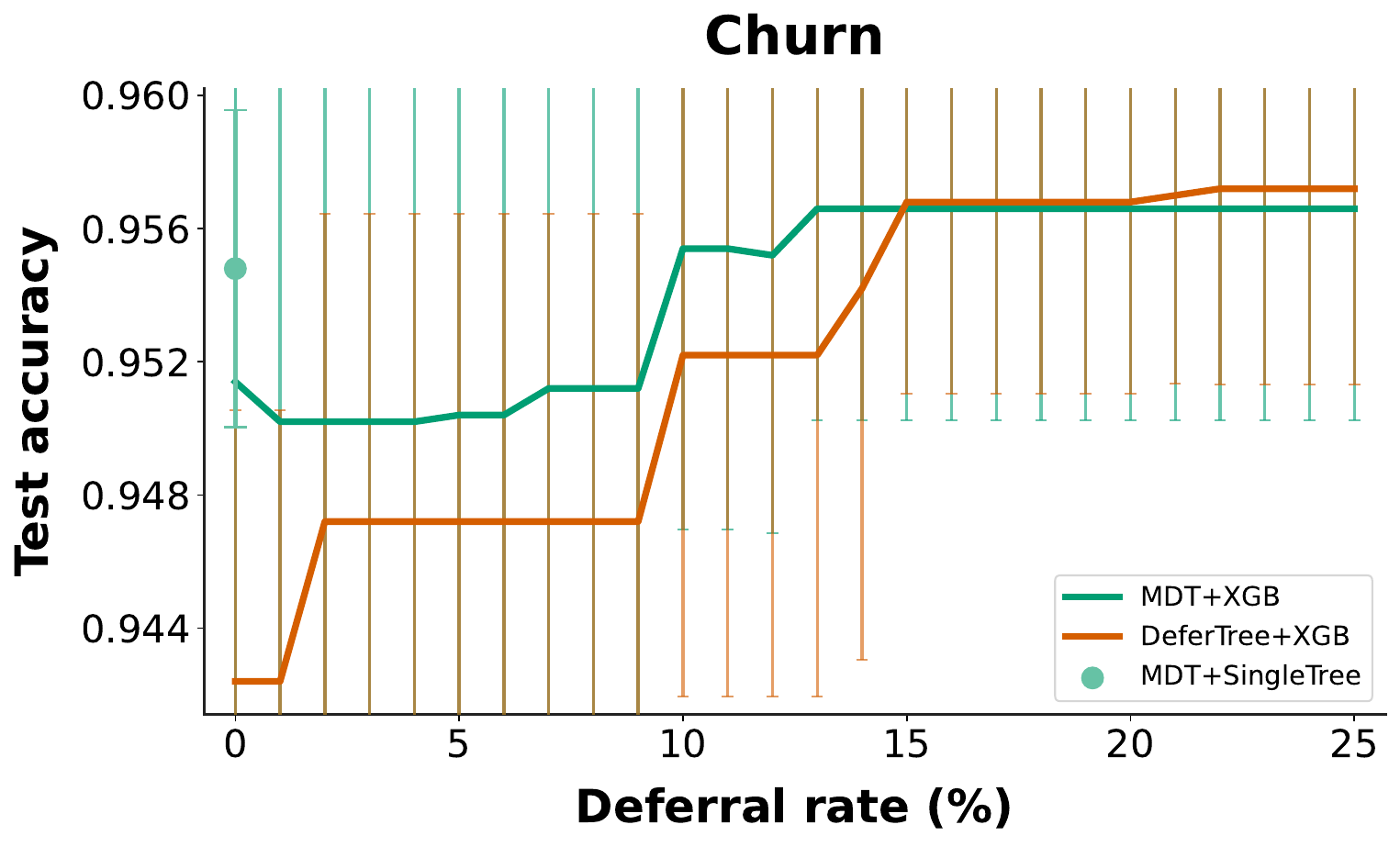}

\vspace{0.3em}

\includegraphics[width=0.49\textwidth,height=0.24\textheight,keepaspectratio]{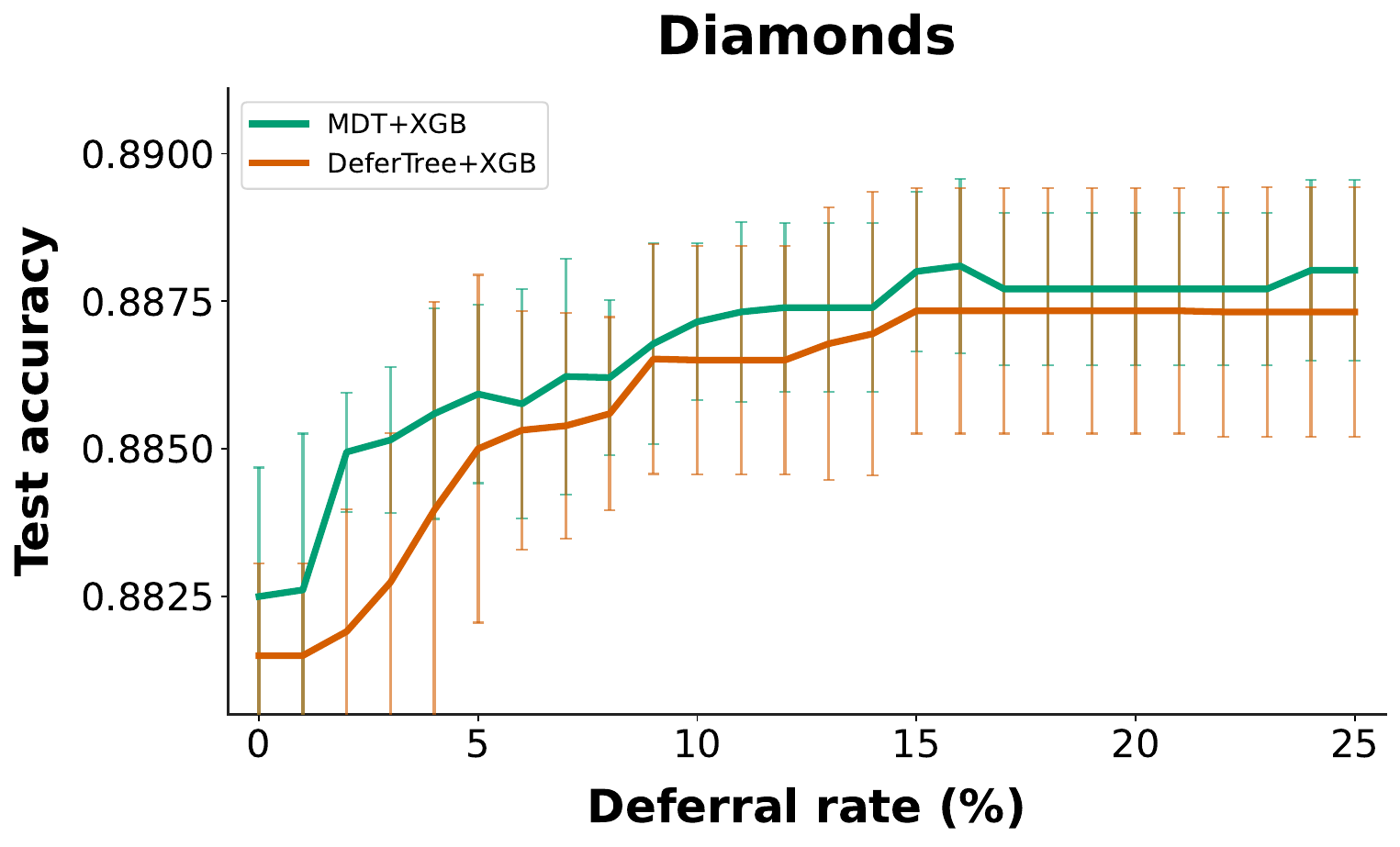}
\hfill
\includegraphics[width=0.49\textwidth,height=0.24\textheight,keepaspectratio]{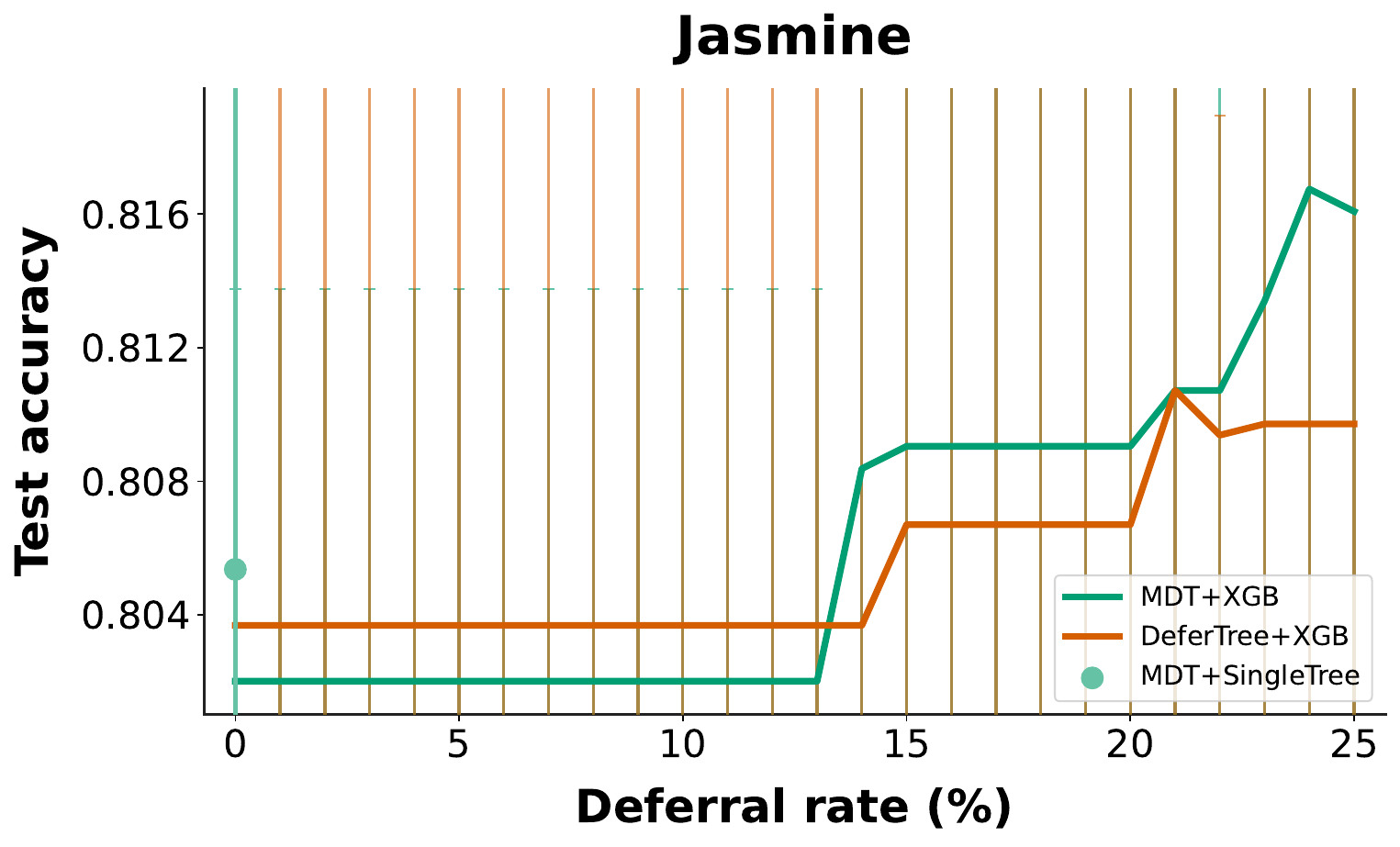}

\vspace{-0.2em}

\caption{Results of our MDT algorithm compared to just using our DeferTree algorithm. Here, we also show results for MDT+SingleTree, which uses a single decision tree algorithm as fallback instead of XGBoost. When the deferral rate is 0, DeferTree+XGB turns into a near-optimal decision tree algorithm. We see that MDT+SingleTree is able to provide a much better, fully interpretable model than a near-optimal tree in most cases.}
\label{fig:appendix-dt-1}
\end{figure}

\clearpage
\thispagestyle{empty}

\begin{figure}[p]
\centering
\vspace*{-0.8cm}

\includegraphics[width=0.49\textwidth,height=0.20\textheight,keepaspectratio]{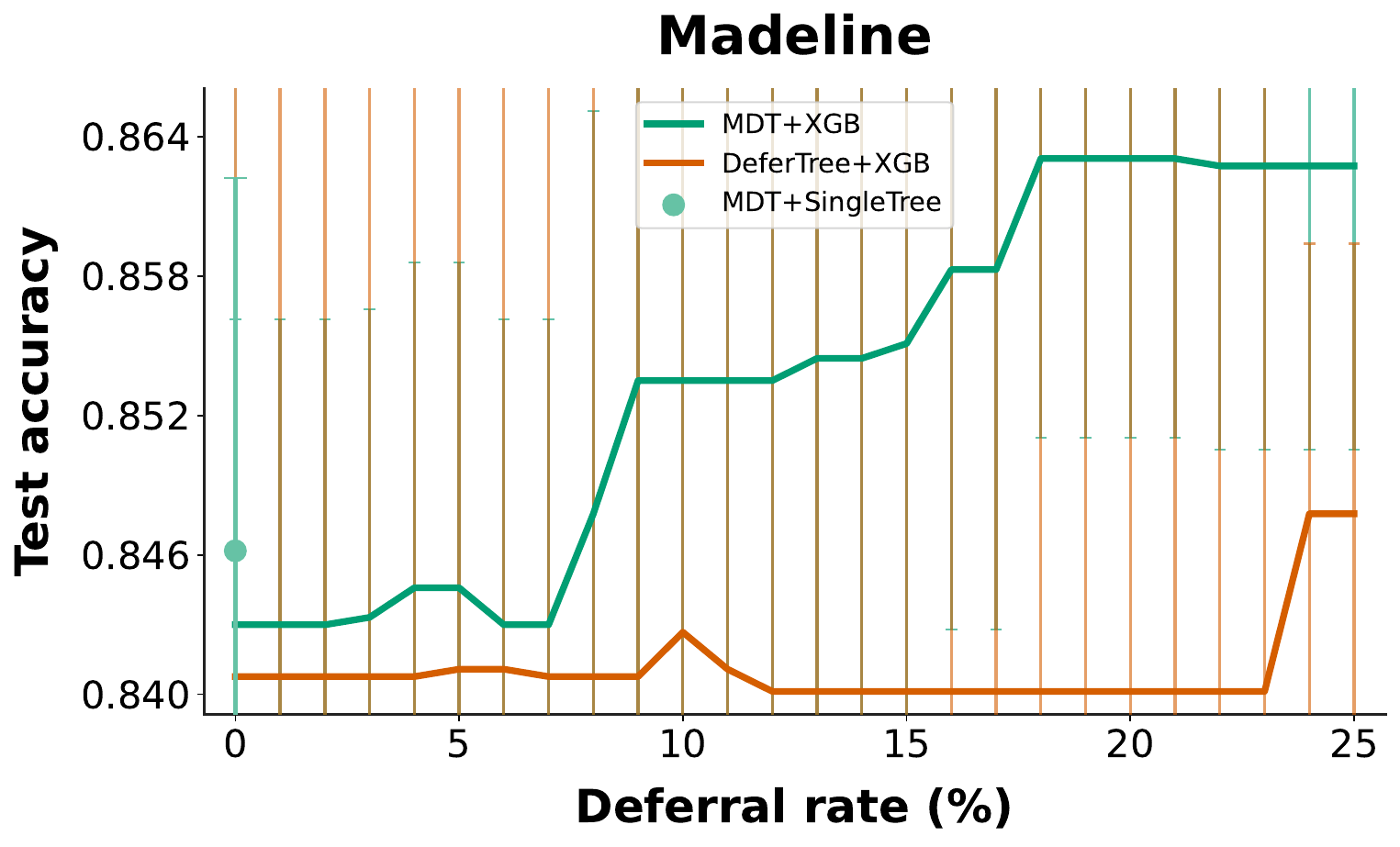}
\hfill
\includegraphics[width=0.49\textwidth,height=0.20\textheight,keepaspectratio]{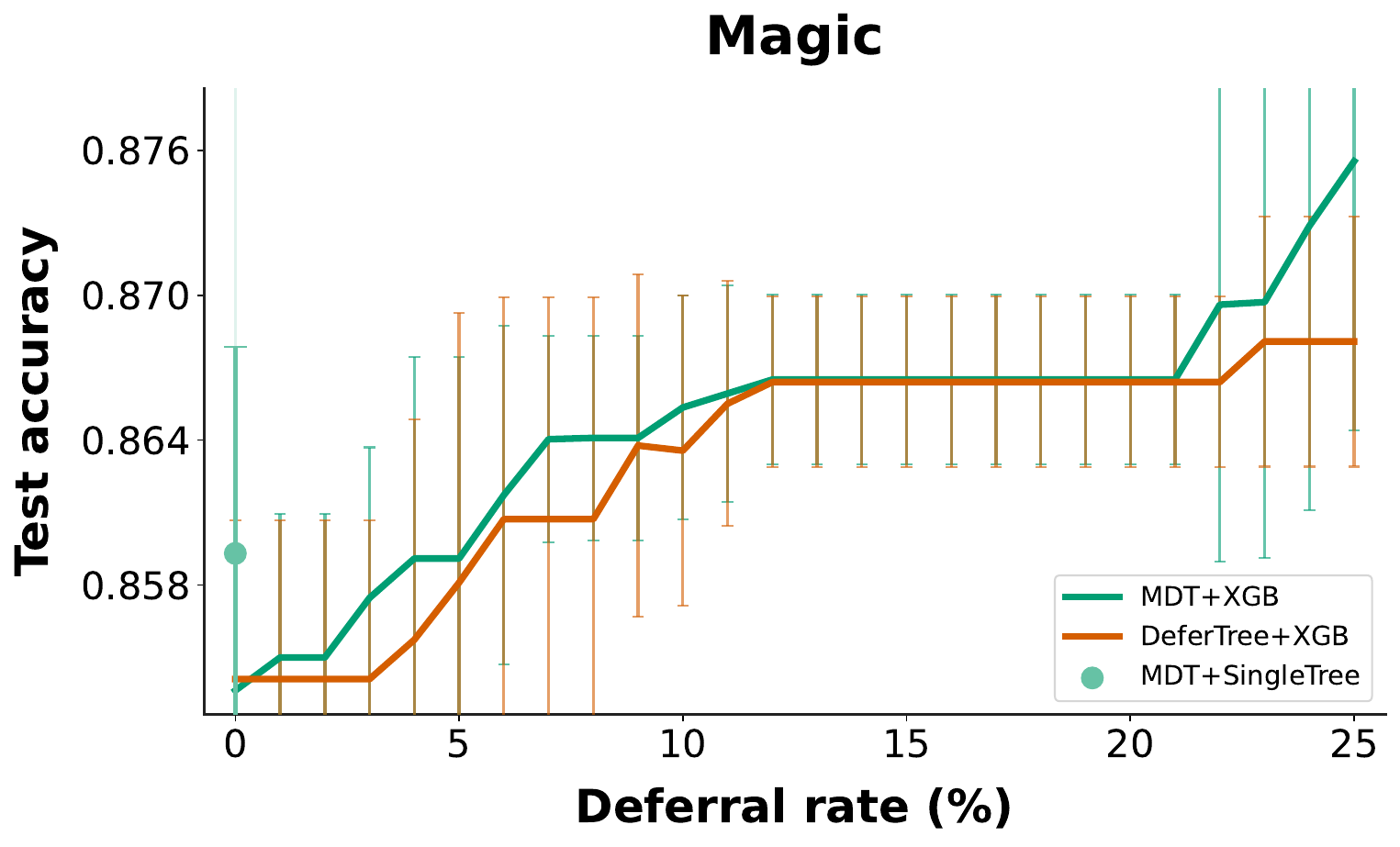}

\vspace{0.3em}

\includegraphics[width=0.49\textwidth,height=0.20\textheight,keepaspectratio]{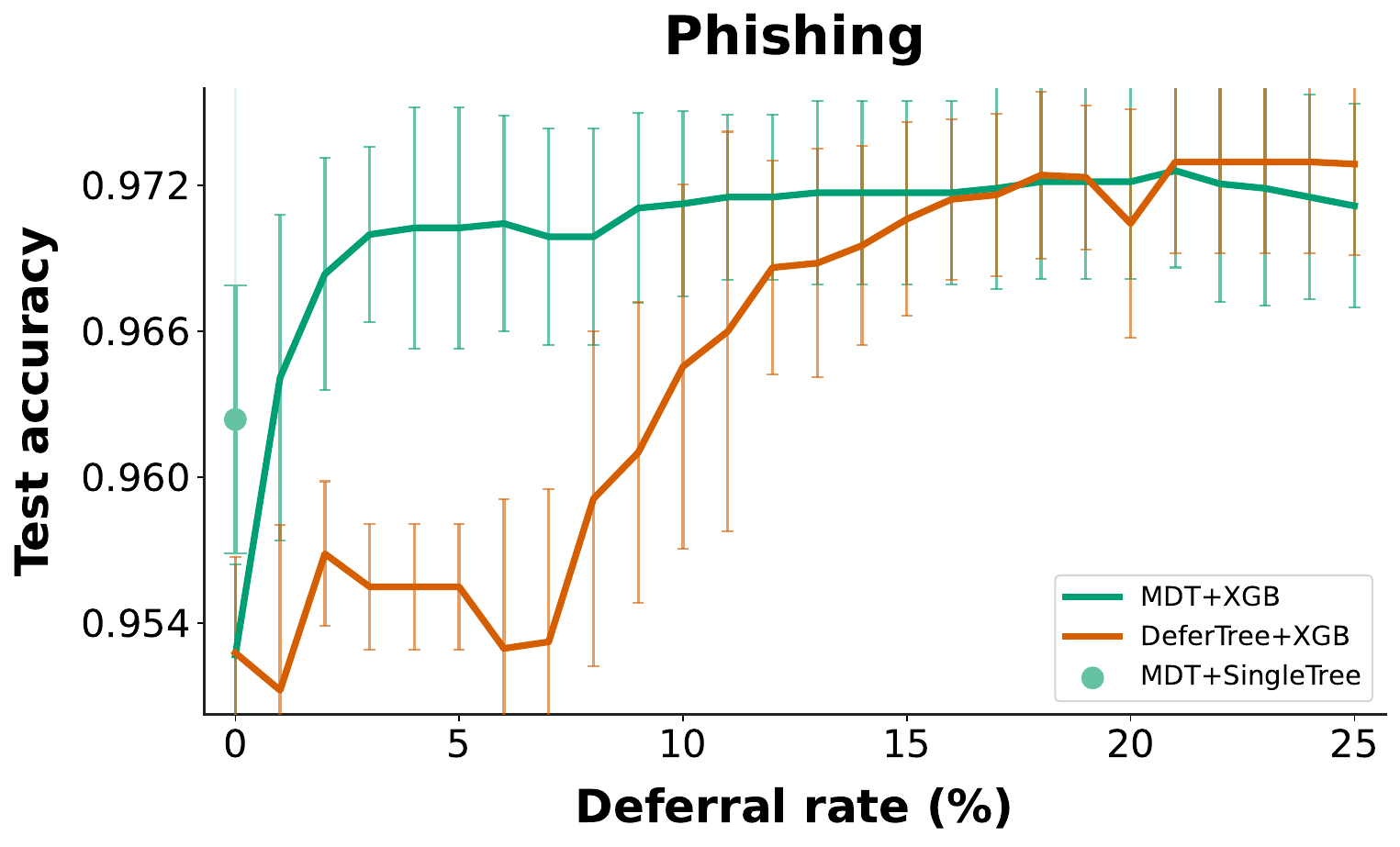}
\hfill
\includegraphics[width=0.49\textwidth,height=0.20\textheight,keepaspectratio]{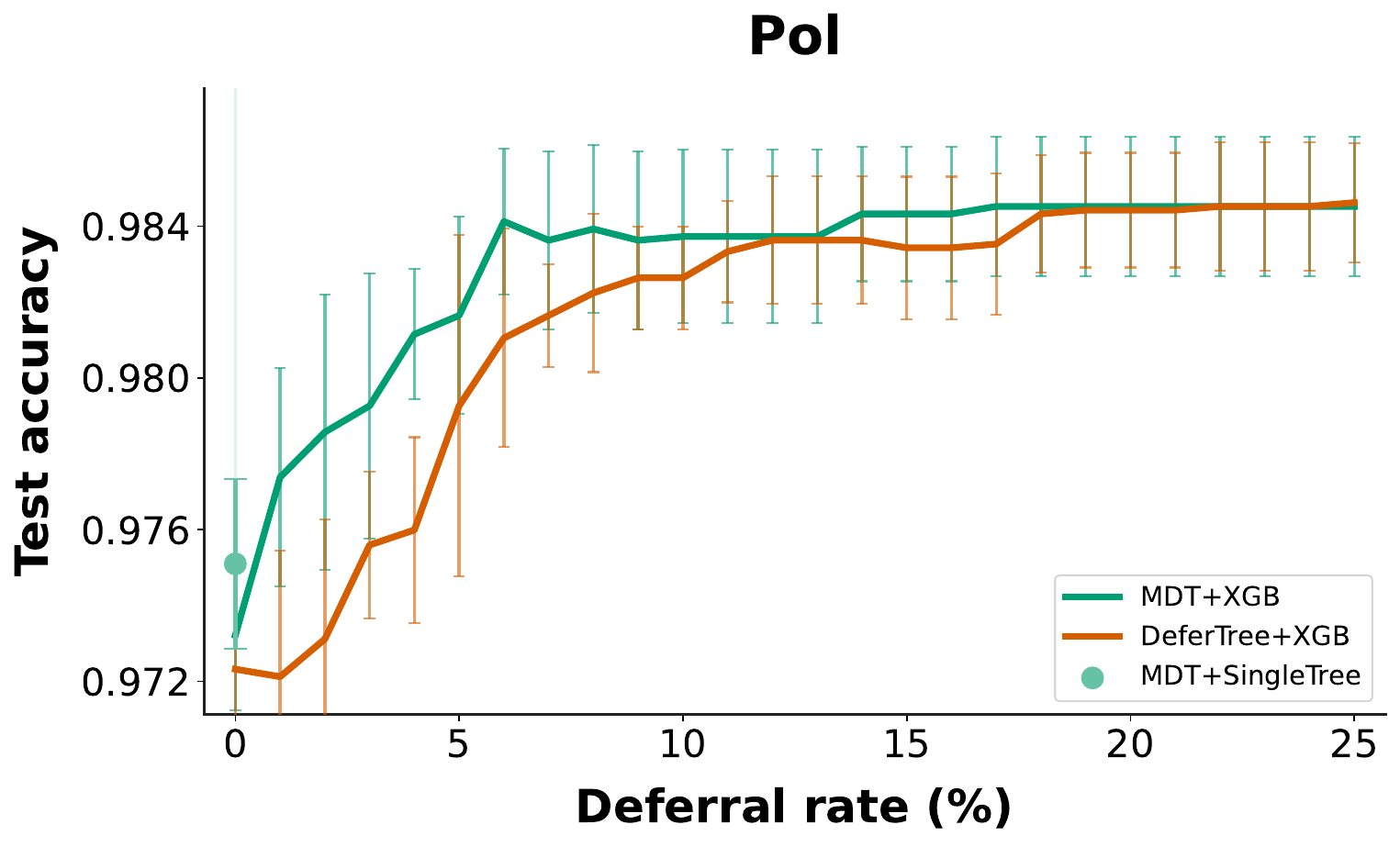}

\vspace{0.3em}

\includegraphics[width=0.49\textwidth,height=0.20\textheight,keepaspectratio]{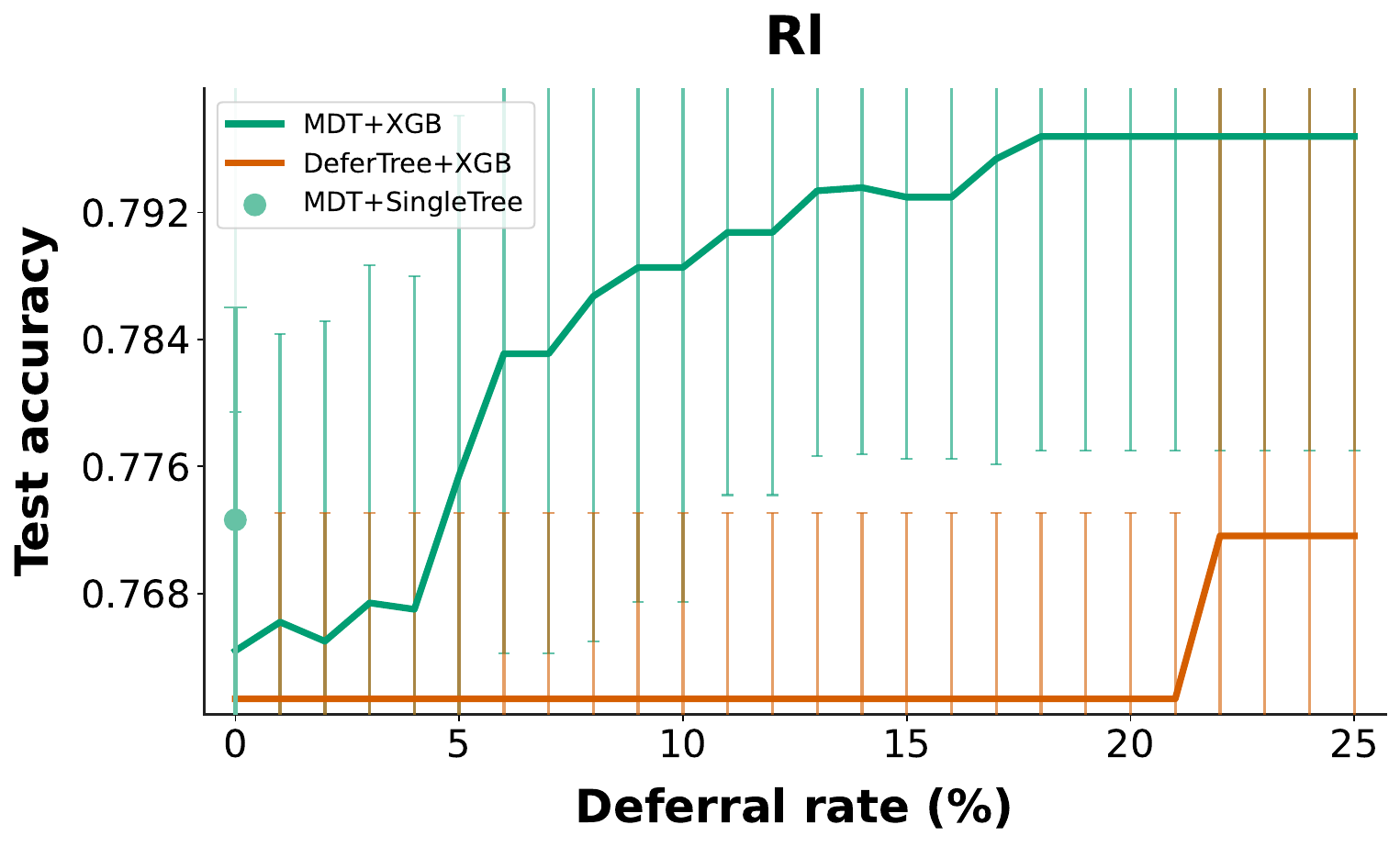}
\hfill
\includegraphics[width=0.49\textwidth,height=0.20\textheight,keepaspectratio]{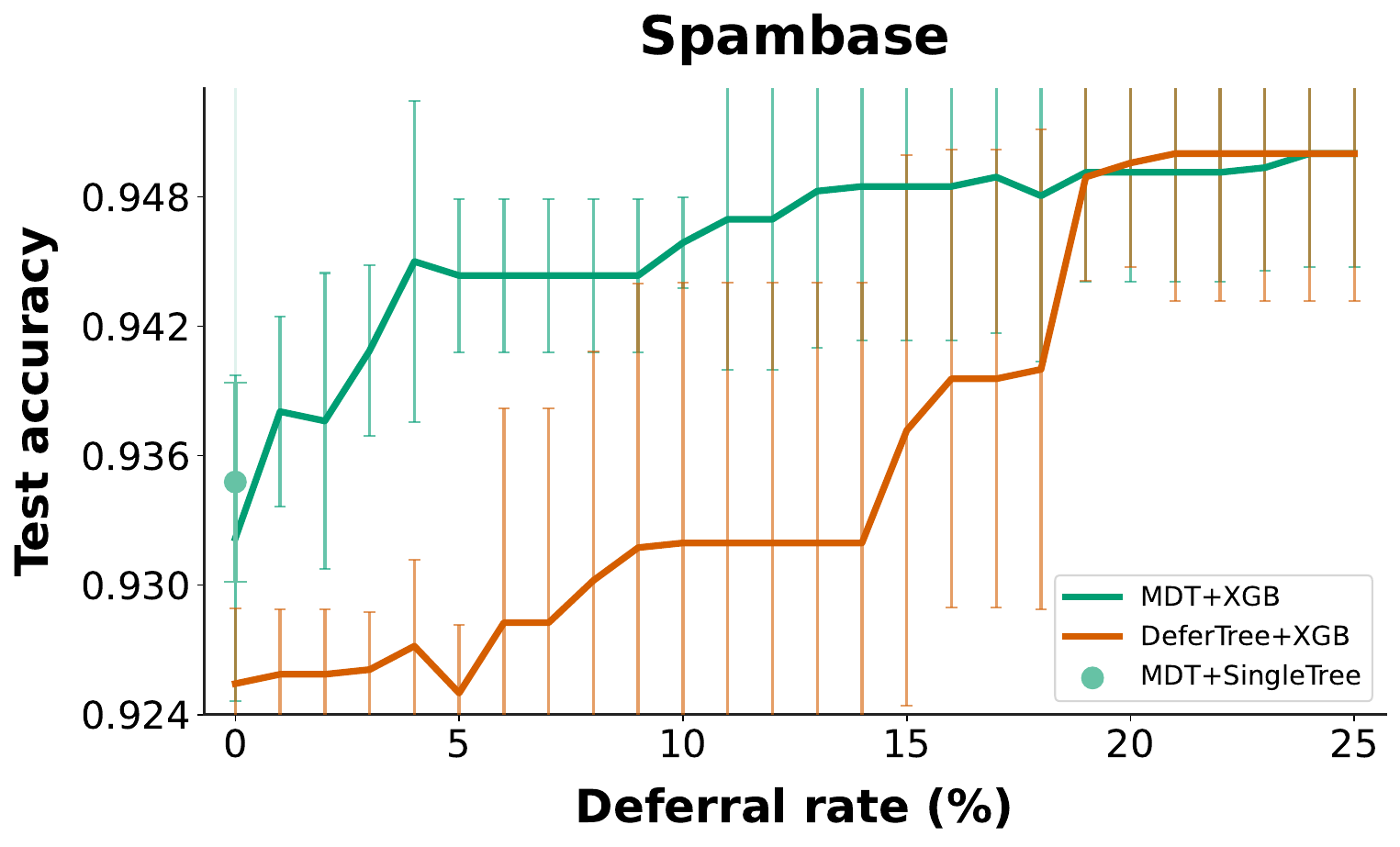}

\vspace{0.3em}

\includegraphics[width=0.62\textwidth,height=0.20\textheight,keepaspectratio]{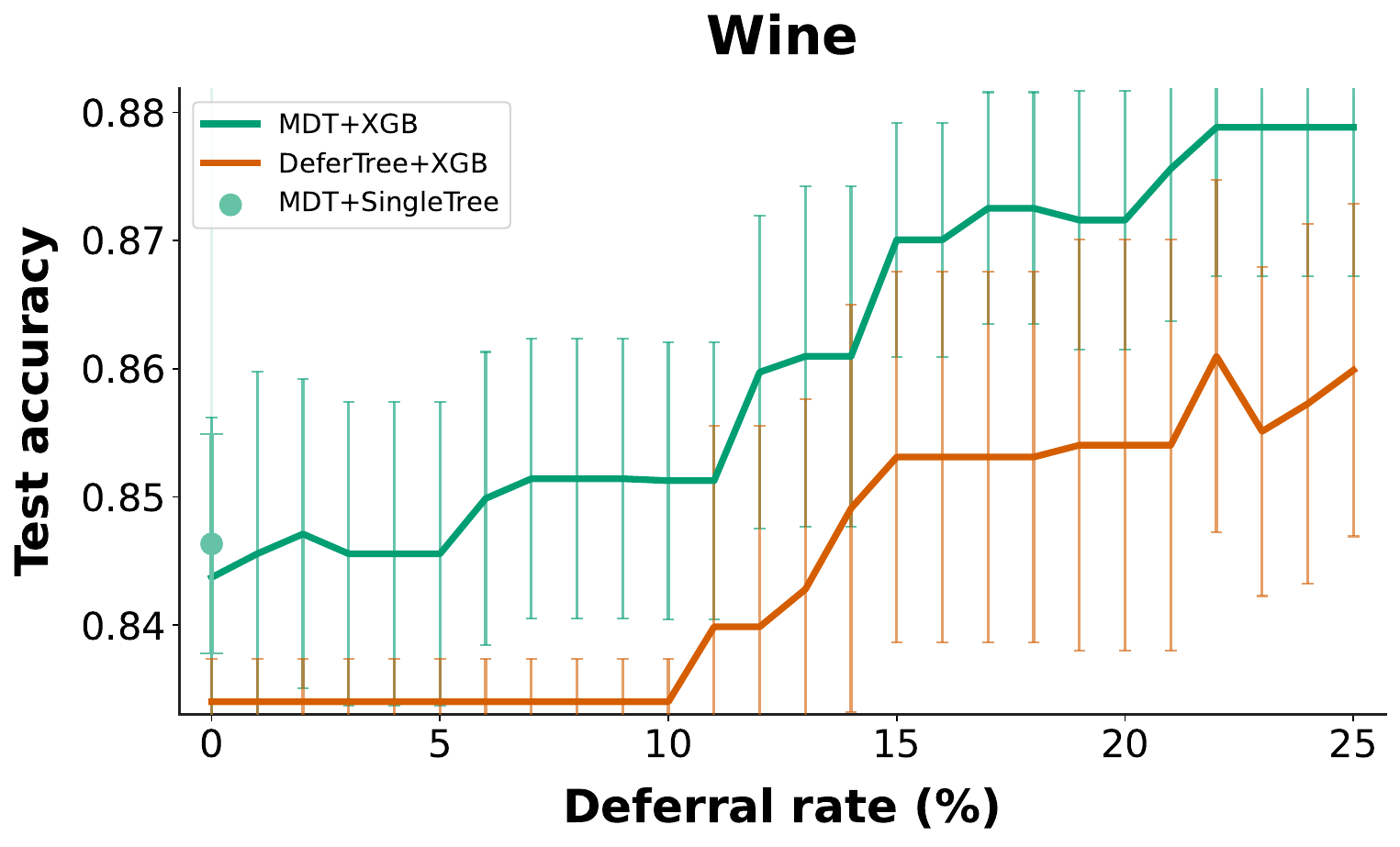}

\vspace{-0.2em}

\caption{Results of our MDT algorithm compared to just using our DeferTree algorithm (continued).}
\label{fig:appendix-dt-2}
\end{figure}

\clearpage

\clearpage
\thispagestyle{empty}

\begin{figure}[p]
\centering
\vspace*{-0.8cm}

\includegraphics[width=0.49\textwidth,height=0.21\textheight,keepaspectratio]{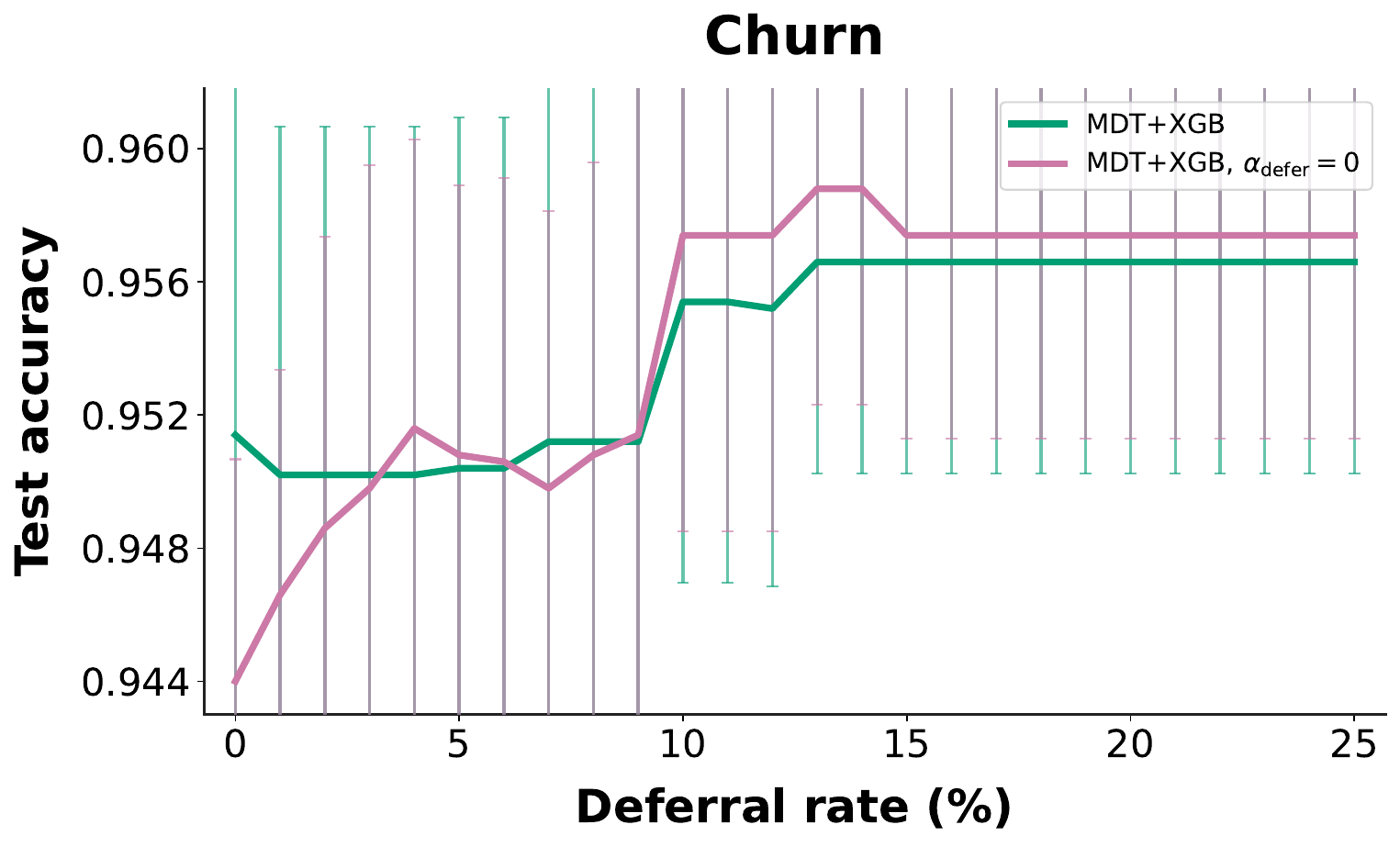}
\hfill
\includegraphics[width=0.49\textwidth,height=0.21\textheight,keepaspectratio]{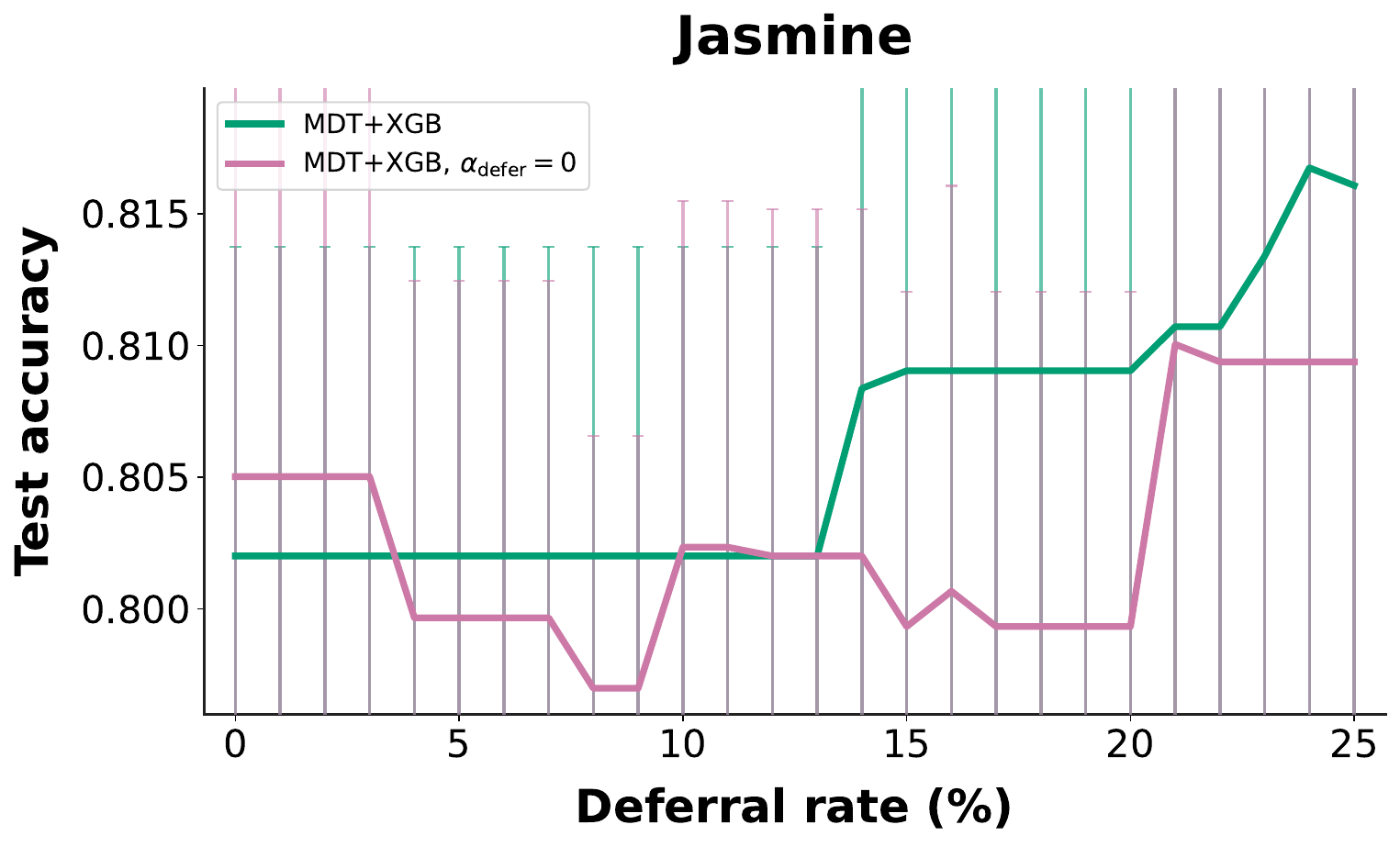}

\vspace{0.3em}

\includegraphics[width=0.49\textwidth,height=0.21\textheight,keepaspectratio]{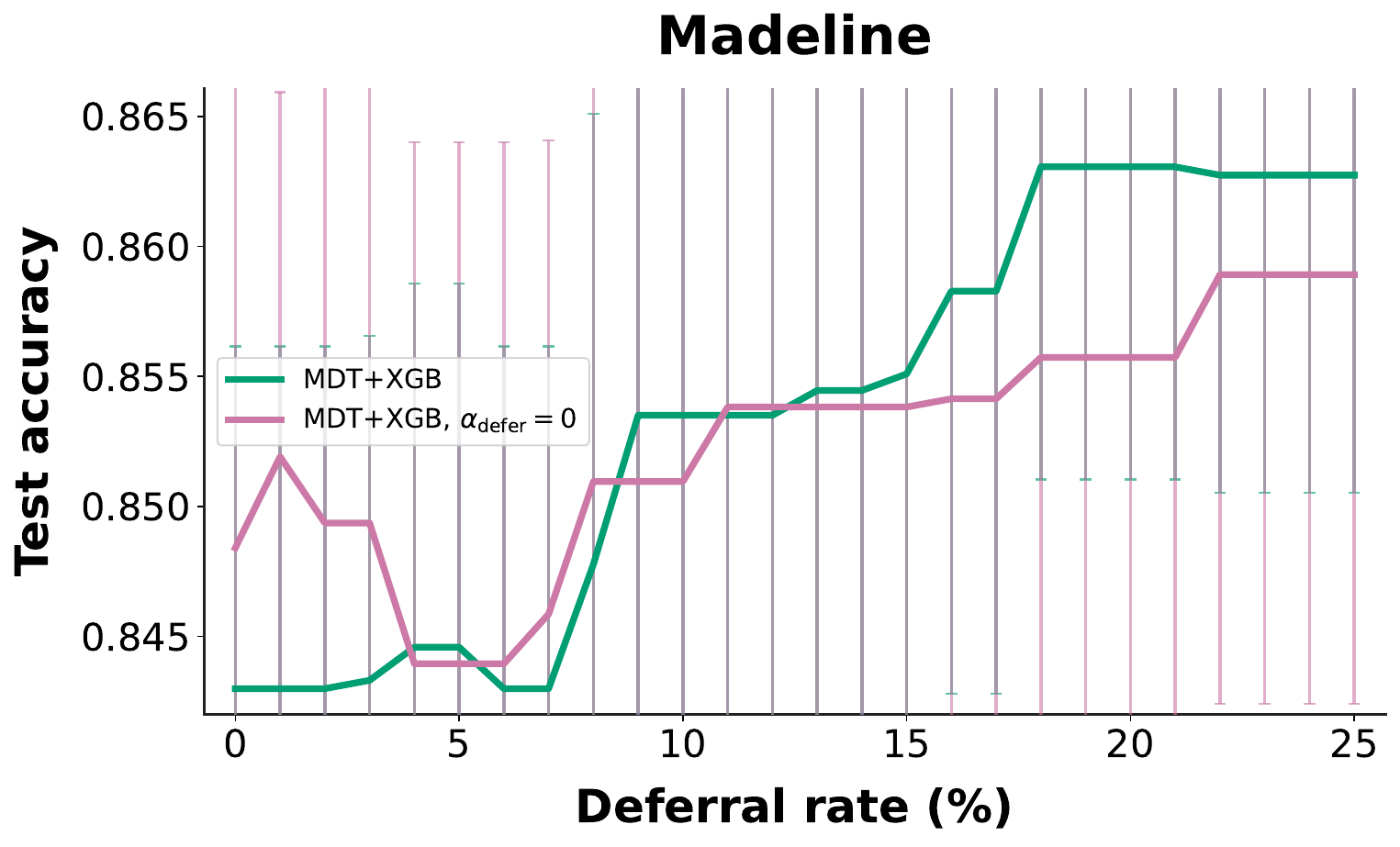}
\hfill
\includegraphics[width=0.49\textwidth,height=0.21\textheight,keepaspectratio]{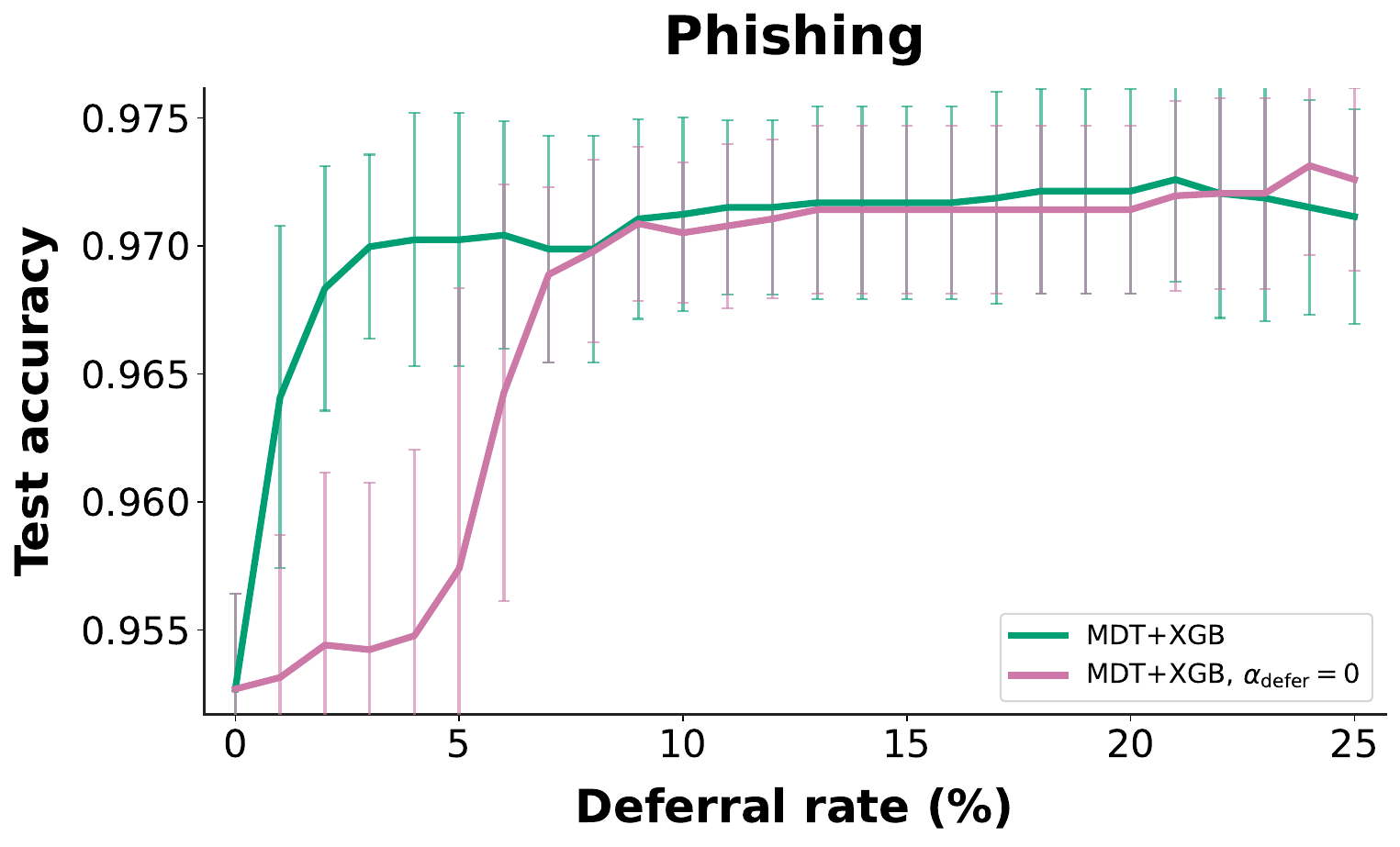}

\vspace{0.3em}

\includegraphics[width=0.49\textwidth,height=0.21\textheight,keepaspectratio]{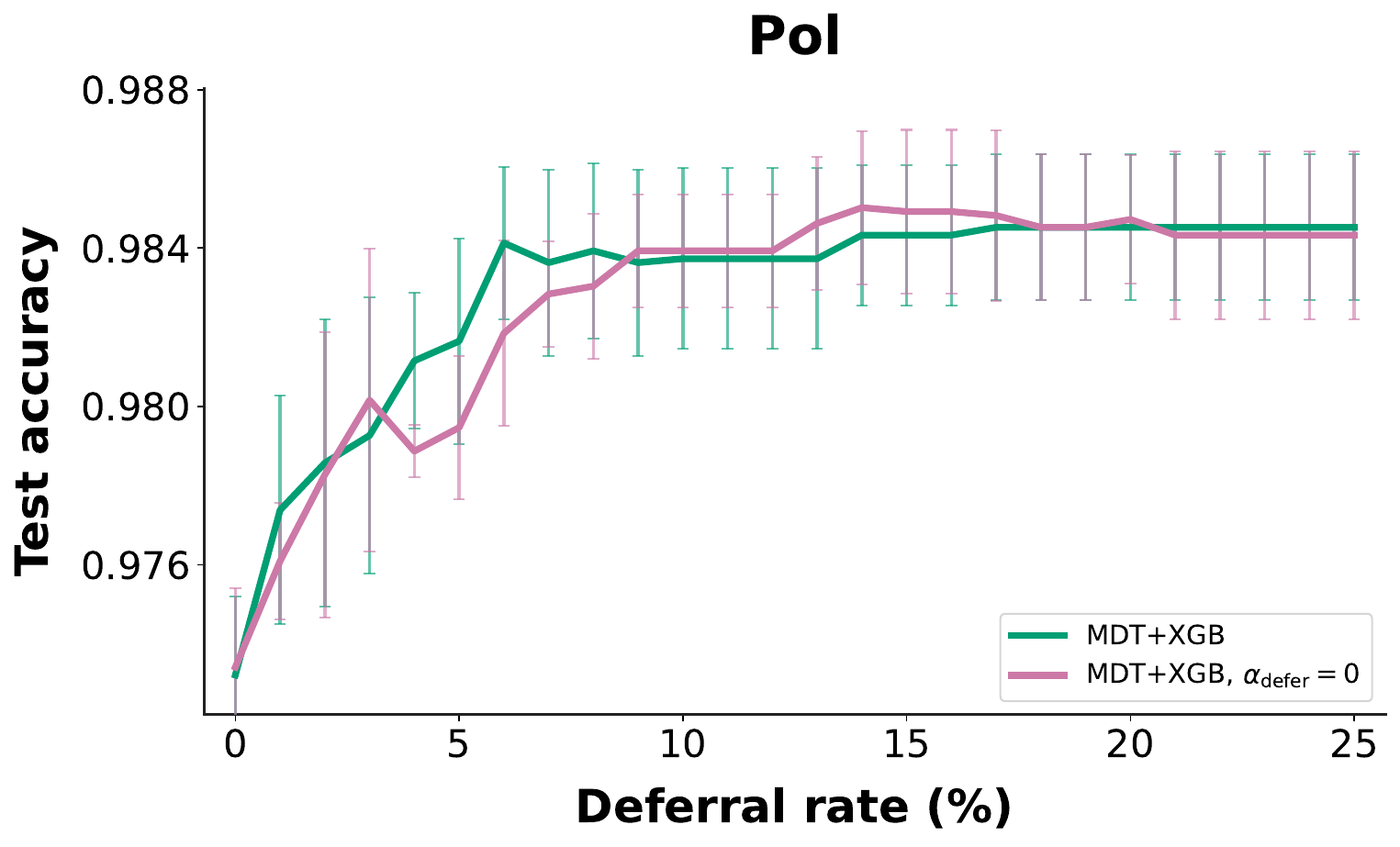}
\hfill
\includegraphics[width=0.49\textwidth,height=0.21\textheight,keepaspectratio]{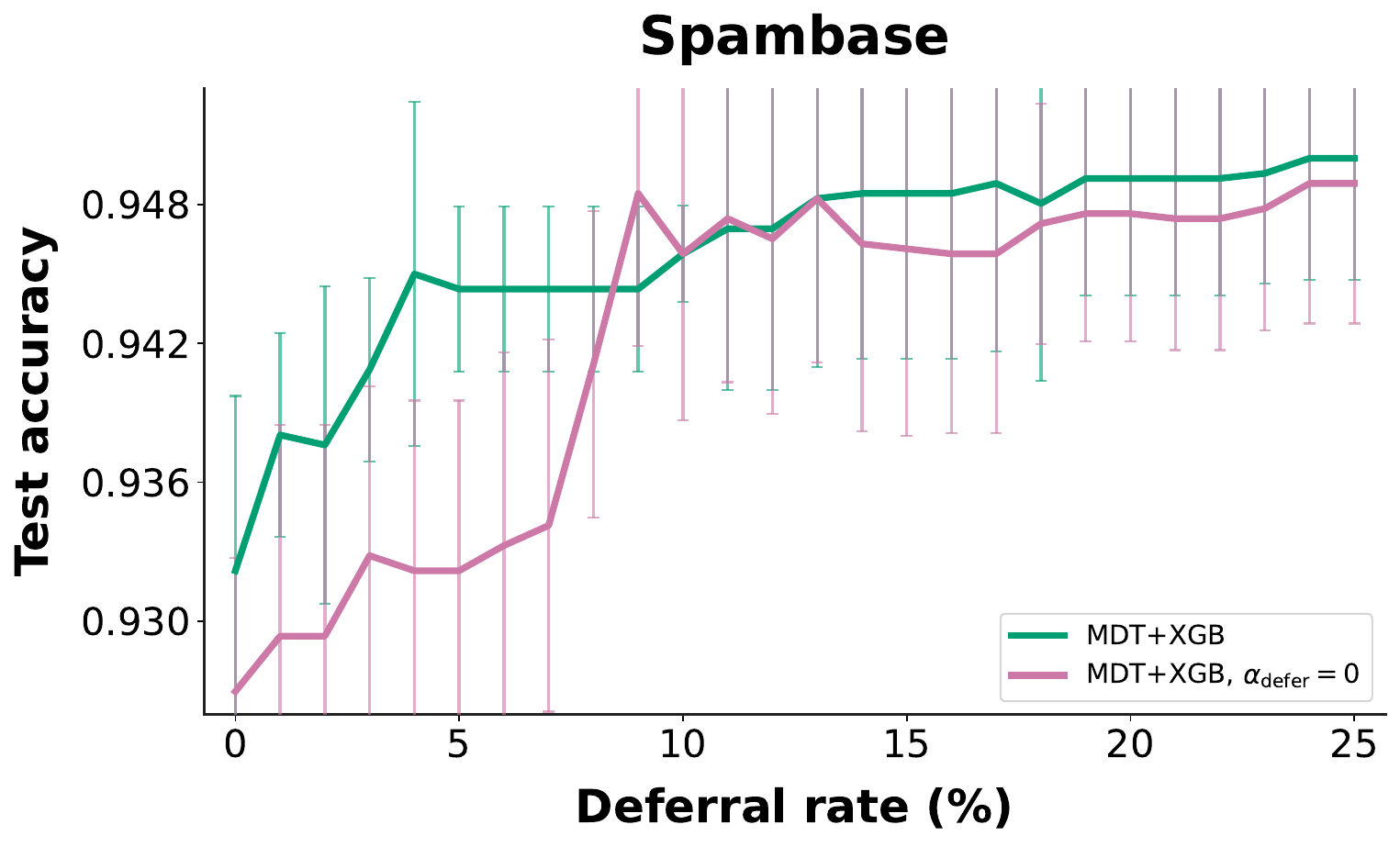}

\vspace{0.3em}

\includegraphics[width=0.62\textwidth,height=0.21\textheight,keepaspectratio]{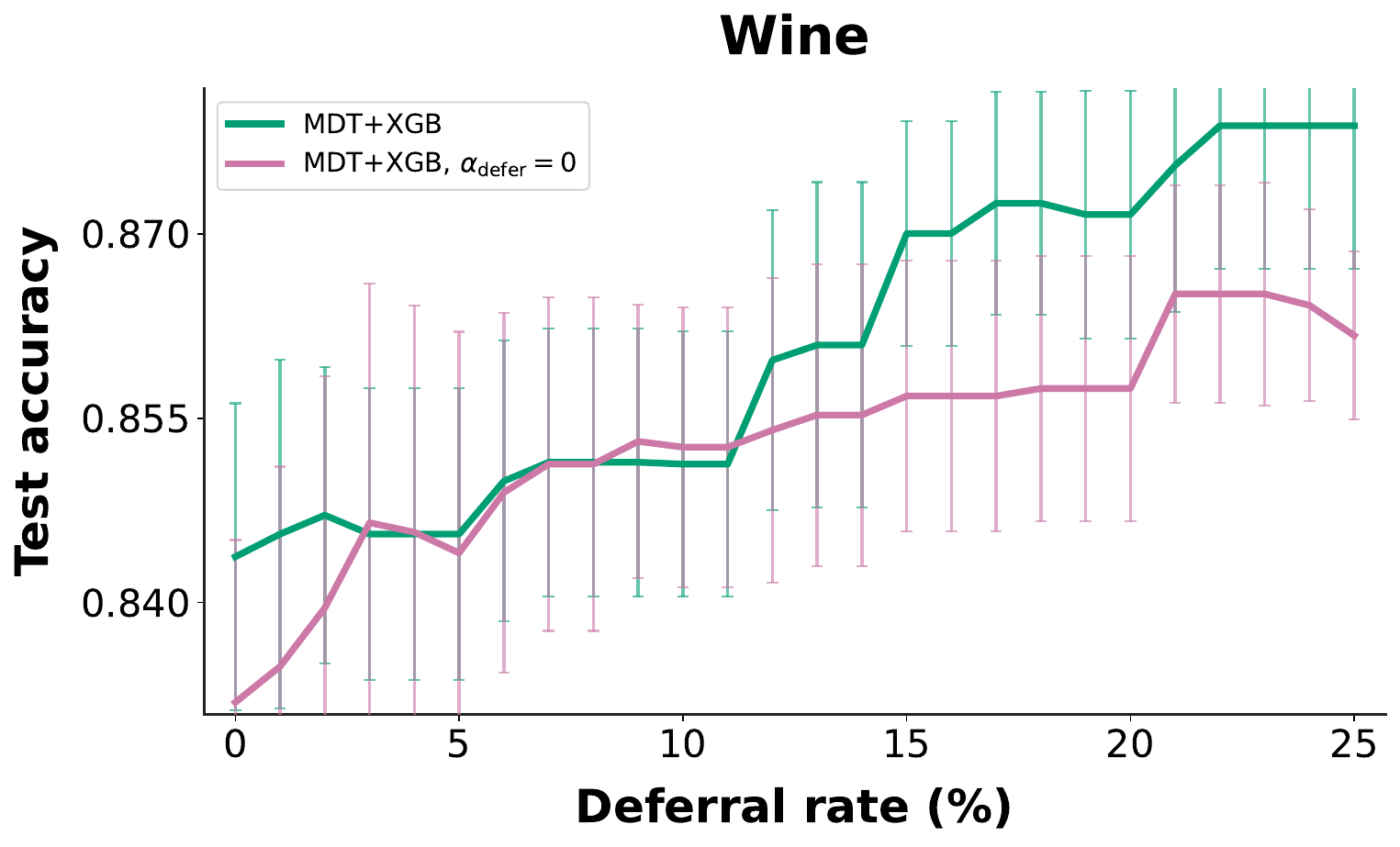}

\vspace{-0.2em}

\caption{Ablation on the impacts of rescaling $\tau$ for later stages of the MDT. The boolean $\texttt{rescale\_tau}$ is set to 0 for the $\alpha_{\text{defer}}=0$ method.}
\label{fig:appendix-strength}
\end{figure}

\clearpage

\clearpage
\thispagestyle{empty}

\begin{figure}[p]
\centering
\vspace*{-1.0cm}

\includegraphics[width=0.49\textwidth,height=0.225\textheight,keepaspectratio]{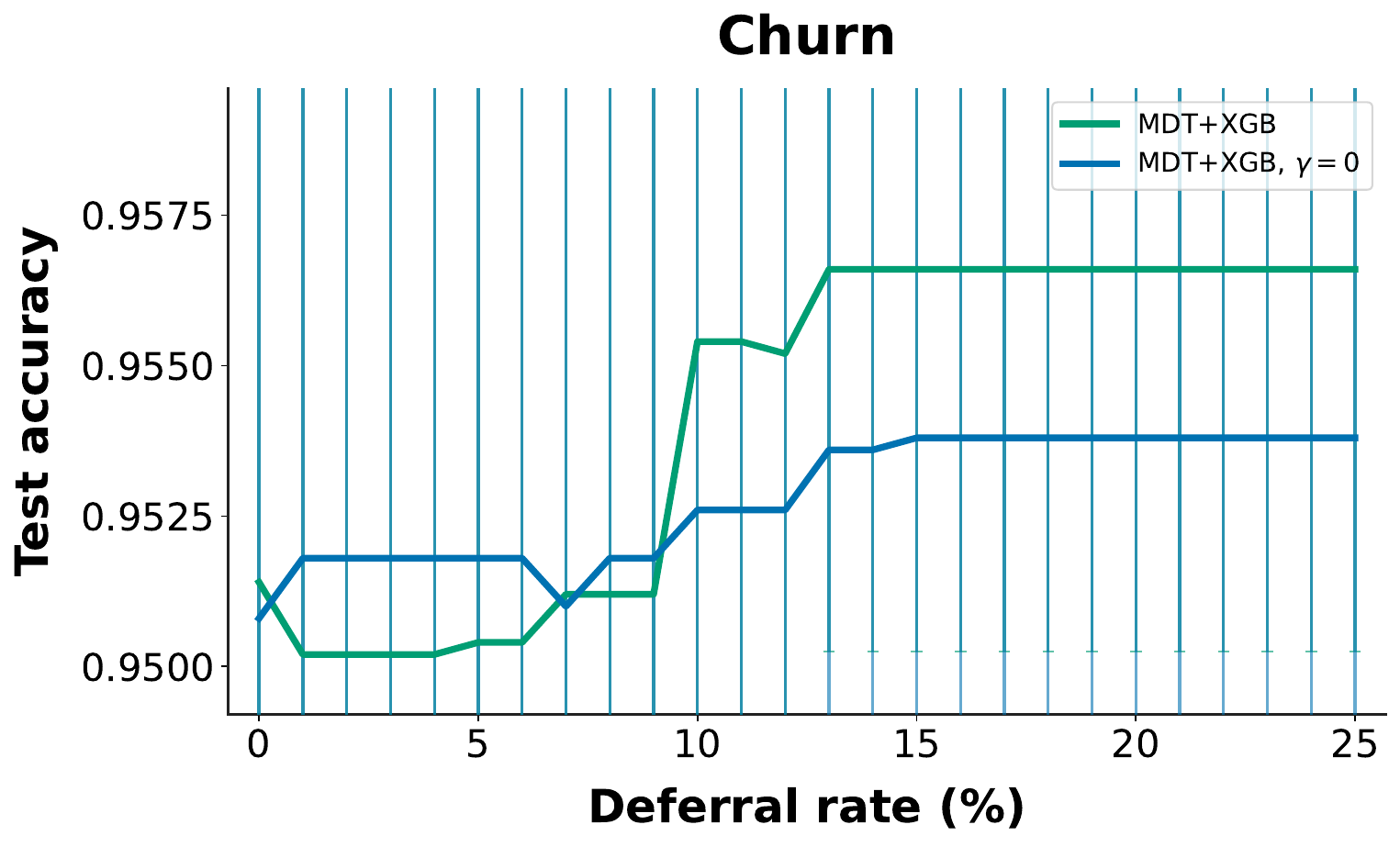}
\hfill
\includegraphics[width=0.49\textwidth,height=0.225\textheight,keepaspectratio]{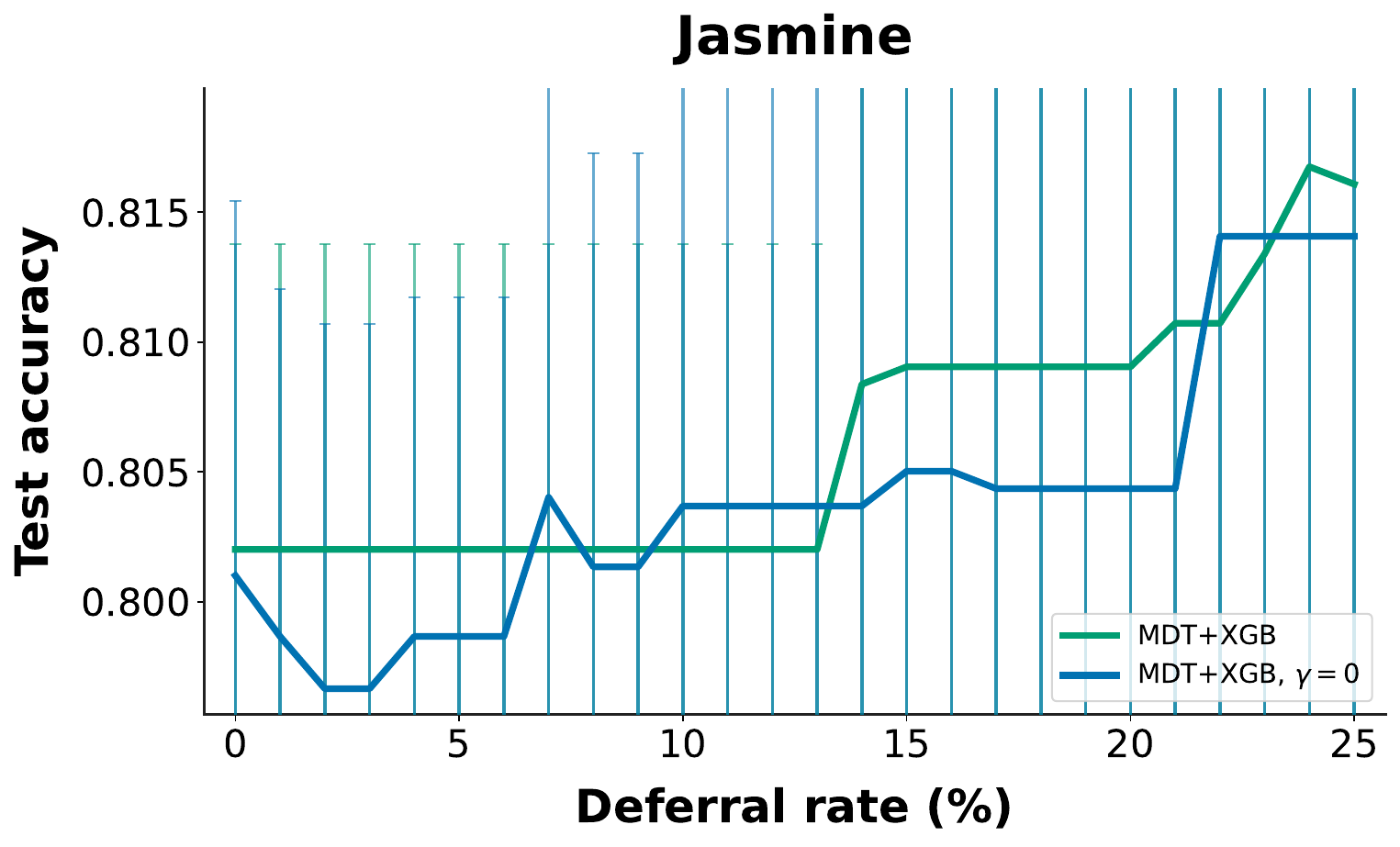}

\vspace{0.25em}

\includegraphics[width=0.49\textwidth,height=0.225\textheight,keepaspectratio]{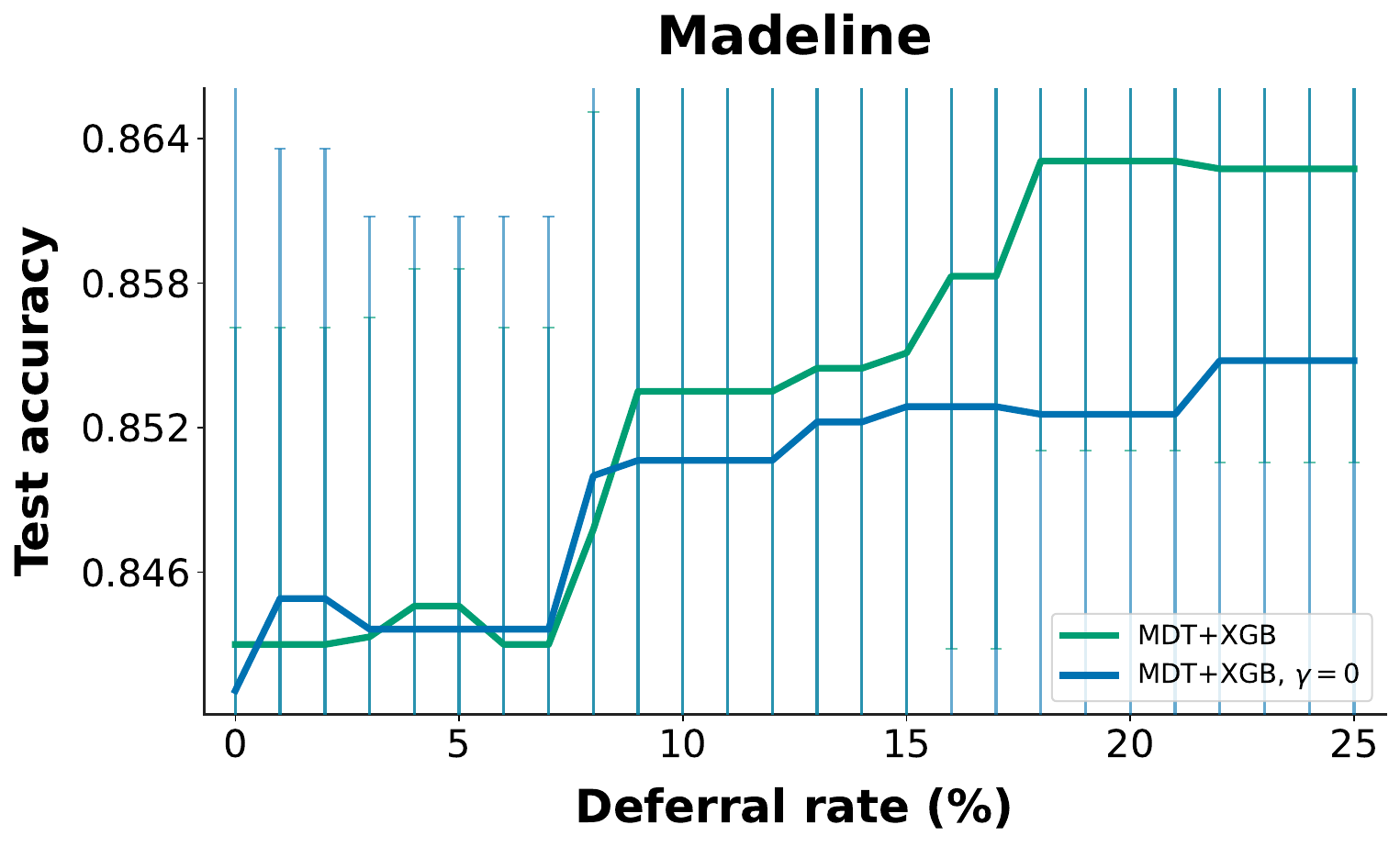}
\hfill
\includegraphics[width=0.49\textwidth,height=0.225\textheight,keepaspectratio]{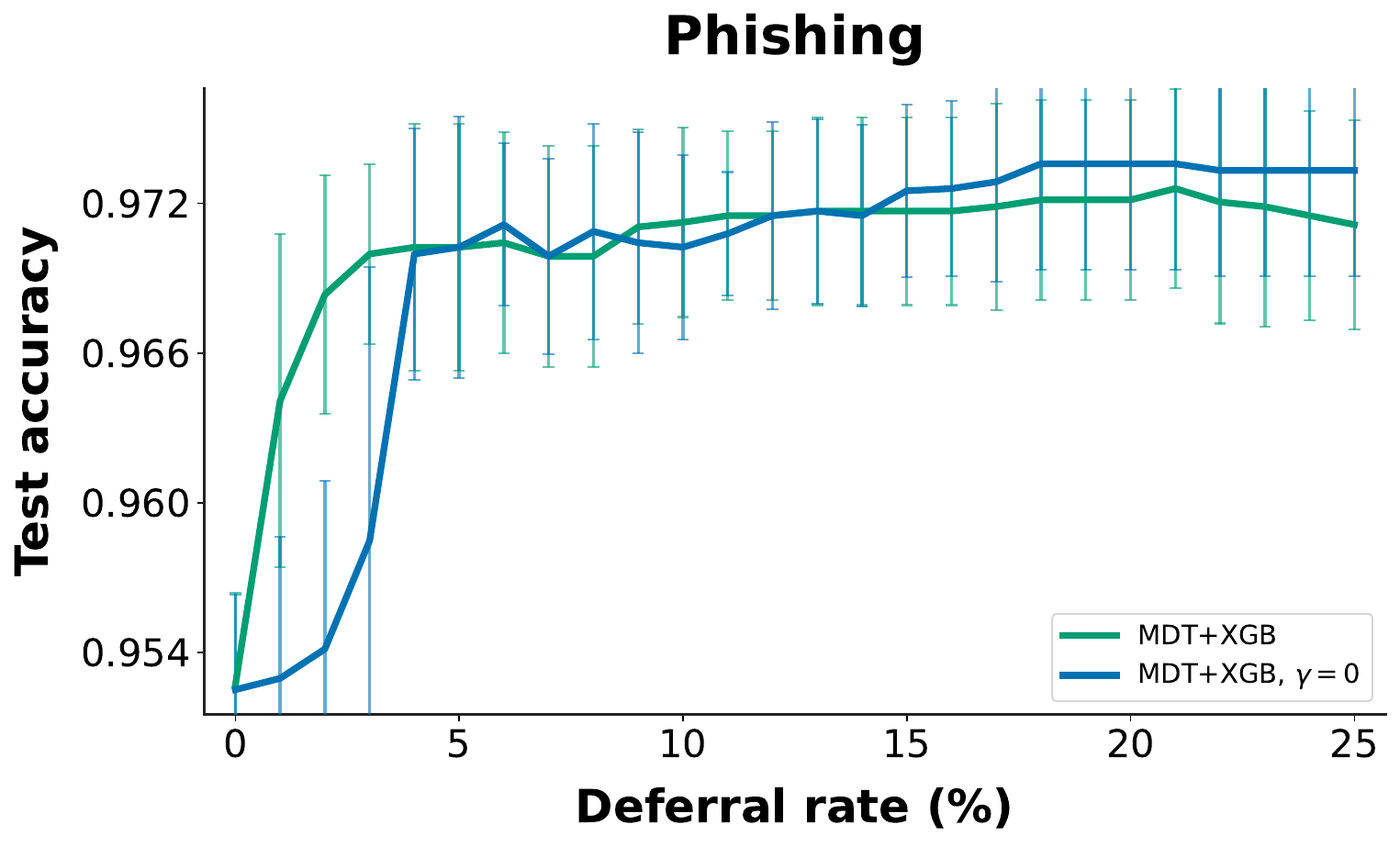}

\vspace{0.25em}

\includegraphics[width=0.49\textwidth,height=0.225\textheight,keepaspectratio]{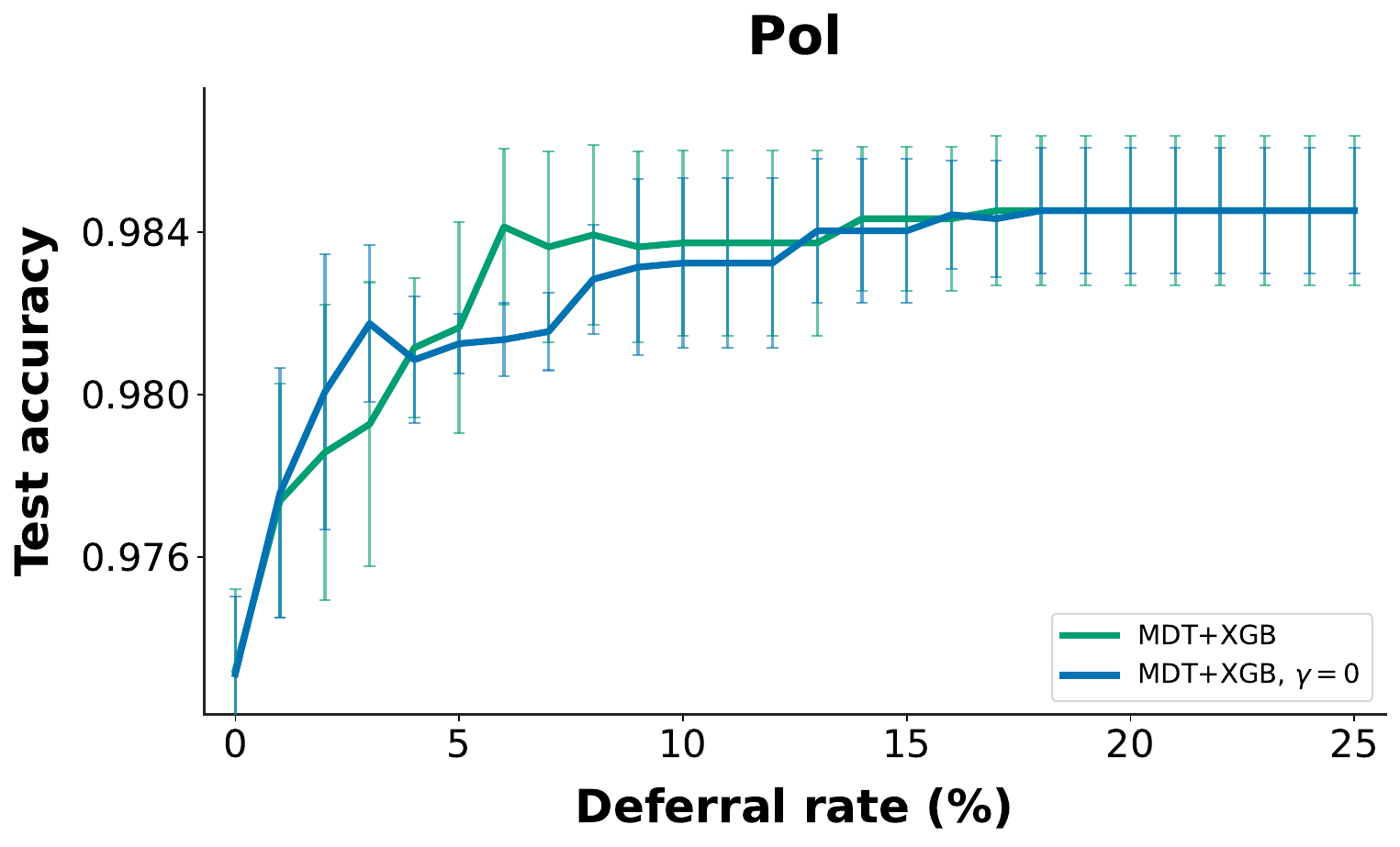}
\hfill
\includegraphics[width=0.49\textwidth,height=0.225\textheight,keepaspectratio]{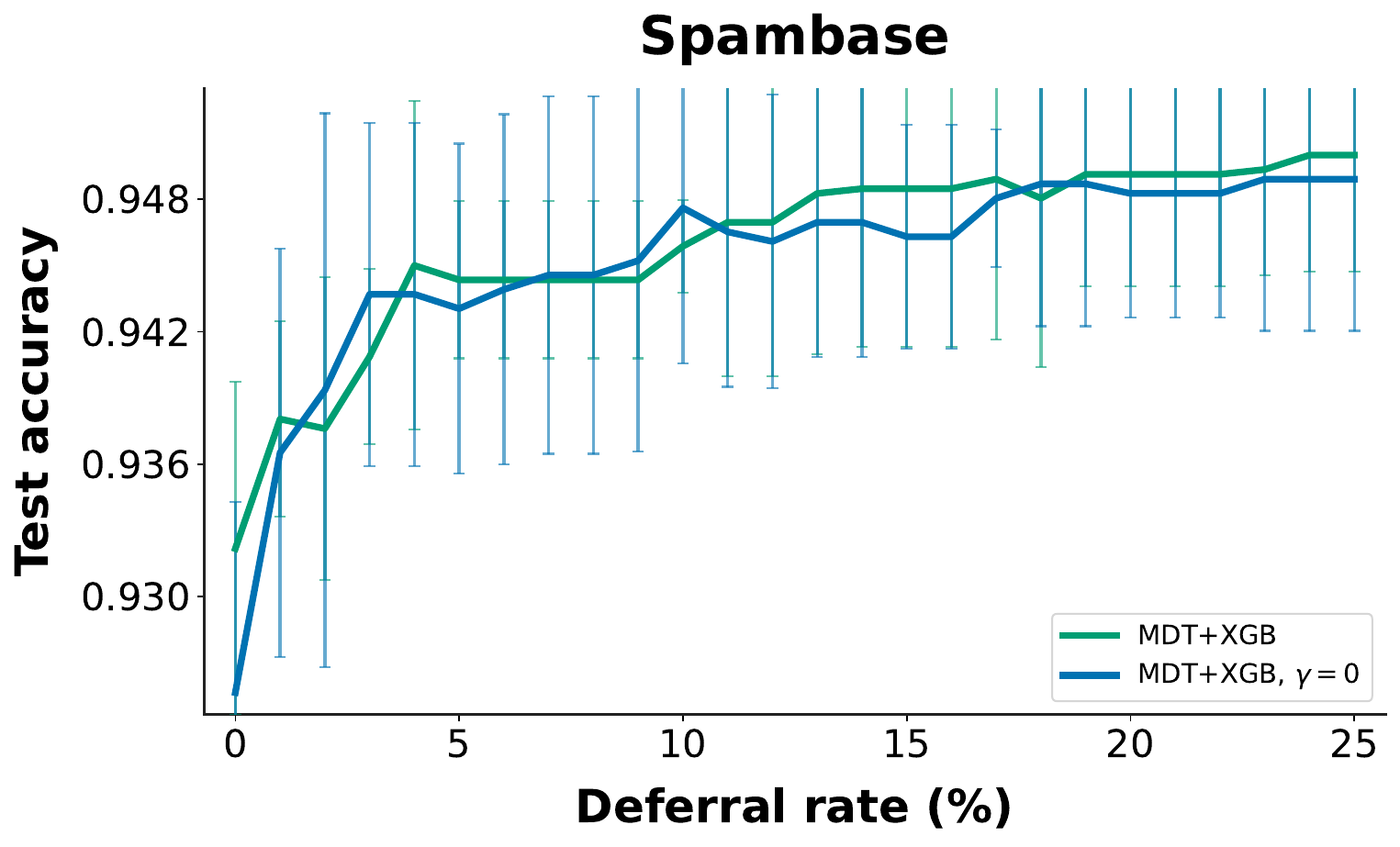}

\vspace{0.25em}

\includegraphics[width=0.68\textwidth,height=0.225\textheight,keepaspectratio]{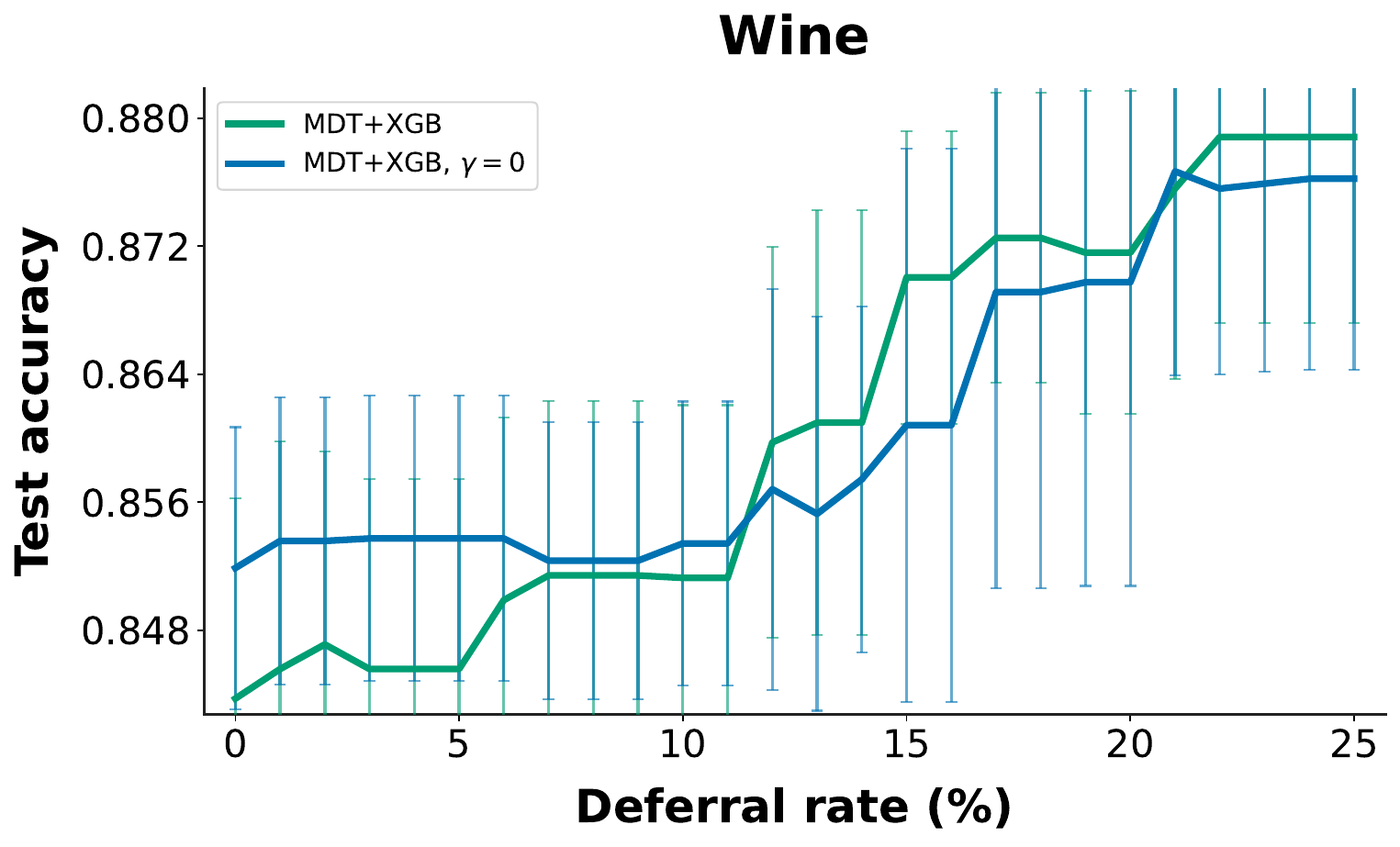}

\vspace{-0.2em}

\caption{Ablation on the distance-based reweighting via tuning $\gamma$ or setting $\gamma=0$.}
\label{fig:appendix-gamma}
\end{figure}

\clearpage

\begin{table}[!t]
\centering
\large
\setlength{\tabcolsep}{5pt}
\renewcommand{\arraystretch}{0.9}
\begin{tabular}{lcc}
\toprule
Dataset & MDT+XGB & MDT+XGB LeaveOneOut \\
\midrule
Abalone   & $0.6333 \pm 0.0155$ & $0.6323 \pm 0.0128$ \\
Adult     & $0.8681 \pm 0.0039$ & $0.8661 \pm 0.0026$ \\
Aging     & $0.8028 \pm 0.0077$ & $0.8028 \pm 0.0077$ \\
Bank      & $0.9090 \pm 0.0013$ & $0.9077 \pm 0.0021$ \\
Bike      & $0.9493 \pm 0.0054$ & $0.9427 \pm 0.0029$ \\
California & $0.9099 \pm 0.0049$ & $0.8997 \pm 0.0075$ \\
\bottomrule
\end{tabular}
\caption{Ablation of using in-sample black-box predictions versus out-of-fold predictions when training the defer model. Using leave-one-out style out-of-fold predictions does not improve performance and slightly worsens accuracy on most datasets. Results show average test accuracy ($\pm$ standard deviation) under a 50\% test deferral constraint during hyperparameter selection.}
\label{tab:mdt_fast_vs_loo_50}
\end{table}

\clearpage
\thispagestyle{empty}

\begin{figure*}[p]
\centering
\vspace*{-1.0cm}

\includegraphics[width=0.49\textwidth,height=0.22\textheight,keepaspectratio]{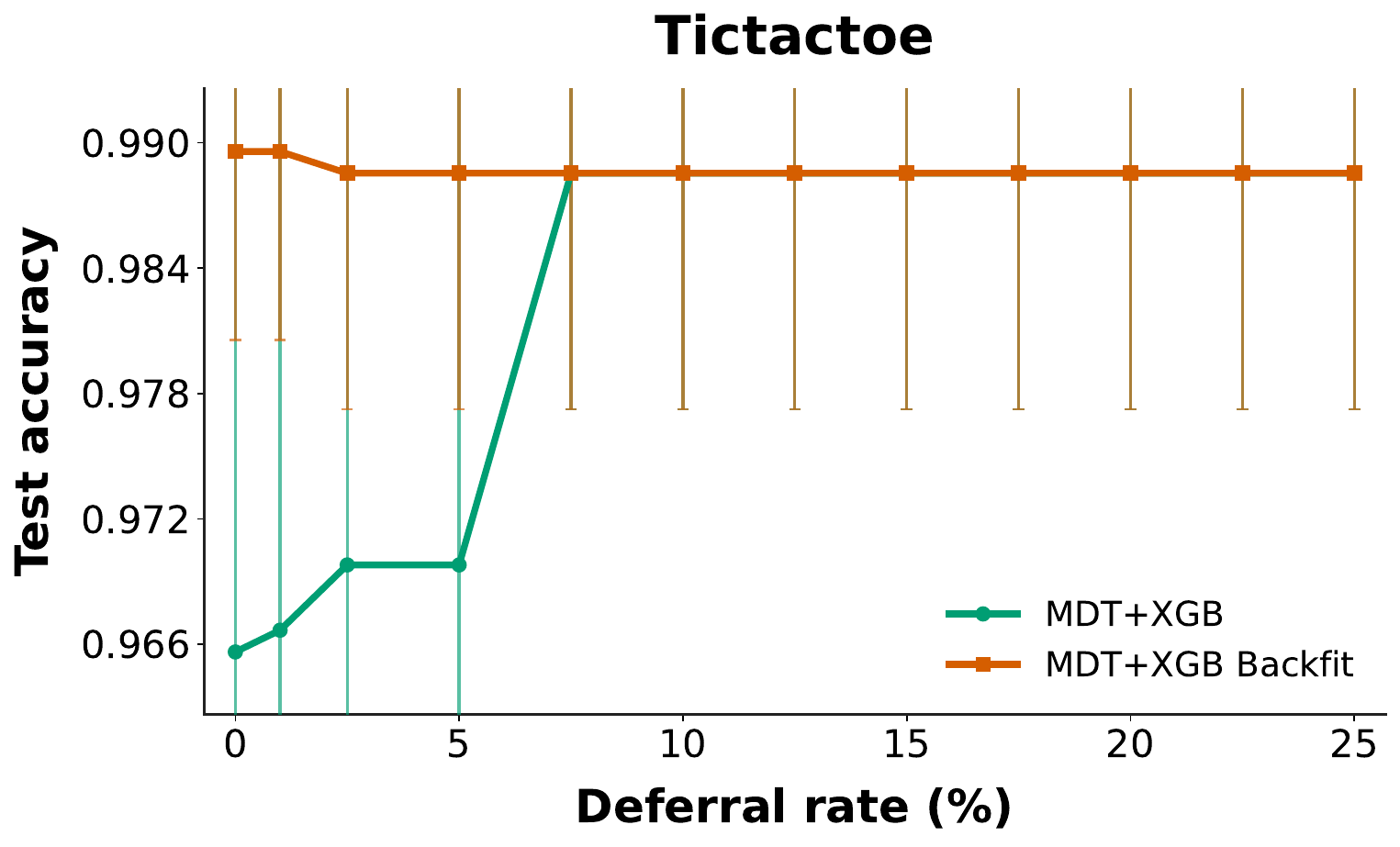}
\hfill
\includegraphics[width=0.49\textwidth,height=0.22\textheight,keepaspectratio]{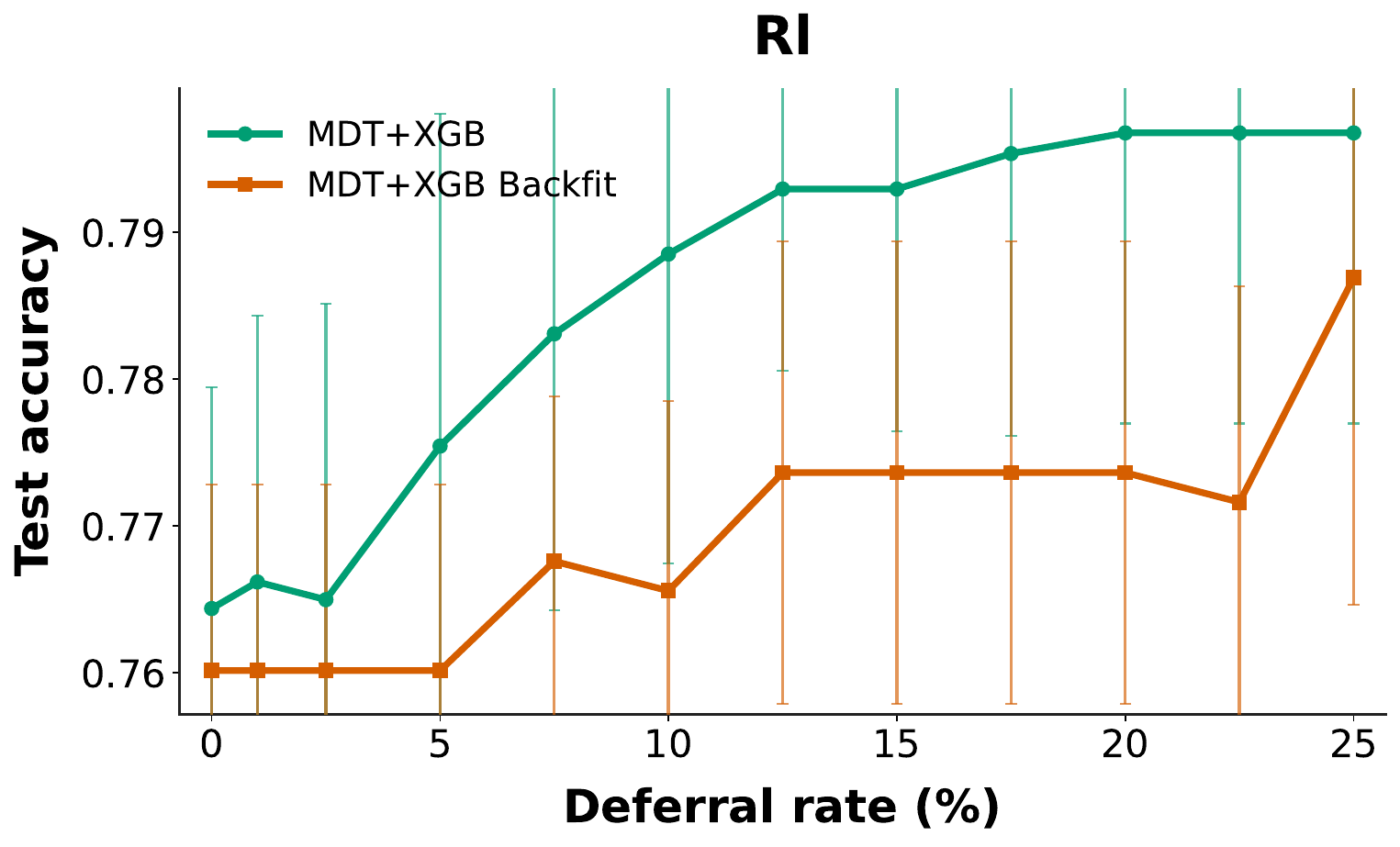}

\vspace{0.25em}

\includegraphics[width=0.49\textwidth,height=0.22\textheight,keepaspectratio]{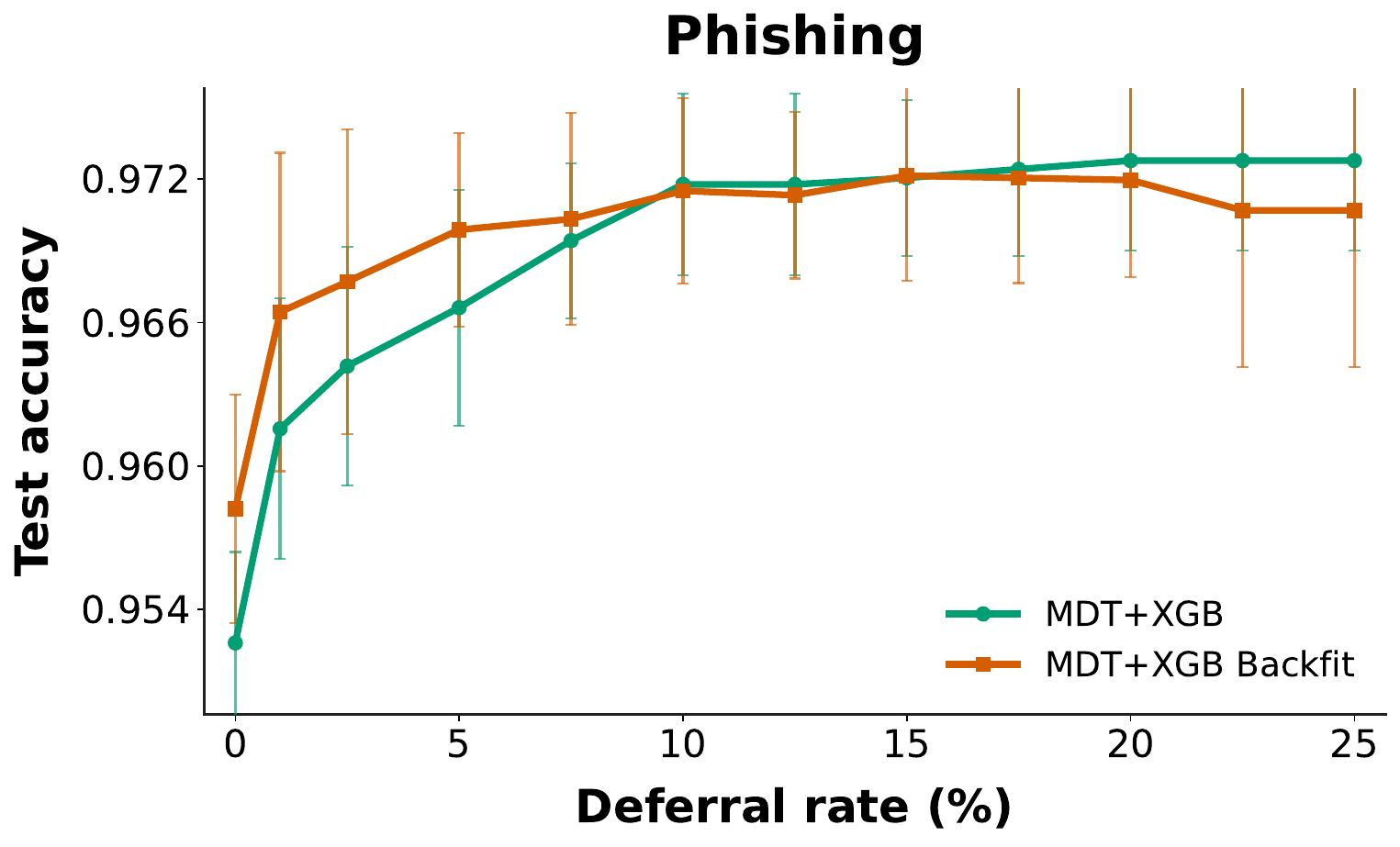}
\hfill
\includegraphics[width=0.49\textwidth,height=0.22\textheight,keepaspectratio]{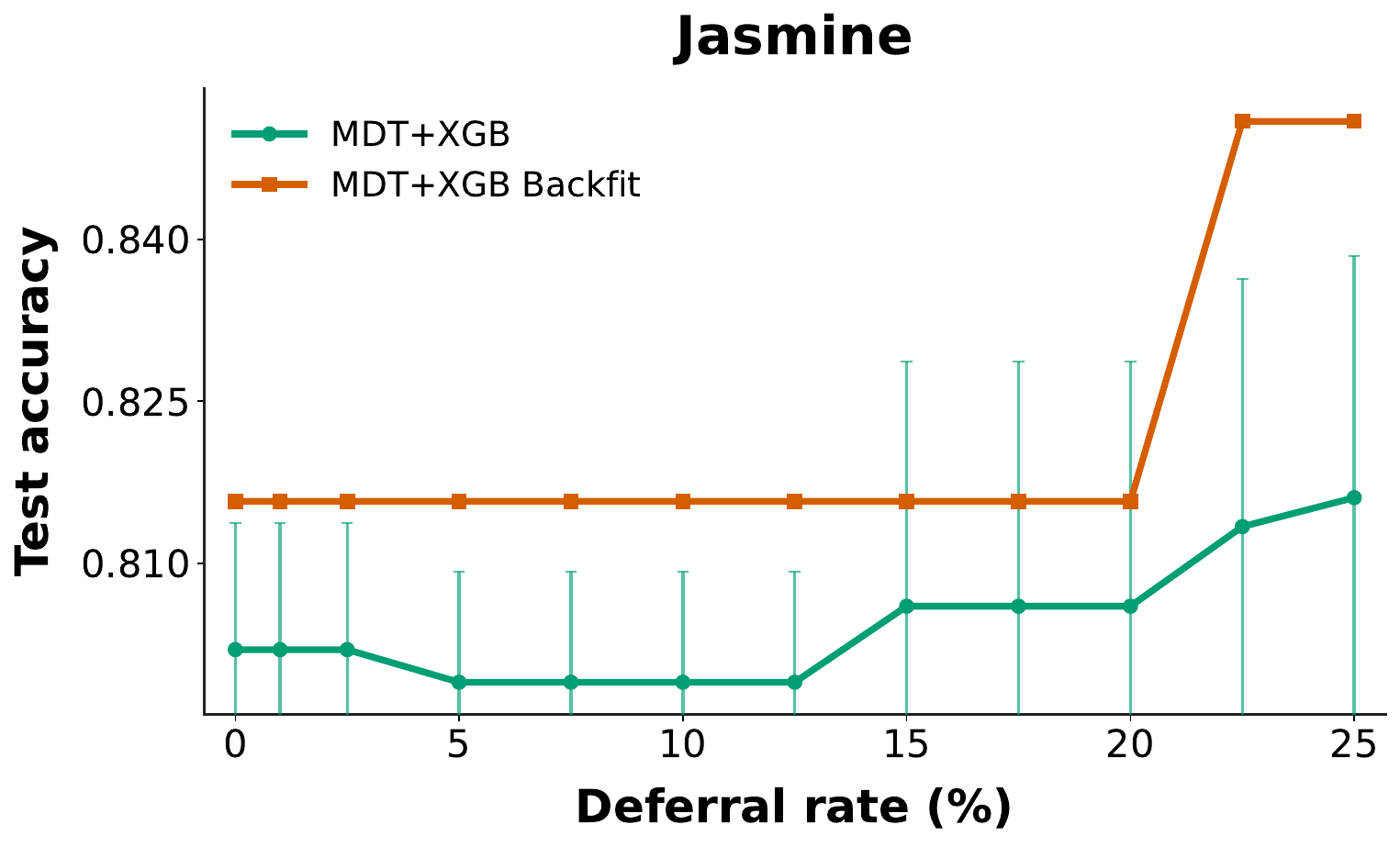}

\vspace{0.25em}

\includegraphics[width=0.49\textwidth,height=0.22\textheight,keepaspectratio]{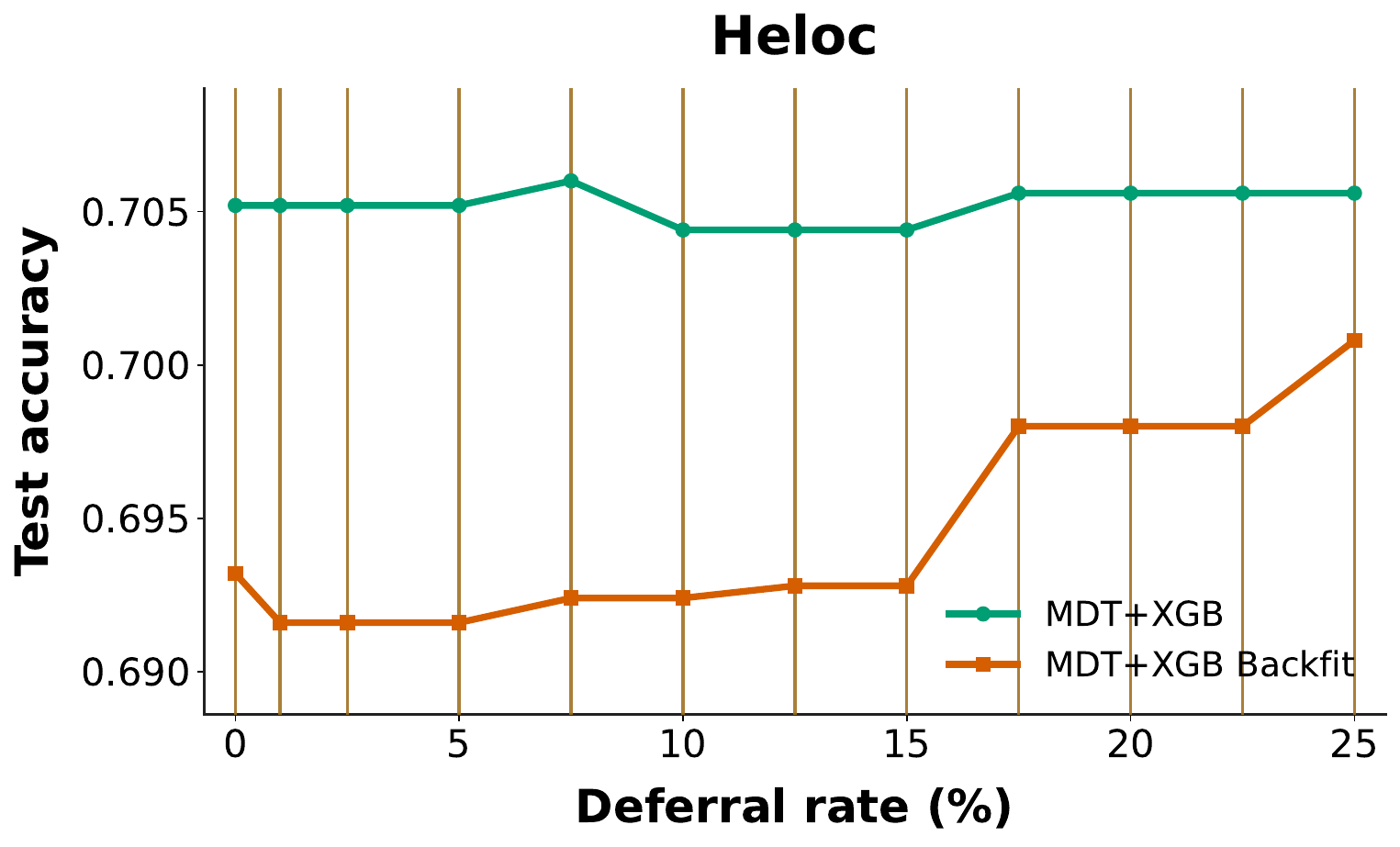}
\hfill
\includegraphics[width=0.49\textwidth,height=0.22\textheight,keepaspectratio]{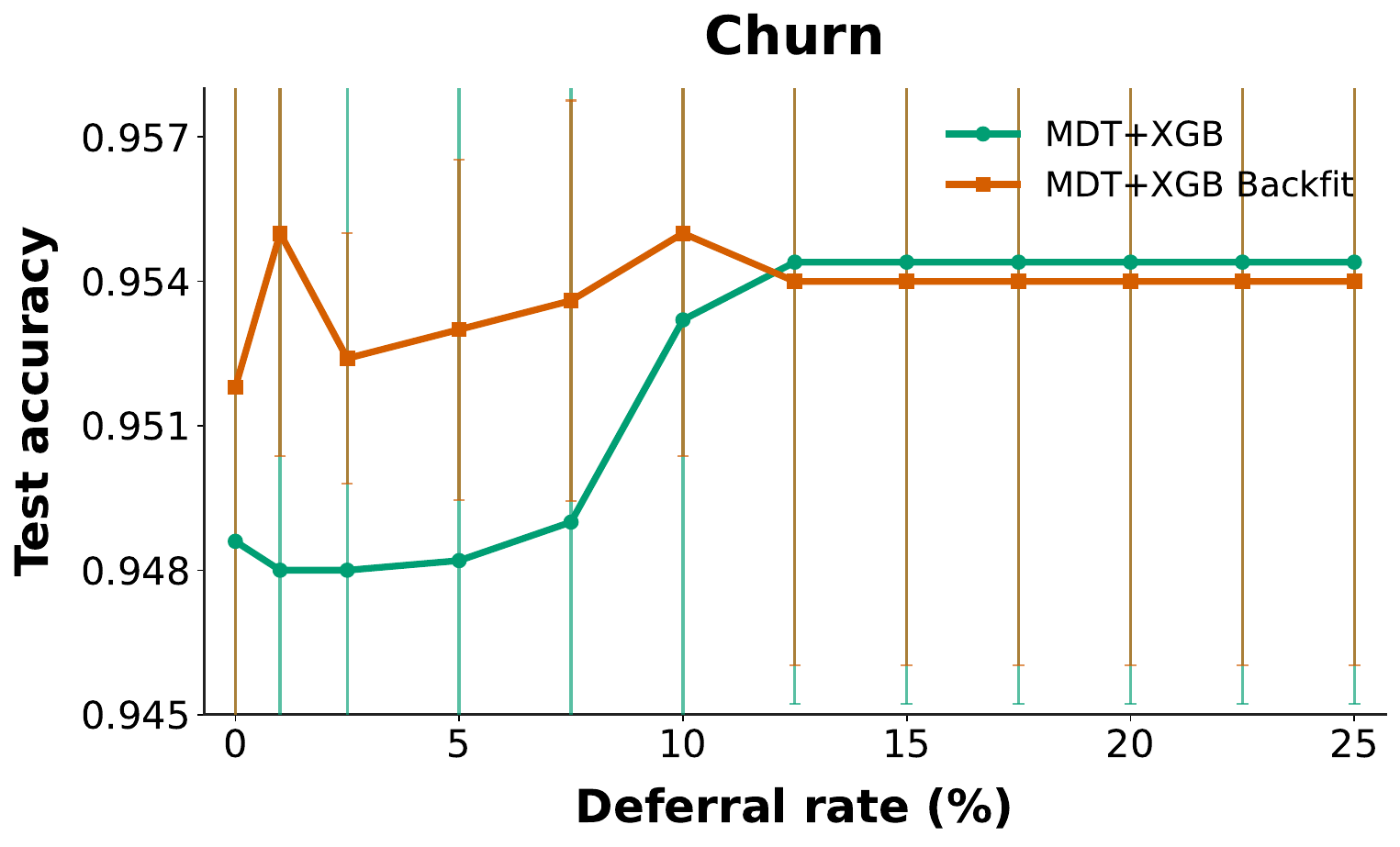}

\vspace{0.25em}

\includegraphics[width=0.62\textwidth,height=0.22\textheight,keepaspectratio]{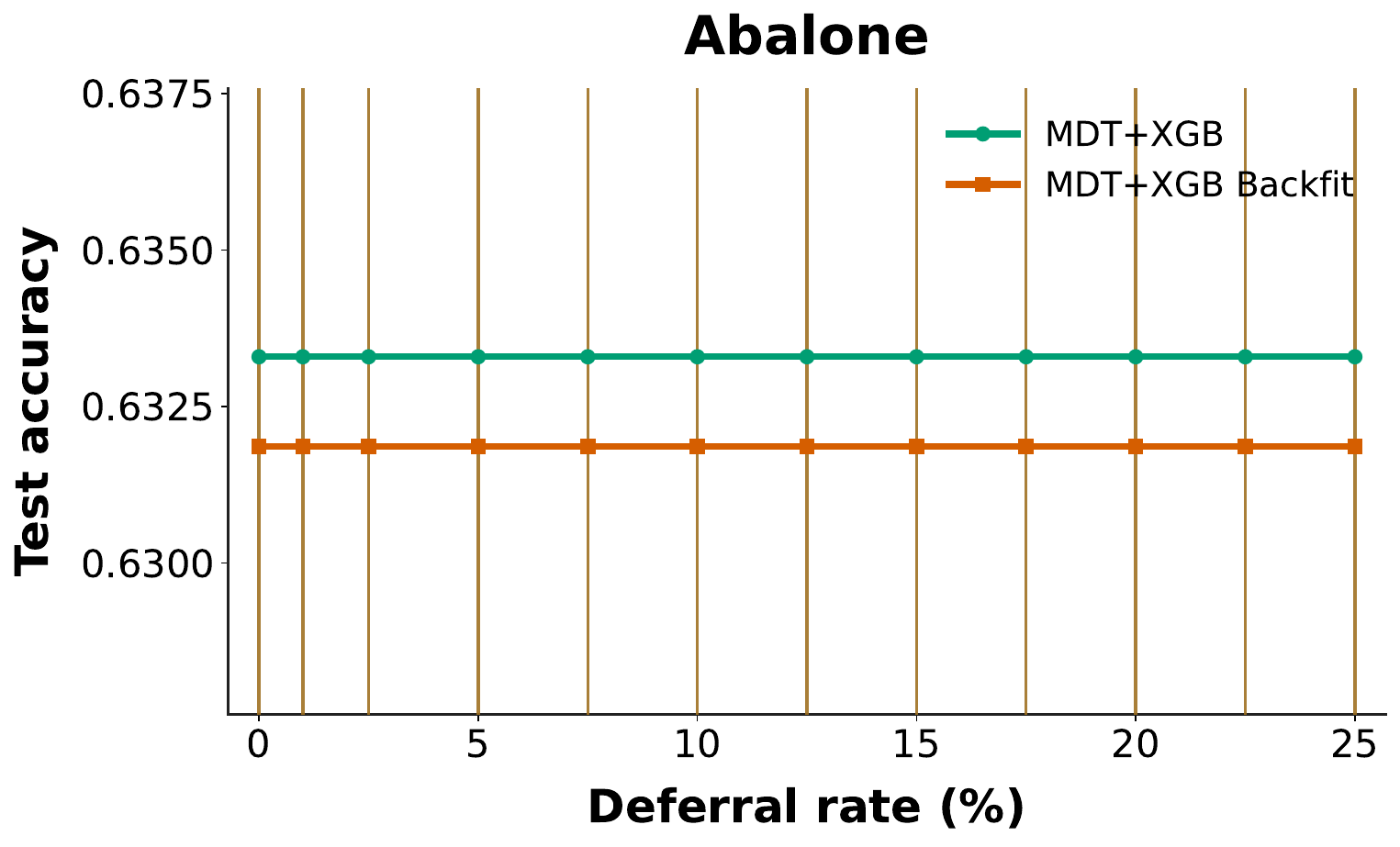}

\vspace{-0.2em}

\caption{Comparison of MDT+XGB and MDT+XGB with backfitting across selected datasets. Each plot shows test accuracy as a function of the deferral rate. Backfitting gives clear improvements on some datasets, such as Tictactoe, where revisiting stages may help capture higher-order interactions, but it degrades performance on several others.}
\label{fig:backfit_mdt_xgb_selected}
\end{figure*}

\clearpage

\clearpage
\thispagestyle{empty}

\begin{figure*}[p]
\centering
\vspace*{-0.9cm}

\includegraphics[width=0.49\textwidth,height=0.17\textheight,keepaspectratio]{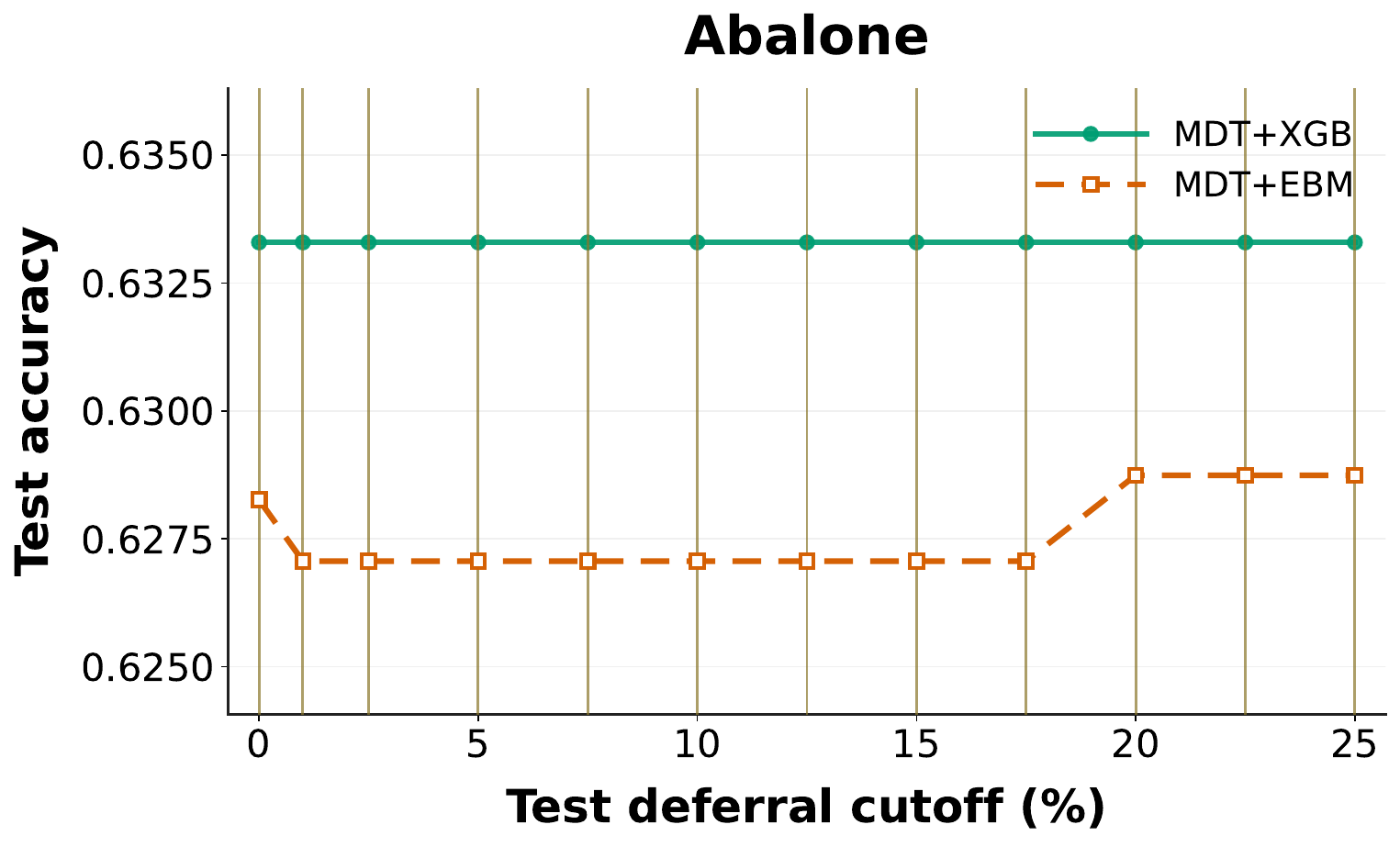}
\hfill
\includegraphics[width=0.49\textwidth,height=0.17\textheight,keepaspectratio]{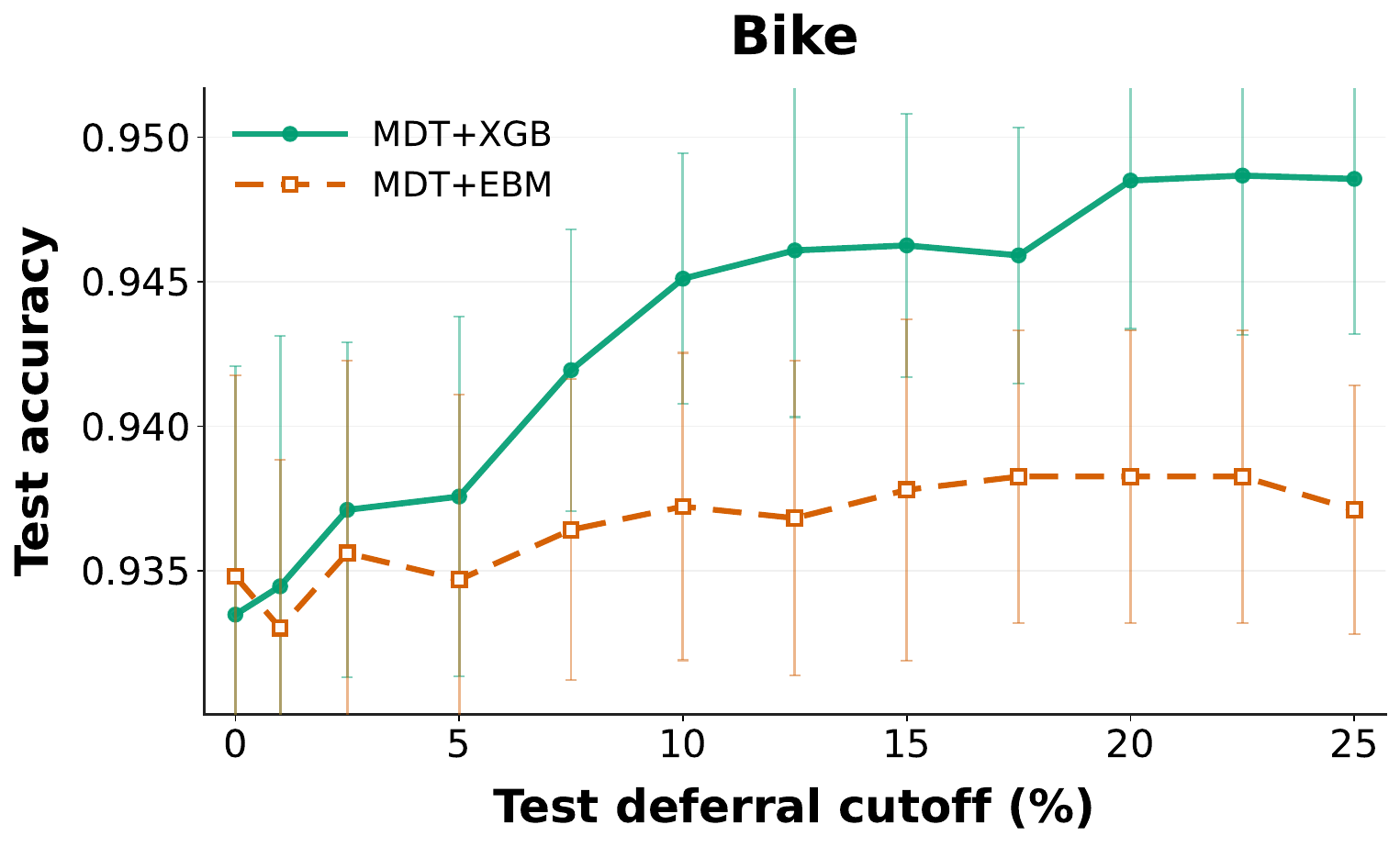}

\vspace{0.25em}

\includegraphics[width=0.49\textwidth,height=0.17\textheight,keepaspectratio]{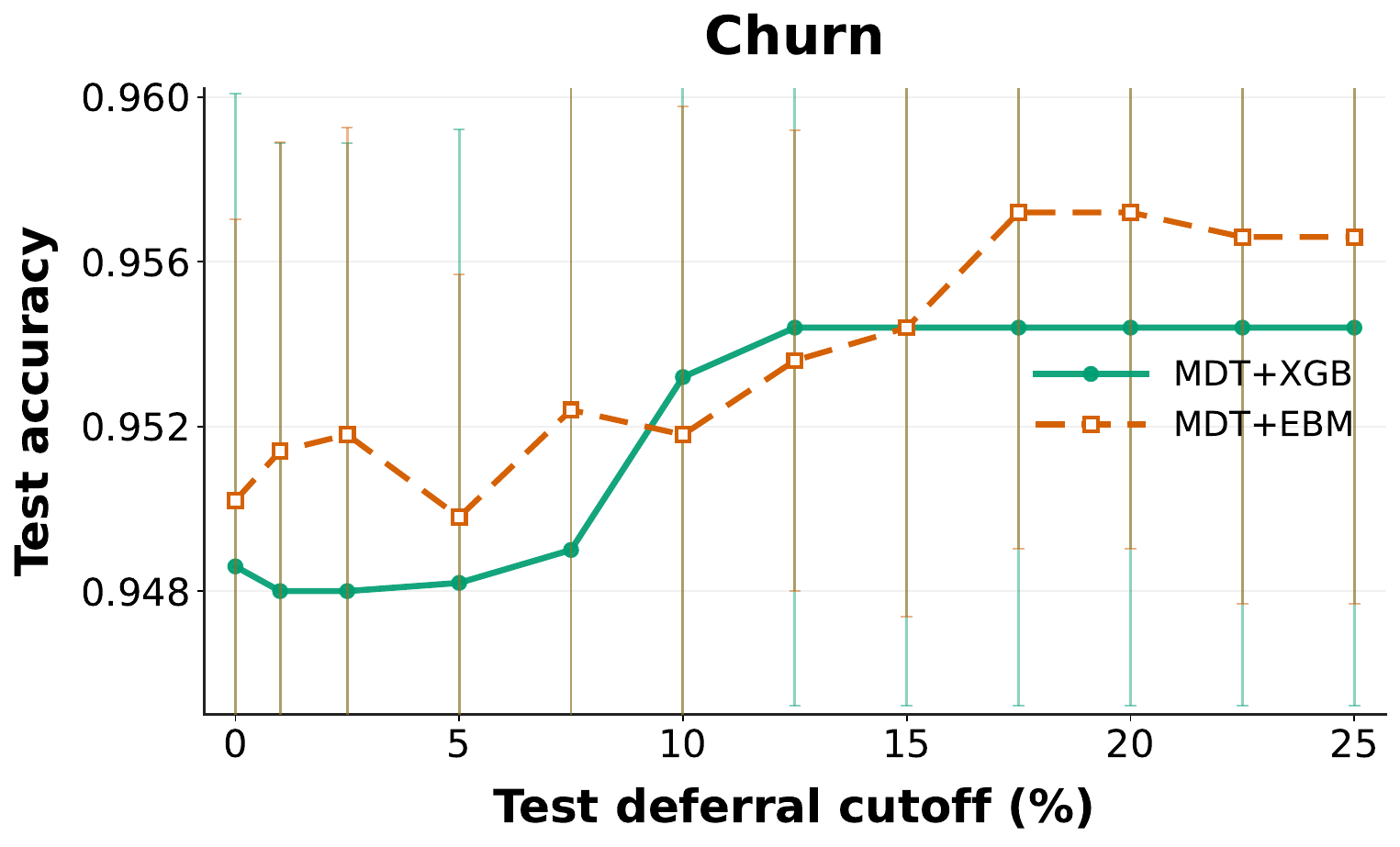}
\hfill
\includegraphics[width=0.49\textwidth,height=0.17\textheight,keepaspectratio]{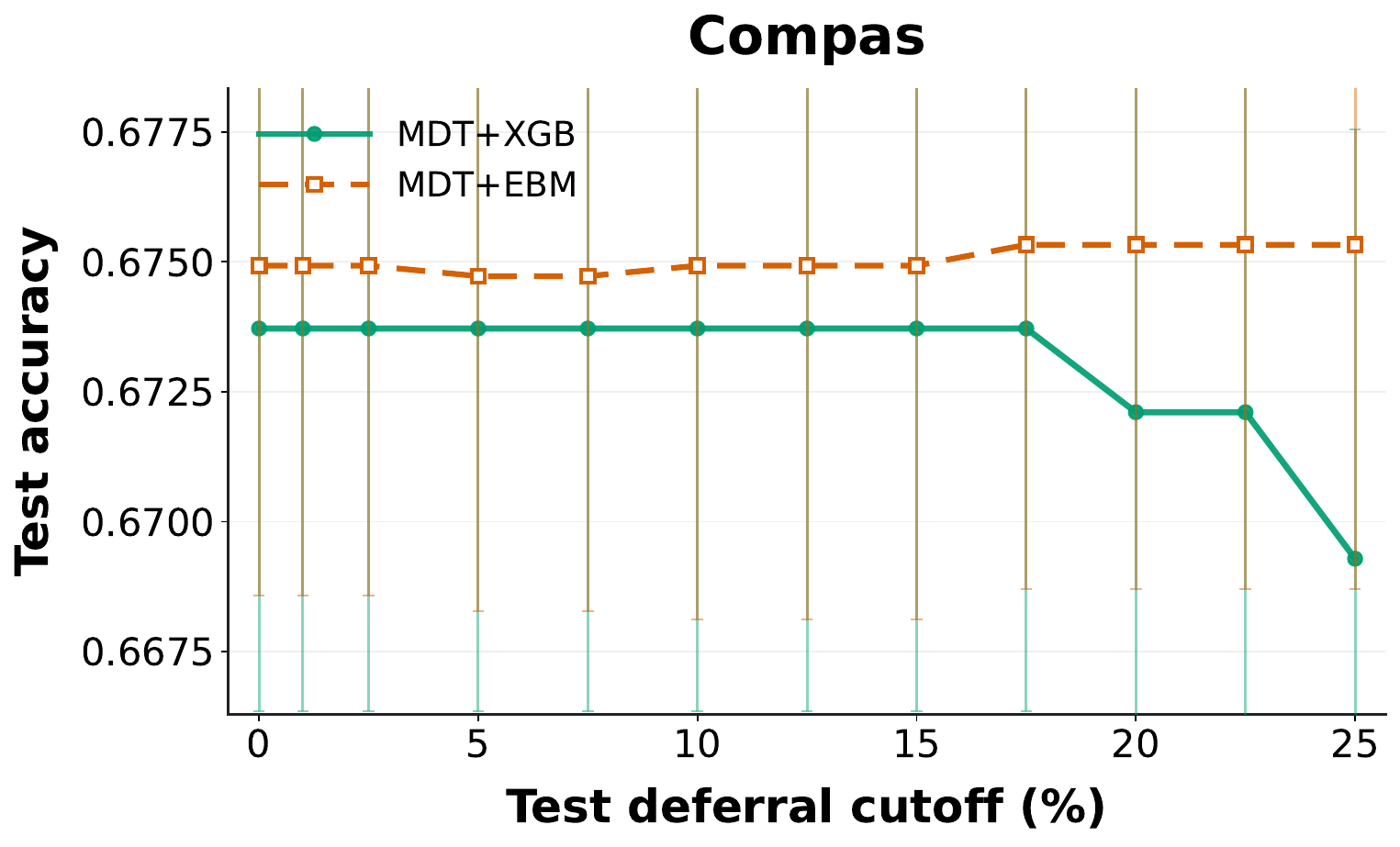}

\vspace{0.25em}

\includegraphics[width=0.49\textwidth,height=0.17\textheight,keepaspectratio]{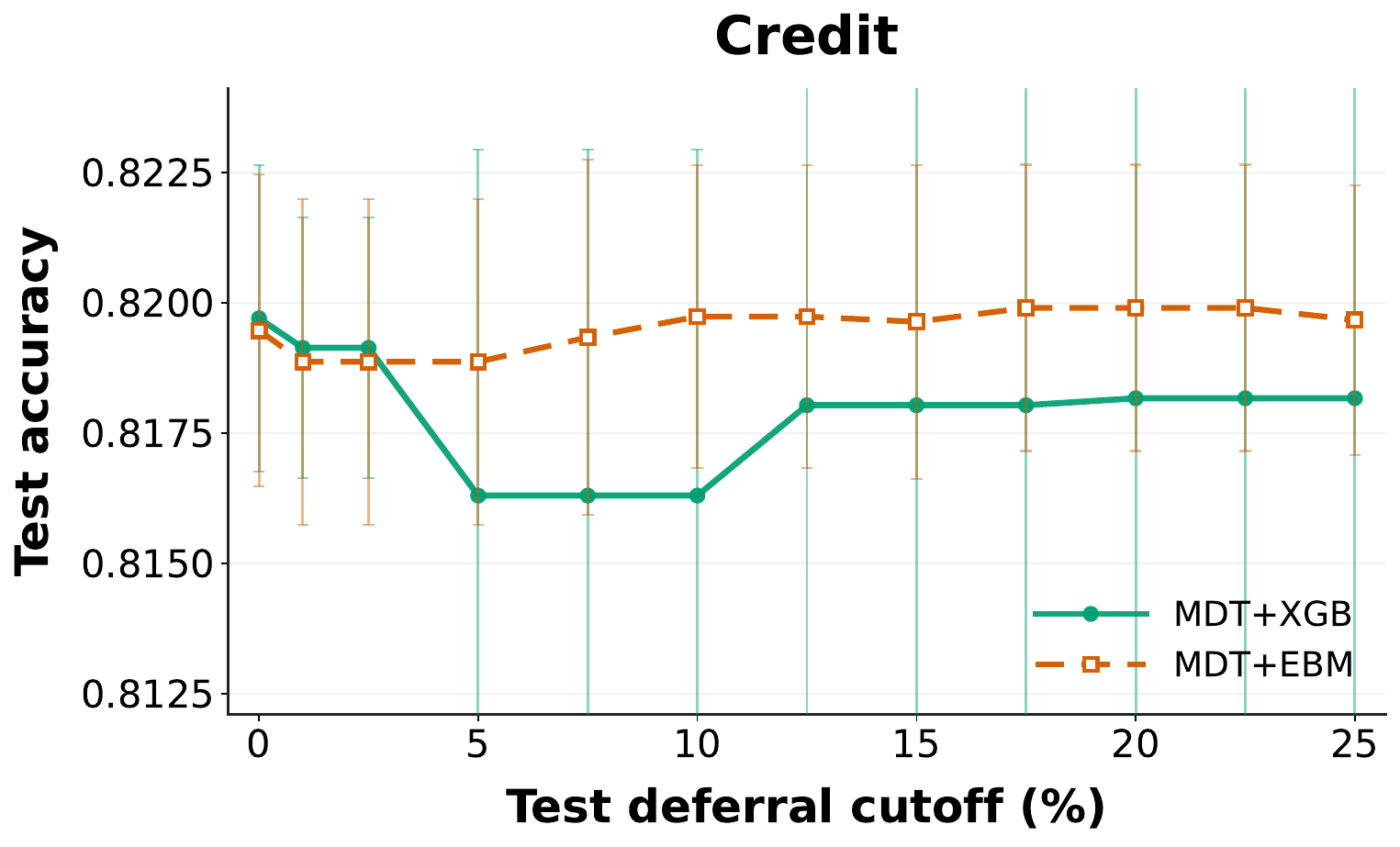}
\hfill
\includegraphics[width=0.49\textwidth,height=0.17\textheight,keepaspectratio]{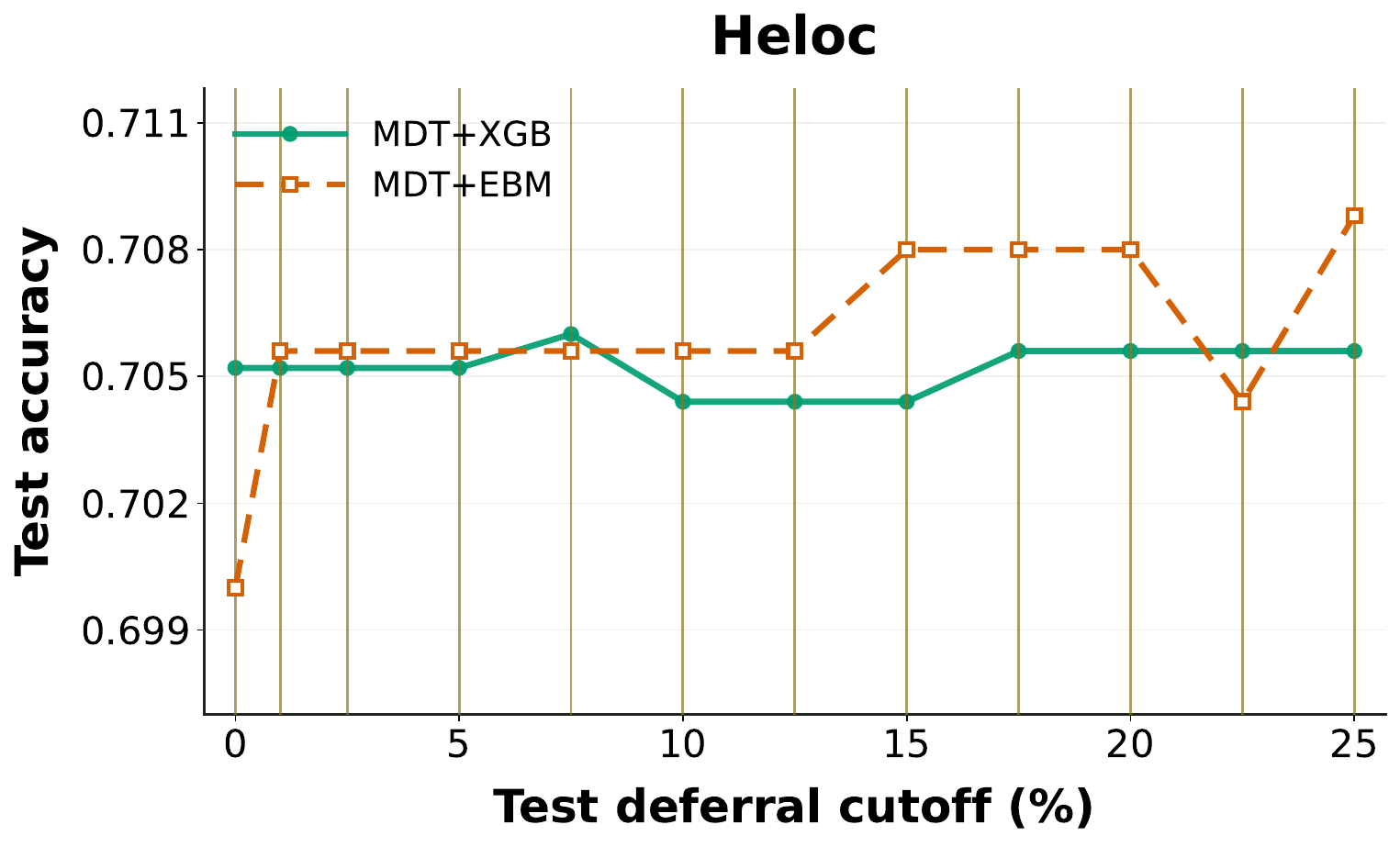}

\vspace{0.25em}

\includegraphics[width=0.49\textwidth,height=0.17\textheight,keepaspectratio]{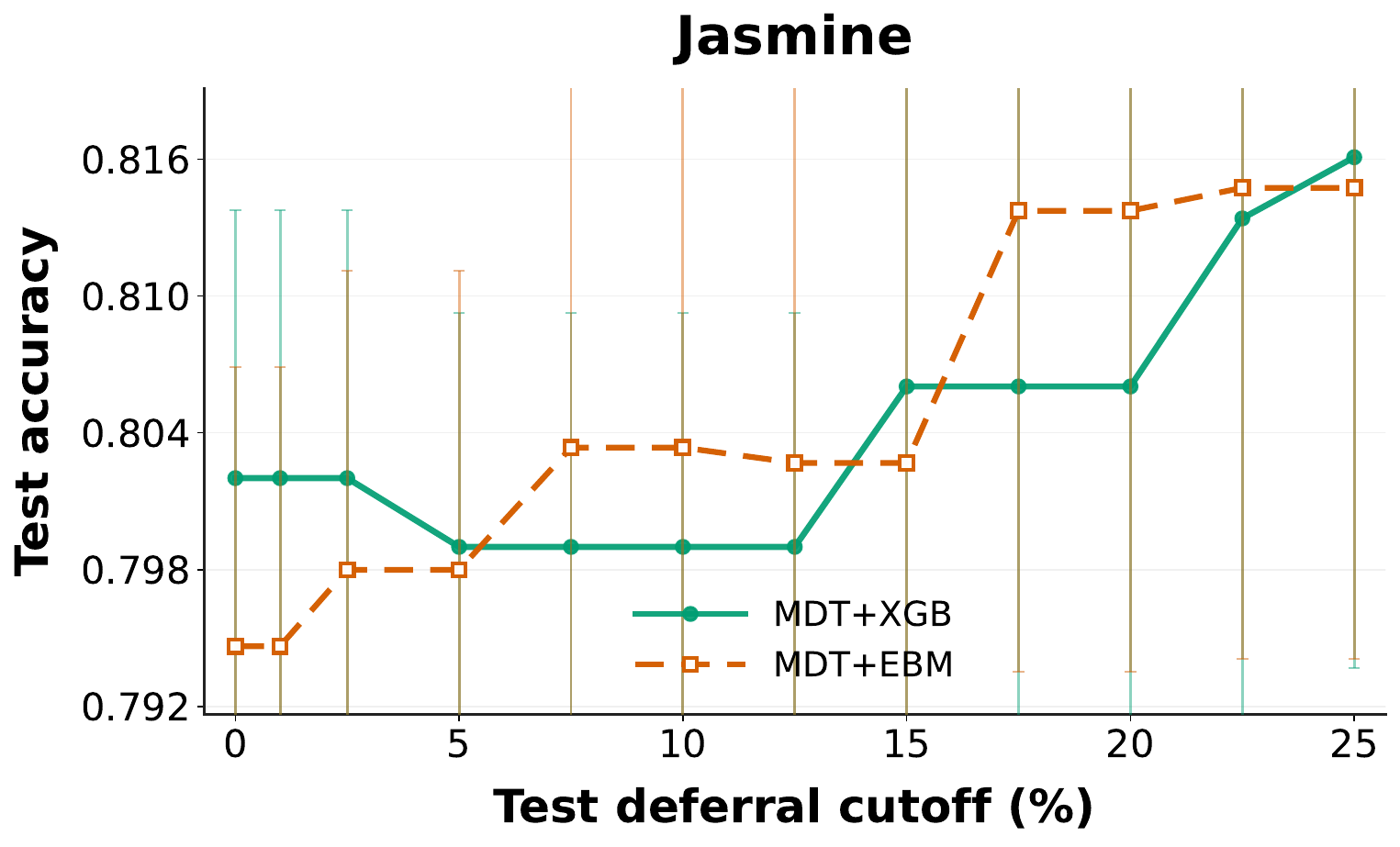}
\hfill
\includegraphics[width=0.49\textwidth,height=0.17\textheight,keepaspectratio]{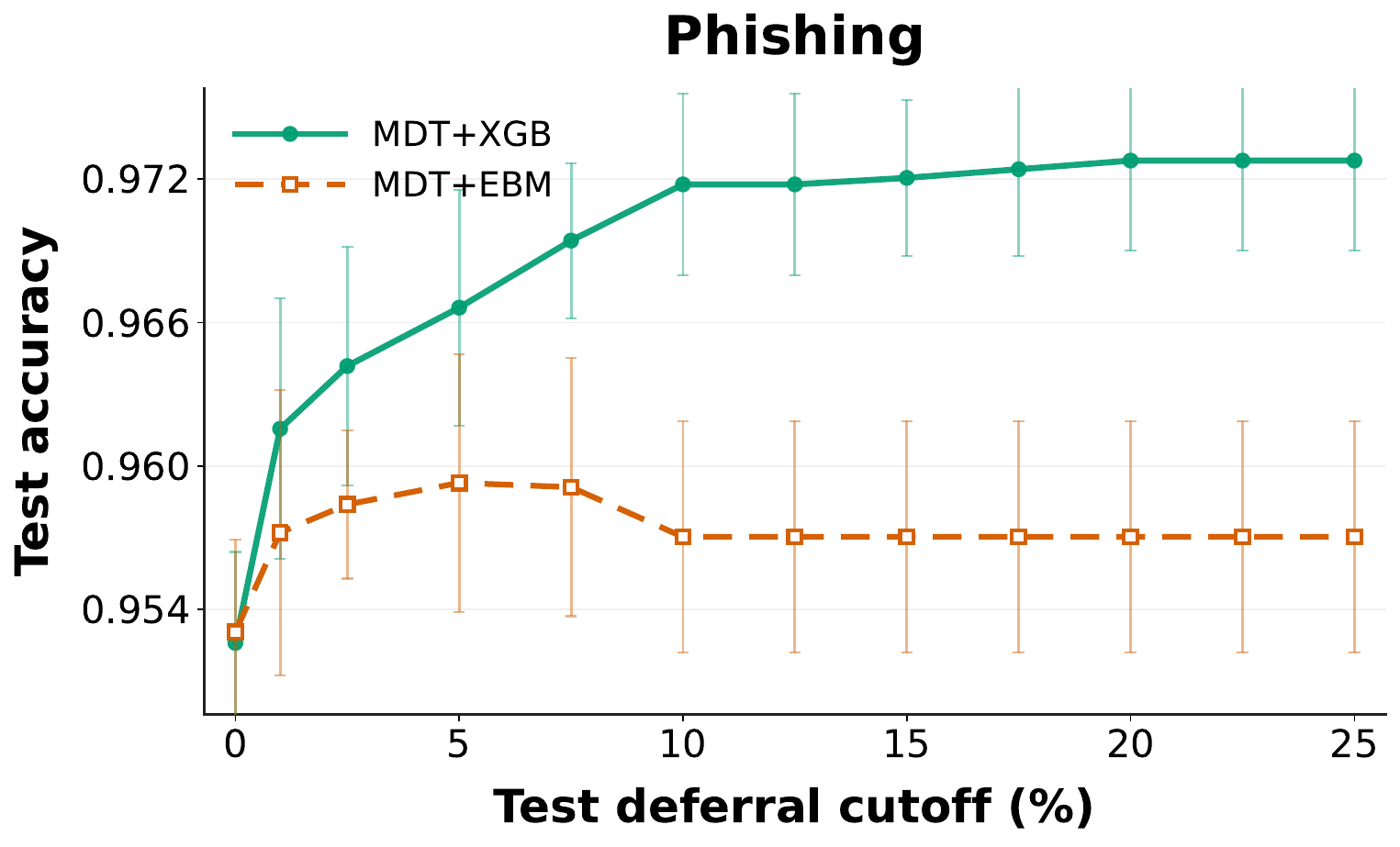}

\vspace{0.25em}

\includegraphics[width=0.58\textwidth,height=0.17\textheight,keepaspectratio]{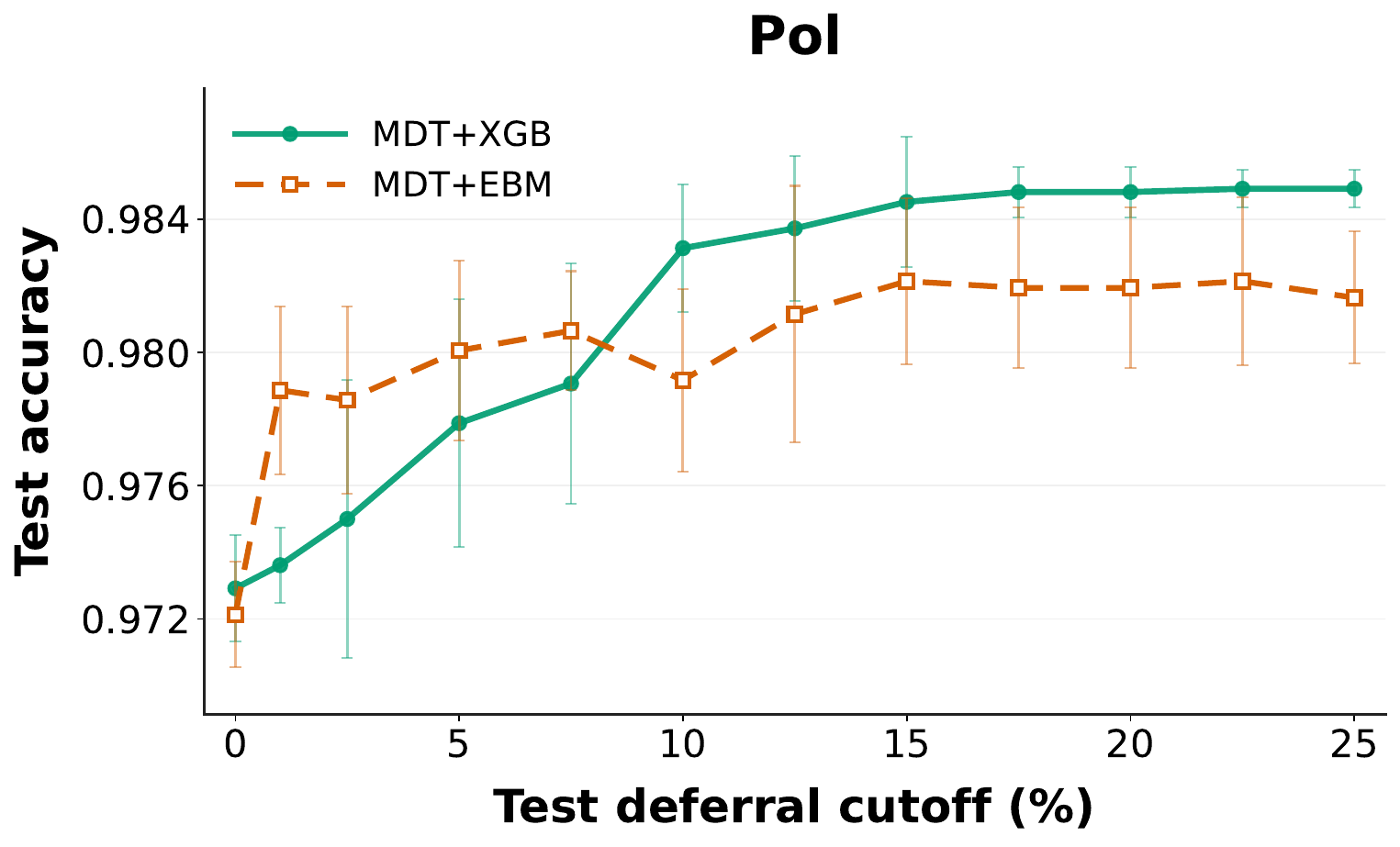}

\vspace{-0.25em}

\caption{
Comparison of MDT+XGB and MDT+EBM across selected datasets.
MDT+EBM uses an Explainable Boosting Machine \citep{Lou2013AccurateIM} as the fallback model, providing a more interpretable alternative to XGBoost.
For some datasets, this amount of fallback complexity is sufficient, and MDT+EBM closely matches MDT+XGB.
For other datasets, most notably Bike and Phishing, the EBM fallback is not expressive enough to recover the same accuracy--deferral tradeoff as MDT+XGB.
}
\label{fig:mdt-xgb-vs-ebm}
\end{figure*}

\begin{table}[t]
    \centering
    \small
    \setlength{\tabcolsep}{5pt}
    \caption{Performance of multistage defer trees (MDTs) trained using a greedy decision tree algorithm operating directly on continuous features. Entries give mean test classification accuracy $\pm$ standard error across five train/test splits (higher is better). $\leq 25\%$ deferral.}
    \label{tab:mdt-continuous-greedy}
    \begin{tabular}{lc}
        \toprule
        Dataset & MDT: greedy (continuous) \\
        \midrule
        Abalone & $0.6240 \pm 0.0028$ \\
        Bike & $0.9460 \pm 0.0024$ \\
        Churn & $0.9544 \pm 0.0026$ \\
        Compas & $0.6723 \pm 0.0057$ \\
        Droid & $0.9706 \pm 0.0005$ \\
        Monk2 & $0.9783 \pm 0.0086$ \\
        Phishing & $0.9708 \pm 0.0018$ \\
        Pol & $0.9824 \pm 0.0009$ \\
        Rl & $0.7777 \pm 0.0089$ \\
        Student & $0.7923 \pm 0.0195$ \\
        Tictactoe & $0.9896 \pm 0.0037$ \\
        Wine & $0.8670 \pm 0.0023$ \\
        \bottomrule
    \end{tabular}
\end{table}

\begin{table*}[t]
    \centering
    \scriptsize
    \setlength{\tabcolsep}{3.5pt}
    \caption{Churn regularization sweep: mean test classification accuracy. Rows vary the sparsity penalty and columns vary the deferral penalty; both penalty values are normalized by the full dataset size $N$. All other hyperparameters are selected to maximize validation accuracy. Values are means across five train/test splits. }
    \label{tab:churn-sweep-accuracy}
    \begin{tabular}{lccccccccc}
        \toprule
        & \multicolumn{9}{c}{Deferral penalty} \\
        \cmidrule(lr){2-10}
        Sparsity penalty & 0.001 & 0.005 & 0.01 & 0.05 & 0.1 & 0.3 & 0.6 & 0.8 & 1 \\
        \midrule
        0.000010 & 0.953 & 0.944 & 0.935 & 0.911 & 0.920 & 0.919 & 0.920 & 0.920 & 0.920 \\
        0.000050 & 0.956 & 0.956 & 0.949 & 0.939 & 0.934 & 0.927 & 0.925 & 0.925 & 0.925 \\
        0.000100 & 0.957 & 0.957 & 0.955 & 0.950 & 0.944 & 0.932 & 0.932 & 0.932 & 0.932 \\
        0.000500 & 0.957 & 0.957 & 0.957 & 0.955 & 0.959 & 0.948 & 0.947 & 0.947 & 0.947 \\
        0.001000 & 0.957 & 0.957 & 0.957 & 0.956 & 0.958 & 0.943 & 0.943 & 0.943 & 0.943 \\
        0.005000 & 0.957 & 0.957 & 0.957 & 0.958 & 0.950 & 0.921 & 0.920 & 0.920 & 0.920 \\
        0.010000 & 0.957 & 0.957 & 0.957 & 0.957 & 0.945 & 0.902 & 0.883 & 0.883 & 0.883 \\
        \bottomrule
    \end{tabular}
\end{table*}

\begin{table*}[t]
    \centering
    \scriptsize
    \setlength{\tabcolsep}{3.5pt}
    \caption{Churn regularization sweep: mean proportion of test examples deferred to the final fallback model. Rows vary the sparsity penalty and columns vary the deferral penalty; both penalty values are normalized by the full dataset size $N$. All other hyperparameters are selected to maximize validation accuracy. Values are means across five train/test splits.}
    \label{tab:churn-sweep-deferral}
    \begin{tabular}{lccccccccc}
        \toprule
        & \multicolumn{9}{c}{Deferral penalty} \\
        \cmidrule(lr){2-10}
        Sparsity penalty & 0.001 & 0.005 & 0.01 & 0.05 & 0.1 & 0.3 & 0.6 & 0.8 & 1 \\
        \midrule
        0.000010 & 0.656 & 0.221 & 0.174 & 0.017 & 0.010 & 0.001 & 0.000 & 0.000 & 0.000 \\
        0.000050 & 0.993 & 0.695 & 0.273 & 0.067 & 0.018 & 0.001 & 0.000 & 0.000 & 0.000 \\
        0.000100 & 1.000 & 1.000 & 0.498 & 0.064 & 0.020 & 0.000 & 0.000 & 0.000 & 0.000 \\
        0.000500 & 1.000 & 1.000 & 1.000 & 0.143 & 0.090 & 0.000 & 0.000 & 0.000 & 0.000 \\
        0.001000 & 1.000 & 1.000 & 1.000 & 0.232 & 0.131 & 0.001 & 0.000 & 0.000 & 0.000 \\
        0.005000 & 1.000 & 1.000 & 1.000 & 0.875 & 0.269 & 0.000 & 0.000 & 0.000 & 0.000 \\
        0.010000 & 1.000 & 1.000 & 1.000 & 1.000 & 0.299 & 0.062 & 0.000 & 0.000 & 0.000 \\
        \bottomrule
    \end{tabular}
\end{table*}

\begin{table*}[t]
    \centering
    \scriptsize
    \setlength{\tabcolsep}{3.5pt}
    \caption{Churn regularization sweep: mean number of leaves reported for the MDT component. Rows vary the sparsity penalty and columns vary the deferral penalty; both penalty values are normalized by the full dataset size $N$. All other hyperparameters are selected to maximize validation accuracy. Values are means across five train/test splits.}
    \label{tab:churn-sweep-leaves}
    \begin{tabular}{lccccccccc}
        \toprule
        & \multicolumn{9}{c}{Deferral penalty} \\
        \cmidrule(lr){2-10}
        Sparsity penalty & 0.001 & 0.005 & 0.01 & 0.05 & 0.1 & 0.3 & 0.6 & 0.8 & 1 \\
        \midrule
        0.000010 & 73.2 & 191.0 & 191.6 & 261.6 & 223.4 & 208.6 & 207.2 & 207.2 & 207.2 \\
        0.000050 & 10.6 & 30.2 & 136.6 & 171.8 & 179.4 & 180.4 & 185.6 & 185.6 & 185.6 \\
        0.000100 & 2.0 & 2.0 & 69.8 & 230.0 & 198.2 & 136.8 & 138.4 & 138.4 & 138.4 \\
        0.000500 & 2.0 & 2.0 & 2.0 & 86.0 & 41.4 & 36.6 & 34.8 & 34.8 & 34.8 \\
        0.001000 & 2.0 & 2.0 & 2.0 & 24.8 & 17.0 & 25.8 & 21.6 & 21.6 & 21.6 \\
        0.005000 & 2.0 & 2.0 & 2.0 & 2.8 & 7.6 & 14.2 & 9.0 & 9.0 & 9.0 \\
        0.010000 & 2.0 & 2.0 & 2.0 & 2.0 & 5.0 & 5.6 & 3.6 & 3.6 & 3.6 \\
        \bottomrule
    \end{tabular}
\end{table*}

\begin{table*}[t]
    \centering
    \scriptsize
    \setlength{\tabcolsep}{3.5pt}
    \caption{Phishing regularization sweep: mean test classification accuracy. Rows vary the sparsity penalty and columns vary the deferral penalty; both penalty values are normalized by the full dataset size $N$. All other hyperparameters are selected to maximize validation accuracy. Values are means across five train/test splits.}
    \label{tab:phishing-sweep-accuracy}
    \begin{tabular}{lccccccccc}
        \toprule
        & \multicolumn{9}{c}{Deferral penalty} \\
        \cmidrule(lr){2-10}
        Sparsity penalty & 0.001 & 0.005 & 0.01 & 0.05 & 0.1 & 0.3 & 0.6 & 0.8 & 1 \\
        \midrule
        0.000010 & 0.973 & 0.971 & 0.971 & 0.965 & 0.961 & 0.951 & 0.952 & 0.952 & 0.952 \\
        0.000050 & 0.973 & 0.973 & 0.974 & 0.970 & 0.969 & 0.954 & 0.953 & 0.953 & 0.953 \\
        0.000100 & 0.974 & 0.973 & 0.973 & 0.972 & 0.969 & 0.954 & 0.952 & 0.952 & 0.952 \\
        0.000500 & 0.973 & 0.974 & 0.973 & 0.970 & 0.965 & 0.946 & 0.942 & 0.942 & 0.942 \\
        0.001000 & 0.973 & 0.973 & 0.974 & 0.969 & 0.966 & 0.943 & 0.937 & 0.937 & 0.937 \\
        0.005000 & 0.973 & 0.973 & 0.973 & 0.970 & 0.963 & 0.938 & 0.914 & 0.914 & 0.914 \\
        0.010000 & 0.973 & 0.973 & 0.973 & 0.970 & 0.943 & 0.938 & 0.911 & 0.911 & 0.911 \\
        \bottomrule
    \end{tabular}
\end{table*}

\begin{table*}[t]
    \centering
    \scriptsize
    \setlength{\tabcolsep}{3.5pt}
    \caption{Phishing regularization sweep: mean proportion of test examples deferred to the final fallback model. Rows vary the sparsity penalty and columns vary the deferral penalty; both penalty values are normalized by the full dataset size $N$. All other hyperparameters are selected to maximize validation accuracy. Values are means across five train/test splits.}
    \label{tab:phishing-sweep-deferral}
    \begin{tabular}{lccccccccc}
        \toprule
        & \multicolumn{9}{c}{Deferral penalty} \\
        \cmidrule(lr){2-10}
        Sparsity penalty & 0.001 & 0.005 & 0.01 & 0.05 & 0.1 & 0.3 & 0.6 & 0.8 & 1 \\
        \midrule
        0.000010 & 0.297 & 0.145 & 0.110 & 0.092 & 0.067 & 0.005 & 0.000 & 0.000 & 0.000 \\
        0.000050 & 0.598 & 0.326 & 0.259 & 0.088 & 0.092 & 0.003 & 0.000 & 0.000 & 0.000 \\
        0.000100 & 0.598 & 0.440 & 0.283 & 0.144 & 0.083 & 0.006 & 0.000 & 0.000 & 0.000 \\
        0.000500 & 1.000 & 0.598 & 0.577 & 0.233 & 0.130 & 0.005 & 0.000 & 0.000 & 0.000 \\
        0.001000 & 1.000 & 1.000 & 0.598 & 0.262 & 0.162 & 0.028 & 0.000 & 0.000 & 0.000 \\
        0.005000 & 1.000 & 1.000 & 1.000 & 0.577 & 0.310 & 0.078 & 0.000 & 0.000 & 0.000 \\
        0.010000 & 1.000 & 1.000 & 1.000 & 0.708 & 0.153 & 0.078 & 0.000 & 0.000 & 0.000 \\
        \bottomrule
    \end{tabular}
\end{table*}

\begin{table*}[t]
    \centering
    \scriptsize
    \setlength{\tabcolsep}{3.5pt}
    \caption{Phishing regularization sweep: mean number of leaves reported for the fitted MDT. Rows vary the sparsity penalty and columns vary the deferral penalty; both penalty values are normalized by the full dataset size $N$. All other hyperparameters are selected to maximize validation accuracy. Values are means across five train/test splits.}
    \label{tab:phishing-sweep-leaves}
    \begin{tabular}{lccccccccc}
        \toprule
        & \multicolumn{9}{c}{Deferral penalty} \\
        \cmidrule(lr){2-10}
        Sparsity penalty & 0.001 & 0.005 & 0.01 & 0.05 & 0.1 & 0.3 & 0.6 & 0.8 & 1 \\
        \midrule
        0.000010 & 77.0 & 203.0 & 235.2 & 228.6 & 285.2 & 296.4 & 293.0 & 293.0 & 293.0 \\
        0.000050 & 5.4 & 34.8 & 59.4 & 218.2 & 187.0 & 230.6 & 238.2 & 238.2 & 238.2 \\
        0.000100 & 5.0 & 14.8 & 57.0 & 133.8 & 140.8 & 161.4 & 139.6 & 139.6 & 139.6 \\
        0.000500 & 2.0 & 5.0 & 6.6 & 41.0 & 89.6 & 66.4 & 29.2 & 29.2 & 29.2 \\
        0.001000 & 2.0 & 2.0 & 5.0 & 13.6 & 41.4 & 37.6 & 16.4 & 16.4 & 16.4 \\
        0.005000 & 2.0 & 2.0 & 2.0 & 4.0 & 5.6 & 5.0 & 4.0 & 4.0 & 4.0 \\
        0.010000 & 2.0 & 2.0 & 2.0 & 3.0 & 4.4 & 5.0 & 3.8 & 3.8 & 3.8 \\
        \bottomrule
    \end{tabular}
\end{table*}

\begin{table}[t]
    \centering
    \small
    \setlength{\tabcolsep}{5pt}
    \caption{Test accuracy of TabPFN \citep{hollmann2025tabpfn} and an MDT with TabPFN as the final fallback model, with at most 25\% deferral. Entries show mean test accuracy $\pm$ standard error over five train/test splits; higher is better.}
    \label{tab:mdt-tabpfn-fallback}
    \begin{tabular}{lcc}
        \toprule
        Dataset & TabPFN & MDT+TabPFN ($\leq$25\%) \\
        \midrule
        Churn & $0.972 \pm 0.003$ & $0.968 \pm 0.004$ \\
        Compas & $0.679 \pm 0.004$ & $0.672 \pm 0.003$ \\
        Heloc & $0.714 \pm 0.010$ & $0.708 \pm 0.010$ \\
        Phishing & $0.975 \pm 0.002$ & $0.975 \pm 0.002$ \\
        Student & $0.815 \pm 0.014$ & $0.812 \pm 0.012$ \\
        \bottomrule
    \end{tabular}
\end{table}

\begin{table}[t]
\centering
\small
\setlength{\tabcolsep}{6pt}
\renewcommand{\arraystretch}{0.95}
\begin{tabular}{lcc}
\toprule
Dataset
& $\Delta$ Accuracy (pp)
& $\Delta$ Deferral (pp) \\
\midrule
Abalone  & $+0.98 \pm 1.65$ & $-0.19 \pm 0.43$ \\
Aging    & $ 0.00 \pm 0.00$ & $ 0.00 \pm 0.00$ \\
Bike     & $+0.07 \pm 0.17$ & $+2.29 \pm 3.94$ \\
Churn    & $+0.28 \pm 0.43$ & $+1.20 \pm 5.76$ \\
Droid    & $-0.02 \pm 0.12$ & $+0.25 \pm 3.41$ \\
Heloc    & $-0.28 \pm 0.87$ & $ 0.00 \pm 0.00$ \\
Phishing & $+0.02 \pm 0.31$ & $-0.43 \pm 5.08$ \\
Pol      & $+0.07 \pm 0.22$ & $-5.34 \pm 3.50$ \\
RL       & $-0.08 \pm 0.61$ & $+8.61 \pm 5.89$ \\
Shopping & $-0.01 \pm 0.13$ & $-5.38 \pm 7.62$ \\
Spambase & $-0.04 \pm 0.16$ & $+0.74 \pm 3.19$ \\
Wine     & $+0.02 \pm 0.80$ & $-2.80 \pm 5.27$ \\
\bottomrule
\end{tabular}
\caption{
Ablation comparing a fixed binarization derived from XGBoost splits at the root
with recomputing the binarization at each stage of the MDT.
Entries report the mean $\pm$ standard deviation of differences
(fixed binarization minus recomputed) across five train/test splits.
For each split, hyperparameters are selected to maximize mean validation
accuracy subject to at most 25\% test deferral. Positive values indicate higher accuracy or deferral under fixed binarization.
}
\label{tab:mdt_binarization_ablation}
\end{table}

\clearpage
\clearpage

%% file: notes/old-methods.tex
\begin{figure}[!t]
    \centering

    \begin{subfigure}{0.92\linewidth}
        \centering
        \includegraphics[width=\linewidth, trim=0 0 0 100, clip]{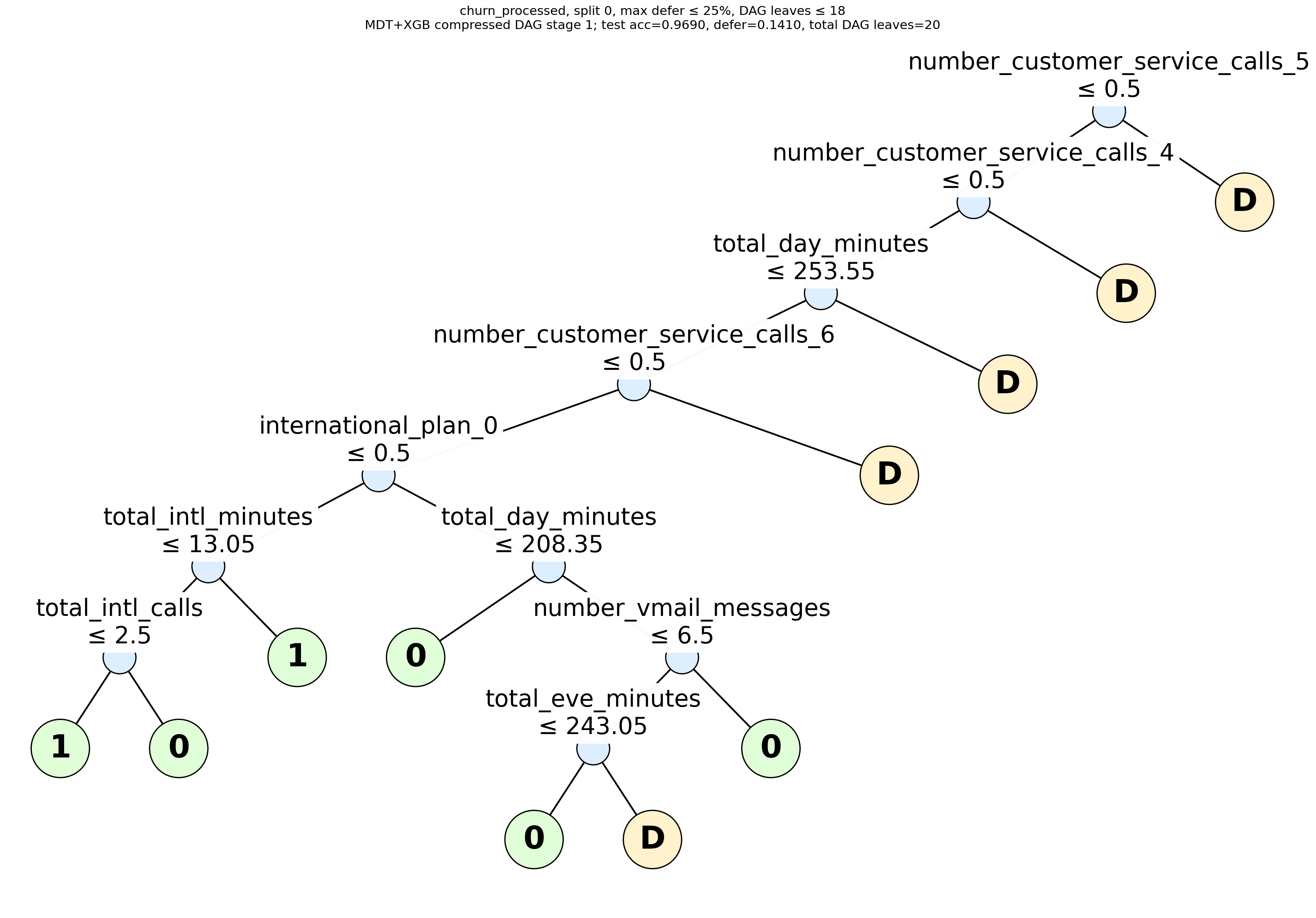}
        \caption{MDT  stage 1}
    \end{subfigure}

    \vspace{0.6em}

    \begin{subfigure}{0.92\linewidth}
        \centering
        \includegraphics[width=\linewidth, trim=0 0 0 100, clip]{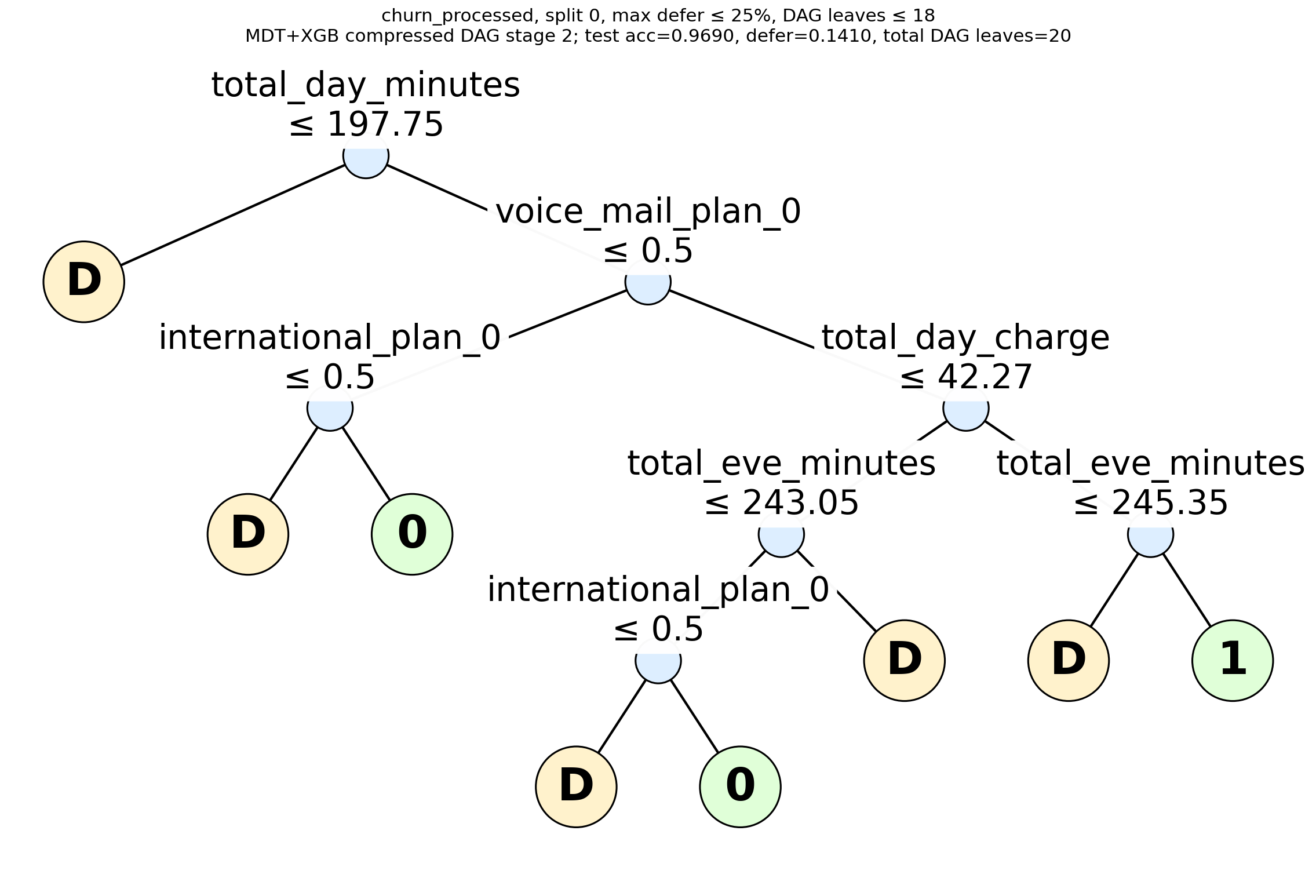}
        \caption{MDT  stage 2}
    \end{subfigure}

    \vspace{0.6em}

    \begin{subfigure}{0.98\linewidth}
        \centering
        \includegraphics[width=\linewidth, trim=0 0 0 100, clip]{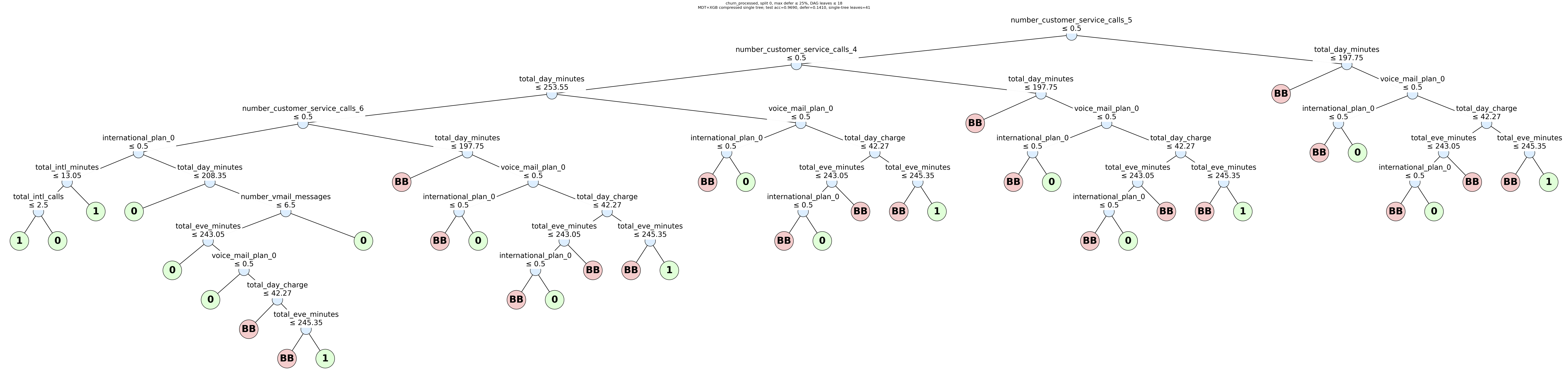}
        \caption{Compressed single-tree representation}
    \end{subfigure}

    \caption{Compressed MDT stages and corresponding compressed single-tree representation for Churn.}
    \label{fig:churn-compressed-mdt-single-tree}
\end{figure}

\begin{figure}[!t]
    \centering

    \begin{subfigure}{0.47\linewidth}
        \centering
        \includegraphics[width=\linewidth, trim=0 0 0 100, clip]{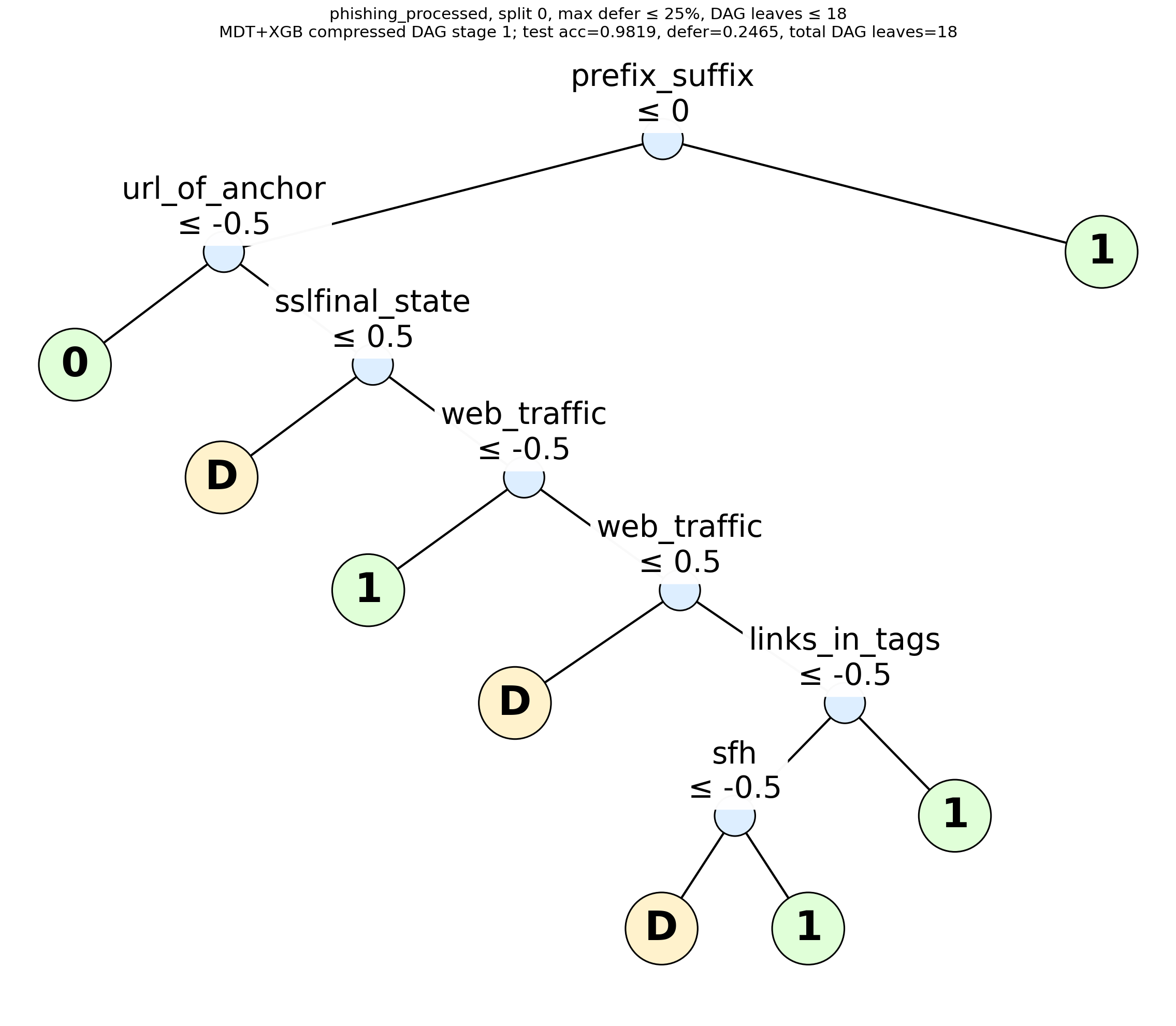}
        \caption{MDT stage 1}
    \end{subfigure}
    \hfill
    \begin{subfigure}{0.47\linewidth}
        \centering
        \includegraphics[width=\linewidth, trim=0 0 0 100, clip]{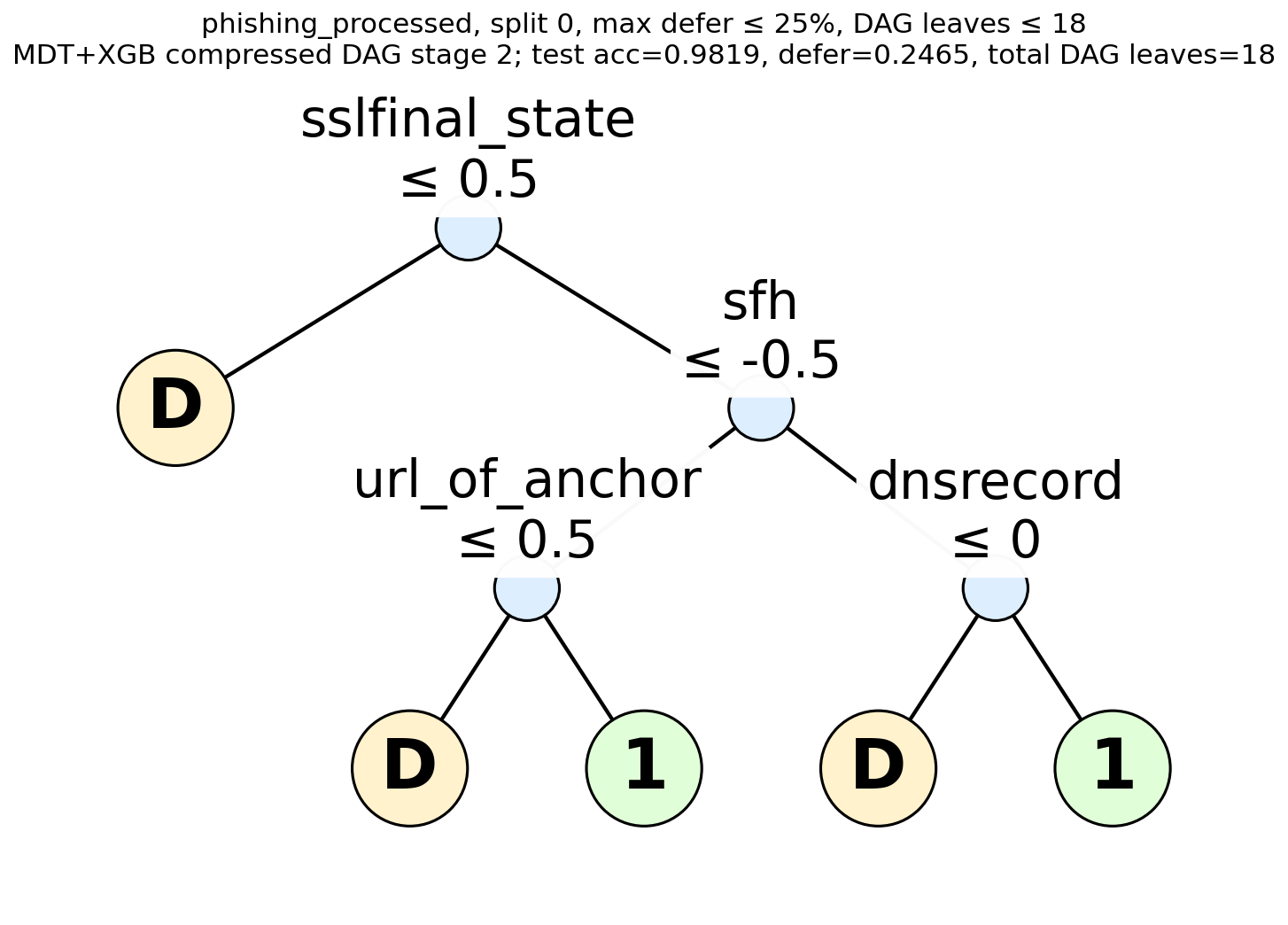}
        \caption{MDT  stage 2}
    \end{subfigure}

    \vspace{0.5em}

    \begin{subfigure}{0.72\linewidth}
        \centering
        \includegraphics[width=\linewidth, trim=0 0 0 100, clip]{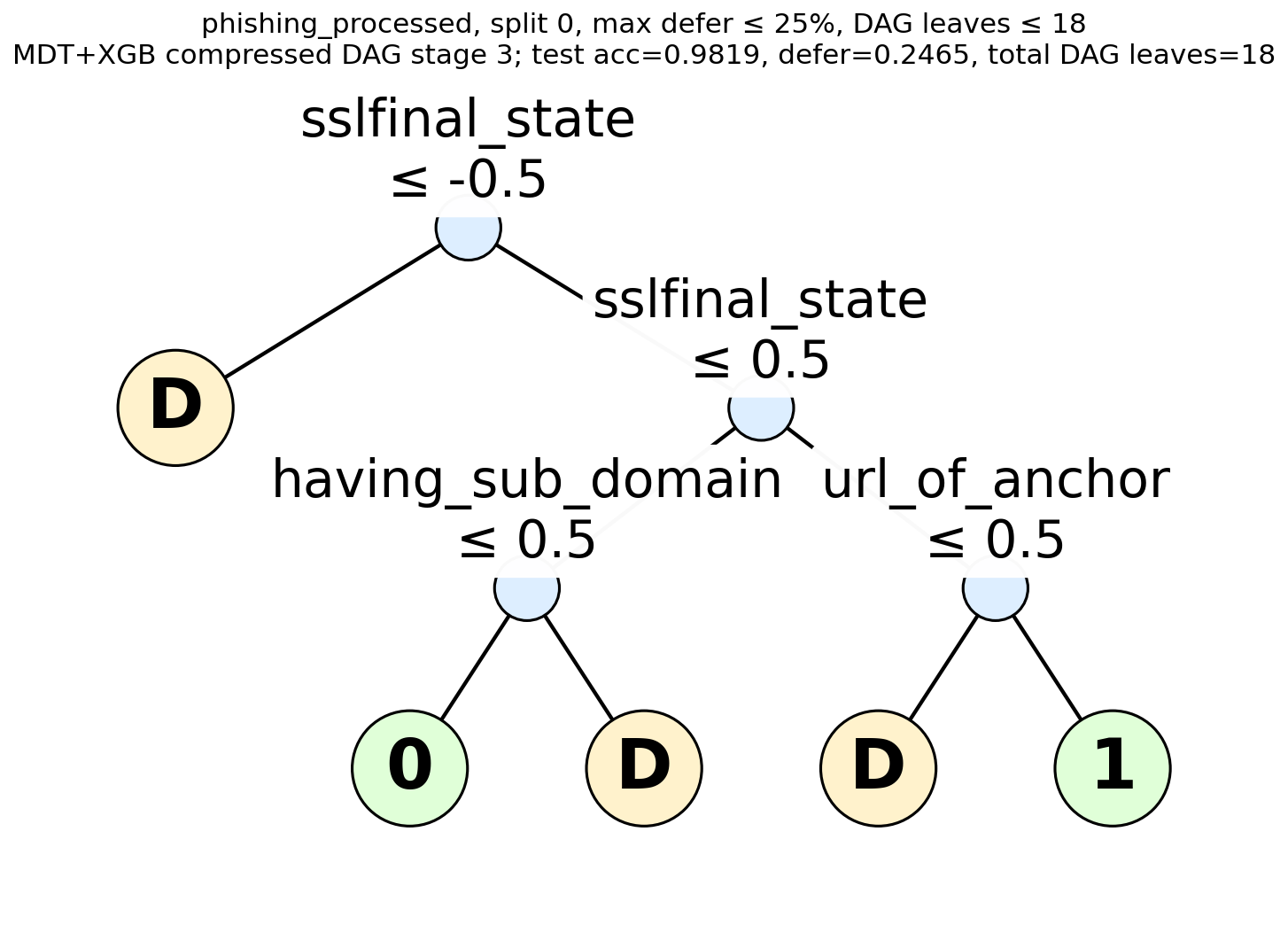}
        \caption{MDT  stage 3}
    \end{subfigure}

    \vspace{0.5em}

    \begin{subfigure}{0.92\linewidth}
        \centering
        \includegraphics[width=\linewidth, trim=0 0 0 100, clip]{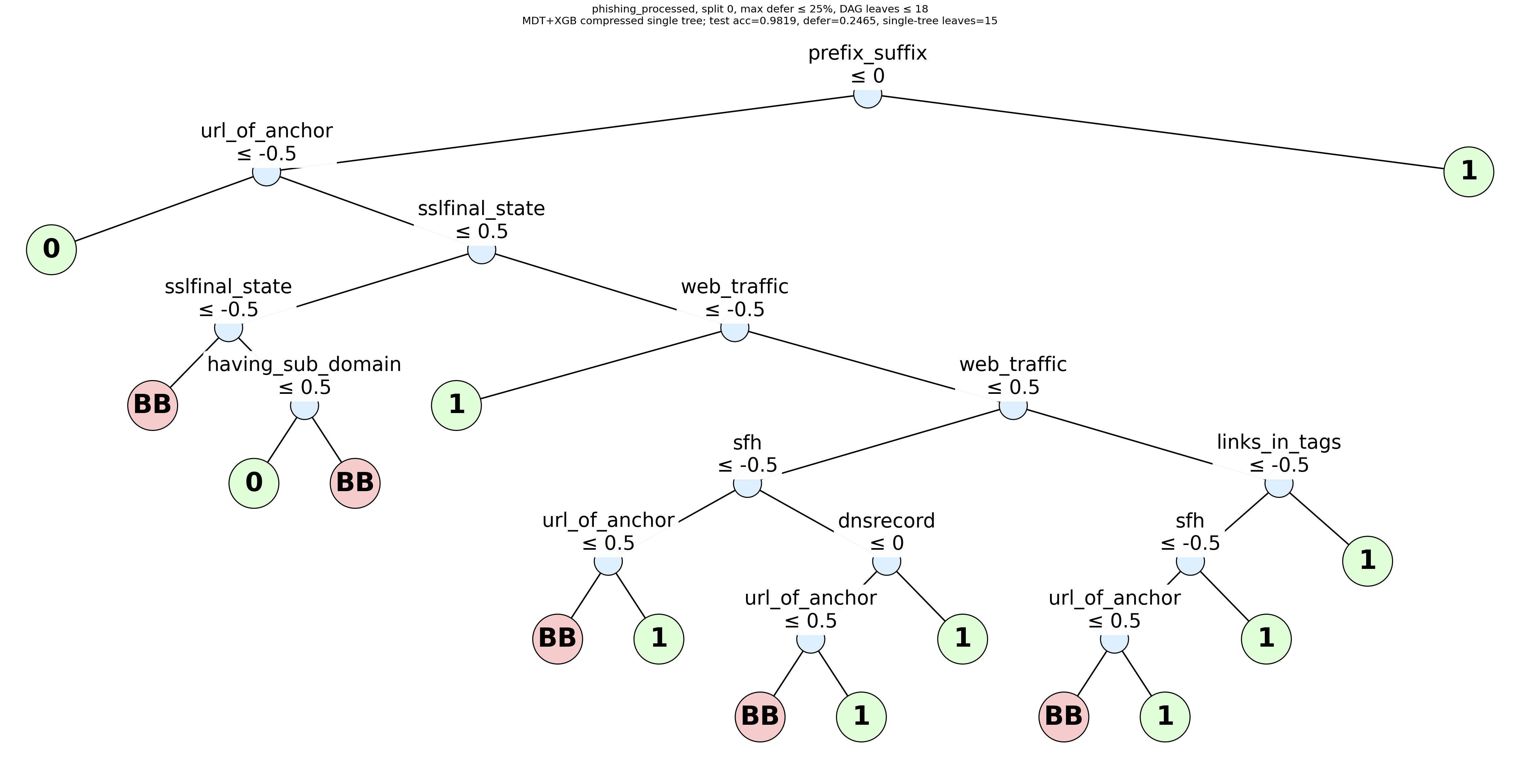}
        \caption{Compressed single-tree representation}
    \end{subfigure}

    \caption{Compressed MDT stages and corresponding compressed single-tree representation for Phishing.}
    \label{fig:phishing-compressed-mdt-single-tree}
\end{figure}

\begin{figure}[!t]\centering\begin{tcolorbox}[    enhanced,    colback=gray!3,    colframe=black!65,    boxrule=0.6pt,    arc=2mm,    left=1.5mm,    right=1.5mm,    top=1mm,    bottom=1mm,    width=\linewidth]\small\noindent\textbf{Rule Bank \quad}\vspace{0.4em}\centering\begin{tabular}{@{}r l c@{}}$r_1$ &\parbox[t]{0.74\linewidth}{$\texttt{prefix\_suffix} > 0$}& $\mapsto 1$ \\[0.6em]$r_2$ &\parbox[t]{0.74\linewidth}{$\texttt{url\_of\_anchor} \le -0.5$}& $\mapsto 0$ \\[0.6em]$r_3$ &\parbox[t]{0.74\linewidth}{$\texttt{sslfinal\_state} > 0.5 \wedge \texttt{web\_traffic} \le -0.5$}& $\mapsto 1$ \\[0.6em]$r_4$ &\parbox[t]{0.74\linewidth}{$\texttt{sslfinal\_state} > 0.5 \wedge \texttt{web\_traffic} > 0.5$\\$\wedge\ \texttt{links\_in\_tags} > -0.5$}& $\mapsto 1$ \\[0.6em]$r_5$ &\parbox[t]{0.74\linewidth}{$\texttt{sslfinal\_state} > 0.5 \wedge \texttt{web\_traffic} > 0.5$\\$\wedge\ \texttt{sfh} > -0.5$}& $\mapsto 1$ \\[0.6em]$r_6$ &\parbox[t]{0.74\linewidth}{$\texttt{sslfinal\_state} > 0.5 \wedge \texttt{sfh} \le -0.5$\\$\wedge\ \texttt{url\_of\_anchor} > 0.5$}& $\mapsto 1$ \\[0.6em]$r_7$ &\parbox[t]{0.74\linewidth}{$\texttt{sslfinal\_state} > 0.5 \wedge \texttt{sfh} > -0.5$\\$\wedge\ \texttt{dnsrecord} > 0$}& $\mapsto 1$ \\[0.6em]$r_8$ &\parbox[t]{0.74\linewidth}{$\texttt{sslfinal\_state} > -0.5 \wedge \texttt{sslfinal\_state} \le 0.5$\\$\wedge\ \texttt{having\_sub\_domain} \le 0.5$}& $\mapsto 0$ \\[0.6em]$r_9$ &\parbox[t]{0.74\linewidth}{$\texttt{sslfinal\_state} > 0.5 \wedge \texttt{url\_of\_anchor} > 0.5$}& $\mapsto 1$\end{tabular}\vspace{0.7em}\hrule\vspace{0.7em}\noindent\textbf{Sequential Rule List \quad}\vspace{0.4em}\centering\begin{tabular}{@{}l p{0.78\linewidth}@{}}\textbf{Stage 1:} &$r_1 \Rightarrow 1,\;  r_2 \Rightarrow 0,\;  r_3 \Rightarrow 1,\;  r_4 \Rightarrow 1,\;  r_5 \Rightarrow 1,$\textit{ (move to next stage)} \\[0.8em]\textbf{Stage 2:} &$r_6 \Rightarrow 1,\;  r_7 \Rightarrow 1,$\textit{ (move to next stage)} \\[0.8em]\textbf{Stage 3:} &$r_8 \Rightarrow 0,\;  r_9 \Rightarrow 1,$\textit{ else defer}\end{tabular}\end{tcolorbox}\caption{Phishing Rule List Representation.}\label{fig:phishing-rule-list-representation}\end{figure}

%% file: references.bib
@inproceedings{wang2019gaining,
  title={Gaining no or low-cost transparency with interpretable partial substitute},
  author={Wang, Tong},
  booktitle={International Conference on Machine Learning},
  year={2019}
}

@article{wang2021hybrid,
author = {Wang, Tong and Lin, Qihang},
title = {Hybrid predictive models: when an interpretable model collaborates with a black-box model},
year = {2021},
issue_date = {January 2021},
publisher = {JMLR.org},
volume = {22},
number = {1},
issn = {1532-4435},
journal = {J. Mach. Learn. Res.},
month = jan,
articleno = {137},
numpages = {38}
}

@inproceedings{pan2020interpretable,
  title={Interpretable companions for black-box models},
  author={Pan, Danqing and Wang, Tong and Hara, Satoshi},
  booktitle={International conference on artificial intelligence and statistics},
  pages={2444--2454},
  year={2020},
  organization={PMLR}
}

@article{ferry2023learning,
  title={Learning hybrid interpretable models: Theory, taxonomy, and methods},
  author={Ferry, Julien and Laberge, Gabriel and A{\"\i}vodji, Ulrich},
  journal={arXiv preprint arXiv:2303.04437},
  year={2023}
}

@inproceedings{vidal2020born,
  title={Born-again tree ensembles},
  author={Vidal, Thibaut and Schiffer, Maximilian},
  booktitle={International conference on machine learning},
  pages={9743--9753},
  year={2020},
  organization={PMLR}
}

@inproceedings{mctavish2022fast,
  title={Fast sparse decision tree optimization via reference ensembles},
  author={McTavish, Hayden and Zhong, Chudi and Achermann, Reto and Karimalis, Ilias and Chen, Jacques and Rudin, Cynthia and Seltzer, Margo},
  booktitle={Proceedings of the {AAAI} Conference on Artificial Intelligence},
  volume={36},
  pages={9604--9613},
  year={2022}
}

@article{sagi2021approximating,
  title={Approximating XGBoost with an interpretable decision tree},
  author={Sagi, Omer and Rokach, Lior},
  journal={Information Sciences},
  volume={572},
  pages={522--542},
  year={2021},
  publisher={Elsevier}
}

@inproceedings{chen2016xgboost,
  title={{Xgboost: A Scalable Tree Boosting System}},
  author={Chen, Tianqi and Guestrin, Carlos},
  booktitle={Proceedings of the 22nd {ACM SIGKDD} International Conference on Knowledge Discovery and Data Mining},
  pages={785--794},
  year={2016}
}

@article{madras2018predict,
  title={Predict responsibly: improving fairness and accuracy by learning to defer},
  author={Madras, David and Pitassi, Toni and Zemel, Richard},
  journal={Advances in neural information processing systems},
  volume={31},
  year={2018}
}

@inproceedings{mozannar2020consistent,
  title={Consistent estimators for learning to defer to an expert},
  author={Mozannar, Hussein and Sontag, David},
  booktitle={International conference on machine learning},
  pages={7076--7087},
  year={2020},
  organization={PMLR}
}

@article{chow2003optimum,
  title={On optimum recognition error and reject tradeoff},
  author={Chow, C},
  journal={IEEE Transactions on information theory},
  volume={16},
  number={1},
  pages={41--46},
  year={2003},
  publisher={IEEE}
}

@article{breiman2001random,
  title={Random Forests},
  author={Breiman, Leo},
  journal={Machine Learning},
  volume={45},
  pages={5--32},
  year={2001},
  publisher={Springer}
}

@misc{breiman1984classification,
  title={Classification and regression trees. Wadsworth \& Brooks},
  author={Breiman, Leo and Friedman, Jerome H and Olshen, Richard A and Stone, Charles J},
  year={1984},
  publisher={Cole Advanced books \& software Pacific Grove, CA}
}

@article{murtree,
  title={Murtree: Optimal decision trees via dynamic programming and search},
  author={Demirovi{\'c}, Emir and Lukina, Anna and Hebrard, Emmanuel and Chan, Jeffrey and Bailey, James and Leckie, Christopher and Ramamohanarao, Kotagiri and Stuckey, Peter J},
  journal={Journal of Machine Learning Research},
  volume={23},
  number={26},
  pages={1--47},
  year={2022}
}

@inproceedings{freund1997using,
  title={Using and combining predictors that specialize},
  author={Freund, Yoav and Schapire, Robert E and Singer, Yoram and Warmuth, Manfred K},
  booktitle={Proceedings of the twenty-ninth annual ACM symposium on Theory of computing},
  pages={334--343},
  year={1997}
}

@article{el2010foundations,
  title={On the Foundations of Noise-free Selective Classification.},
  author={El-Yaniv, Ran and others},
  journal={Journal of Machine Learning Research},
  volume={11},
  number={5},
  year={2010}
}

@article{bertsimas2017optimal,
  title={Optimal classification trees},
  author={Bertsimas, Dimitris and Dunn, Jack},
  journal={Machine Learning},
  volume={106},
  pages={1039--1082},
  year={2017},
  publisher={Springer}
}

@inproceedings{verwer2019learning,
  title={Learning optimal classification trees using a binary linear program formulation},
  author={Verwer, Sicco and Zhang, Yingqian},
  booktitle={Proceedings of the AAAI Conference on Artificial Intelligence},
  volume={33},
  pages={1625--1632},
  year={2019}
}

@inproceedings{dl85,
  title={Learning optimal decision trees using caching branch-and-bound search},
  author={Aglin, Ga{\"e}l and Nijssen, Siegfried and Schaus, Pierre},
  booktitle={Proceedings of the AAAI Conference on Artificial Intelligence},
  volume={34},
  pages={3146--3153},
  year={2020}
}

@article{quinlan1986induction,
  title={Induction of decision trees},
  author={Quinlan, J. Ross},
  journal={Machine learning},
  volume={1},
  number={1},
  pages={81--106},
  year={1986},
  publisher={Springer}
}

@inproceedings{kiossou2022time,
  title={Time constrained dl8.5 using limited discrepancy search},
  author={Kiossou, Harold and Schaus, Pierre and Nijssen, Siegfried and Houndji, Vinasetan Ratheil},
  booktitle={Joint European Conference on Machine Learning and Knowledge Discovery in Databases},
  pages={443--459},
  year={2022},
  organization={Springer}
}

@inproceedings{demirovic2023blossom,
  title={Blossom: an anytime algorithm for computing optimal decision trees},
  author={Demirovi{\'c}, Emir and Hebrard, Emmanuel and Jean, Louis},
  booktitle={International Conference on Machine Learning},
  pages={7533--7562},
  year={2023},
  organization={PMLR}
}

@article{topk,
  title={Harnessing the power of choices in decision tree learning},
  author={Blanc, Guy and Lange, Jane and Pabbaraju, Chirag and Sullivan, Colin and Tan, Li-Yang and Tiwari, Mo},
  journal={Advances in Neural Information Processing Systems},
  volume={36},
  year={2024}
}

@inproceedings{osdt,
 author = {Hu, Xiyang and Rudin, Cynthia and Seltzer, Margo},
 booktitle = {Advances in Neural Information Processing Systems},
 pages = {7265-7273},
 title = {Optimal Sparse Decision Trees},
 volume = {32},
 year = {2019}
}

@InProceedings{quantbnb,
  title = 	 {Quant-{B}n{B}: A Scalable Branch-and-Bound Method for Optimal Decision Trees with Continuous Features},
  author =       {Mazumder, Rahul and Meng, Xiang and Wang, Haoyue},
  booktitle = 	 {International Conference on Machine Learning},
  pages = 	 {15255--15277},
  year = 	 {2022},
  volume = 	 {162},
  month = 	 {17--23 Jul},
  publisher =    {PMLR}
}

@inproceedings{brița2025optimal,
  title={Optimal classification trees for continuous feature data using dynamic programming with branch-and-bound},
  author={Brița, C{\u{a}}t{\u{a}}lin E and van der Linden, Jacobus GM and Demirovi{\'c}, Emir},
  booktitle={Proceedings of the AAAI Conference on Artificial Intelligence},
  volume={39},
  pages={11131--11139},
  year={2025}
}

@inproceedings{
babbar2025near,
title={Near-Optimal Decision Trees in a {SPLIT} Second},
author={Varun Babbar and Hayden McTavish and Cynthia Rudin and Margo Seltzer},
booktitle={International Conference on Machine Learning},
year={2025},
}

@inproceedings{kiossou2024efficient,
  title={Efficient Lookahead Decision Trees},
  author={Kiossou, Harold and Schaus, Pierre and Nijssen, Siegfried and Aglin, Ga{\"e}l},
  booktitle={International Symposium on Intelligent Data Analysis},
  pages={133--144},
  year={2024},
  organization={Springer}
}

@article{kiossou2025generic,
  title={A generic complete anytime beam search for optimal decision tree},
  author={Kiossou, Harold Silvere and Nijssen, Siegfried and Schaus, Pierre},
  journal={arXiv preprint arXiv:2508.06064},
  year={2025}
}

@inproceedings{gosdt,
  title={Generalized and scalable optimal sparse decision trees},
  author={Lin, Jimmy and Zhong, Chudi and Hu, Diane and Rudin, Cynthia and Seltzer, Margo},
  booktitle={International Conference on Machine Learning},
  pages={6150--6160},
  year={2020},
  organization={PMLR}
}

@inproceedings{frost2024partially,
  title={Partially interpretable models with guarantees on coverage and accuracy},
  author={Frost, Nave and Lipton, Zachary and Mansour, Yishay and Moshkovitz, Michal},
  booktitle={International conference on algorithmic learning theory},
  pages={590--613},
  year={2024},
  organization={PMLR}
}

@article{semenova2023path,
  title={A Path to Simpler Models Starts With Noise},
  author={Semenova, Lesia and Chen, Harry and Parr, Ronald and Rudin, Cynthia},
  journal={Advances in Neural Information Processing Systems},
  volume={36},
  year={2023}
}

@inproceedings{semenova2022existence,
  title={On the existence of simpler machine learning models},
  author={Semenova, Lesia and Rudin, Cynthia and Parr, Ronald},
  booktitle={2022 ACM Conference on Fairness, Accountability, and Transparency},
  pages={1827--1858},
  year={2022}
}

@inproceedings{
boner2024using,
title={Using Noise to Infer Aspects of Simplicity Without Learning},
author={Zachery Boner and Harry Chen and Lesia Semenova and Ronald Parr and Cynthia Rudin},
booktitle={Advances In Neural Information Processing Systems},
year={2024}
}

@article{tan2025fast,
  author = {Yan Shuo Tan and Chandan Singh and Keyan Nasseri and Abhineet Agarwal and James Duncan and Omer Ronen and Matthew Epland and Aaron Kornblith and Bin Yu},
  title = {Fast Interpretable Greedy-Tree Sums},
  journal = {Proceedings of the National Academy of Sciences},
  volume = {122},
  number = {7},
  pages = {e2310151122},
  year = {2025},
  doi = {10.1073/pnas.2310151122}
}

@article{Li2026, 
author = {Yifan Li and Shuhan Qi and Lei Cui and Chao Xing and Lei Zhang and Xuan Wang},
title = {Interpret When Possible: A Tree-Based Hybrid Framework for Interpretable Classification},
year = {2026},
journal = {Big Data Mining and Analytics},
volume = {9},
number = {1},
pages = {263-283},
url = {https://www.sciopen.com/article/10.26599/BDMA.2025.9020055},
doi = {10.26599/BDMA.2025.9020055},
}

@article{SANTOSPEREIRA2005943,
title = {On optimal reject rules and ROC curves},
journal = {Pattern Recognition Letters},
volume = {26},
number = {7},
pages = {943-952},
year = {2005},
issn = {0167-8655},
doi = {https://doi.org/10.1016/j.patrec.2004.09.042},
url = {https://www.sciencedirect.com/science/article/pii/S0167865504002892},
author = {Carla M. Santos-Pereira and Ana M. Pires}
}

@inproceedings{schapire,
author = {Schapire, Robert E. and Singer, Yoram},
title = {Improved boosting algorithms using confidence-rated predictions},
year = {1998},
isbn = {1581130570},
publisher = {Association for Computing Machinery},
address = {New York, NY, USA},
url = {https://doi.org/10.1145/279943.279960},
doi = {10.1145/279943.279960},
booktitle = {Proceedings of the Eleventh Annual Conference on Computational Learning Theory},
pages = {80–91},
numpages = {12},
location = {Madison, Wisconsin, USA},
series = {COLT' 98}
}

@misc{adult_2,
  author       = {Becker, Barry and Kohavi, Ronny},
  title        = {{Adult}},
  year         = {1996},
  howpublished = {UCI Machine Learning Repository},
  note         = {{DOI}: https://doi.org/10.24432/C5XW20}
}

@misc{malani2017npha,
  author    = {Malani, Preeti N. and Kullgren, Jeffrey and Solway, Erica},
  title     = {National Poll on Healthy Aging (NPHA), United States, April 2017},
  year      = {2019},
  publisher = {Inter-university Consortium for Political and Social Research},
  doi       = {10.3886/ICPSR37305.v1},
  url       = {https://doi.org/10.3886/ICPSR37305.v1},
  note      = {Distributor}
}

@misc{bank_marketing_222,
  author       = {Moro, S. and Rita, P. and Cortez, P.},
  title        = {{Bank Marketing}},
  year         = {2014},
  howpublished = {UCI Machine Learning Repository},
  note         = {{DOI}: https://doi.org/10.24432/C5K306}
}

@article{Moro2014ADA,
  title={A data-driven approach to predict the success of bank telemarketing},
  author={S{\'e}rgio Moro and P. Cortez and Paulo Rita},
  journal={Decis. Support Syst.},
  year={2014},
  volume={62},
  pages={22-31},
  url={https://api.semanticscholar.org/CorpusID:14181100}
}

@article{FanaeeT2013EventLC,
  title={Event labeling combining ensemble detectors and background knowledge},
  author={Hadi Fanaee-T and Jo{\~a}o Gama},
  journal={Progress in Artificial Intelligence},
  year={2013},
  volume={2},
  pages={113 - 127},
  url={https://api.semanticscholar.org/CorpusID:256282956}
}

@misc{bike_sharing_275,
  author       = {Fanaee-T, Hadi},
  title        = {{Bike Sharing}},
  year         = {2013},
  howpublished = {UCI Machine Learning Repository},
  note         = {{DOI}: https://doi.org/10.24432/C5W894}
}

@inproceedings{automlchallenges,

 author = {Isabelle Guyon and Lisheng Sun-Hosoya and Marc Boull\'e and Hugo Jair Escalante and Sergio Escalera and Zhengying Liu and Damir Jajetic and Bisakha Ray and Mehreen Saeed and Mich\'ele Sebag and Alexander Statnikov and WeiWei Tu and Evelyne Viegas},

 title = {Analysis of the AutoML Challenge series 2015-2018},

 booktitle = {AutoML},

 series = {Springer series on Challenges in Machine Learning},

 year = {2019},

url = {https://www.automl.org/wp-content/uploads/2018/09/chapter10-challenge.pdf}

}

@misc{erickson2025tabarenalivingbenchmarkmachine,
      title={TabArena: A Living Benchmark for Machine Learning on Tabular Data}, 
      author={Nick Erickson and Lennart Purucker and Andrej Tschalzev and David Holzmüller and Prateek Mutalik Desai and David Salinas and Frank Hutter},
      year={2025},
      eprint={2506.16791},
      archivePrefix={arXiv},
      primaryClass={cs.LG},
      url={https://arxiv.org/abs/2506.16791}, 
}

@book{marcoulides2005data,
  title={Discovering knowledge in data: an introduction to data mining},
  author={Larose, Daniel T and Larose, Chantal D},
  year={2015},
  publisher={John Wiley \& Sons}
}

@inproceedings{bao2021compaslicated,
  author    = {Bao, Michelle and Zhou, Angela and Zottola, A. S. and Brubach, Brian and Desmarais, Sarah and Horowitz, Seth A. and Lum, Kristian and Venkatasubramanian, Suresh},
  title     = {It's COMPASlicated: The Messy Relationship between RAI Datasets and Algorithmic Fairness Benchmarks},
  booktitle = {Proceedings of the Thirty-Fifth Conference on Neural Information Processing Systems (NeurIPS)},
  year      = {2021},
  note      = {Datasets and Benchmarks Track (Round 1)}
}

@misc{default_of_credit_card_clients_350,
  author       = {Yeh, I-Cheng},
  title        = {{Default of Credit Card Clients}},
  year         = {2009},
  howpublished = {UCI Machine Learning Repository},
  note         = {{DOI}: https://doi.org/10.24432/C55S3H}
}

@article{Yeh2009TheCO,
  title={The comparisons of data mining techniques for the predictive accuracy of probability of default of credit card clients},
  author={I-Cheng Yeh and Che-hui Lien},
  journal={Expert Syst. Appl.},
  year={2009},
  volume={36},
  pages={2473-2480},
  url={https://api.semanticscholar.org/CorpusID:15696161}
}

@inproceedings{Mathur2021PosterNA,
  title={Poster: NATICUSdroid: A malware detection framework for Android using native and custom},
  author={Akshay Mathur and Mounika Podila and Keyur Kulkarni and Quamar Niyaz and Ahmad Y. Javaid},
  year={2021},
  url={https://api.semanticscholar.org/CorpusID:232063483}
}

@article{breiman1996born,
  title={Born again trees},
  author={Breiman, Leo and Shang, Nong},
  journal={University of California, Berkeley, Berkeley, CA, Technical Report},
  volume={1},
  number={2},
  pages={4},
  year={1996}
}

@misc{fico2018heloc,
  author       = {{FICO}},
  title        = {Home Equity Line of Credit (HELOC) Dataset},
  year         = {2018},
  publisher    = {FICO},
  url          = {https://community.fico.com/s/explainable-machine-learning-challenge},
  note         = {FICO Explainable Machine Learning Challenge}
}

@misc{openml_jasmine_41143,
  author       = {{OpenML}},
  title        = {Jasmine Dataset},
  year         = {2018},
  version      = {1},
  publisher    = {OpenML},
  url          = {https://www.openml.org/d/41143},
  note         = {OpenML Dataset ID 41143; dataset from the ChaLearn Automatic Machine Learning (AutoML) Challenge}
}

@misc{openml_madeline_41144,
  author       = {{OpenML}},
  title        = {Madeline Dataset},
  year         = {2018},
  version      = {1},
  publisher    = {OpenML},
  url          = {https://www.openml.org/d/41144},
  note         = {OpenML Dataset ID 41144; dataset from the ChaLearn Automatic Machine Learning (AutoML) Challenge}
}

@misc{magic_gamma_telescope_159,
  author       = {Bock, R.},
  title        = {{MAGIC Gamma Telescope}},
  year         = {2004},
  howpublished = {UCI Machine Learning Repository},
  note         = {{DOI}: https://doi.org/10.24432/C52C8B}
}

@inproceedings{Thrun1991TheMP,
  title={The MONK''s Problems-A Performance Comparison of Different Learning Algorithms, CMU-CS-91-197, Sch},
  author={Sebastian Thrun},
  year={1991},
  url={https://api.semanticscholar.org/CorpusID:59699060}
}

@misc{phishing_websites_327,
  author       = {Mohammad, Rami and McCluskey, Lee},
  title        = {{Phishing Websites}},
  year         = {2012},
  howpublished = {UCI Machine Learning Repository},
  note         = {{DOI}: https://doi.org/10.24432/C51W2X}
}

@article{Mohammad2012AnAO,
  title={An assessment of features related to phishing websites using an automated technique},
  author={Rami Mustafa A. Mohammad and Fadi A. Thabtah and Lee Mccluskey},
  journal={2012 International Conference for Internet Technology and Secured Transactions},
  year={2012},
  pages={492-497},
  url={https://api.semanticscholar.org/CorpusID:5716727}
}

@misc{online_shoppers_purchasing_intention_dataset_468,
  author       = {Sakar, C. and Kastro, Yomi},
  title        = {{Online Shoppers Purchasing Intention Dataset}},
  year         = {2018},
  howpublished = {UCI Machine Learning Repository},
  note         = {{DOI}: https://doi.org/10.24432/C5F88Q}
}

@article{Sakar2018RealtimePO,
  title={Real-time prediction of online shoppers’ purchasing intention using multilayer perceptron and LSTM recurrent neural networks},
  author={Cemal Okan Sakar and Suleyman Olcay Polat and Mete Katircioglu and Yomi Kastro},
  journal={Neural Computing and Applications},
  year={2018},
  volume={31},
  pages={6893 - 6908},
  url={https://api.semanticscholar.org/CorpusID:13682776}
}

@misc{spambase_94,
  author       = {Hopkins, Mark and Reeber, Erik and Forman, George and Suermondt, Jaap},
  title        = {{Spambase}},
  year         = {1999},
  howpublished = {UCI Machine Learning Repository},
  note         = {{DOI}: https://doi.org/10.24432/C53G6X}
}

@misc{student_performance_320,
  author       = {Cortez, Paulo},
  title        = {{Student Performance}},
  year         = {2008},
  howpublished = {UCI Machine Learning Repository},
  note         = {{DOI}: https://doi.org/10.24432/C5TG7T}
}

@inproceedings{Cortez2008UsingDM,
  title={Using data mining to predict secondary school student performance},
  author={P. Cortez and A. M. Gonçalves Silva},
  year={2008},
  url={https://api.semanticscholar.org/CorpusID:16621299}
}

@misc{tic-tac-toe_endgame_101,
  author       = {Aha, David},
  title        = {{Tic-Tac-Toe Endgame}},
  year         = {1991},
  howpublished = {UCI Machine Learning Repository},
  note         = {{DOI}: https://doi.org/10.24432/C5688J}
}

@misc{openml_wine_47041,
  author       = {{OpenML}},
  title        = {Wine Dataset},
  year         = {2025},
  version      = {9},
  publisher    = {OpenML},
  url          = {https://www.openml.org/d/47041},
  note         = {OpenML Dataset ID 47041}
}

@misc{openml_rl_43949,
  author       = {{OpenML}},
  title        = {RL Dataset},
  year         = {2022},
  version      = {2},
  publisher    = {OpenML},
  url          = {https://www.openml.org/d/43949},
  note         = {OpenML Dataset ID 43949; dataset used in the tabular data benchmark and derived from the ChaLearn AutoML Challenge}
}

@misc{openml_pol_44082,
  author       = {{OpenML}},
  title        = {Pol Dataset},
  year         = {2022},
  version      = {8},
  publisher    = {OpenML},
  url          = {https://www.openml.org/d/44082},
  note         = {OpenML Dataset ID 44082; dataset used in the tabular data benchmark and derived from a binarized version of the original regression dataset}
}

@misc{openml_diamonds_42225,
  author       = {{OpenML}},
  title        = {Diamonds Dataset},
  year         = {2019},
  version      = {1},
  publisher    = {OpenML},
  url          = {https://www.openml.org/d/42225},
  note         = {OpenML Dataset ID 42225; dataset containing prices and attributes of nearly 54,000 diamonds}
}

@misc{in-vehicle_coupon_recommendation_603,
  title        = {{In-Vehicle Coupon Recommendation}},
  year         = {2017},
  howpublished = {UCI Machine Learning Repository},
  note         = {{DOI}: https://doi.org/10.24432/C5GS4P}
}

@misc{openml_california_44090,
  author       = {{OpenML}},
  title        = {California Dataset},
  year         = {2022},
  version      = {5},
  publisher    = {OpenML},
  url          = {https://www.openml.org/d/44090},
  note         = {OpenML Dataset ID 44090; dataset used in the tabular data benchmark and derived from the California Housing dataset}
}

@misc{openml_abalone_44956,
  author       = {{OpenML}},
  title        = {Abalone Dataset},
  year         = {2022},
  version      = {15},
  publisher    = {OpenML},
  url          = {https://www.openml.org/d/44956},
  note         = {OpenML Dataset ID 44956; dataset for predicting abalone age from physical measurements}
}

@article{Lou2013AccurateIM,
  title={Accurate intelligible models with pairwise interactions},
  author={Yin Lou and Rich Caruana and Johannes Gehrke and Giles Hooker},
  journal={Proceedings of the 19th ACM SIGKDD international conference on Knowledge discovery and data mining},
  year={2013},
  url={https://api.semanticscholar.org/CorpusID:11246170}
}

@inproceedings{
heile2026,
title={From Rashomon Theory to {PRAXIS}: Efficient Decision Tree Rashomon Sets},
author={Zakk Heile and Hayden McTavish and Varun Babbar and Margo Seltzer and Cynthia Rudin},
booktitle={Forty-third International Conference on Machine Learning},
year={2026},
url={https://openreview.net/forum?id=Sgwd0l1u2V}
}

@article{hollmann2025tabpfn,
 title={Accurate predictions on small data with a tabular foundation model},
 author={Hollmann, Noah and M{\"u}ller, Samuel and Purucker, Lennart and
         Krishnakumar, Arjun and K{\"o}rfer, Max and Hoo, Shi Bin and
         Schirrmeister, Robin Tibor and Hutter, Frank},
 journal={Nature},
 year={2025},
 month={01},
 day={09},
 doi={10.1038/s41586-024-08328-6},
 publisher={Springer Nature},
 url={https://www.nature.com/articles/s41586-024-08328-6},
}

@article{adaboost,
title = {A Decision-Theoretic Generalization of On-Line Learning and an Application to Boosting},
journal = {Journal of Computer and System Sciences},
volume = {55},
number = {1},
pages = {119-139},
year = {1997},
issn = {0022-0000},
doi = {https://doi.org/10.1006/jcss.1997.1504},
url = {https://www.sciencedirect.com/science/article/pii/S002200009791504X},
author = {Yoav Freund and Robert E Schapire}
}

@inproceedings{
wang2026claritree,
title={{CLARIT}ree: Cholesky and Lookahead Accelerations for Regression with Interpretable Piecewise Linear Trees},
author={Yixiao Wang and Hayden McTavish and Varun Babbar and Margo Seltzer and Cynthia Rudin},
booktitle={Forty-third International Conference on Machine Learning},
year={2026},
url={https://openreview.net/forum?id=JjBozF4i2w}
}

@article{scikit-learn,
  title={Scikit-learn: Machine Learning in {P}ython},
  author={Pedregosa, F. and Varoquaux, G. and Gramfort, A. and Michel, V.
          and Thirion, B. and Grisel, O. and Blondel, M. and Prettenhofer, P.
          and Weiss, R. and Dubourg, V. and Vanderplas, J. and Passos, A. and
          Cournapeau, D. and Brucher, M. and Perrot, M. and Duchesnay, E.},
  journal={Journal of Machine Learning Research},
  volume={12},
  pages={2825--2830},
  year={2011}
}

@article{grinsztajn2022tree,
  title={Why do tree-based models still outperform deep learning on typical tabular data?},
  author={Grinsztajn, L{\'e}o and Oyallon, Edouard and Varoquaux, Ga{\"e}l},
  journal={Advances in neural information processing systems},
  volume={35},
  pages={507--520},
  year={2022}
}

@article{LarsonMaKiAn16,
title={How We Analyzed the {COMPAS} Recidivism Algorithm},
__author={Jeff Larson and Surya Mattu and Lauren Kirchner and Julia Angwin},
author={J. Larson and S. Mattu and L. Kirchner and J. Angwin},
journal={ProPublica},
year={2016}
}
